\documentclass[10pt]{report}
\usepackage[numbers]{natbib}
\usepackage{pdfpages}
\pdfoutput=1

\usepackage{amsmath,amsfonts,amssymb,amsthm, color}
\usepackage{algorithm,algorithmicx}
\usepackage{epsfig,wrapfig,epsf,subfig}
\usepackage[outercaption]{sidecap} 
\usepackage{times}
\usepackage{fancyvrb}
\usepackage{enumerate}
\usepackage{authblk}
\usepackage{tikz}
\usepackage{bm}
\usetikzlibrary{arrows,shapes}
\usetikzlibrary{patterns}
\usepackage{graphicx}
\usepackage{fullpage}
\usepackage{ltxtable}
\usepackage{nicefrac}
\usepackage{multirow}
\usepackage{setspace}
\usepackage[colorlinks=true, linkcolor=black, bookmarks=true]{hyperref}
\usepackage[noend]{algpseudocode}
\usepackage[utf8]{inputenc} 
\usepackage[T1]{fontenc}    
\usepackage{url}            
\usepackage{booktabs}       
\usepackage{microtype}      
\usepackage{pbox}
\usepackage{natbib}
\renewcommand{\ref}{\hyperref}

\makeatletter
\newtheorem*{rep@theorem}{\rep@title}
\newcommand{\newreptheorem}[2]{%
\newenvironment{rep#1}[1]{%
 \def\rep@title{#2 \ref{##1}}%
 \begin{rep@theorem}}%
 {\end{rep@theorem}}}
\makeatother

\newtheorem{theorem}{Theorem}
\newtheorem{lemma}[theorem]{Lemma}
\newtheorem{claim}[theorem]{Claim}
\newtheorem{corollary}[theorem]{Corollary}

\newreptheorem{theorem}{Theorem}
\newreptheorem{lemma}{Lemma}
\newreptheorem{definition}{Definition}
\newenvironment{sketch}{\paragraph{Proof sketch}\mbox{}}{\hfill $\blacksquare$}

\theoremstyle{definition}
\newtheorem{definition}{Definition}

\newcounter{saveenumi}

\algrenewcommand\algorithmicrequire{\textbf{Parameters:}}
\algrenewcommand\algorithmicensure{\textbf{Initialization:}}

\newcommand{\calA}{\mathcal{A}}

\newcommand{\calD}{\mathcal{D}}

\newcommand{\calF}{\mathcal{F}}
\newcommand{\calG}{\mathcal{G}}

\newcommand{\calK}{\mathcal{K}}
\newcommand{\calL}{\mathcal{L}}
\newcommand{\calM}{\mathcal{M}}
\newcommand{\calN}{\mathcal{N}}
\newcommand{\calO}{\mathcal{O}}
\newcommand{\calP}{\mathcal{P}}

\newcommand{\calR}{\mathcal{R}}
\newcommand{\calS}{\mathcal{S}}
\newcommand{\calT}{\mathcal{T}}

\newcommand{\calW}{\mathcal{W}}
\newcommand{\calX}{\mathcal{X}}
\newcommand{\calY}{\mathcal{Y}}

\newcommand{\mathE}{\mathbb{E}}

\newcommand{\mathP}{\mathbb{P}}

\newcommand{\mathR}{\mathbb{R}}

\newcommand{\vecH}{\mathbf{H}}

\newcommand{\vecU}{\mathbf{U}}
\newcommand{\vecV}{\mathbf{V}}
\newcommand{\vecW}{\mathbf{W}}
\newcommand{\vecX}{\mathbf{X}}
\newcommand{\vecY}{\mathbf{Y}}

\newcommand{\vecGamma}{\boldsymbol{\gamma}}

\newcommand{\vecDelta}{\boldsymbol{\Delta}}

\newcommand{\vecphi}{\boldsymbol{\phi}}
\newcommand{\vectheta}{\boldsymbol{\theta}}

\newcommand{\vecb}{\mathbf{b}}
\newcommand{\vecc}{\mathbf{c}}

\newcommand{\vech}{\mathbf{h}}

\newcommand{\vecp}{\mathbf{p}}

\newcommand{\vecr}{\mathbf{r}}

\newcommand{\vecu}{\mathbf{u}}

\newcommand{\vecw}{\mathbf{w}}
\newcommand{\vecx}{\mathbf{x}}
\newcommand{\vecy}{\mathbf{y}}
\newcommand{\vecz}{\mathbf{z}}

\DeclareMathOperator{\sign}{sign}
\DeclareMathOperator{\conv}{conv}
\DeclareMathOperator{\cov}{Cov}
\DeclareMathOperator{\diag}{diag}
\DeclareMathOperator{\length}{len}

\newcommand{\removed}[1]{}
\newcommand{\defeq}{\stackrel{\text{def}}{=}}

\newcommand{\argmin}[1]{\underset{#1}{\mathrm{argmin}} \:}

\newcommand{\inner}[2]{\left\langle #1, #2 \right\rangle}
\newcommand{\E}[1]{{\mathbb{E}\left[{#1}\right]}}

\newcommand{\Ep}[2]{{\mathbb{E}_{#1}\left[{#2}\right]}}

\newcommand{\norm}[1]{\left\lVert{#1}\right\rVert}
\newcommand{\abs}[1]{{\left\lvert{#1}\right\rvert}}
\newcommand{\margin}{\text{margin}}

\newcommand{\R}{\mathbb{R}}
\newcommand{\pr}[2]{{\mathbb{P}_{#1}\left[{#2}\right]}}

\newcommand{\net}{\text{net}}
\newcommand{\eps}{\epsilon}

\newcommand{\In}{\text{in}}
\newcommand{\Out}{\text{out}}
\newcommand{\Win}{\vecW_{\text{in}}}
\newcommand{\Wout}{\vecW_{\text{out}}}
\newcommand{\Wrec}{\vecW_{\text{rec}}}
\newcommand{\relu}{\sigma_{\textsc{relu}}}
\newcommand{\err}{Err}
\newcommand{\error}{\epsilon}
\newcommand{\hatD}{\widehat{D}}
\newcommand{\hatL}{\hat{L}}
\newcommand{\hatl}{\hat{\ell}}

\newcommand{\vin}{v_{\textrm{in}}}
\newcommand{\vout}{v_{\textrm{out}}}
\newcommand{\layer}{\text{layer}}
\newcommand{\gnorm}[1]{\norm{#1}}
\newcommand{\pathr}{\phi}
\newcommand{\DAG}{\text{DAG}}
\newcommand{\convexnn}{\nu}
\newcommand{\nout}{n_\text{out}}
\newcommand{\nin}{n_\text{in}}
\newcommand{\netdepth}{d}
\newcommand{\netwidth}{H}
\newcommand{\nparam}{n_{\text{param}}}

\usepackage[left=1.65in, right=1.65in, top=1in, bottom=1in]{geometry}
\usepackage{makeidx}

\makeindex

\newcommand{\picwidth}{1.84in}

\numberwithin{equation}{section}

   \def\url#1{}

\newcommand{\thesisTitle}{Implicit Regularization in Deep Learning}
\newcommand{\thesisAuthor}{Behnam Neyshabur}
\newcommand{\thesisInstitute}{Toyota Technological Institute at Chicago}
\newcommand{\thesisDate}{August, 2017}
\newcommand{\thesisAdvisor}{Nathan Srebro}
\newcommand{\thesisAdvisorTitle}{Professor}

\newcommand{\committeeFirst}{Yury Makarychev}
\newcommand{\committeeSecond}{Ruslan Salakhutdinov}
\newcommand{\committeeThird}{Gregory Shakhnarovich}

\newcommand{\thesisAbstract}{In an attempt to better understand generalization in deep learning, we study several possible explanations. We show that implicit regularization induced by the optimization method is playing a key role in generalization and success of deep learning models. Motivated by this view, we study how different complexity measures can ensure generalization and explain how optimization algorithms can implicitly regularize complexity measures. We empirically investigate the ability of these measures to explain different observed phenomena in deep learning. We further study the invariances in neural networks, suggest complexity measures and optimization algorithms that have similar invariances to those in neural networks and evaluate them on a number of learning tasks.}

\begin{document}

\newgeometry{left=1.65in, right=1.65in, top=1.65in, bottom=1.65in}

\begin{center}

\huge{\bf \thesisTitle} \\[1cm]
\Large{by} \\[1cm]
\LARGE{\thesisAuthor} \\[2cm]
\Large{A thesis submitted\\ in partial fulfillment of the requirements for\\the degree of}\\[0.8cm]
\Large{Doctor of Philosophy in Computer Science}\\[0.8cm]
\Large{at the}\\[0.8cm]
\Large{\MakeUppercase \thesisInstitute}\\[0.8cm]
\Large{\thesisDate}\\[2cm]
\Large{Thesis Committee: \\ {\thesisAdvisor \;(Thesis advisor)},\\ {\committeeFirst},\\ {\committeeSecond},\\ {\committeeThird}}	

\end{center}

\thispagestyle{empty}
\includepdf[pages=-, width=1.7\textwidth]{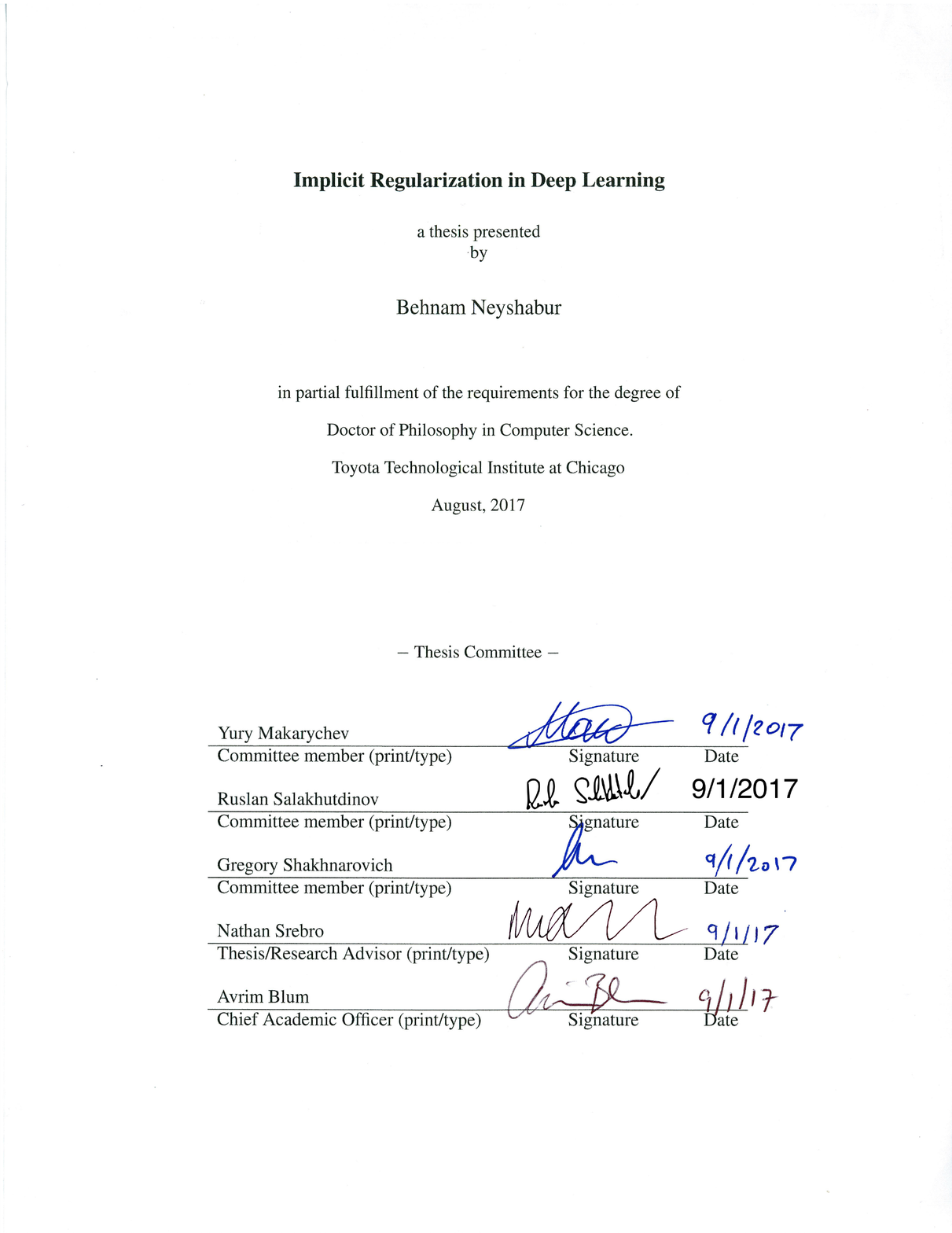}
\newgeometry{left=1.5in, right=1.5in, top=1in, bottom=1.25in}

\begin{center}

\Large{\bf \thesisTitle} \\[0.2cm]
\large{by} \\[0.2cm]
\Large{\thesisAuthor}

\end{center}

\section*{Abstract}
\thesisAbstract\par

\noindent Thesis Advisor: \thesisAdvisor\\
\noindent Title: \thesisAdvisorTitle

\thispagestyle{empty}
\newpage

\pdfbookmark[0]{Dedication}{dedication}
	\begin{center} 
~~~\\
~~~\\
~~~\\
~~~\\
~~~\\
~~~\\
~~~\\
~~~\\
~~~\\
~~~\\
~~~\\
\includegraphics[width=0.85\linewidth]{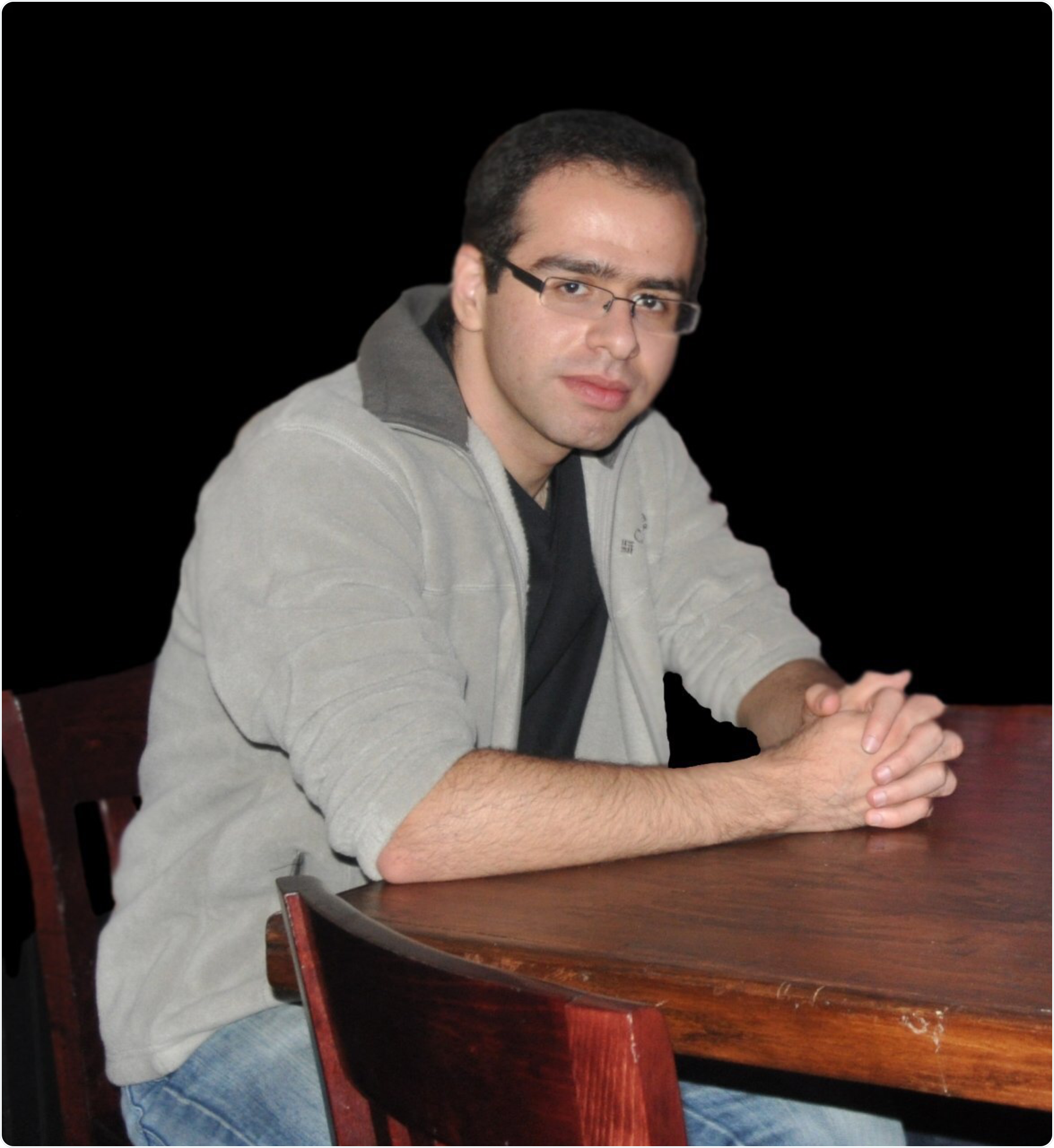}
~~~\\
~~~\\
{In memory of my dear friend, Sina Masihabadi}
 \end{center}
\thispagestyle{empty}
\newpage

\section*{Acknowledgments}
I would like to thank my advisor Nati Srebro without whom I would not be even pursuing a PhD in the first place. Nati's classes made me interested in optimization and machine learning and later I was delighted when he agreed to be my advisor. From the very first few meetings, I realized how his excitement and passion for research energizes me. I will not forget our countless long meetings that did not seem long to me at all. Looking back, I think Nati was the best mentor I could have hoped for. Among the skills I started to pick up from Nati, my favorite is asking the right question and formalizing it. About three years ago, excited about recent advances in deep learning, I walked into Nati's office and I told him that I want to understand what makes deep learning models perform well in practice. He helped me to ask the right questions and formalize them. Nati has had a great influence in shaping my thoughts and therefore this dissertation. I am forever grateful for all I have learned from him.

I would like to thank all of my committee members from whom I benefited during my PhD study. I am thankful to Ruslan Salakhutdinov who has been advising me in several projects and I have benefited a lot from discussions with him and his advices regarding my career. I thank Yury Makarychev for his collaborations and advices during early years of PhD. I always felt free to knock at his office door whenever I was stuck with theoretical questions and needed more insight. I thank Greg Shakhnarovich who has been kind enough to answer my questions on deep learning and metric learning whenever I showed up at his office. Moreover, I think his well-prepared machine learning course had a great impact on making me interested in machine learning.

TTIC's faculty have a close and friendly relationship with students and I have benefited from that. I am thankful to Jinbo Xu, my interim advisor who encouraged me to continue PhD in the area that fits my interests the best. I am pleased that we continued collaboration on computational biology projects. Madhur Tulsiani deserves a special thanks for being a great director of graduate studies. Madhur's desire to improve the PhD program at TTIC and his support made me feel comfortable and free to discuss my thoughts and suggestions with him many times. I would like to acknowledge that I really enjoyed the mathematical foundation course taught by David McAllester and his views on deep learning which he discussed in his deep learning course at TTIC. I regret that I only started collaborating with him in last few months of my PhD and I wish I would have had more discussions with him during these years. I was not able to collaborate with Karen Livescu and Kevin Gimpel during my PhD years but I have enjoyed many conversations with them and I felt supported by them. I also thank Sadaoki Furui, Julia Chuzhoy and Matthew Walter for their effort in enhancing TTIC's PhD program. I had the pleasure of chatting with Avrim Blum a few times since he has joined TTIC and I am excited about his new appointment at TTIC.

During my PhD years, I have enjoyed collaborating with several research faculties at TTIC. I would like to thank Srinadh Bhojanapalli with whom I have spent a lot of time in the last two years for being a good friend, mentor and collaborator. I am very grateful for everything he has offered me. My earlier works in deep learning was in collaboration with Ryota Tomioka and I am thankful his help and support. Suriya Gunasekar is another research faculty at TTIC who has been an amazing friend and collaborator. Other than her help and support, I have also enjoyed countless discussions with her on random topics. I would like to thank Aly Khan and Ayan Chakrabarti for what I have learned from them during our collaborations on different projects. Finally, I have learned from and enjoyed chatting with many other research faculty and postdocs at TTIC including but not limited to Michael Maire, Hammad Naveed, Mesrob Ohannessian, Mehrdad Mahdavi, Weiran Wang, Herman Kamper, Ofer Meshi, Qixing Huangi, Subhransu Maji, Raman Arora, and George Papandreou.

I would like to thank TTIC students. I am indebted to Payman Yadollahpour who kindly hosted me and my wife, Somaye, for several days after we arrived to US and helped us numerous times in various occasions. I am grateful for knowing him and for all the moments we have shared. Hao Tang also deserves a special thanks. He was the person I used to discuss all my ideas and thoughts with. Therefore, everything presented in this thesis is somehow impacted by the discussions with him. I thank Mrinalkanti Ghosh for the countless times I interrupted his work with a theoretical question and he patiently helped me try to investigate it. Shubhendu Trivedi is another student whom I had several fruitful discussion with. I am also thankful for things I have learned from and memories I have shared with Bahador Nooraei, Mohammadreza Mostajabi, Haris Angelidakis, Vikas Garg, Kaustav Kundu, Abhishek Sen, Siqi Sun, Hai Wang, Heejin Choi, Qingming Tang, Lifu Tu, Blake Woodworth, Shubham Toshniwal, Shane Settle, Nicholas Kolkin, Falcon Dai, Charles Schaff, Rachit Nimavat, and Ruotian Luo.

I want to thank TTIC staff for making my life much easier and helping me whenever I had any problems. On top of the list is Chrissy Novak. She has always patiently and kindly listened to my complaints, suggestions, and problems, and tried her best to resolve or improve any issues. Adam Bohlander has always been helpful and quick in resolving any IT issues. I think Adam's great expertise has saved me several hundreds of hours. Liv Leader was one of TTIC's staff when we arrived to US. She helped us during the first few months of being in US. The first party we were invited in US was Liv's housewarming party. I also thank Mary Marre, Amy Minick, Jessica Johnston and other TTIC staffs for their effort.

Beyond TTIC, I am grateful for my internship at MSR Silicon Valley with Rina Panigrahy. Perhaps, several hundred hours of meeting and discussions with Rina in this internship whose goal was to better understand neural networks theoretically had a great influence on making me interested in deep learning. I am also thankful for collaboration and discussions I had with Robert Schapire, Alekh Agarwal, Haipeng Luo and John Langford during my internship at MSR New York City. I am grateful to Anna Choromanska for several discussions on neural networks and for great career advices I received from her. I thank Tony Wu for being a good friend and collaborator. I have learned a lot about deep learning from several meetings with him.

My deepest gratitude goes to my parents and my brother who have shaped who I am today. Words cannot capture how grateful I am to Somaye without whom everything in my life would have been drastically different. I will thank her in person.

I am dedicating this thesis to the memory of my dear friend Sina Massihabadi who went to the same high school and university as me. Sina was very brilliant and one of the finest human beings I have ever met. Later, he moved to United States to start a PhD in industrial engineering at Texas A\&M university. However, he did not make it to the end of the program. He died in a tragic car accident when he was driving to the airport to pick up his mother whom he was not able to visit for two years due to visa restrictions on Iranians.

\newgeometry{left=1.4in, right=1.4in, top=1in, bottom=1.25in}

{\large \tableofcontents }
{\large \listoffigures}
{\large \listoftables}

\chapter{Introduction}  \label{chap:intro}

Deep learning refers to training typically complex and highly over-parameterized models that benefit from learning a hierarchy of representations. The terms ``neural networks'' and ``deep learning'' are often used interchangeably as many modern deep learning models are slight modifications of different types of neural networks suggested originally around 1950-2000~\cite{schimidhuber15}. Interest in deep learning was revived around 2006 ~\cite{hinton06,bengio07} and since then, it has had enormous practical successes in different areas~\cite{lecun15}. The rapid growth of practical works and new concepts in this field has created a considerable gap between our theoretical understanding of deep learning and practical advances. 

From the learning viewpoint, we often look into three different properties to investigate the effectiveness of a model: expressive power, optimization, and generalization. Given a function class/model the expressive power is about understanding what functions can be realized or approximated by the functions in the function class. Given a loss function as an evaluation measure, the optimization aspect refers to the ability to \emph{efficiently} find a function with a minimal loss on the training data and generalization addresses the model's ability to perform well on the new unseen data.

All above aspects of neural networks have been studied before. Neural networks have great expressive power. Universal approximation theorem states that for any given precision, feed-forward networks with a single hidden layer containing a finite number of hidden units can approximate any continuous function~\cite{hornik1989multilayer}. More broadly, any $O(T)$ time computable function can be captured by an $O(T^2)$ sized network, and so the expressive power of such networks is indeed great~\cite[Theorem~9.25]{sipser2012}.

Generalization of neural networks as a function of network size is also fairly well understood. With hard-threshold activations, the VC-dimension, and hence sample complexity, of the class of functions realizable with a feed-forward network is equal, up to logarithmic factors, to the number of edges in the
network \cite{anthony09,shalev14}, corresponding to the number of parameters.  With continuous activation functions the VC-dimension could be higher, but is fairly well understood and is still controlled by the size of the network.\footnote{Using weights with very high precision and vastly different magnitudes it is possible to shatter a number of points quadratic in the number of edges when activations such as the sigmoid, ramp or hinge are used \cite[Chapter 20.4]{shalev14}.  But even with such activations, the VC dimension can still be bounded by the size and depth \cite{bartlett98,anthony09,shalev14}.}

At the same time, we also know that learning even moderately sized
networks is computationally intractable---not only is it NP-hard to
minimize the empirical error, even with only three hidden units, but
it is hard to learn small feed-forward networks using {\em any}
learning method (subject to cryptographic assumptions).  That is, even
for binary classification using a network with a single hidden layer and a logarithmic (in the
input size) number of hidden units, and even if we know the true
targets are {\em exactly} captured by such a small network, there is
likely no efficient algorithm that can ensure error better than 1/2
\cite{klivans2006cryptographic,Daniely14}---not if the algorithm
tries to fit such a network, not even if it tries to fit a much larger
network, and in fact no matter how the algorithm represents
predictors.  And so, merely knowing that some not-too-large
architecture is excellent in expressing does {\em not} explain
why we are able to learn using it, nor using an even larger network.
These results, however, are not suggesting any insights on the practical success of deep learning. In contrast to our theoretical understanding, it is possible to train (optimize) very large neural networks and despite their large sizes, they generalize to unseen data. Why is it then that we succeed in learning very large neural networks?  Can we identify a property that makes them
possible to train? Why do these networks generalize to unseen data despite their large capacity in terms of the number of parameters?

In such an over-parameterized setting, the objective has multiple global minima, all minimize the training error, but many of them do not generalize well.
Hence, just minimizing the training error is not sufficient for learning: picking the wrong global minima can lead to bad generalization behavior. In such situations, generalization behavior depends implicitly on the algorithm used to minimize the training error. Different algorithmic choices for optimization such as the initialization, update rule, learning rate, and stopping condition, will lead to different global minima with different generalization behavior \cite{chaudhari2016entropy, keskar2016large, NeySalSre15}.

What is the bias introduced by these algorithmic choices for neural
networks? What is the relevant notion of complexity or capacity control?

The goal of this dissertation is to understand the implicit regularization by studying the optimization, regularization, and generalization in deep learning and the relationship between them. The dissertation is divided into two parts. In the first part, we study different approaches to explain generalization in neural networks. We discuss how some of the complexity measures derived by these approaches can explain implicit regularization. In the second part, we investigate the transformations under which the function computed by a network remains the same and therefore argue for complexity measures and optimization algorithms that have similar invariances. We find complexity measures that have similar invariances to neural networks and optimization algorithms that implicitly regularize them. Using these optimization algorithms for different learning tasks, we indeed observe that they have better generalization abilities.

\section{Main Contributions}
\begin{enumerate}
\item Part I:
\begin{enumerate}
\item {\em The Role of Implicit Regularization (Chapter 4)} We design experiments to highlight the role of implicit regularization in the success of deep learning models.
\item {\em Norm-based capacity control (Chapter 5):} We prove generalization bounds for the class of fully connected feedforward networks with the bounded norm. We further show that for some norms, this bound is independent of the number of hidden units.
\item {\em Generalization Guarantee by PAC-Bayes Framework (Chapter 6):} We show how PAC-Bayes framework can be employed to obtain generalization bounds for neural networks by making a connection between sharpness and PAC-Bayes theory.
\item {\em Implicit Regularization by SGD (Chapter 6):} We show that networks learned by SGD satisfy several conditions that lead to flat minima. 
\item {\em Empirical Investigation of Generalization in Deep Learning (Chapter 7):} We design experiments to compare the ability of different complexity measures to explain the implicit regularization and generalization in deep learning.
\end{enumerate}
\item Part II:
\begin{enumerate}
\item {\em Invariances in neural networks (Chapter 8):}  We characterize a large class of invariances in feedforward and recurrent neural networks caused by rescaling issues and suggest a measure called the Path-norm that is invariant to the rescaling of the weights.
\item {\em Path-normalized optimization (Chapter 9 and 10):} Inspired by our understanding of invariances in neural networks and the importance of implicit regularization, we suggest a new method called Path-SGD whose updates are the approximate steepest descent direction with respect to the Path-norm. We show Path-SGD achieves better generalization error than SGD in both fully connected and recurrent neural networks on different benchmarks.
\item {\em Data-dependent path normalization (Chapter 11):} We propose a unified framework for neural net normalization, regularization, and optimization, which includes Path-SGD and Batch-Normalization and interpolates between them across two different dimensions. Through this framework, we investigate the issue of invariance of the optimization, data dependence and the connection with natural gradient.
\end{enumerate}
\end{enumerate}
\chapter{Preliminaries} \label{chap:prelim}
In this chapter, we present the basic setup and notations used throughout this dissertation.

\section{The Statistical Learning Framework}\label{sec:learning}
In this section, we briefly review the statistical learning framework. More details on the formal model can be found in \citet{shalev14}.

In the statistical batch learning framework, the learner is given a training set $\calS=\{(\vecx_1,\vecy_1),\dots,(\vecx_m,\vecy_m)\}$ of $m$ training points in $\calX \times \calY$ that are independently and identically distributed (i.i.d.) according to an unknown distribution $\calD$. For simplicity, we will focus on the task of classification where the goal of the learner is to output a predictor $f:\calX\rightarrow\calY$ with minimum expected error on samples generated from the distribution $\calD$:
\begin{equation}
L_\calD(f) = \pr{(\vecx,\vecy)\sim \calD}{f(\vecx)\neq\vecy}
\end{equation}
Since the learner does not have access to the distribution $\calD$, it cannot evaluate or minimize the expected loss. It can however, obtain an estimate of the expected loss of a predictor $f$ using the training set $\calS$:
\begin{equation}\label{eq:training-err}
L_\calS(f) = \frac{\abs{\left\{(\vecx,\vecy)\in \calS \mid f(\vecx)\neq \vecy\right\}}}{m}
\end{equation}
When the distribution $\calD$ and training set $\calS$ is clear from the context, we use $L(f)$ and $\hatL(f)$ instead of $L_\calD$ and $L_\calS(f)$ respectively. We also define the expected margin loss for any margin $\gamma>0$, as follows:
\begin{equation}\label{eq:loss}
L_{\gamma}(f)=\pr{(\vecx,y)\sim \calD}{f(\vecx)[y]\leq \gamma + \max_{j\neq y}f(\vecx)[j]}
\end{equation}
Let $\hatL_{\gamma}(f)$ be the empirical estimate of the above expected margin loss. Since setting $\gamma=0$ corresponds to the classification loss, we will use $L_{0}(f)=L(f)$ and $\hatL_{0}(f_{\vecw})=\hatL(f)$ to refer to the expected risk and the training error.

Minimizing the loss in the equation \eqref{eq:training-err} which is called the training error does not guarantee low expected error. For example, a predictor that only memorizes the set $\calS$ to output the right label for the data in the training set can get zero training error while its expected loss might be very close to the random guess. We are therefore interested in controlling the difference $L_\calD(f)-L_\calS(f)$ which we will refer to as {\em generalization error}. This quantity reflects the difference between memorizing and learning.

An interesting observation is that if $f$ is chosen in advance and is not dependent on the distribution $\calD$ or training set $\calS$, then the {\em generalization error} can be simply bounded by concentration inequalities such as Hoeffding's inequality and relatively small number of samples are required to get small generalization error. However, since the predictor is chosen by the learning algorithm using the training set $\calS$, one need to make sure that this bound holds for the set all predictors that could be chosen by the learning algorithm. It is therefore preferred to limit the search space of the learning algorithm to a small enough set $\calF$ of predictors called {\em model class} to be able to bound the generalization.

We consider the statistical {\em capacity} of a model class in terms of the number of examples required to ensure {\em generalization}, i.e.~that the population (or test error) is close to the training error, even when minimizing the training error.  This also roughly corresponds to the maximum number of examples on which one can obtain small training error even with random labels.

In the next section, we define a meta-model class of feedforward networks with shared weights that include several well-known model classes such as fully connected, convolutional and recurrent neural networks.

\section{Feedforward Neural Networks with Shared Weights}\label{sec:feedforward}

We denote a feedforward network by a triple $(G,\vecw,\sigma)$ where $G=(V,E)$ is a directed acyclic graph over the set of nodes $V$ that corresponds to units $v\in V$ in the network, including special input nodes $V_\In\subset V$ with no incoming edges and special output nodes $V_\Out\subseteq V$ with no outgoing edges, $\vecw:E\rightarrow \R$ is the weights assigned to the edges and $\sigma:\R\rightarrow \R$ is an activation function. 

Feedforward network $(G,\vecw,\sigma)$ computes the function $f_{G,\vecw,\sigma}:\R^{n_\In}\rightarrow \R^{n_\Out}$ for a given input vector $\vecx\in \R^{n_\In}$ as follows: For any input node $v\in V_{\In}$, its output $h_v$ is the corresponding coordinate of $\vecx$~\footnote{We might want to also add a special ``bias'' node with $h_{\text{bias}}=1$, or just rely on the inputs having a fixed ``bias coordinate''.}; for any internal node $v$ (all nodes except the input and output nodes) the output value is defined according to the forward propagation equation:
\begin{equation}
h_v=\sigma\left(\sum_{(u\rightarrow v)\in E} w_{u\rightarrow v}\cdot h_u\right)
\end{equation}

and for any output node $v\in V_\Out$, no non-linearity is applied and its output $h_v=\sum_{(u\rightarrow v)\in E} w_{u\rightarrow v}\cdot h_u$ corresponds to coordinates of the computed function $f_{G,\vecw,\sigma}(\vecx)$. When the graph structure $G$ and the activation function $\sigma$ is clear from the context, we use the shorthand $f_\vecw=f_{G,\vecw,\sigma}$ to refer to the function computed by weights $\vecw$.

We will focus mostly on the hinge, or RELU (REctified Linear Unit) activation, which is currently in popular use \citep{nair10,glorot11,zeiler13}, $\relu(z) = [z]_+ = \max(z,0)$.
When the activation will not be specified, we will
implicitly be referring to the RELU.  The RELU has several convenient
properties which we will exploit, some of them shared with other
activation functions:
\begin{description}\itemsep1pt \parskip0pt \parsep0pt
\item[Lipshitz] ReLU is Lipschitz continuous with Lipschitz
  constant one.  This property is also shared by the sigmoid and the
  ramp activation $\sigma(z)=\min(\max(0,z),1)$.
\item[Idempotency] ReLU is idempotent,
  i.e.~$\relu(\relu(z))=\relu(z)$. This property is also shared by the ramp
  and hard threshold activations.
\item[Non-Negative Homogeneity] For a non-negative scalar $c \geq 0$
  and any input $z\in\R$ we have $\relu(c\cdot z)=c\cdot \relu(z)$.
  This property is important as it allows us to scale the incoming
  weights to a unit by $c>0$ and scale the outgoing edges by $1/c$
  without changing the function computed by the network.  For
  layered graphs, this means we can scale $\vecW^i$ by $c$ and compensate
  by scaling $\vecW^{i+1}$ by $1/c$.
\end{description}

When investigating a class of feedforward networks, in order to account for weight sharing, we separate the weights from actual parameters. Given a parameter vector $\vectheta\in \R^\nparam$ and a mapping $\pi:E\rightarrow \{1,\dots, \nparam\}$ from edges to parameter indices, the weight of any edge $e\in E$ is $w_e=\theta_{\pi(e)}$. We also refer to the set of edges that share the $i$th parameter $\theta_i$ as $E_i=\left\{e\in E| \pi(e)=i\right\}$. That is, for any $e_1,e_2\in E_i$, $\pi(e_1)=\pi(e_2)$ and therefore $w_{e_1}=w_{e_2}=\theta_{\pi(e_1)}$. Given a graph $G$, activation function $\sigma$ and mapping $\pi$, we consider the hypothesis class $\calF^{G,\sigma,\pi}=\left\{f_{G,w,\sigma}|\vectheta\in \R^k; \forall_{e\in E}\;w(e)=\theta_{\pi(e)}\right\}$ of functions computable using some setting of parameters. When $\pi$ is a one-to-one mapping, we use weights $\vecw$ to refer to the parameters $\vectheta$ and drop $\pi$ and use $\calF^{G,\sigma}$ to refer to the hypothesis class.

We will refer to the {\em size} of the network, which is the overall number of edges $n_{\text{edge}}=\abs{E}$, the {\em depth} $\netdepth$ of the network, which is the length of the longest directed path in $G$, and the {\em in-degree} (or width) $\netwidth$ of a network, which is the maximum in-degree of a vertex in $G$.

If the mapping $\pi$ is a one-to-one mapping, then there is no weight sharing and it corresponds to standard feedforward networks. Fully connected neural networks (FCNNs) are a well-known family of standard feedforward networks in which every hidden unit in each layer is connected to all hidden units in the previous and next layers. On the other hand, weight sharing exists if $\pi$ is a many-to-one mapping. Two well-known examples of feedforward networks with shared weights are convolutional neural networks (CNNs) and recurrent neural networks (RNNs). We mostly use the general notation of feedforward networks with shared weights as this will be more comprehensive. However, when focusing on FCNNs or RNNs, it is helpful to discuss them using a more familiar notation which we briefly introduce next.

\paragraph{Fully Connected Neural Networks}

Let us consider a layered fully-connected network where nodes are partitioned into layers. Let $n_i$ be the number of nodes in layer $i$. For all nodes $v$ on layer $i$, we recover the layered recursive formula $\vech^i=\sigma\left(\mathbf{W}^i\vech^{i-1}\right)$ where $\vech^i\in \R^{n_i}$ is the vector of outputs in layer $i$ and $\mathbf{W}^i\in \R^{n_{i}\times n_{i-1}}$ is the weight matrix in layer $i$ with entries $w_{u\rightarrow v}$, for each $u$ in layer $i-1$ and $v$ in layer $i$.  This description ignores the bias term, which could be modeled as a direct connection from $v_{\rm bias}$ into every node on every layer, or by introducing a bias unit (with output fixed to 1) at each layer.

\removed{
\paragraph{Convolutional Neural Networks} CNNs are a family of feedforward networks that the input data has sequence or grid-like structure. Their main features are special weight sharing and local connectivity patterns. For simplicity of notation, we will focus on CNNs defined for 2D structures such as images. Let $\vech^i$ be the outputs of hidden units in layer $i$. Then $\vech^i$ is 3-dimensional where $h^i[j_1,j_2,j_3]$ is $j_3$-th feature for the position $(j_1,j_2)$ of the image representation in layer $i$. Local connectivity of CNNs allows each node in layer $i$ to be connected to all nodes within $F\times F$ neighborhood of that node where $F$ is the filter size. To save computation and reduce the image dimensionality, sometimes the filter jump over some fixed number of pixels. If $s_i$ is the stride, then we define the operator $\ast_s$ to be convolution with stride $s$ which can be calculated as:
\begin{equation}
(W^i \ast_{s_i} h^{i-1})[j_1,j_2] = \sum_{\alpha_1=1}^{F}  \sum_{\alpha_2=1}^{F} \inner{\vecW^{i}[\alpha_1,\alpha_2]}{\vech^{i-1}[(j_1-1)s_i+\alpha_1,(j_2-1)s_i+\alpha_2]}
\end{equation}
}

\paragraph{Recurrent Neural Networks}
Time-unfolded RNNs are feedforward networks with shared weights that map an input sequence to an output sequence. Each input node corresponds to either a coordinate of the input vector at a particular time step or a hidden unit at time $0$. Each output node also corresponds to a coordinate of the output at a specific time step. Finally, each internal node refers to some hidden unit at time $t\geq 1$. When discussing RNNs, it is useful to refer to different layers and the values calculated at different time-steps. We use a notation for RNN structures in which the nodes are partitioned into layers and $\vech_t^i$ denotes the output of nodes in layer $i$ at time step $t$. Let $\vecx=(\vecx_1,\dots,\vecx_T)$ be the input at different time steps where $T$ is the maximum number of propagations through time and we refer to it as the length of the RNN. For $0\leq i <d$, let $\Win^i \in \R^{n_{i}\times n_{i-1}}$ and $\Wrec^i\in \R^{n_{i}\times n_{i}}$ be the input and recurrent parameter matrices of layer $i$ and $\Wout\in \R^{n_{d}\times n_{d-1}}$ be the output parameter matrix.The output of the function implemented by RNN can then be calculated as $f_{\vecw,t}(x)=h_t^d$. Note that in this notations, weight matrices $\Win$, $\Wrec$ and $\Wout$ correspond to ``free'' parameters of the model that are shared in different time steps. Table \ref{tab:notation} shows forward computations for layered feedforward networks and RNNs.

\begin{table}
\begin{center}
\small
\begin{tabular}{|c|c|c|c|}
\hline
 & Input nodes &  Internal nodes & Output nodes \rule{0pt}{2ex}\\
\hline
FF (shared weights)& $h_v=x[v]$ &  $h_v=\sigma\left(\sum_{(u\rightarrow v)\in E} w_{u\rightarrow v} h_u\right)$ & $h_v=\sum_{(u\rightarrow v)\in E} w_{u\rightarrow v} h_u$ \\
\hline
FCNN notation& $\vech^0=\vecx$& $\vech^i=\sigma\left(\vecW^i \vech^{i-1}\right)$& $\vech^d=\vecW^d \vech^{d-1}$\\
\removed{
\hline
CNN notation& $\vech^0=\vecx$& $\vech^i=\sigma\left(\vecW^i \ast_{s_i}\vech^{i-1}\right)$& $\vech^d=\vecW^d \vech^{d-1}$\\
}
\hline
RNN notation& $\vech_t^0=\vecx_t,\textbf{h}_0^i=0$& $\vech_t^i=\sigma\left(\Win^i \vech_t^{i-1} + \Wrec^i \vech_{t-1}^i\right)$& $\vech_t^d=\Wout \vech_t^{d-1}$\\
\hline
\end{tabular}
  \caption[\small Forward computations for feedforward and recurrent networks]{Forward computations for feedforward nets with shared weights.}
  \label{tab:notation}
 \end{center}
\end{table}

\part[\Large Implicit Regularization and Generalization]{Implicit Regularization and Generalization}\label{part:generalization}
\chapter{Generalization and Capacity Control} \label{chap:generalization}
In section \ref{sec:learning} we briefly discussed viewing
the statistical {\em capacity} of a model class in terms
of the number of examples required to ensure {\em generalization}. 
Given a model class $\calF$, such as all the functions representable by
some feedforward or convolutional networks, one can consider the
capacity of the entire class $\calF$---this corresponds to learning
with a uniform ``prior'' or notion of complexity over all models in
the class.  Alternatively, we can also consider some {\em complexity
  measure}, which we take as a mapping that assigns a non-negative
number to every predictor in the class - $\mu: \{ \calF, \calS \}
\rightarrow \mathR^+$, where $\calS$ is the training set.  It is then sufficient to consider the capacity
of the restricted class $\calF_{\mu,\alpha}=\{f: f\in \calF,
\mu(f) \leq \alpha\}$ for a given $\alpha \geq 0$.  One can then
ensure generalization of a learned predictor $f$ in terms of the
capacity of $\calF_{\mu,\mu(f)}$.  Having a good predictor with
low complexity, and being biased toward low complexity (in terms of
$\mu$) can then be sufficient for learning, even if the capacity of
the entire $\calF$ is high.  And if we are indeed relying on $\mu$
for ensuring generalization (and in particular, biasing toward models
with lower complexity under $\mu$), we would expect a learned $f$
with a lower value of $\mu(f)$ to generalize better.

For some complexity measures, we allow $\mu$ to depend also on
the training set.  If this is done carefully, we can still ensure
generalization for the restricted class $\calF_{\mu,\alpha}$.

When considering a complexity measure $\mu$, we can 
investigate whether it is sufficient for
generalization, and analyze the capacity of $\calF_{\mu,\alpha}$.
Understanding the capacity corresponding to different complexity
measures also allows us to relate between different measures and
provides guidance as to what and how we should measure: From the above
discussion, it is clear that any monotone transformation of a
complexity measures leads to an equivalent notion of complexity.
Furthermore, complexity is meaningful only in the context of a
specific model class $\calF$, e.g.~specific architecture or
network size.  The capacity, as we consider it (in units of sample
complexity), provides a yardstick by which to measure complexity (we
should be clear though, that we are vague regarding the scaling of the
generalization error itself, and only consider the scaling in terms of
complexity and model class, thus we obtain only a very crude yardstick
sufficient for investigating trends and relative phenomena, not a
quantitative yardstick).

We next look at different ways of controlling the capacity.
\section{VC Dimension: A Cardinality-Based Arguments}\label{sec:vc}
Consider a finite model class $\calF$. Given any predictor $f\in \calF$, training error $L(f)$
is the average of independent random variables and the expected error the excepted value of
the training error. We can therefore use Hoeffding's inequality upper bounds the generalization error with 
high probability:
\begin{equation}
\pr{}{L(f)-\hatL(f)\geq t} \leq e^{-2mt^2}
\end{equation}
The above bound is for any given $f$. However, since the learning algorithm can output any predictor
from class $\calF$, we need to make sure that all predictors in $\calF$ have low generalization error which
can be done through a union bound over model class $\calF$:
\begin{equation}
\pr{}{\exists_{f\in \calF}\; L(f)-\hatL(f)\geq t} \leq \sum_{f\in \calF}\pr{}{L(f)-\hatL(f)\geq t} \leq \abs{\calF}e^{-2mt^2}
\end{equation}
Setting the r.h.s. of the above inequality to small probability $0<\delta<1$, we can say that with probability $1-\delta$ over the choice of samples in the training set, the following generalization bound holds:
\begin{equation}\label{eq:gen-cardinality}
L(f) \leq \hatL(f) + \sqrt{\frac{\ln \abs{\calF}+\ln(1/\delta)}{m}}
\end{equation}
The above simple yet effective approach gives us an intuition about the relationship between
the capacity and generalization. Many of the approaches of controlling the capacity that we will study
later follow similar arguments. Here, the term $\ln \abs{\calF}$ corresponds to the complexity of the model class.

Even though many model classes that we consider are not finite based on the definition, one can argue that
all parametrized model classes used in practice are finite since the parameters are stored with finite precision
For any model, if $b$ bits are used to store each parameter, then we have $\ln \abs{\calF}\leq b\nparam$ which is
is linear in the total number of parameters. 

Even without making an assumption on the precision of parameters, it is possible to get
similar generalization bound using Vapnik-Chervonenkis dimension (VC dimension) which can be
thought as the logarithm of the``intrinsic'' cardinality. VC-dimension is defined as the size of the largest set $\calW=\{\vecx_i\}_{i=1}^m$ such that for any mapping $g:\calW\rightarrow \{\pm\}^m$, there is a predictor in $\calF$ that achieves zero training error on the training set $\calS=\left\{(\vecx_i,g(\vecx_i)\mid \vecx_i\in \calW\right\}$. The VC-dimension of many known model classes is a linear or low-degree polynomial of the number of parameters. The following generalization bound then holds with probability $1-\delta$~\cite{vapnik1971uniform,blumer1987occam}:
\begin{equation}\label{eq:gen-vc}
L(f) \leq \hatL(f) + \calO\left(\sqrt{\frac{\text{VC-dim}(\calF)\ln m+\ln(1/\delta)}{m}}\right)
\end{equation}

\paragraph{Feedforward Networks} The VC dimension of feedforward networks can also be bounded in terms of the number of parameters 
$\nparam$\cite{anthony2009neural,bartlett1998sample,bartlett1998almost,
  shalev2014understanding}. In particular, \citet{bartlet2017} and
\citet{harvey2017nearly}, following \citet{bartlett1998almost}, give the following tight
(up to logarithmic factors) bound on the VC
dimension and hence capacity of feedforward networks with ReLU activations:
\begin{equation}
\text{VC-dim} = \tilde{O}(d * \nparam)
\end{equation}
In the over-parametrized settings, where the number of parameters is
more than the number of samples, complexity measures that depend on
the total number of parameters are too weak and cannot explain the
generalization behavior.  Neural networks used in practice often have
significantly more parameters than samples, and indeed can perfectly
fit even random labels, obviously without
generalizing~\cite{zhang2017understanding}.  Moreover, measuring
complexity in terms of number of parameters cannot explain the
reduction in generalization error as the number of hidden units
increase \cite{neyshabur15b}. We will discuss more details about network size as the capacity control in Chapter~\ref{chap:implicit}.



\section{Norms and Margins: Counting Real-Valued Functions}\label{sec:margin}
The model classes that we learn are often functions with real-valued outputs and for each task, we use
a different loss and prediction method based on the predicted scores. For example, for the binary classification, thresholding the only real-valued output gives us the binary labels. For the multi-class classification, the output dimension is usually equal to the number of classes and the class with maximum score is chosen as the predicted label. For simplicity, we focus on binary classification here. Since the model class has real-valued output, can not directly use VC-dimension here. Instead, we can use a similar concept called subgraph VC-dimension which is similar to VC-dimension with the difference being that here we count the number of different behavior with a given margin. This means for the binary case, we require $yf(\vecx)\geq \eta$ for some margin $\eta$. There are different techniques that bound subgraph-VC dimension such as Covering Numbers and Rademacher Complexities. Here, we focus on the Rademacher Complexity since most of the results by Covering Numbers can be also proved through Rademacher complexities with less effort. The empirical Rademacher complexity of a class $\calF$ of function mapping from $\calX$ to $\R$ with respect to a set $\{x_1,\dots,x_m\}$ is defined as:
\begin{equation}
\calR_m(\calF) = \Ep{\xi \in \{\pm 1\}^m}{\frac{1}{m}
\sup_{f\in \calF}  \left\lvert \sum_{i=1}^m \xi_i f(x_i) \right\rvert}
\end{equation}
The relationship between Rademacher complexity and subgraph VC-dimension is as follows:
\begin{equation}
\calR_m(\calF) = \calO\left(\sqrt{\frac{\text{VC-dim}(\calF)}{m}}\right)
\end{equation}
It is possible to get the following generalization error for any margin $\gamma>0$ with probability $1-\delta$ over the choice of training examples for every $f\in \calF$:
\begin{equation}
L_{0}(f) \leq \hatL_{\gamma}(f) + 2\frac{\calR_m(\calF)}{\gamma}+ \sqrt{\frac{8\ln(2/\delta)}{m}}
\end{equation}

\paragraph{Feedforward Networks}
\cite{bartlett2002rademacher} proved that the Rademacher complexity of fully connected
feedforward networks on set $\calS$ can be bounded based on the $\ell_1$ norm of the
weights of hidden units in each layer as follows:
\begin{equation}\label{eq:L1-rademacher}
\calR_m(\calF) \leq \sqrt{\frac{4^d\ln \left(n_{\In}\right)\prod_{i=1}^d \norm{W_i}^2_{1,\infty} \max_{\vecx\in \calS} \norm{\vecx}_\infty}{m}}
\end{equation}
where $\norm{W_i}_{1,\infty}$ is the maximum over hidden units in layer $i$ of the $\ell_1$ norm of
incoming weights to the hidden unit \cite{bartlett2002rademacher}. This suggests that the capacity scales roughly as $\prod_{i=1}^d \norm{W_i}^2_{1,\infty}$. In Chapter~\ref{chap:norm-based} we
show how the capacity can be controlled for a large family of norms.

\section{Robustness: Lipschitz Continuity with Respect to Input}\label{subsec:lipschitz}
Some of the measures/norms also control the Lipschitz constant of the model class with respect to its input such as the capacity based on \eqref{eq:L1-rademacher}. Is the capacity
control achieved through the bound on the Lipschitz constant?  Is bounding the
Lipschitz constant alone enough for generalization?  To answer these
questions, and in order to understand capacity control in terms of
Lipschitz continuity more broadly, we review here the relevant guarantees.

Given an input space $\calX$ and metric $\calM$, a function $f:\calX\rightarrow \R$ on a metric space $(\calX,\calM)$ is called
a Lipschitz function if there exists a constant $C_\calM$, such that $\abs{f(x)-f(y)}\leq C_\calM \calM(x,y)$.
\citet{luxburg2004distance} studied the capacity of functions with
bounded Lipschitz constant on metric space $(\calX,\calM)$ with a finite diameter $\text{diam}_\calM(\calX)=\sup_{x,y\in X} \calM(x,y)$ 
and showed that the capacity is proportional to $\left(\frac{C_\calM}{\gamma}\right)^n \text{diam}_\calM(\calX)$ where $\gamma$ is the margin. This capacity bound is weak as it has an
exponential dependence on input size.

Another related approach is through algorithmic robustness as suggested by \citet{xu2012robustness}. Given $\epsilon>0$, the model $f_\vecw$ found by a learning algorithm is $K$ robust if $\calX$ can be partitioned into $K$ disjoint sets, denoted as $\{C_i\}_{i=1}^K$, such that for any pair $(\vecx,y)$ in the training set $\mathbf{s}$ ,\footnote{\citet{xu2012robustness} have defined the robustness as a property of learning algorithm given the model class and the training set. Here since we are focused on the learned model, we introduce it as a property of the model.}
\begin{equation}
\vecx,\vecz\in C_i \Rightarrow \abs{\ell(\vecw,\vecx)-\ell(\vecw,\vecz)}\leq \epsilon
\end{equation}
\citet{xu2012robustness} showed the capacity of a model class whose
models are $K$-robust scales as $K$. For the model class of functions
with bounded Lipschitz $C_{\norm{.}}$, $K$ is proportional to
${\frac{C_{\norm{.}}}{\gamma}}$-covering number of the
input domain $\calX$ under norm $\norm{.}$ where $\gamma$ is the margin to get error $\epsilon$. However, the covering
number of the input domain can be exponential in the input dimension
and the capacity can still grow as
$\left(\frac{C_{\norm{.}}}{\gamma}\right)^n$~\footnote{Similar to margin-based bounds, we drop the term that depends on the diameter of the input space.}.

\paragraph{Feedforward Networks} Returning to our original question, the $C_{\ell_\infty}$ and $C_{\ell_2}$
Lipschitz constants of the network can be bounded by $\prod_{i=1}^d
\norm{W_i}_{1,\infty}$ (hence $\ell_1$-path norm) and $\prod_{i=1}^d
\norm{W_i}_{2}$,
respectively~\cite{xu2012robustness,sokolic2016generalization}. This
will result in a very large capacity bound that scales as
$\left(\frac{\prod_{i=1}^d \norm{W_i}_{2}}{\gamma}\right)^n$,
which is exponential in both the input dimension and depth of the
network. This shows that simply bounding the Lipschitz constant of the
network is not enough to get a reasonable capacity control and the
capacity bounds of the previous Section are not merely a consequence
of bounding the Lipschitz constant.

\section{PAC-Bayesian Framework: Sharpness with Respect to Parameters}\label{sec:pac-bayes-general}
The notion of sharpness as a generalization measure was recently suggested by \citet{keskar2016large} and corresponds to robustness to adversarial perturbations on the parameter space:
\begin{equation}
\zeta_\alpha(\vecW) = \frac{\max_{\abs{\vecu_i}\leq \alpha (\abs{\vecw_i}+\mathbf{1})}\hatL(f_{\vecw+\vecu}) - \hatL(f_\vecw)}{1+\hatL(f_\vecw)} \simeq \max_{\abs{\vecu_i}\leq \alpha (\abs{\vecw_i}+\mathbf{1})}\hatL(f_{\vecw+\vecu}) - \hatL(f_\vecw),
\end{equation}
where the training error $\hatL(f_\vecw)$ is generally very small in the case of neural networks in practice, so we can simply drop it from the denominator without a significant change in the sharpness value. 

Instead, we advocate viewing a related notion of expected sharpness in
the context of the PAC-Bayesian framework.  Viewed this way, it
becomes clear that sharpness controls only one of two relevant terms,
and must be balanced with some other measure such as norm.  Together,
sharpness and norm do provide capacity control and can explain many of
the observed phenomena.  This connection between sharpness and the
PAC-Bayes framework was also recently noted by \citet{dziugaite2017computing}. 

The PAC-Bayesian framework~\cite{mcallester1998some,mcallester1999pac}
provides guarantees on the expected error of a randomized predictor
(hypothesis), drawn from a distribution denoted $\mathcal{Q}$ and
sometimes referred to as a ``posterior'' (although it need {\em not}
be the Bayesian posterior), that depends on the training data. Let $f_\vecw$
be any predictor (not necessarily a neural network) learned from training data.
We consider a distribution $\mathcal{Q}$ over predictors with
weights of the form $\vecw+\vecu$, where $\vecw$ is a single predictor
learned from the training set, and $\vecu$ is a random variable.  Then,
given a ``prior'' distribution $P$ over the hypothesis that is
independent of the training data, with probability at least $1-\delta$
over the draw of the training data, the expected error of
$f_{\vecw+\vecu}$ can be bounded as follows~\cite{mcallester2003simplified}:
\begin{equation}\label{eq:pac-bayes-general}
\Ep{\vecu}{L(f_{\vecw+\vecu})} \leq \Ep{\vecu}{\hatL(f_{\vecw+\vecu})}+\sqrt{\Ep{\vecu}{\hatL(f_{\vecw+\vecu})}\calK}+\calK
\end{equation}
where $\calK=\frac{2\left(KL\left(\vecw+\vecu\|P\right)+\ln\frac{2m}{\delta}\right)}{m-1}$. When the training loss $\Ep{\vecu}{\hatL(f_{\vecw+\vecu})}$ is smaller than $\calK$, then the last term dominates. This is often the case for neural networks with small enough perturbation. One can also get the the following weaker bound:
\begin{small}
\begin{equation}\label{eq:pac-bayes-simple}
\Ep{\vecu}{L(f_{\vecw+\vecu})} \leq \Ep{\vecu}{\hatL(f_{\vecw+\vecu})} +2\sqrt{\frac{2\left(KL\left(\vecw+\vecu\|P\right)+\ln\frac{2m}{\delta}\right)}{m-1}}
\end{equation}
\end{small}
The above inequality clearly holds for $\calK\geq 1$ and for $\calK<1$ it can be derived from Equation~\eqref{eq:pac-bayes-general} by upper bounding the loss in the second term by $1$. We can rewrite the above bound as follows:
\begin{small}
\begin{equation}
\Ep{\vecu}{L(f_{\vecw+\vecu})} \leq \hatL(f_\vecw) + \underbrace{\Ep{\vecu}{\hatL(f_{\vecw+\vecu})} -\hatL(f_\vecw)}_{\text{expected sharpness}} +2\sqrt{\frac{2}{m-1}\left(KL\left(\vecw+\vecu\|P\right)+\ln\frac{2m}{\delta}\right)}
\end{equation}
\label{eq:pacbayes}
\end{small}
As we can see, the PAC-Bayes bound depends on two quantities - i) the
expected sharpness and ii) the Kullback Leibler (KL) divergence to the
``prior'' $P$.  The bound is valid for any distribution measure $P$,
any perturbation distribution $\vecu$ and any method of choosing
$\vecw$ dependent on the training set. 

Next, we present a result that gives a margin-based generalization bound using the PAC-Bayesian framework. The proof of the lemma uses similar ideas as in the proof for the case of linear separators, discussed by \citet{langford2003pac} and \citet{mcallester2003simplified}. This is a general result that holds for any hypothesis class and not specific to neural networks.

\begin{lemma}\label{lem:general-bound}
Let $f_\vecw(\vecx):\calX\rightarrow \R^{k}$ be any predictor (not necessarily a neural network) with parameters $\vecw$ and $P$ be any distribution on the parameters that is independent of the training data. For any $\gamma>0$, consider any set ${\calS_{\vecw}}$ of perturbations with the following property:
\begin{equation*}
{\calS_{\vecw}} \subseteq \left\{\vecw+\vecu  \;\bigg|\max_{\vecx \in \calX}\abs{f_{\vecw+\vecu}(\vecx)-f_{\vecw}(\vecx)}_\infty < \frac{\gamma}{4} \right\}
\end{equation*}
Let $\vecu$ be a random variable such that $\mathP\left[{\vecu}\in {\calS_{\vecw}}\right]\geq \frac{1}{2}$. Then, for any $\delta>0$, with probability $1-\delta$ over the training set, the generalization error can be bounded as follows:
\begin{equation*}
L_0(f_{\vecw}) \leq \hatL_{\gamma}(f_{\vecw})+ 4\sqrt{\frac{KL_{{\calS_{\vecw}}}\left(\vecw+\vecu\|P\right)+\ln\frac{4m}{\delta}}{m-1}}\\
\end{equation*}
where $KL_{\calS_{\vecw}}(Q||P)=\int_{{\calS_{\vecw}}} q(x)\ln\frac{q(x)}{p(x)}dx$.
\end{lemma}
\begin{proof}
Let $q$ be the probability density function for $\vecw+\vecu$. We consider the distribution $\tilde{Q}$ with the following probability density function:
\begin{equation*}
\tilde{q}(\vecr)=\frac{1}{Z}
\begin{cases}
q(\vecr) & \vecr \in {\calS_{\vecw}}\\
0 & \text{otherwise}.
\end{cases}
\end{equation*}
where $Z$ is a normalizing constant and by the lemma assumption $Z = \mathP \left[ \vecw+\vecu \in {\calS_{\vecw}}\right] \geq \frac{1}{2}$. Therefore, we have:
\begin{equation}
KL(\tilde{Q}\|P) =\int \tilde{q}(\vecr) \ln\frac{\tilde{q}(\vecr)}{p(\vecr)}d\vecr \leq 2\int_{{\calS_{\vecw}}} q(\vecr)\ln\frac{q(\vecr)}{p(\vecr)}d\vecr + 1
\end{equation}
Consider $\vecw+{\tilde{\vecu}}$ to be the random perturbation centered at $\vecw$ drawn from $\tilde{Q}$. By the definition of $\tilde{Q}$, we know that for any perturbation ${\tilde{\vecu}}$:
\begin{equation}
\max_{\vecx \in \calX}\abs{f_{\vecw+{\tilde{\vecu}}}(\vecx)-f_{\vecw}(\vecx)}_\infty < \frac{\gamma}{4}
\end{equation}
Therefore, the perturbation ${\tilde{\vecu}}$ can change the margin between two output units of $f_\vecw$ by at most $\frac{\gamma}{2}$; i.e. for any perturbation ${\tilde{\vecu}}$ drawn from $\tilde{Q}$:
\begin{equation*}
\max_{i,j\in[k],\vecx \in \calX}\abs{\left(\abs{f_{\vecw+{\tilde{\vecu}}}(\vecx)[i]-f_{\vecw+{\tilde{\vecu}}}(\vecx)[j]}\right)-\left(\abs{f_{\vecw}(\vecx)[i]-f_{\vecw}(\vecx)[j]}\right)} < \frac{\gamma}{2}
\end{equation*}
Since the above bound holds for any $\vecx$ in the domain $\calX$, we can get the following inequalities:
\begin{align*}
&L_0(f_{\vecw}) \leq L_{\frac{\gamma}{2}}(f_{\vecw+{\tilde{\vecu}}})\\
&\hatL_{\frac{\gamma}{2}}(f_{\vecw+{\tilde{\vecu}}}) \leq \hatL_{\gamma}(f_{\vecw})
\end{align*}
Now using the above inequalities together with the equation~\eqref{eq:pacbayes}, with probability $1-\delta$ over the training set we have:
\begin{align*}
L_0(f_{\vecw}) &\leq \Ep{{\tilde{\vecu}}}{L_{\frac{\gamma}{2}}(f_{\vecw+{\tilde{\vecu}}})}\\
&\leq \Ep{{\tilde{\vecu}}}{\hatL_{\frac{\gamma}{2}}(f_{\vecw+{\tilde{\vecu}}})}+ 2\sqrt{\frac{2(KL\left(\vecw+{\tilde{\vecu}}\|P\right)+\ln\frac{2m}{\delta})}{m-1}}\\
&\leq \hatL_{\gamma}(f_{\vecw})+ 2\sqrt{\frac{2(KL\left(\vecw+{\tilde{\vecu}}\|P\right)+\ln\frac{2m}{\delta})}{m-1}}\\
&\leq \hatL_{\gamma}(f_{\vecw})+ 4\sqrt{\frac{KL_{\calS_{\vecw}}\left(\vecw+\vecu\|P\right)+\ln\frac{4m}{\delta}}{m-1}},
\end{align*}
\end{proof}
 
\paragraph{Feedforward Networks} This connection between sharpness and the PAC-Bayesian framework was
also recently noticed by \citet{dziugaite2017computing}, who optimize
the PAC-Bayes generalization bound over a family of multivariate
Gaussian distributions, extending the work of \citet{langford2001not}.
They show that the optimized PAC-Bayes bounds are numerically
non-vacuous for feedforward networks trained on a binary classification
variant of MNIST dataset.
\chapter{On the Role of Implicit Regularization in Generalization} \label{chap:implicit}

Central to any form of learning is an inductive bias that induces some
sort of capacity control (i.e.~restricts or encourages predictors to
be ``simple'' in some way), which in turn allows for generalization.
The success of learning then depends on how well the inductive bias
captures reality (i.e.~how expressive is the hypothesis class of
``simple'' predictors) relative to the capacity induced, as well as on
the computational complexity of fitting a ``simple'' predictor to the
training data.

Let us consider learning with feed-forward networks from this
perspective. If we search for the weights minimizing the training
error, we are essentially considering the hypothesis class of
predictors representable with different weight vectors, typically for
some fixed architecture.  We showed in Section~\ref{sec:vc}
that the capacity can then be controlled by the size
(number of weights) of the network.  Our justification for using such networks is
then that many interesting and realistic functions can be represented
by not-too-large (and hence bounded capacity) feed-forward networks.
Indeed, in many cases we can show how specific architectures can
capture desired behaviors.  More broadly, any $O(T)$ time computable
function can be captured by an $O(T^2)$ sized network, and so the
expressive power of such networks is indeed great~\cite[Theorem~9.25]{sipser2012}.

At the same time, we also know that learning even moderately sized
networks is computationally intractable---not only is it NP-hard to
minimize the empirical error, even with only three hidden units, but
it is hard to learn small feed-forward networks using {\em any}
learning method (subject to cryptographic assumptions).  That is, even
for binary classification using a network with a single hidden layer and a logarithmic (in the
input size) number of hidden units, and even if we know the true
targets are {\em exactly} captured by such a small network, there is
likely no efficient algorithm that can ensure error better than 1/2
\citep{klivans2006cryptographic,Daniely14}---not if the algorithm
tries to fit such a network, not even if it tries to fit a much larger
network, and in fact no matter how the algorithm represents
predictors.  And so, merely knowing that some not-too-large
architecture is excellent in expressing reality does {\em not} explain
why we are able to learn using it, nor using an even larger network.
Why is it then that we succeed in learning using multilayer
feed-forward networks?  Can we identify a property that makes them
possible to learn?  An alternative inductive bias?

Here, we make our first steps at shedding light on this question by
going back to our understanding of network size as the capacity
control at play.  

Our main observation, based on empirical experimentation with
single-hidden-layer networks of increasing size (increasing number of
hidden units), is that size does {\em not} behave as a capacity
control parameter, and in fact there must be some other, implicit,
capacity control at play.  We suggest that this hidden capacity
control might be the real inductive bias when learning with deep
networks.

In order to try to gain an understanding at the possible inductive
bias, we draw an analogy to matrix factorization and understand
dimensionality versus norm control there.  Based on this analogy we
suggest that implicit norm regularization might be central also for
deep learning, and also there we should think of bounded-norm models
with capacity independent of number of hidden units.

\removed{We then also demonstrate how (implicit) $\ell_2$
weight decay in an infinite two-layer network gives rise to a ``convex
neural net'', with an infinite hidden layer and $\ell_1$ (not $\ell_2$)
regularization in the top layer. }

\section{Network Size and Generalization}\label{sec:netsize}

Consider training a feedforward network by finding the weights
minimizing the training error.  Specifically, we will consider a fully connected
feedforward networks with one hidden layer that includes $H$ hidden units.
The weights learned by minimizing a soft-max cross entropy loss
\footnote{When using soft-max
  cross-entropy, the loss is never exactly zero for correct
  predictions with finite margins/confidences.  Instead, if the data
  is seperable, in order to minimize the loss the weights need to be
  scaled up toward infinity and the cross entropy loss goes to zero,
  and a global minimum is never attained.  In order to be able to say
  that we are actually reaching a zero loss solution, and hence a
  global minimum, we use a slightly modified soft-max which does not
  noticeably change the results in practice.  This truncated loss
  returns the same exact value for wrong predictions or correct
  prediction with confidences less than a threshold but returns zero
  for correct predictions with large enough margins: Let
  $\{s_i\}_{i=1}^k$ be the scores for $k$ possible labels and $c$ be
  the correct labels. Then the soft-max cross-entropy loss can be
  written as $\ell(s,c) = \ln \sum_{i} \exp(s_i - s_c)$ but we instead
  use the differentiable loss function $\hat{\ell}(s,c) = \ln \sum_{i}
  f(s_i-s_c)$ where $f(x)=\exp(x)$ for $x\geq -11$ and $f(x)
  =\exp(-11) [x+13]_+^2/4$ otherwise. Therefore, we only deviate from
  the soft-max cross-entropy when the margin is more than $11$, at
  which point the effect of this deviation is negligible (we always
  have $\abs{\ell(s,c)-\hat{\ell}(s,c)}\leq 0.000003k$)---if there are
  any actual errors the behavior on them would completely dominate
  correct examples with margin over $11$, and if there are no errors
  we are just capping the amount by which we need to scale up the
  weights.} on $n$ labeled training examples.  The total number of
weights is then $H(\abs{V_\In}+\abs{V_\Out})$.

What happens to the training and test errors when we increase the
network size $H$? The training error will necessarily decrease.  The
test error might initially decrease as the approximation error is
reduced and the network is better able to capture the targets.
However, as the size increases further, we loose our capacity control
and generalization ability, and should start overfitting.  This is the
classic approximation-estimation tradeoff behavior.

Consider, however, the results shown in Figure \ref{fig:inductive}, where
we trained networks of increasing size on the MNIST and CIFAR-10
datasets.  Training was done using stochastic gradient descent with
momentum and diminishing step sizes, on the training error and without
any explicit regularization.  As expected, both training and test
error initially decrease.  More surprising is that if we increase the
size of the network past the size required to achieve zero training
error, the test error continues decreasing!  This behavior is not at
all predicted by, and even contrary to, viewing learning as fitting a
hypothesis class controlled by network size.  For example for MNIST, 32 units
are enough to attain zero training error.  When we allow more units,
the network is not fitting the training data any better, but the
estimation error, and hence the generalization error, should increase
with the increase in capacity.  However, the test error goes down.  In
fact, as we add more and more parameters, even beyond the number
of training examples, the generalization error does not go up.

\begin{figure}
\centering
\subfloat[MNIST]{
  \includegraphics[width=65mm]{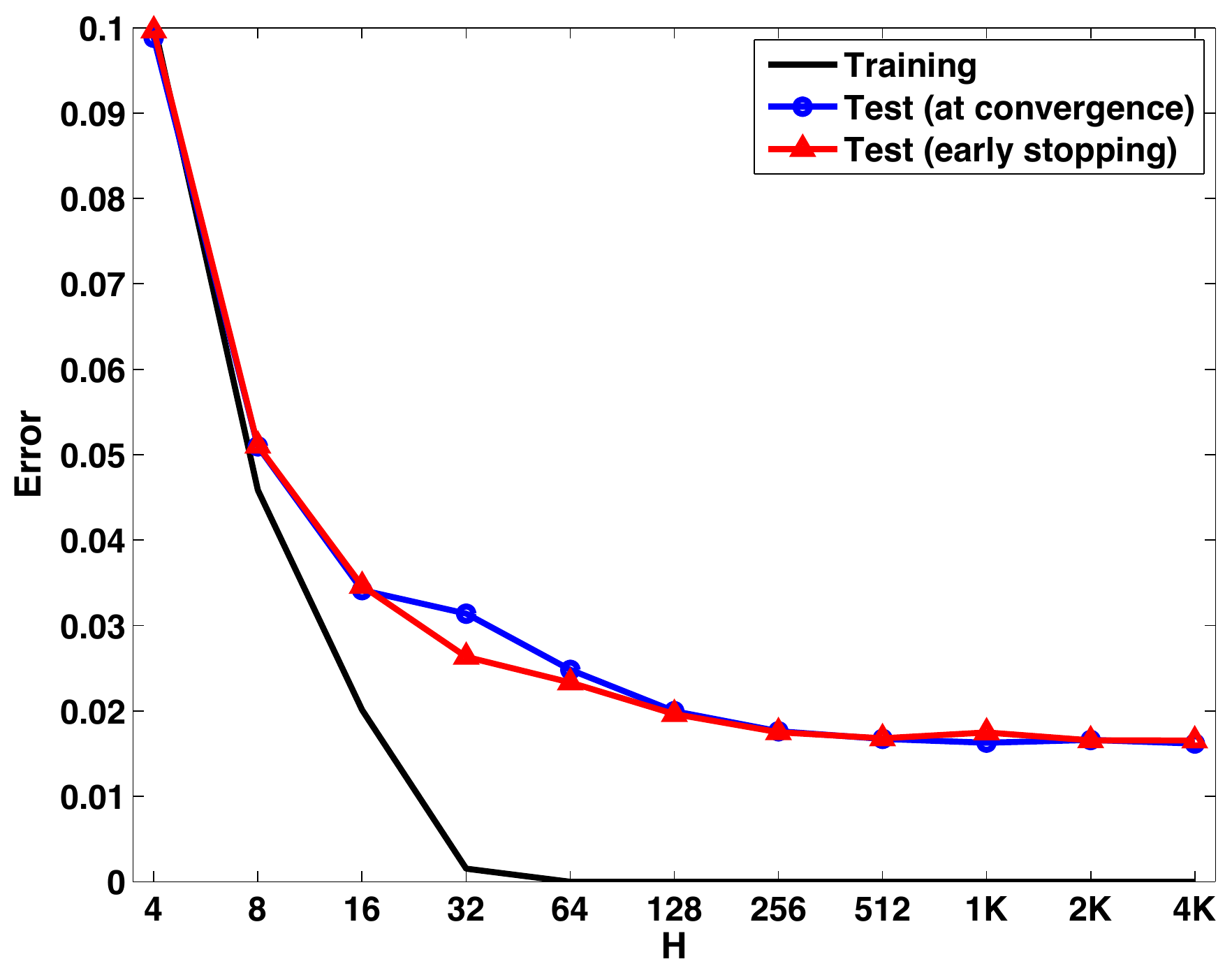}
}
\subfloat[CIFAR10]{
  \includegraphics[width=65mm]{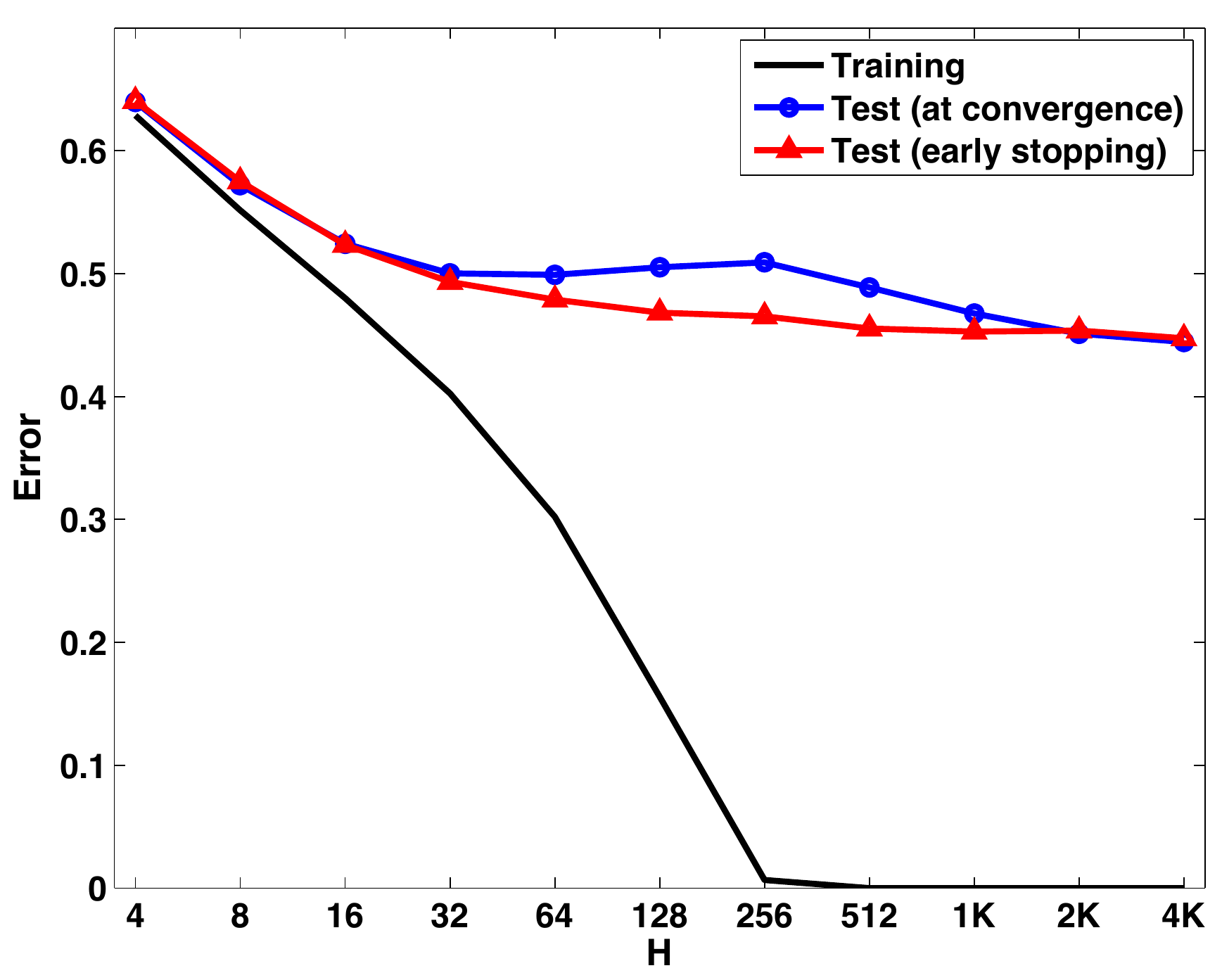}
}
\caption[\small The training and test error of a two-layer perceptron for varying number of hidden units]{\small The training error and the test error based on different stopping
  criteria when 2-layer NNs with different number of hidden units
  are trained on MNIST and CIFAR-10. Images in both datasets are downsampled
  to 100 pixels. The size of the training set is 50000 for MNIST and 40000 for CIFAR-10.
  The early stopping is based on the error on a validation set
  (separate from the training and test sets) of size 10000. The training was
  done using stochastic gradient descent with momentum and mini-batches
  of size 100. The network was initialized with weights generated randomly from
  the Gaussian distribution. The initial step size and momentum were set to 0.1 and 0.5
  respectively. After each epoch, we used the update rule $\mu^{(t+1)}=0.99\mu^{(t)}$
  for the step size and $m^{(t+1)}=\min\{0.9,m^{(t)}+0.02\}$ for the momentum. 
  \label{fig:inductive}}
\end{figure}

We also further tested this phenomena under some artificial
mutilations to the data set.  First, we wanted to artificially ensure
that the approximation error was indeed zero and does not decrease as
we add more units.  To this end, we first trained a network with a
small number $H_0$ of hidden units ($H_0=4$ on MNIST and $H_0=16$ on
CIFAR) on the entire dataset (train+test+validation).  This network
did have some disagreements with the correct labels, but we then
switched all labels to agree with the network creating a ``censored''
data set.  We can think of this censored data as representing an
artificial source distribution which can be exactly captured by a
network with $H_0$ hidden units.  That is, the approximation error is zero
for networks with at least $H_0$ hidden units, and so does not
decrease further. Still, as can be seen in the middle row of
Figure~\ref{fig:netsize}, the test error continues decreasing even
after reaching zero training error.

Next, we tried to force overfitting by adding random label noise to
the data.  We wanted to see whether now the network will use its
higher capacity to try to fit the noise, thus hurting generalization.
However, as can be seen in the bottom row of Figure~\ref{fig:netsize},
even with five percent random labels, there is no significant
overfitting and test error continues decreasing as network size
increases past the size required for achieving zero training error.

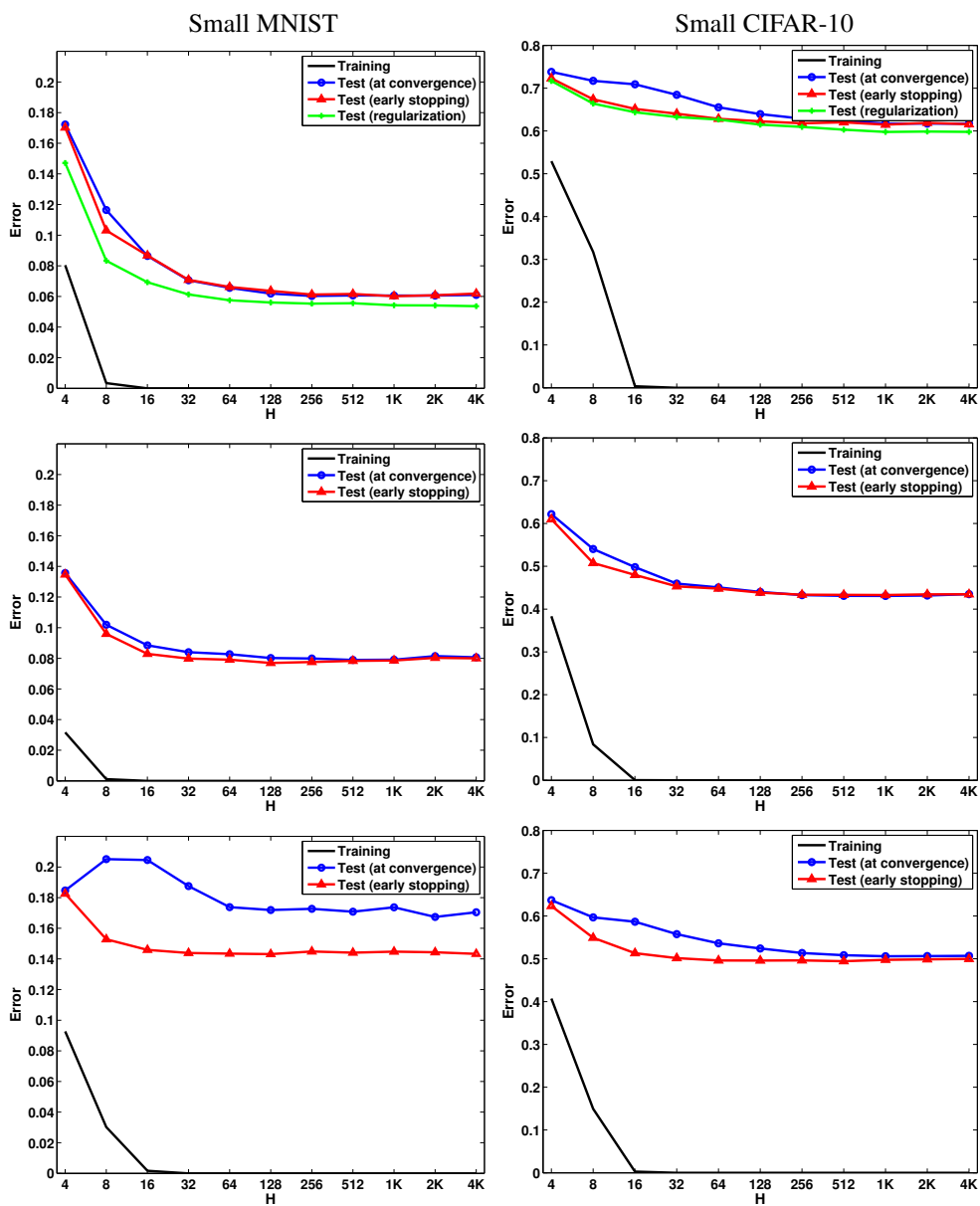
\begin{figure}
\includegraphics[width=65mm]{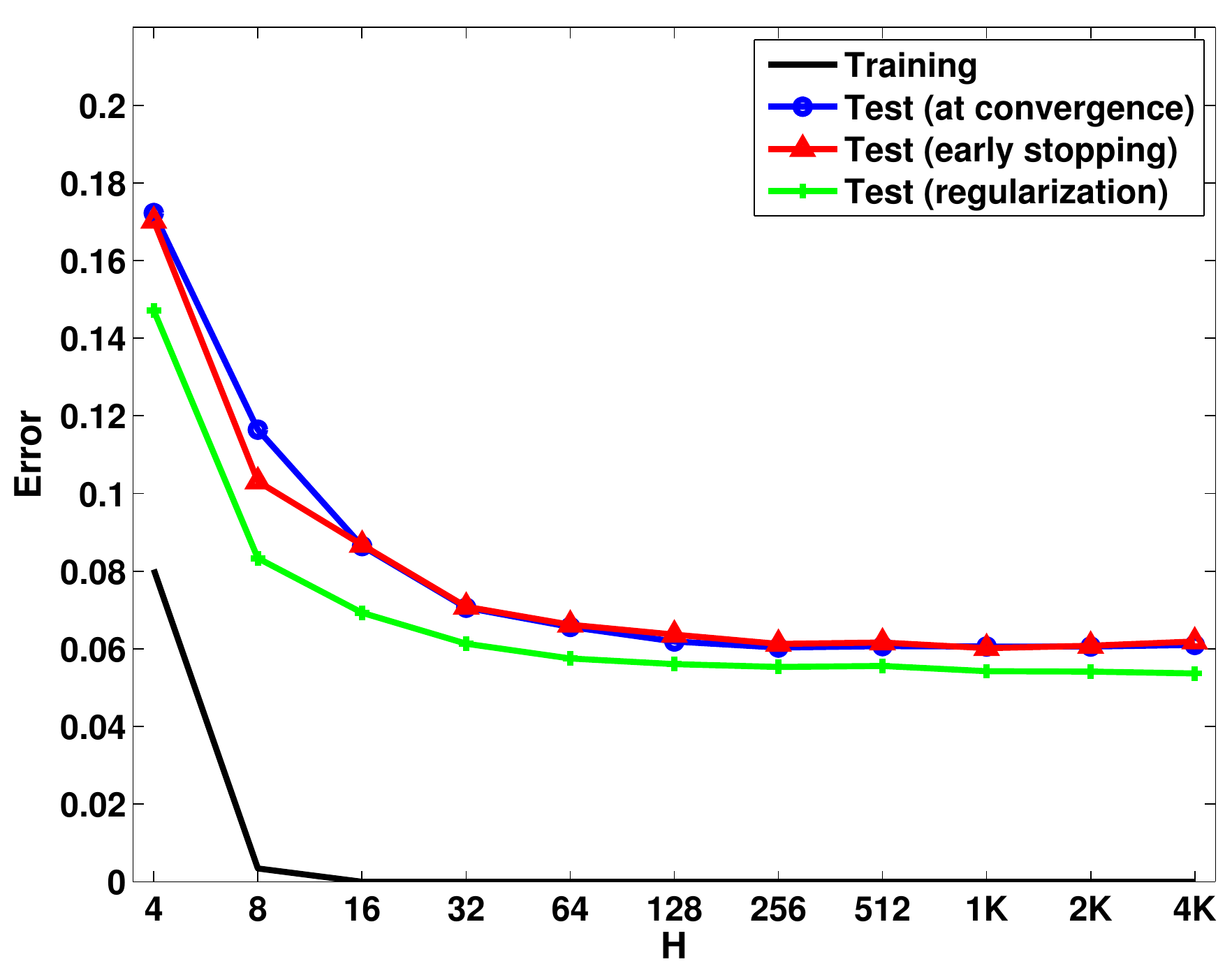}
\includegraphics[width=65mm]{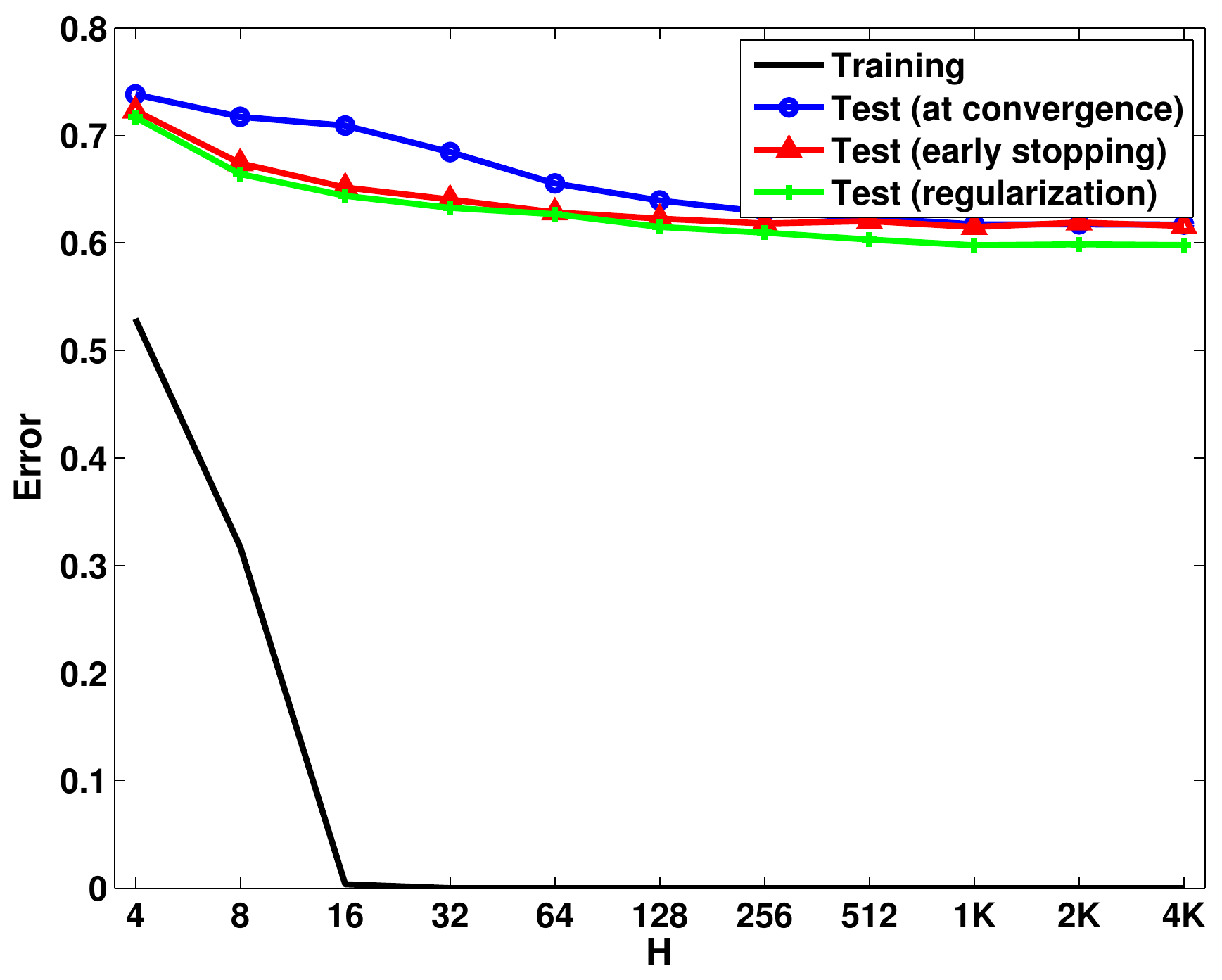}\\
\includegraphics[width=65mm]{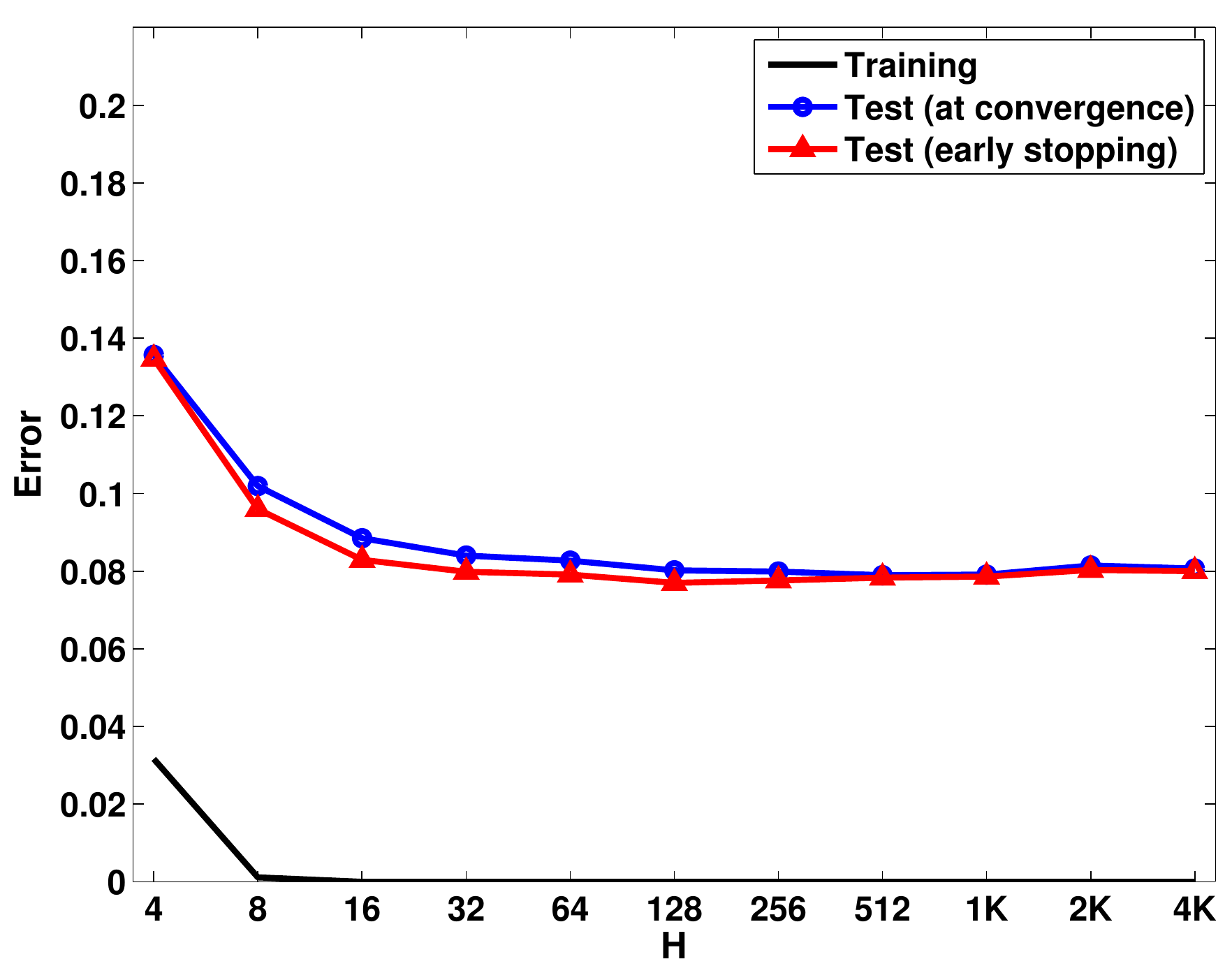}
\includegraphics[width=65mm]{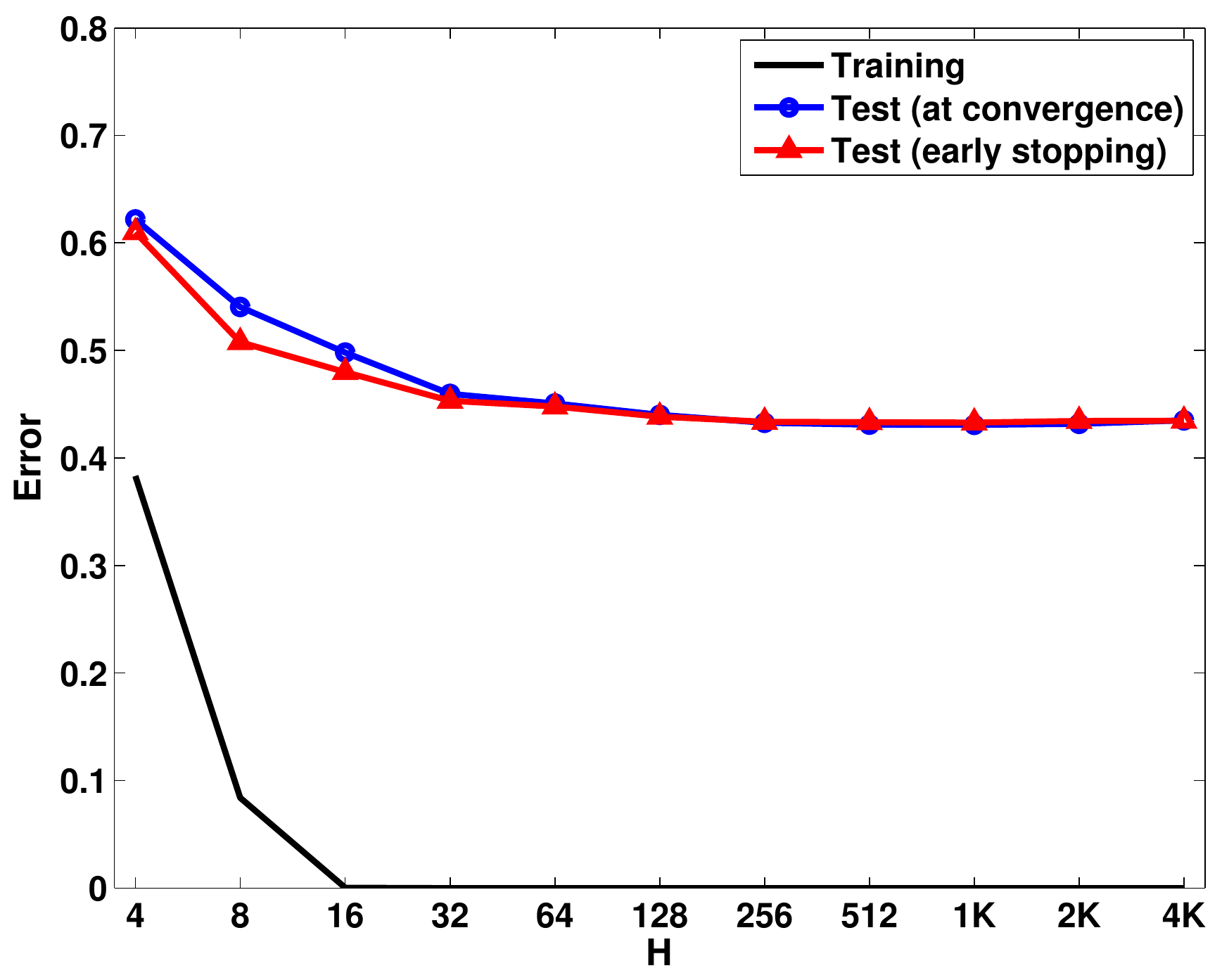}\\
\includegraphics[width=65mm]{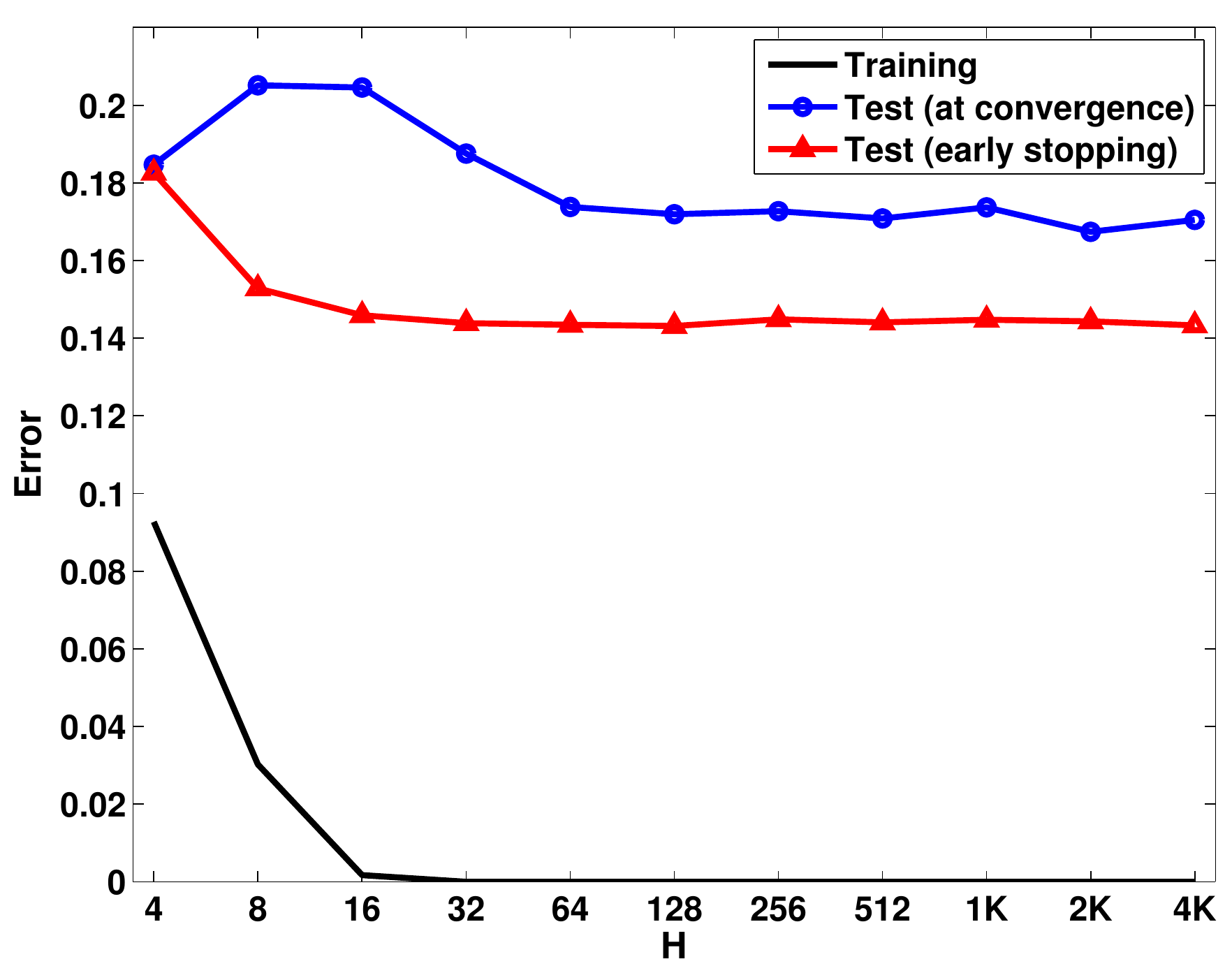}
\includegraphics[width=65mm]{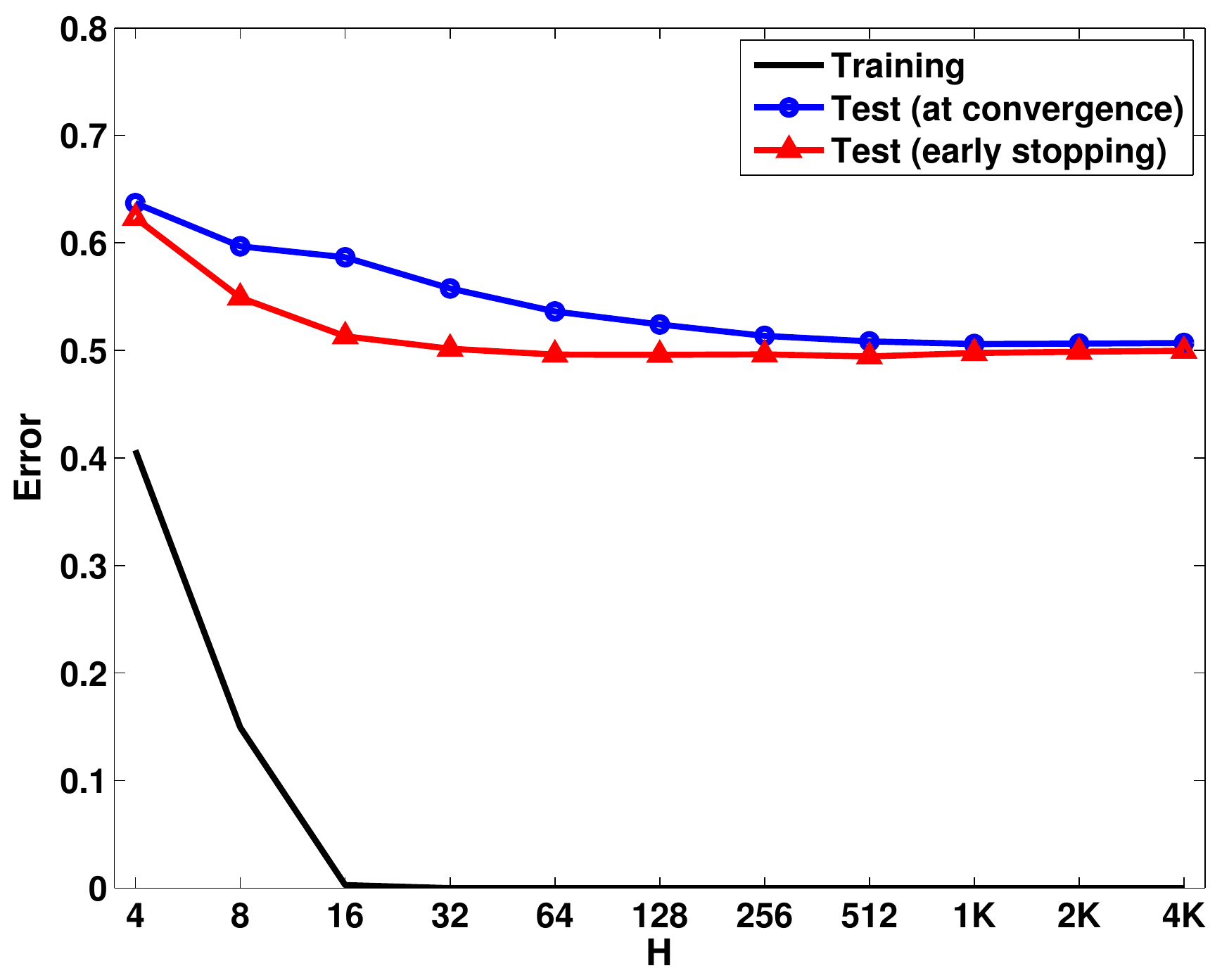}\\
\begin{picture}(0,0)(0,0)
\put(70, 460){Small MNIST}
\put(255, 460){Small CIFAR-10}
\end{picture}
\caption[\small The generalization behavior a two-layer perceptron for varying number of hidden units]{\small The training error and the test error based on different stopping
  criteria when 2-layer NNs with different number of hidden units
  are trained on small subsets of MNIST and CIFAR-10. Images in both
  datasets are downsampled to 100 pixels. The sizes of the training
  and validation sets are 2000 for both MNIST and CIFAR-10 and the early stopping
  is based on the error on the validation set. The top plots are
  the errors for the original datasets with and without explicit regularization.The best
  weight decay parameter is chosen based on the validation error.
  The middle plots are on the censored data set that is constructed
  by switching all the labels to agree with the predictions of a trained network
  with a small number $H_0$ of hidden units  $H_0=4$ on MNIST and $H_0=16$
  on CIFAR-10) on the  entire dataset (train+test+validation).
  The plots on the bottom are also for the censored data except we also add 5 percent
  noise to the labels by randomly changing 5 percent of the labels. 
  The optimization method is the same as the in Figure 1.  The results in this figure are the average error over 5 random repetitions.
\label{fig:netsize}
}
\end{figure}

What is happening here?  A possible explanation is that the
optimization is introducing some implicit regularization.  That is, we are
implicitly trying to find a solution with small ``complexity'', for
some notion of complexity, perhaps norm.  This can explain why we do
not overfit even when the number of parameters is huge.  Furthermore,
increasing the number of units might allow for solutions that actually
have lower ``complexity'', and thus generalization better.  Perhaps an
ideal then would be an infinite network controlled only through this
hidden complexity.

We want to emphasize that we are not including any explicit
regularization, neither as an explicit penalty term nor by modifying
optimization through, e.g., drop-outs, weight decay, or with
one-pass stochastic methods.  We are using a stochastic method, but we
are running it to convergence---we achieve zero surrogate loss and zero training error. In fact, we also tried training using batch conjugate
gradient descent and observed almost identical behavior.  But it seems
that even still, we are not getting to some random global
minimum---indeed for large networks the vast majority of the many
global minima of the training error would horribly overfit.  Instead,
the optimization is directing us toward a ``low complexity'' global
minimum.

Although we do not know what this hidden notion of complexity is, as a
final experiment we tried to see the effect of adding explicit
regularization in the form of weight decay.  The results are shown in the top
row of figure~\ref{fig:netsize}. There is a slight improvement in generalization
but we still see that increasing the network size helps generalization.

\section{A Matrix Factorization Analogy}

To gain some understanding at what might be going on, let us consider
a slightly simpler model which we do understand much better.  Instead of
rectified linear activations, consider a feed-forward network with a
single hidden layer, and {\em linear} activations, i.e.:
\begin{equation}
  \label{eq:ykl}
  f_{\vecU,\vecV}(\vecx) = \vecU\vecV\vecx
\end{equation}
This is of course simply a matrix-factorization model, where $ f_{\vecW}(\vecx)=\vecW\vecx$
and $\vecW=\vecV\vecU$.  Controlling capacity by limiting the number of hidden
units exactly corresponds to constraining the rank of $\vecW$,
i.e.~biasing toward low dimensional factorizations.  Such a low-rank
inductive bias is indeed sensible, though computationally intractable
to handle with most loss functions.

However, in the last decade we have seen much success for learning
with low {\em norm} factorizations.  In such models, we do not
constrain the inner dimensionality $H$ of $\vecU,\vecV$, and instead only
constrain, or regularize, their norm.  For example, constraining the
Frobenius norm of $\vecU$ and $\vecV$ corresponds to using the {\em
  trace-norm} as an inductive bias \citep{Srebro04}:
\begin{equation}\label{eq:tracenorm}
  \norm{\vecW}_{\text{tr}} = \min_{\vecW=\vecV\vecU} \frac{1}{2}(\norm{\vecU}_F^2+\norm{\vecV}_F^2).
\end{equation}
Other norms of the factorization lead to different regularizers.  

Unlike the rank, the trace-norm (as well as other factorization norms)
is convex, and leads to tractable learning problems
\citep{Fazel01,Srebro04}.  In fact, even if learning is done by a
local search over the factor matrices $\vecU$ and $\vecV$ (i.e.~by a local
search over the weights of the network), if the dimensionality is high
enough and the norm is regularized, we can ensure convergence to a
global minima \citep{Burer06}. This is in stark contrast to the
dimensionality-constrained low-rank situation, where the limiting
factor is the number of hidden units, and local minima are abundant
\citep{Srebro03}.

Furthermore, the trace-norm and other factorization norms are
well-justified as sensible inductive biases.  We can ensure
generalization based on having low trace-norm, and a low-trace norm
model corresponds to a realistic factor model with many factors of
limited overall influence.  In fact, empirical evidence suggests that
in many cases low-norm factorization are a more appropriate inductive
bias compared to low-rank models.

We see, then, that in the case of linear activations (i.e.~matrix
factorization), the norm of the factorization is in a sense a better
inductive bias than the number of weights: it ensures generalization,
it is grounded in reality, and it explain why the models can be learned
tractably.

Recently, \citet{gunasekar2017implicit} provided empirical and theoretical evidence on the implicit
regularization of gradient descent for matrix factorization. They showed that 
gradient descent on the full dimensional factorizations without any explicit regularization
 indeed converges to the minimum trace norm solution with initialization close enough
 to the origin and small enough step size.

Let us interpret the experimental results of Section \ref{sec:netsize} in this
light.  Perhaps learning is succeeding not because there is a good
representation of the targets with a small number of units, but rather
because there is a good representation with small overall norm, and
the optimization is implicitly biasing us toward low-norm models.
Such an inductive bias might potentially explain both the
generalization ability {\em and} the computational tractability of
learning, even using local search.

Under this interpretation, we really should be using infinite-sized
networks, with an infinite number of hidden units.  Fitting a finite
network (with implicit regularization) can be viewed as an
approximation to fitting the ``true'' infinite network.  This
situation is also common in matrix factorization: e.g., a very
successful approach for training low trace-norm models, and other
infinite-dimensional bounded-norm factorization models, is to
approximate them using a finite dimensional representation
\cite{rennie2005fast,srebro2010collaborative}.  The finite
dimensionality is then not used at all for capacity (statistical
complexity) control, but purely for computational reasons.  Indeed, increasing
the allowed dimensionality generally improves generalization performance, as it
allows us to better approximate the true infinite model. Inspired by these experiments, in order to understand the implicit regularization in deep learning, we will next look at ways of controlling the capacity independent of the number of hidden units.

\chapter{Norm-based Capacity Control} \label{chap:norm-based}
As we discussed in Section~\ref{sec:vc} statistical complexity, or capacity, of {\em unregularized}
feed-forward neural networks, as a function of the network size and
depth, is fairly well understood. But feedforward networks are often trained with some
kind of regularization, such as weight decay,
early stopping, ``max regularization'', or more exotic
regularization such as drop-outs. We also showed in Chapter~\ref{chap:implicit} that
even without any explicit regularization, the capacity of neural networks is being controlled
by a form of implicit regularization caused by optimization which does not depend on the size
of the network. What is the effect of such
regularization on the induced model class and its capacity?

For linear prediction (a one-layer feed-forward network) we know that
using regularization the capacity of the class can be bounded only in
terms of the norms, with no (or a very weak) dependence on the number
of edges (i.e.~the input dimensionality or
number of linear coefficients).  E.g., we understand very well how the
capacity of $\ell_2$-regularized linear predictors can be bounded in
terms of the norm alone (when the norm of the data is also bounded),
even in infinite dimension.

A central question we ask is: can we bound the capacity of
feed-forward network in terms of norm-based regularization alone,
without relying on network size and even if the network size (number
of nodes or edges) is unbounded or infinite?  What type of
regularizers admit such capacity control?  And how does the capacity
behave as a function of the norm, and perhaps other network
parameters such as depth?

Beyond the central question of capacity control, we also analyze the
convexity of the resulting model class---unlike unregularized
size-controlled feed-forward networks, infinite magnitude-controlled
networks have the potential of yielding convex model classes
(this is the case, e.g., when we move from rank-based control on
matrices, which limits the number of parameters to magnitude based
control with the trace-norm or max-norm).  A convex class might be
easier to optimize over and might be convenient in other ways.

In this chapter we focus on two natural types of norm regularization: bounding the norm of the
incoming weights of each unit (per-unit regularization) and bounding
the overall norm of all the weights in the system jointly (overall
regularization, e.g.~limiting the overall sum of the magnitudes, or
square magnitudes, in the system).  We generalize both of these with a
single notion of group-norm regularization: we take the $\ell_p$
norm over the weights in each unit and then the $\ell_q$ norm over
units.  In Section \ref{sec:group} we present this regularizer and
obtain a tight understanding of when it provides for size-independent
capacity control and a characterization of when it induces convexity.
We then apply these generic results to per-unit regularization
(Section \ref{sec:path}) and overall regularization (Section
\ref{sec:overall}), noting also other forms of regularization that are
equivalent to these two.  In particular, we show how per-unit
regularization is equivalent to a novel path-based regularizer
and how overall $\ell_2$ regularization for two-layer networks is
equivalent to so-called ``convex neural networks''
\citep{Bengio05}.  In terms of capacity control, we show that
per-unit regularization allows size-independent capacity-control only
with a per-unit $\ell_1$-norm, and that overall $\ell_p$
regularization allows for size-independent capacity control only when
$p \leq 2$, even if the depth is bounded.  In any case, even if we
bound the sum of all magnitudes in the system, we show that an
exponential dependence on the depth is unavoidable.

As far as we are aware, prior work on size-independent capacity
control for feed-forward networks considered only per-unit $\ell_1$
regularization, and per-unit $\ell_2$ regularization for two-layered
networks (see discussion and references at the beginning of Section
\ref{sec:path}). Recently, \citet{bartlett2017spectrally} have shown a
generalization bound based on the product of spectral norm of the layers
using covering numbers. In Chapter~\ref{chap:sharpness}, we show a simpler prove for a tighter bound.
Here, we extend the scope significantly, and
provide a broad characterization of the types of regularization
possible and their properties.  In particular, we consider overall
norm regularization, which is perhaps the most natural form of
regularization used in practice (e.g.~in the form of weight decay).
We hope our study will be useful in thinking about, analyzing and
designing learning methods using feed-forward networks.  Another
motivation for us is that complexity of large-scale optimization is
often related to scale-based, not dimension-based complexity.
Understanding when the scale-based complexity depends exponentially on
the depth of a network might help shed light on understanding the
difficulties in optimizing deep networks.

\paragraph{Preliminaries and Notations}
We denote by $\calF^{d,\netwidth}$ the class of fully connected feedforward networks with a single output node and use
the shorthand $\calF^d=\calF^{d,\infty}$. We will consider various measures $\mu(\vecw)$ of the magnitude of the
weights $\vecw$.  Such a measure induces a complexity measure on functions $f\in\calF^{d,\netwidth}$ defined by $\mu^{d,\netwidth}(f)=\inf_{f_{\vecw}=f} \mu(\vecw)$.
The sublevel sets of the complexity measure $\mu^{d,\netwidth}$ form a family of hypothesis classes
$\calF^{d,\netwidth}_{\mu\leq a} = \{ f\in\calF^{d,\netwidth} \;|\;
\mu^{d,\netwidth}(f)\leq a \}$.

For binary function $g:\{\pm 1\}^{n_{\In}} \rightarrow {\pm 1}$ we say that $g$ is
realized by $f$ with unit margin if $\forall_x f(x)g(x)\geq 1$.  A set
of points $\calS_{\In}$ is shattered with unit margin by a model class
$\calF$ if all $g:\calS_{\In}\rightarrow {\pm 1}$ can be realized with unit
margin by some $f_\vecw\in\calF$.

\section{Group Norm Regularization}
\label{sec:group}
Considering the grouping of weights going into each edge of the
network, we will consider the following generic group-norm type
regularizer, parametrized by $1\leq p,q \leq\infty$:
\begin{equation}
  \label{eq:mu}
  \mu_{p,q}(\vecw) = \left(\sum_{v \in V}\left(\sum_{(u\rightarrow v) \in E} \left\lvert w_{u\rightarrow v}\right\rvert ^p\right)^{q/p}\right)^{1/q}.
\end{equation}
Here and elsewhere we allow $q=\infty$ with the usual conventions that
$(\sum z_i^q)^{1/q}=\sup z_i$ and $1/q=0$ when it appears in other
contexts.  When $q=\infty$ the group regularizer \eqref{eq:mu} imposes
a per-unit regularization, where we constrain the norm of the incoming
weights of each unit separately, and when $q=p$ the regularizer
\eqref{eq:mu} is an ``overall'' weight regularizer, constraining the
overall norm of all weights in the system.  E.g., when $q=p=1$ we are
paying for the sum of all magnitudes of weights in the network, and
$q=p=2$ corresponds to overall weight-decay where we pay for the sum
of square magnitudes of all weights (i.e.~the overall Euclidean norm
of the weights).

For a layered graph, we have:
\begin{align}
\mu_{p,q}(\vecw) &= \left(\sum_{k=1}^d\sum_{i=1}^\netwidth \left(
\sum_{j=1}^\netwidth\abs{W^k[i,j]}^p
\right)^{q/p}\right)^{1/q} 
\!\!=d^{1/q} \left(\frac{1}{d} \sum_{k=1}^d
\gnorm{W^k}^q_{p,q}\right)^{1/q} \notag \\ 
&\geq d^{1/q} \left( \prod_{k=1}^d \gnorm{W^k}_{p,q} \right)^{1/d} \defeq
d^{1/q} \sqrt[d]{\psi_{p,q}(\vecw)} \label{eq:mugeqgamma}
\end{align}
where $\displaystyle \psi_{p,q}(\vecw) = \prod_{k=1}^d \gnorm{W^k}_{p,q}$
aggregates the layers by multiplication instead of summation.
The inequality~\eqref{eq:mugeqgamma} holds regardless of the
activation function, and so for any $\sigma$ we
have:
\begin{equation}
  \label{eq:mugammaforf}
  \psi_{p,q}^{d,\netwidth}(f)\leq \left(\frac{\mu^{d,\netwidth}(f)_{p,q}}{d^{1/q}}\right)^d.
\end{equation}
But due to the homogeneity of the RELU activation, when this
activation is used we can always balance the norm between the
different layers without changing the computed function so as to
achieve equality in \eqref{eq:mugeqgamma}:
\begin{claim}\label{clm:mugamma}
For any $f_\vecw\in\calF^{d,\netwidth}$, 
$\displaystyle \mu^{d,\netwidth}_{p,q}(f) = d^{1/q}
\sqrt[d]{\psi_{p,q}^{d,\netwidth}(f)}$.
\end{claim}
\begin{proof}
  Let $\vecw$ be weights that realizes $f$ and are optimal with respect to
  $\psi_{p,g}$; i.e.~$\psi_{p,q}(\vecw) = \psi^{d,\netwidth}_{p,q}(\vecw)$.  Let
  $\widetilde{W}^k = \sqrt[d]{\psi_{p,q}(\vecw)}W^k/\norm{W^k}_{p,q}$,
  and observe that they also realize $f$.  We now have:
\begin{align}
\mu^{d,\netwidth}_{p,q}(f) \leq \mu_{p,q}(\widetilde{W}) = \Bigl(\sum\nolimits_{k=1}^d \norm{\widetilde{W}^k}^q_{p,q}\Bigr)^{1/q}
= \Bigl(d\Bigl(\psi_{p,q}(\vecw)\Bigr)^{q/d}\Bigr)^{1/q} = d^{1/q}\sqrt[d]{\psi^{d,\netwidth}_{p,q}(f)}\notag
\end{align}
which together with \eqref{eq:mugeqgamma} completes the proof.
\end{proof}
The two measures are therefore equivalent when we use RELUs, and
define the same level sets, or family of model classes, which we
refer to simply as $\calF^{d,\netwidth}_{p,q}$.  In the remainder of this
Section, we investigate convexity and generalization properties of
these model classes.

\subsection{Generalization and Capacity}

In order to understand the effect of the norm on the sample
complexity, we bound the Rademacher complexity of the classes
$\calF^{d,\netwidth}_{p,q}$.  Recall that the Rademacher Complexity is
a measure of the capacity of a model class on a specific sample,
which can be used to bound the difference between empirical and
expected error, and thus the excess generalization error of empirical
risk minimization (see, e.g., \cite{bartlett03} for a
complete treatment, and Section \ref{sec:rademacher} for the exact definitions we
use).  In particular, the Rademacher complexity typically scales as
$\sqrt{C/m}$, which corresponds to a sample complexity of
$O(C/\epsilon^2)$, where $m$ is the sample size and $C$ is the
effective measure of capacity of the model class.

\begin{theorem}\label{thm:l-norm}
For any $d,q\geq 1$, any $1\leq p <\infty$ and any set $\calS=\{x_1,\dots,x_m\}\subseteq\R^{\nin}$:
\begin{align*}
\calR_m(\calF_{\psi_{p,q}\leq \psi}^{d,\netwidth}) &\leq 
\psi \left( 2 \netwidth^{[\frac{1}{p^*} -
        \frac{1}{q}]_+}\right)^{(d-1)} \calR^{\text{linear}}_{m,p,\nin}\\
&\leq \sqrt{ \frac{\psi^2 \left( 2 \netwidth^{[\frac{1}{p^*} -
        \frac{1}{q}]_+}\right)^{2(d-1)}\min\{p^*,4\log(2\nin)\} \max_i \norm{x_i}_{p^*}^2}{m}}
\end{align*}
and so:
\begin{align*}
\calR_m(\calF_{\mu_{p,q}\leq \mu}^{d,\netwidth}) &\leq 
\mu^{d} \left( 2 \netwidth^{[\frac{1}{p^*} -
        \frac{1}{q}]_+} /\sqrt[q]{d}\right)^{(d-1)}\calR^{\text{linear}}_{m,p,\nin}\\
&\leq \sqrt{ \frac{\mu^{2d} \left( 2 \netwidth^{[\frac{1}{p^*} -
        \frac{1}{q}]_+} /\sqrt[q]{d}\right)^{2(d-1)}\min\{p^*,4\log(2\nin)\} \max_i \norm{x_i}_{p^*}^2}{m}}     
\end{align*}
where the second inequalities hold only if $1\leq p \leq 2$, $\calR^{\text{linear}}_{m,p,\nin}$ is the Rademacher complexity of $\nin$-dimensional linear predictors with unit $\ell_p$ norm with respect to a set of $m$ samples and $p^*$ is such that $\frac{1}{p^*} + \frac{1}{p}=1$.
\end{theorem}
\begin{sketch}
We prove the bound by induction, showing that for any $q,d>1$ and $1\leq p < \infty$,
$$
\calR_m(\calF_{\psi_{p,q}\leq \psi}^{d,\netwidth}) \leq 2 \netwidth^{[\frac{1}{p^*} -\frac{1}{q}]_+}\calR_m(\calF_{\psi_{p,q}\leq \psi}^{d-1,\netwidth}).
$$
The intuition is that when $p^*<q$, the Rademacher complexity increases
 by simply distributing the weights among neurons and if $p^*\geq q$
 then the supremum is attained when the output neuron is connected to a neuron with highest Rademacher complexity in the lower layer and all other weights in the top layer are set to zero. For a complete proof, see Section \ref{sec:rademacher}.
\end{sketch}

Note that for $2\leq p < \infty$, the bound on the Rademacher complexity scales with $m^{\frac{1}{p}}$ (see Section \ref{sec:linear}) because:
\begin{equation}
\calR^{\text{linear}}_{m,p,\nin} \leq \frac{\sqrt{2}\norm{X}_{2,p^*}}{m} \leq \frac{\sqrt{2}\max_i\norm{x_i}_{p^*}}{m^{\frac{1}{p}}}
\end{equation}
The bound in Theorem \ref{thm:l-norm} depends on both the magnitude
of the weights, as captured by $\mu_{p,q}(\vecw)$ or $\psi_{p,q}(\vecw)$, and also on
the width $\netwidth$ of the network (the number of nodes in each layer).
However, the dependence on the width $\netwidth$ disappears, and the bound
depends only on the magnitude, as long as $q \leq p^*$
(i.e. $1/p+1/q\geq 1$). This happens, e.g., for overall $\ell_1$ and $\ell_2$ regularization, for per-unit $\ell_1$ regularization, and whenever $1/p+1/q=1$.
In such cases, we can omit the size constraint and state the theorem for an infinite-width layered network (i.e.~a network with an infinitely countable number of units, when the number of units is allowed to be as large as needed):
\begin{corollary}\label{cor:noH}
For any $d\geq 1$, $1\leq p < \infty$ and $1\leq q\leq p^*=p/(p-1)$, and any set $\calS=\{x_1,\dots,x_m\}\subseteq\R^{\nin}$,
\begin{align*}
\calR_m(\calF_{\psi_{p,q}\leq \psi}^{d,\netwidth}) &\leq 
\psi 2^{(d-1)} \calR^{\text{linear}}_{m,p,\nin}\\
&\leq \sqrt{ \frac{\psi^2 \left( 2 \netwidth^{[\frac{1}{p^*} -
        \frac{1}{q}]_+}\right)^{2(d-1)}\min\{p^*,4\log(2\nin)\} \max_i \norm{x_i}_{p^*}^2}{m}}
\end{align*}
and so:
\begin{align*}
\calR_m(\calF_{\mu_{p,q}\leq \mu}^{d,\netwidth}) &\leq 
\left( 2 \mu /\sqrt[q]{d}\right)^{d}\calR^{\text{linear}}_{m,p,\nin}\\
&\leq \sqrt{ \frac{ \left( 2 \mu /\sqrt[q]{d}\right)^{2d} \min\{p^*,4\log(2\nin)\} \max_i \norm{x_i}_{p^*}^2}{m}}     
\end{align*}
where the second inequalities hold only if $1\leq p \leq 2$ and $\calR^{\text{linear}}_{m,p,\nin}$ is the Rademacher complexity of $\nin$-dimensional linear predictors with unit $\ell_p$ norm with respect to a set of $m$ samples.
\end{corollary}

\subsection{Tightness}\label{sec:tight}
We next investigate the tightness of the complexity bound in Theorem
\ref{thm:l-norm}, and show that when $1/p+1/q<1$ the dependence on the
width $\netwidth$ is indeed unavoidable.  We show not only that the bound on
the Rademacher complexity is tight, but that the implied bound on the
sample complexity is tight, even for binary classification with a
margin over binary inputs.  To do this, we show how we can shatter the
$m=2^{\nin}$ points $\{\pm 1\}^{\nin}$ using a network with small group-norm:

\begin{theorem}\label{thm:shattering}
For any $p,q \geq 1$ (and $1/p^*+1/p=1$) and any depth $d\geq 2$, the
$m=2^{\nin}$ points $\{ \pm 1 \}^{\nin}$ can be shattered with unit margin by
$\calF^{d,\netwidth}_{\psi_{p,q}\leq\psi}$ with:
$$
\psi \leq \nin^{1/p} \, m^{1/p+1/q} \, \netwidth^{-(d-2)[1/p^*-1/q]_+}
$$
\end{theorem}
\begin{proof}
Consider a size $m$ subset $S_m$ of $2^{\nin}$ vertices of the $\nin$
 dimensional hypercube  $\{-1,+1\}^{\nin}$. We construct the first layer
 using $m$ units. Each unit 
has a unique weight vector consisting of $+1$ and $-1$'s and will output
a positive value if and only if the sign pattern of the input $x\in\calS_m$ matches
that of the weight vector. The second layer has a single unit and
connects to all $m$ units in the first layer. For any $m$
dimensional sign pattern $b\in\{-1,+1\}^{m}$, we can choose the
weights of the second layer to be $b$, and the network will output the
desired sign for each $x\in\calS_m$ with unit margin. The norm
 of the network is at most
$
(m\cdot \nin^{q/p})^{1/q}\cdot m^{1/p}=\nin^{1/p}\cdot m^{(1/p+1/q)}.
$
This establishes the claim for $d=2$. For $d>2$ and $1/p+1/q\geq 1$, we
 obtain the same norm and unit margin by adding $d-2$ layers with one
 unit in each layer connected to the previous layer by a unit weight.
 For $d>2$ and $1/p+1/q<1$, we show
the dependence on $\netwidth$ by recursively replacing the top unit with $\netwidth$
copies of it and adding an averaging unit on top of that. More
specifically, given the above $d=2$ layer network, we make $\netwidth$ copies of
the output unit with rectified linear activation and add a 3rd layer with
one output unit with uniform weight $1/\netwidth$ to all the copies in the 2nd
layer. Since this operation does not change the output of the network,
we have the same margin and now the norm of the network is
$
(m\cdot \nin^{q/p})^{1/q}\cdot (\netwidth m^{q/p})^{1/q}\cdot (\netwidth(1/\netwidth^p))^{1/p}
=\nin^{1/p}\cdot m^{(1/p+1/q)}\cdot \netwidth^{1/q-1/p^*}.
$
That is, we have reduced the norm by factor $\netwidth^{1/q-1/p^*}$. By repeating
 this process, we get the geometric reduction in the norm
 $\netwidth^{(d-2)(1/q-1/p^*)}$, which concludes the proof.
\end{proof}

To understand this lower bound, first consider the bound without the
dependence on the width $\netwidth$.  We have that for any depth $d\geq 2$,
$\psi \leq m^r \nin = m^r \log m$ (since $1/p\leq 1$ always) where
$r=1/p+1/q\leq 2$.  This means that for any depth $d\geq 2$ and any
$p,q$ the sample complexity of learning the class scales as
$m=\Omega(\psi^{1/r}/\log \psi) \geq
\tilde{\Omega}(\sqrt{\psi})$.  This shows a polynomial dependence on
$\psi$, though with a lower exponent than the $\psi^2$ 
(or higher for $p>2$) dependence
in Theorem \ref{thm:l-norm}.  Still, if we now consider the complexity
control as a function of $\mu_{p,q}$ we get a sample complexity of at
least $\Omega(\mu^{d/2}/\log \mu)$, establishing that if we control
the group-norm as in \eqref{eq:mu}, we cannot avoid a sample
complexity which depends exponentially on the depth.  Note that in our
construction, all other factors in Theorem \ref{thm:l-norm}, namely
$\max_i\norm{x_i}$ and $\log \nin$, are logarithmic (or double-logarithmic) in
$m$.

Next we consider the dependence on the width $\netwidth$ when $1/p+1/q<1$.
Here we have to use depth $d\geq 3$, and we see that indeed as the
width $\netwidth$ and depth $d$ increase, the magnitude control $\psi$ can
decrease as $\netwidth^{(1/p^*-1/q)(d-2)}$ without decreasing the capacity,
matching Theorem 1 up to an offset of 2 on the depth.  In particular,
we see that in this regime we can shatter an arbitrarily large number
of points with arbitrarily low $\psi$ by using enough hidden units,
and so the capacity of $\calF^d_{p,q}$ is indeed infinite and it
cannot ensure any generalization.

\subsection{Convexity}

Finally we establish a sufficient condition for the model classes
$\calF^d_{p,q}$ to be convex.  We are referring to convexity of the
functions in the $\calF^d_{p,q}$ independent of a specific
representation.  If we consider a, possibly regularized, empirical
risk minimization problem on the weights, the objective (the empirical
risk) would never be a convex function of the weights (for depth
$d\geq 2$), even if the regularizer is convex in $w$ (which it always
is for $p,q\geq 1$).  But if we do not bound the width of the network,
and instead rely on magnitude-control alone, we will see that the
resulting model class, and indeed the complexity measure, may be
convex (with respect to taking convex combinations of functions, {\em
  not} of weights).

\begin{theorem}\label{thm:cvx}
  For any $d,p,q \geq 1$ such that $\frac{1}{q}\leq
  \frac{1}{d-1}\big(1-\frac{1}{p}\big)$, $\psi^d_{p,q}(f)$ is a semi-norm
  in $\calF^d$.
\end{theorem}
In particular, under the condition of the Theorem, $\psi^d_{p,q}$ is
convex, and hence its sublevel sets $\calF^d_{p,q}$ are convex, and so
$\mu^d_{p,q}$ is quasi-convex (but not convex). \removed{ The condition holds
for per-unit regularization ($q=\infty$) for any $p\geq 1$, and for
overall regularization ($q=p$) whenever $p=q \geq d$.  However, }

\begin{sketch}
  To show convexity, consider two functions
  $f,g\in\calF^d_{\psi_{p,q}\leq \psi}$ and $0<\alpha<1$, and
  let $U$ and $V$ be the weights realizing $f$ and $g$ respectively
  with $\psi_{p,q}(U)\leq \psi$ and $\psi_{p,q}(V)\leq\psi$.  We will construct
  weights $\vecw$ realizing $\alpha f+(1-\alpha)g$ with
  $\psi_{p,q}(\vecw)\leq \psi$.  This is done by first balancing $U$
  and $V$ s.t.~at each layer
  $\norm{U_i}_{p,q}=\sqrt[d]{\psi_{p,q}(U)}$ and $\norm{V_i}_{p,q}=\sqrt[d]{\psi_{p,q,}(V)}$ and then
  placing $U$ and $V$ side by side, with no interaction between the
  units calculating $f$ and $g$ until the output layer.  The output
  unit has weights $\alpha U_d$ coming in from the $f$-side and
  weights $(1-\alpha)V_d$ coming in from the $g$-side.  In Section \ref{sec:proof-cvx} we show that under the condition in the theorem, $\psi_{p,q}(\vecw)\leq\psi$.  To complete the proof, we also show $\psi^d_{p,q}$ is homogeneous and that this is sufficient for convexity.
\end{sketch}

\section{Per-Unit and Path Regularization}\label{sec:path}

In this Section we will focus on the special case of $q=\infty$,
 i.e.~when we constrain the norm of the incoming weights of each unit
separately.  

Per-unit $\ell_1$-regularization was studied by
\cite{bartlett98,koltchinskii02,bartlett03} who showed generalization
guarantees.  A two-layer network of this form with RELU activation was
also considered by \cite{Bach14}, who studied its
approximation ability and suggested heuristics for learning
it.  Per-unit $\ell_2$ regularization in a two-layer network was
considered by \cite{Cho09}, who showed it is equivalent to using a
specific kernel. We now introduce \emph{Path regularization} and discuss
its equivalence to Per-Unit regularization.

\paragraph{Path Regularization}
Consider a regularizer which looks at the sum over all paths from
input nodes to the output node, of the product of the weights along
the path:
\begin{equation}
  \label{eq:pathr}
  \pathr_p(\vecw) = \Bigl(
    \sum_{\vin[i]\overset{e_1}{\rightarrow}v_1\overset{e_2}{\rightarrow}v_2\cdots\overset{e_k}{\rightarrow}\vout} \prod_{i=1}^k \abs{w_{e_i}}^p \Bigr)^{1/p}
\end{equation}
where $p\geq 1$ controls the norm used to aggregate the paths.  We can
motivate this regularizer as follows: if a node does not have any
high-weight paths going out of it, we really don't care much about
what comes into it, as it won't have much effect on the output.  The
path-regularizer thus looks at the aggregated influence of all the
weights.

Referring to the induced regularizer $\pathr^G_p(f) = \min_{f_{\vecw}=f}
\pathr_p(\vecw)$ (with the usual shorthands for layered graphs), we now
observe that for layered graphs, path regularization and per-unit
regularization are equivalent:
\begin{theorem}\label{thm:path-layer}
For $p\geq 1$, any $d$ and (finite or infinite) $\netwidth$, for any $f_\vecw\in\calF^{d,\netwidth}$:  $\pathr_p^{d,\netwidth}(f) = \psi^{d,\netwidth}_{p,\infty}$
\end{theorem}
It is important to emphasize that even for layered graphs, it is not
the case that for all weights $\pathr_p(\vecw)=\psi_{p,\infty}(\vecw)$.
E.g., a high-magnitude edge going into a unit with no non-zero
outgoing edges will affect $\psi_{p,\infty}(\vecw)$ but not
$\pathr_p(\vecw)$, as will having high-magnitude edges on different layers
in different paths.  In a sense path regularization is as more
careful regularizer less fooled by imbalance.  Nevertheless, in the
proof of Theorem \ref{thm:path-layer} in Section \ref{sec:path-layer}, we show we
can always balance the weights such that the two measures are equal.

The equivalence does not extend to non-layered graphs, since the
lengths of different paths might be different.  Again, we can think of
path regularizer as more refined regularizer taking into account the
local structure.  However, if we consider all DAGs of depth at most
$d$ (i.e.~with paths of length at most $d$), the notions are again
equivalent (see proof in Section \ref{sec:proof-path-dag}):
\begin{theorem}\label{thm:path-dag}
  For any $p\geq 1$ and any $d$: $\displaystyle \psi^d_{p,\infty}(f) =
  \min_{\textrm{$G\in \DAG(d)$}} \pathr^G_p(f)$.
\end{theorem}

In particular, for any graph $G$ of depth $d$, we have that
$\pathr^G_p(f) \geq\psi^d_{p,\infty}(f)$.  Combining this
observation with Corollary \ref{cor:noH} allows us to immediately obtain a
generalization bound for path regularization on any, even non-layered,
graph:
\begin{corollary}
  For any graph $G$ of depth $d$ and any set
  $\calS=\{x_1,\dots,x_m\}\subseteq\R^{\nin}$:
$$\calR_m(\calF^G_{\pathr_1\leq \pathr}) \leq \sqrt{\frac{4^{d-1} \pathr^2
    \cdot 4\log(2\nin) \sup \norm{x_i}_\infty^2}{m}}$$
\end{corollary}
Note that in order to apply Corollary \ref{cor:noH} and obtain a
width-independent bound, we had to limit ourselves to $p=1$.  We
further explore this issue next.

\paragraph{Capacity}

As was previously noted, size-independent generalization bounds for
bounded depth networks with bounded per-unit $\ell_1$ norm have long
been known (and make for a popular homework problem).  These
correspond to a specialization of Corollary \ref{cor:noH} for the case
$p=1,q=\infty$.  Furthermore, the kernel view of \cite{Cho09} allows
obtaining size-independent generalization bound for {\em two-layer}
networks with bounded per-unit $\ell_2$ norm (i.e.~a single infinite
hidden layer of all possible unit-norm units, and a bounded
$\ell_2$-norm output unit).  However, the lower bound of Theorem
\ref{thm:shattering} establishes that for any $p>1$, once we go beyond
two layers, we cannot ensure generalization without also controlling
the size (or width) of the network.

\paragraph{Convexity}
An immediately consequence of Theorem \ref{thm:cvx} is that per-unit
regularization, if we do not constrain the network width, is convex
for any $p\geq 1$.  In fact, $\psi^d_{p,\infty}$ is a (semi)norm.
However, as discussed above, for depth $d>2$ this is meaningful only
for $p=1$, as $\psi^d_{p,\infty}$ collapses for $p>1$.

\paragraph{Hardness} Since the classes $\calF^d_{1,\infty}$ are
convex, we might hope that this might make learning computationally
easier.  Indeed, one can consider functional-gradient or boosting-type
strategies for learning a predictor in the class \citep{lee96}.  However, as
\citet{Bach14} points out, this is not so easy as it
requires finding the best fit for a target with a RELU unit, which is
not easy.  Indeed, applying results on hardness of learning
intersections of halfspaces, which can be represented with small
per-unit norm using two-layer networks, we can conclude that, subject
to certain complexity assumptions, it is not possible to
efficiently PAC learn $\calF^d_{1,\infty}$, even for depth $d=2$ when $\psi_{1,\infty}$ increases superlinearly:
\begin{corollary}\label{cor:hard1}
  Subject to the the strong random CSP assumptions in \cite{Daniely14}, it is
  not possible to efficiently PAC learn (even improperly) functions
  $\{\pm 1\}^{\nin} \rightarrow \{ \pm 1 \}$ realizable with unit margin by
  $\calF^2_{1,\infty}$ when $\psi_{1,\infty}=\omega(\nin)$ (e.g.~when
  $\psi_{1,\infty}=\nin \log \nin$). Moreover, subject to intractability of 
  $\tilde{Q}(\nin^{1.5})$-unique shortest vector problem, for any $\epsilon>0$,
  it is not possible to efficiently PAC learn (even improperly) functions
  $\{\pm 1\}^{\nin} \rightarrow \{ \pm 1 \}$ realizable with unit margin by
  $\calF^2_{1,\infty}$ when $\psi_{1,\infty}=\nin^{1+\epsilon}$.
\end{corollary}
This is a corollary of Theorem \ref{thm:hardness} in the Section
\ref{sec:hardness}.  Either versions of corollary \ref{cor:hard1} precludes the
possibility of learning in time polynomial in $\psi_{1,\infty}$,
though it still might be possible to learn in $\textrm{poly}(\nin)$ time
when $\psi_{1,\infty}$ is sublinear.

\paragraph{Sharing} We conclude this Section with an observation on
the type of networks obtained by per-unit, or equivalently path,
regularization.
\begin{theorem}\label{thm:opt-tree}
  For any $p\geq 1$ and $d>1$ and any $f_\vecw\in\calF^d$, there exists a
  layered graph $G(V,E)$ of depth $d$, such that $f_\vecw\in\calF^G$ and
  $\psi^G_{p,\infty}(f)=\pathr^G_p(f)=\psi^d_{p,\infty}(f)$, and
  the out-degree of every internal (non-input) node in $G$ is one.
  That is, the subgraph of $G$ induced by the non-input vertices is a
  tree directed toward the output vertex.
\end{theorem}
What the Theorem tells us is that we can realize every function as a
tree with optimal per-unit norm.  If we think of learning with an
infinite fully-connected layered network, we can always restrict
ourselves to models in which the non-zero-weight edges form a tree.
This means that when using per-unit regularization we have no
incentive to ``share'' lower-level units---each unit will only have a
single outgoing edge and will only be used by a single down-stream
unit.  This seems to defy much of the intuition and power of using
deep networks, where we expect lower layers to represent generic
feature useful in many higher-level features.  In effect, we are not
encouraging any transfer between learning different aspects of the
function (or between different tasks or classes, if we do have
multiple output units).  Per-unit regularization therefore misses out
on much of the inductive bias that we might like to impose when using
deep learning (namely, promoting sharing).

\begin{proof} {[of Theorem \ref{thm:opt-tree}]}
For any $f_{\vecw}\in \calF^{\DAG(d)}$, we show how to construct such $\widetilde{G}$ and $\widetilde{w}$. 
We first sort the vertices of $G$ based on topological ordering such that the out-degree of the first vertex is zero.
Let $G_0=G$ and $w_0=w$. At each step $i$, we first set $G_i=G_{i-1}$ and $w_i=w_{i-1}$ and then pick the vertex $u$ that is the $i$th vector in the topological ordering.
If the out-degree of $u$ is at most 1. Otherwise, for any edge $(u\rightarrow v)$ we create a copy of vertex $u$ that we call it $u_v$, add the edge $(u_v\rightarrow v)$ to $G_i$ and connect all incoming edges of $u$ with the same weights to every such $u_v$ and finally we delete the vertex $u$ from $G_i$ together with all incoming and outgoing edges of $u$. It is easy to indicate that $f_{G_{i},w_i}=f_{G_{i-1},w_{i-1}}$. After at most $|V|$ such steps, all internal nodes have out-degree one and hence the subgraph induced by non-input vertices will be a tree.
\end{proof}

\section{Overall Regularization}
\label{sec:overall}
In this Section, we will focus on ``overall'' $\ell_p$ regularization,
corresponding to the choice $q=p$, i.e.~when we bound the overall
(vectorized) norm of all weights in the system:
$$\mu_{p,p}(\vecw)=\Bigl( \sum_{e\in E} \abs{w(e)}^p \Bigr)^{1/p}.$$

\paragraph{Capacity}
For $p\leq 2$, Corollary \ref{cor:noH} provides a generalization
guarantee that is independence of the width---we can conclude that if
we use weight decay (overall $\ell_2$ regularization), or any tighter
$\ell_p$ regularization, there is no need to limit ourselves to
networks of finite size (as long as the corresponding dual-norm of the
inputs are bounded).  However, in Section \ref{sec:tight} we saw that
with $d \geq 3$ layers, the regularizer degenerates and leads to
infinite capacity classes if $p>2$.  In any case, even if we bound the
overall $\ell_1$-norm, the complexity increases exponentially with the
depth.

\paragraph{Convexity} The conditions of Theorem \ref{thm:cvx} for
convexity of $\calF^d_{2,2}$ are ensured when $p \geq d$.  For depth
$d=1$, i.e.~a single unit, this just confirms that
$\ell_p$-regularized linear prediction is convex for $p\geq 1$.  For
depth $d=2$, we get convexity with $\ell_2$ regularization, but not
$\ell_1$.  For depth $d>2$ we would need $p>d\geq 3$, however for such
values of $p$ we know from Theorem \ref{thm:shattering} that
$\calF^d_{p,p}$ degenerates to an infinite capacity class if we do not
control the width (if we do control the width, we do not get
convexity).  This leaves us with $\calF^2_{2,2}$ as the interesting
convex class.  Below we show an explicit convex characterization of
$\calF^2_{2,2}$ by showing it is equivalent to so-called ``convex neural
nets''.

{\em Convex Neural Nets} \citep{Bengio05} over inputs in $\R^{\nin}$ are
two-layer networks with a fixed infinite hidden layer consisting of
all units with weights $w\in\calG$ for some base class $\calG\in\R^{\nin}$, and
a second $\ell_1$-regularized layer.  Since over finite data the
weights in the second layer can always be taken to have finite support
(i.e.~be non-zero for only a finite number of first-layer units), and
we can approach any function with countable support, we
can instead think of a network in $\calF^2$ where the bottom layer is
constraint to $\calG$ and the top layer is $\ell_1$ regularized.
Focusing on $\calG=\{ w \,|\, \norm{w}_p \leq 1 \}$, this corresponds to
imposing an $\ell_p$ constraint on the bottom layer, and $\ell_1$
regularization on the top layer and yields the following complexity
measure over $\calF^2$:
\begin{equation}
  \label{eq:convexNN}
  \convexnn_p(f) = \inf_{f_{\layer(d),W}=f, \textrm{s.t.} \forall_j
    \norm{W_1[j,:]}_p \leq 1} \norm{W_2}_1.
\end{equation}
This is similar to per-unit regularization, except we impose different
norms at different layers (if $p\not=1$).  We can see that
$\calF^2_{\convexnn_p \leq \convexnn} = \convexnn \cdot
\overline{\conv}(\sigma(\calG))$, and is thus convex for any $p$.
Focusing on RELU activation we have the equivalence:
\begin{theorem}
  $\displaystyle \mu^2_{2,2}(f) = 2 \convexnn_2(f).$
\end{theorem}
That is, overall $\ell_2$ regularization with two layers is equivalent
to a convex neural net with $\ell_2$-constrained units on the bottom
layer and $\ell_1$ (not $\ell_2$!) regularization on the output.
\begin{proof}We can calculate:
\begin{align}
\min_{f_W=f}
\mu_{2,2}^2(\vecw)&=\min_{f_W=f}\sum_{j=1}^{\netwidth}\left(\sum_{i=1}^{\nin}|W_1[j,i]|^2+|W_2[j]|^2\right)
\notag \\
&=\min_{f_W=f}
\sum_{j=1}^{\netwidth}2\sqrt{\sum\nolimits_{i=1}^{\nin}|W_1[j,i]|^2}\cdot|W_2[j]|
\label{eq:convexnnproof1} \\
&=2 \min_{f_W=f} \sum_{j=1}^{\netwidth}\left|W_2[j]\right|\quad \text{s.t.}\quad 
\sqrt{\sum\nolimits_{i=1}^{\nin}|W_1[j,i]|^2}\leq 1. \label{eq:convexnnproof2}
\end{align}
Here \eqref{eq:convexnnproof1} is the arithmetic-geometric mean
inequality for which we can achieve equality by balancing the weights
(as in Claim \ref{clm:mugamma}) and \eqref{eq:convexnnproof2} again
follows from the homogeneity of the RELU which allows us to rebalance
the weights.
\end{proof}

\paragraph{Hardness} As with $\calF^d_{1,\infty}$, we might hope that the convexity of
$\calF^2_{2,2}$ might make it computationally easy to learn.  However,
by the same reduction from learning intersection of halfspaces
(Theorem \ref{thm:hardness} in Section \ref{sec:hardness}) we can
again conclude that we cannot learn in time polynomial in $\mu^2_{2,2}$:

\begin{corollary}\label{cor:hard2}
  Subject to the the strong random CSP assumptions in \cite{Daniely14}, it is
  not possible to efficiently PAC learn (even improperly) functions
  $\{\pm 1\}^{\nin} \rightarrow \{ \pm 1 \}$ realizable with unit margin by
  $\calF^2_{p,p}$ when $\mu^2_{p,p}=\omega(\nin^{\frac{1}{p}})$. (e.g.~when
  $\psi_{1,\infty}=\nin \log \nin$). Moreover, subject to intractability of 
  $\tilde{Q}(\nin^{1.5})$-unique shortest vector problem, for any $\epsilon>0$,
  it is not possible to efficiently PAC learn (even improperly) functions
  $\{\pm 1\}^{\nin} \rightarrow \{ \pm 1 \}$ realizable with unit margin by
  $\calF^2_{1,\infty}$ when $\psi_{1,\infty}=\nin^{\frac{1}{p}+\epsilon}$.
\end{corollary}


\section{Depth Independent Regularization}
\label{sec:nod}
Up until now we discussed relying on magnitude-based regularization
instead of directly controlling network size, thus allowing unbounded
and even infinite width. But we still relied on a finite bound on the
depth in all our derivations.  Can the explicit dependence on the
depth be avoided, and replaced with only a measure of scale of the
weights?

We already know we cannot rely only on a bound on the group-norm
$\mu_{p,q}$ when the depth is unbounded, as we know from Theorem
\ref{thm:shattering} that in terms of $\mu_{p,q}$ the sample
complexity necessarily increases exponentially with the depth: if we
allow arbitrarily deep graphs we can shrink $\mu_{p,q}$ toward zero
without changing the scale of the computed function.  However,
controlling the $\psi$-measure, or equivalently the path-regularizer
$\pathr$, in arbitrarily-deep graphs is sensible, and we can define:
\begin{equation}
  \label{eq:infgamma}
  \psi_{p,q} = \inf_{d\geq 1} \psi^d_{p,q}(f) =
  \lim_{d\rightarrow\infty} \psi^d_{p,q}(f) \quad\quad\text{or:}\quad
  \pathr_p = \inf_G \pathr^G_p(f)
\end{equation}
where the minimization is over {\em any} DAG.  From Theorem \ref{thm:path-dag}
we can conclude that $\pathr_p(f)=\psi_{p,\infty}(f)$.  In any case,
$\psi_{p,q}(f)$ is a sensible complexity measure, that does not
collapse despite the unbounded depth.  Can we obtain generalization
guarantees for the class $\calF_{\psi_{p,q}\leq\psi}$ ?

Unfortunately, even when $1/p+1/q \geq 1$ and we can obtain
width-independent bounds, the bound in Corollary \ref{cor:noH} still
has a dependence on $4^d$, even if $\psi_{p,q}$ is bounded.  Can
such a dependence be avoided?

For {\em anti-symmetric} Lipschitz-continuous activation functions
(i.e.~such that $\sigma(-z)=-\sigma(z)$), such as the ramp, and for
per-unit $\ell_p$-regularization $\mu^d_{1,\infty}$ we can avoid the factor of $4^d$
\begin{theorem}\label{thm:antisym}
For any anti-symmetric 1-Lipschitz function $\sigma$ and any set $\calS=\{x_1,\dots,x_m\}\subseteq\R^{\nin}$:
$$
\calR_m(\calF_{\mu_{1,\infty}\leq \mu}^{d}) \leq 
\sqrt{\frac{4\mu^{2d} \log(2\nin) \sup \norm{x_i}_{\infty}^2}{m}}
$$
\end{theorem}
The proof is again based on an inductive argument similar to Theorem \ref{thm:l-norm} and you can find it in Section \ref{sec:antisym}.

However, the ramp is not homogeneous and so the equivalent between
$\mu$, $\psi$ and $\phi$ breaks down.  Can we obtain such a bound
also for the RELU?  At the very least, what we can say is that an
inductive argument such that used in the proofs of Theorems
\ref{thm:l-norm} and \ref{thm:antisym} cannot be used to avoid an
exponential dependence on the depth.  To see this, consider
$\psi_{1,\infty}\leq 1$ (this choice is arbitrary if we are
considering the Rademacher complexity), for which we have
\begin{equation}\label{eq:Nd1}
\calF^{d+1}_{\psi_{1,\infty}<1} = \left[
  \overline{\conv}(\calF^d_{\psi_{1,\infty}<1}) \right]_+,
\end{equation}
where $\overline\conv(\cdot)$ is the symmetric convex hull, and
$[\cdot]_+ = \max(z,0)$ is applied to each function in the class.  In
order to apply the inductive argument without increasing the
complexity exponentially with the depth, we would need the operation
$[ \overline{\conv}(\cal\netwidth) ]_+$ to preserve the Rademacher complexity,
at least for non-negative convex cones $\cal\netwidth$.  However we show a
simple example of a non-negative convex cone $\cal\netwidth$ for which
$\calR_m\left( [ \overline{\conv}(\cal\netwidth) ]_+ \right) > \calR_m\left(
  \cal\netwidth \right)$.

We will specify $\cal\netwidth$ as a set of vectors in $\R^m$, corresponding
to the evaluation of $h(x_i)$ of different functions in the class on
the $m$ points $x_i$ in the sample.  In our construction, we will have
only $m=3$ points.  Consider $\cal\netwidth = \conv(\{ (1,0,1),(0,1,1)
\})$, in which case $\cal\netwidth' \defeq [ \overline{\conv}(\cal\netwidth) ]_+ =
\conv(\{ (1,0,1),(0,1,1),(0.5,0,0) \})$.  It is not hard to verify
that $\calR_m(\cal\netwidth')=\frac{13}{16}>\frac{12}{16}=\calR_m(\cal\netwidth)$.

\section{Proofs}\label{sec:norm-based-proofs}
\subsection{Rademacher Complexities}\label{sec:rademacher}
The sample based Rademacher complexity of a class $\calF$ of function mapping from $\calX$ to $\R$ with respect to a set $\calS=\{x_1,\dots,x_m\}$ is defined as:
$$
\calR_m(\calF) = \E{\xi \in \{\pm 1\}^m}{\frac{1}{m}
\sup_{f_\vecw\in \calF}  \left\lvert \sum_{i=1}^m \xi_i f(x_i) \right\rvert}
$$

In this section, we prove an upper bound for the Rademacher complexity
of the class $\calF_{\psi_{p,q}\leq
\psi}^{d,\netwidth_{\text{RELU}}}$, i.e., the class of functions that
can be represented as depth $d$, width $\netwidth$ network with rectified linear
activations, and
the layer-wise group norm complexity $\psi_{p,q}$ bounded by
$\psi$. As mentioned in the main text, our proof is an induction with
respect to the depth $d$.
We start with $d=1$ layer neural networks, which is essentially  the class of linear separators.
\subsubsection{$\ell_p$-regularized Linear Predictors}\label{sec:linear}

For completeness, we prove the upper bounds on the Rademacher complexity of class of linear separators with bounded $\ell_p$ norm. The upper bounds presented here are particularly similar to generalization bounds in \cite{kakade09} and \cite{balcan14}. We first mention two already established lemmas that we use in the proofs.
\begin{theorem}(Khintchine-Kahane Inequality)
For any $0<p<\infty$ and $\calS=\{z_1,\dots,z_m\}$, if the random 
variable $\xi$ is uniform over $\{\pm 1\}^m$, then
$$
\left(\E{\xi}{\left\lvert \sum_{i=1}^m \xi_i z_i
\right\rvert ^p}\right)^\frac{1}{p} \leq C_p 
\left(\sum_{i=1}^m |z_i|^2\right)^\frac{1}{2}
$$
where $C_p$ is a constant depending only on $p$.
\end{theorem}
The sharp value of the constant $C_p$ was found by 
\citet{haagerup81} but for our analysis, it is enough to note
 that if $p\geq 1$ we have $C_p \leq \sqrt{p}$.
\begin{lemma}(Massart Lemma)
Let $A$ be a finite set of $m$ dimensional vectors. Then
$$
\E{\xi}{
\max_{a\in A}\frac{1}{m}\sum_{i=1}^{m}\xi_ia_i
}\leq \max_{a\in A} \norm{a}_2 \frac{\sqrt{2\log |A|}}{m},
$$
where $|A|$ is the cardinality of $A$.
\end{lemma}

We are now ready to show upper bounds on Rademacher complexity of linear separators with bounded $\ell_p$ norm.

\begin{lemma}\label{lem:layer1}(Rademacher complexity of linear separators with bounded $\ell_p$ norm)
For any $d,q\geq 1$,
For any $1\leq p\leq 2$,
$$
\calR_m
(\calF^1_{\psi_{p,q}\leq \psi}) \leq \sqrt{ \frac{ \psi^2\min\{p^*,4\log(2\nin)\} \max_i \norm{x_i}_{p^*}^2}{m}}
$$
and for any $2<p<\infty$
$$
\calR_m
(\calF^1_{\psi_{p,q}\leq \psi})\leq \frac{\sqrt{2}\psi\norm{X}_{2,p^*}}{m} \leq \frac{\sqrt{2}\psi\max_i\norm{x_i}_{p^*}}{m^{\frac{1}{p}}}
$$
where $p^*$ is such that $\frac{1}{p^*} + \frac{1}{p}=1$.
\end{lemma}
\begin{proof}
First, note that $\calF^1$ is the class of linear functions
and hence for any function $f_{w}\in\calF^1$, we have
that $\psi_{p,q}(\vecw)=\norm{w}_p$. Therefore,
we can write the Rademacher complexity for a set 
$\calS=\{x_1,\dots,x_m\}$ as:
\begin{align*}
\calR_m(\calF^1_{\psi_{p,q}\leq \psi})
&= \E{\xi \in \{\pm 1\}^m}{\frac{1}{m}\sup_{\norm{w}_p\leq \psi}
  \left\lvert \sum_{i=1}^m \xi_i w^\top x_i \right\rvert }\\
&= \E{\xi \in \{\pm 1\}^m}{\frac{1}{m}\sup_{\norm{w}_p\leq \psi}
  \left\lvert w^\top\sum_{i=1}^m \xi_i x_i \right\rvert}\\
&= \psi\E{\xi \in \{\pm 1\}^m}{\frac{1}{m}
\norm{ \sum_{i=1}^m \xi_i x_i}_{p^*}}
\end{align*}
For $1\leq p\leq \min\left\{2,\frac{2\log(2\nin)}{2\log(2\nin)-1}\right\}$ (and therefore $2\log(2\nin) \leq p^*$), we have
\begin{align*}
\calR_m(\calF^1_{\psi_{p,q}\leq \psi}) &=\psi\E{\xi \in \{\pm 1\}^m}
{\frac{1}{m}\norm{ \sum_{i=1}^m \xi_i x_i}_{p^*}}\\
&\leq \nin^{\frac{1}{p^*}}\psi\E{\xi \in \{\pm 1\}^m}
\left[\frac{1}{m}\norm{ \sum_{i=1}^m \xi_i x_i}_{\infty}\right]\\
&\leq \nin^{\frac{1}{2\log(2\nin)}}\psi\E{\xi \in \{\pm 1\}^m}
\frac{1}{m}\norm{ \sum_{i=1}^m \xi_i x_i}_{\infty}\\
&\leq \sqrt{2}\psi\E{\xi \in \{\pm 1\}^m}
{\frac{1}{m}\norm{ \sum_{i=1}^m \xi_i x_i}_{\infty}}\\
\end{align*}
We now use  the Massart Lemma viewing each feature $(x_i[j])_{i=1}^{m}$ for
 $j=1,\ldots,\nin$ as a  member of a finite model class and obtain
 \begin{align*}
\calR_m(\calF^1_{\psi_{p,q}\leq \psi}) &\leq  \sqrt{2}\psi\E{\xi \in \{\pm 1\}^m}
{\frac{1}{m}\norm{ \sum_{i=1}^m \xi_i x_i}_{\infty}}\\
&\leq 2\psi\frac{\sqrt{\log(2\nin)}}{m}\max_{j=1\ldots,\nin}\norm{(x_i[j])_{i=1}^{m}}_2\\
&\leq 2\psi \sqrt{\frac{\log(2\nin)}{m}}\max_{i=1,\ldots,m}\norm{x_i}_{\infty}\\
&\leq 2\psi \sqrt{\frac{\log(2\nin)}{m}}\max_{i=1,\ldots,m}\norm{x_i}_{p^*}
\end{align*}
If $\min\left\{2,\frac{2\log(2\nin)}{2\log(2\nin)-1}\right\} <p<\infty$, by Khintchine-Kahane inequality we have
\begin{align*}
\calR_m(\calF^1_{\psi_{p,q}\leq \psi}) &=\psi\E{\xi \in \{\pm 1\}^m}
\left[\frac{1}{m}\norm{ \sum_{i=1}^m \xi_i x_i}_{p^*}\right]\\
&\leq \psi\frac{1}{m}\left(\sum_{j=1}^{\nin} \E{\xi \in \{\pm 1\}^m}\left[
\left\lvert \sum_{i=1}^m \xi_i x_i[j]\right\rvert ^{p^*}\right]\right)^{1/p^*}\\
&\leq\psi\frac{\sqrt{p^*}}{m}\left(\sum\nolimits_{j=1}^{\nin}\norm{(x_i[j])_{i=1}^{m}}_2^{p^*}\right)^{1/p^*} =\psi\frac{\sqrt{p^*}}{m}\norm{X}_{2,p^*}
\end{align*}
If $p^*\geq 2$, by Minskowski inequality we have that $\norm{X}_{2,p^*} \leq m^{1/2} \max_i \norm{x_i}_{p^*}$. Otherwise, by subadditivity of the function $f(z)=z^{\frac{p^*}{2}}$, we get $\norm{X}_{2,p^*} \leq m^{1/{p^*}} \max_i \norm{x_i}_{p^*}$.

\end{proof}

\subsubsection{Theorem~\ref{thm:l-norm}}
We define the model class $\calF^{d,\netwidth,\netwidth}$ to be the class of functions from $\calX$ to $\R^\netwidth$ computed by a layered network of depth $d$, layer size $\netwidth$ and $\netwidth$ outputs.

For the proof of theorem~\ref{thm:l-norm}, we need the following two
technical lemmas.
The first is the well-known contraction lemma:
\begin{lemma}(Contraction Lemma)
Let function $\phi:\R\rightarrow \R$ be Lipschitz with constant $\calL_\phi$ such that $\phi$ satisfies $\phi(0)=0$. Then for any class $\calF$ of functions mapping from $\calX$ to $\R$ and any set $\calS=\{x_1,\dots,x_m\}$:
$$
\E{\xi \in \{\pm 1\}^m}\left[\frac{1}{m}
\sup_{f_\vecw\in \calF}  \left\lvert \sum_{i=1}^m \xi_i \phi( f(x_i) ) \right\rvert \right] \leq 2 \calL_\phi \E{\xi \in \{\pm 1\}^m}\left[\frac{1}{m}
\sup_{f_\vecw\in \calF}  \left\lvert \sum_{i=1}^m \xi_i f(x_i) ) \right\rvert \right]
$$
\end{lemma}
Next, the following lemma reduces the maximization over a matrix
$W\in\R^{\netwidth\times \netwidth}$ that
appears in the computation of Rademacher complexity to $\netwidth$ independent
maximizations over a vector $w\in\R^{\netwidth}$ (the proof is deferred to
subsubsection \ref{sec:proof-lem-l-norm}):
\begin{lemma}\label{lem:l-norm}
For any $p,q\geq 1$, $d\geq 2$, $\xi \in \{\pm 1\}^m$ and
 $f_\vecw\in \calF^{d,\netwidth,\netwidth}$ we have
$$
\sup_{W}\frac{1}{\norm{W}_{p,q}}
\norm{\sum_{i=1}^m \xi_i [W[f(x_i)]_+]_+}_{p^*}=\netwidth^{[\frac{1}{p^*} - \frac{1}{q}]_+}
 \sup_{w} \frac{1}{\norm{w}_{p}} \left\lvert \sum_{i=1}^m \xi_i 
 [w^\top [f(x_i)]_+]_+\right\rvert 
$$
where $p^*$ is such that $\frac{1}{p^*} + \frac{1}{p}=1$.
\end{lemma}

\begin{reptheorem}{thm:l-norm}
For any $d,p,q\geq 1$ and any set $\calS=\{x_1,\dots,x_m\}\subseteq\R^{\nin}$:
$$
\calR_m(\calF_{\psi_{p,q}\leq \psi}^{d,\netwidth}) \leq 
\sqrt{\frac{\psi^2 \left( 2 \netwidth^{[\frac{1}{p^*} -
        \frac{1}{q}]_+}\right)^{2(d-1)} \min\{p^*,2\log(2\nin)\} \sup \norm{x_i}_{p^*}^2}{m}}
$$
and so:
$$
\calR_m(\calF_{\mu_{p,q}\leq \mu}^{d,\netwidth}) \leq 
\sqrt{\frac{\mu^{2d} \left( 2 \netwidth^{[\frac{1}{p^*} -
        \frac{1}{q}]_+} /\sqrt[q]{d}\right)^{2(d-1)} \min\{p^*,2\log(2\nin)\} \sup \norm{x_i}_{p^*}^2}{m}}
$$
where $p^*$ is such that $\frac{1}{p^*} + \frac{1}{p}=1$.
\end{reptheorem}
\begin{proof}
By the definition of Rademacher complexity if $\xi$ is uniform over $\{\pm 1\}^m$, we have:
\begin{align}\notag
\calR_m(\calF^{d,\netwidth}_{\psi_{p,q}\leq \psi}) &= \E{\xi}\left[\frac{1}{m} \sup_{f_\vecw\in \calF^{d,\netwidth}_{\psi_{p,q}\leq \psi}} \left\lvert \sum_{i=1}^m 
\xi_i f(x_i) \right\rvert \right]\\ \notag
&= \E{\xi}\left[\frac{1}{m} 
\sup_{f_\vecw\in\calF^{d,\netwidth}}  \frac{\psi}{\psi_{p,q}(\vecw)}
\left\lvert \sum_{i=1}^m \xi_i f(x_i) \right\rvert \right]\\ \notag
&= \E{\xi}\left[\frac{1}{m}  
\sup_{g\in \calF^{d-1,\netwidth,\netwidth}}  \sup_{w} 
\frac{\psi}{\psi_{p,q}(g)\norm{w}_p} \left\lvert \sum_{i=1}^m \xi_i 
w^\top [g(x_i)]_+ \right\rvert \right]\\ \notag
&=\E{\xi }\left[\frac{1}{m}  
\sup_{g\in \calF^{d-1,\netwidth,\netwidth}}   \frac{\psi}{\psi_{p,q}(g)} 
\norm{\sum_{i=1}^m \xi_i [g(x_i)]_+}_{p^*}\right]\\ \notag
&=\E{\xi }\left[\frac{1}{m}  
\sup_{h\in \calF^{d-2,\netwidth,\netwidth}}  \frac{\psi}{\psi_{p,q}(h)}\sup_{W}\frac{1}{\norm{W}_{p,q}}
\norm{\sum_{i=1}^m \xi_i [W[h(x_i)]_+]_+}_{p^*}\right]\\ \label{eq:lemlayer}
&= \netwidth^{[\frac{1}{p^*} - \frac{1}{q}]_+}\E{\xi}
\left[\frac{1}{m}  \sup_{h\in \calF^{d-2,\netwidth,\netwidth}} \frac{\psi}{\psi_{p,q}(h)} \sup_{w} 
\frac{1}{\norm{w}_{p}} \left\lvert \sum_{i=1}^m \xi_i [w^\top [h(x_i)]_+]_+
\right\rvert \right]\\\notag
&= \netwidth^{[\frac{1}{p^*} - \frac{1}{q}]_+}\E{\xi }
\left[ \frac{1}{m} \sup_{g\in \calF^{d-1,\netwidth}_{\psi_{p,q}\leq \psi}} \left\lvert  \sum_{i=1}^m \xi_i [g(x_i)]_+\right\rvert  \right]\\ \label{eq:contraction}
&\leq 2 \netwidth^{[\frac{1}{p^*} - \frac{1}{q}]_+}\E{\xi}
\left[ \frac{1}{m} \sup_{g\in \calF^{d-1,\netwidth}_{\psi_{p,q}\leq \psi}} \left\lvert  \sum_{i=1}^m \xi_i g(x_i)\right\rvert  \right]\\ \notag
&=2 \netwidth^{[\frac{1}{p^*} - \frac{1}{q}]_+} \calR_m(\calF^{d-1,\netwidth}_{\psi_{p,q}\leq \psi}) \\ \notag
\end{align}
where the equality~\eqref{eq:lemlayer} is obtained by lemma~\ref{lem:l-norm} and inequality~\eqref{eq:contraction} is by Contraction Lemma.
This will give us the bound on Rademacher complexity of $\calF^{d,\netwidth}_{\psi_{p,q}\leq \psi}$ based on the Rademacher complexity of  $\calF^{d-1,\netwidth}_{\psi_{p,q}\leq \psi}$. Applying the same argument on  all layers and using lemma~\ref{lem:layer1} to bound the  complexity of the first layer completes the proof.
\end{proof}

\subsubsection{Proof of Lemma~\ref{lem:l-norm}}
\label{sec:proof-lem-l-norm}
\begin{proof}
It is immediate that the right hand side of the equality in the
 statement is always less than or equal to the left hand side because
 given any vector $w$ in the right hand side, by setting each row of
 matrix $\vecw$ in the left hand side we get the equality. Therefore, it is
 enough to prove that the left hand side is less than or equal to the
 right hand side. For the convenience of notations, let $g(\vecw) \defeq |\sum_{i=1}^m \xi_i 
w^\top [f(x_i)]_+|$. Define $\widetilde{w}$ to be:
$$
\widetilde{w} \defeq \arg\max_{w} \frac{g(\vecw)}{\norm{w}_{p}}
$$
If $q\leq p^*$, then the right hand side of equality in the lemma statement will reduce to $g(\widetilde{w})/\norm{\widetilde{w}}_p$ and therefore we need to show that for any matrix $V$, 
$$
\frac{g(\widetilde{w})}{\norm{\widetilde{w}}_p} \geq 
\frac{\norm{g(V)}_{p^*}}{\norm{V}_{p,q}}.
$$
Since $q\leq p^*$, we have $\norm{V}_{p,{p^*}} \leq 
 \norm{V}_{p,q}$ and hence it is enough to prove the
  following inequality:
$$
\frac{g(\widetilde{w})}{\norm{\widetilde{w}}_p} \geq 
\frac{\norm{g(V)}_{p^*}}{\norm{V}_{p,{p^*}}}.
$$
On the other hand, if $q>{p^*}$, then we need to prove the following 
inequality holds:
$$
\netwidth^{\frac{1}{{p^*}} - \frac{1}{q}}\frac{g(\widetilde{w})}
{\norm{\widetilde{w}}_p} \geq \frac{\norm{g(V)}_{p^*}}{\norm{V}_{p,q}}
$$
Since $q>{p^*}$, we have that $\norm{V}_{p,{p^*}} \leq  
\netwidth^{\frac{1}{{p^*}} - \frac{1}{q}}\norm{V}_{p,q}$. Therefore,
 it is again enough to show that:
$$
\frac{g(\widetilde{w})}{\norm{\widetilde{w}}_p} \geq 
\frac{\norm{g(V)}_{p^*}}{\norm{V}_{p,{p^*}}}.
$$
We can rewrite the above inequality in the following form:
$$
\sum_{i=1}^\netwidth \left(\frac{g(\widetilde{w})\norm{V_i}_p}
{\norm{\widetilde{w}}_p}\right)^{p^*} \geq \sum_{i=1}^\netwidth g(V_i)^{p^*}
$$
By the definition of $\widetilde{w}$, we know that the
 above inequality holds for each term in the sum and 
 hence the inequality is true.
\end{proof}

\subsubsection{Theorem \ref{thm:antisym}}\label{sec:antisym}

The proof is similar to the proof of theorem \ref{thm:l-norm} but here bounding $\mu_{1,\infty}$ by $\mu$ means the $\ell_1$ norm of input weights to each neuron is bounded by $\mu$. We use a different version of Contraction Lemma in the proof that is without the absolute value:
\begin{lemma}(Contraction Lemma (without the absolute value))
Let function $\phi:\R\rightarrow \R$ be Lipschitz with constant $\calL_\phi$. Then for any class $\calF$ of functions mapping from $\calX$ to $\R$ and any set $\calS=\{x_1,\dots,x_m\}$:
$$
\E{\xi \in \{\pm 1\}^m}\left[\frac{1}{m}
\sup_{f_\vecw\in \calF}\sum_{i=1}^m \xi_i \phi( f(x_i) )\right] \leq \calL_\phi \E{\xi \in \{\pm 1\}^m}\left[\frac{1}{m}
\sup_{f_\vecw\in \calF} \sum_{i=1}^m \xi_i f(x_i) )  \right]
$$
\end{lemma}
\begin{reptheorem}{thm:antisym}
For any anti-symmetric 1-Lipschitz function $\sigma$ and any set $\calS=\{x_1,\dots,x_m\}\subseteq\R^{\nin}$:
$$
\calR_m(\calF_{\mu_{1,\infty}\leq \mu}^{d}) \leq 
\sqrt{\frac{2\mu^{2d} \log(2\nin) \sup \norm{x_i}_{\infty}^2}{m}}
$$
\end{reptheorem}
\begin{proof}
Assuming $\xi$ is uniform over $\{\pm 1\}^m$, we have:
\begin{align}\notag
\calR_m(\calF^{d,\netwidth}_{\mu_{1,\infty}\leq \mu}) &= \E{\xi}\left[\frac{1}{m} \sup_{f_\vecw\in \calF^{d,\netwidth}_{\mu_{1,\infty}\leq \mu}} \left\lvert \sum_{i=1}^m 
\xi_i f(x_i) \right\rvert \right]\\ \notag
&= \E{\xi}\left[\frac{1}{m} \sup_{f_\vecw\in \calF^{d,\netwidth}_{\mu_{1,\infty}\leq \mu}} \sum_{i=1}^m 
\xi_i f(x_i) \right]\\ \notag
&= \E{\xi}\left[\frac{1}{m}  
\sup_{g\in \calF^{d-1,\netwidth,\netwidth}_{\mu_{1,\infty}\leq \mu}}  \sup_{\norm{w}_1\leq \mu} 
w^\top\sum_{i=1}^m \xi_i 
\sigma(g(x_i))\right]\\ \notag
&=\E{\xi }\left[\frac{1}{m}  
\sup_{g\in \calF^{d-1,\netwidth,\netwidth}_{\mu_{1,\infty}\leq \mu}} 
\norm{\sum_{i=1}^m \xi_i \sigma(g(x_i))}_{\infty}\right]\\ \label{eq:anti-sym}
&= \E{\xi }
\left[ \frac{1}{m} \sup_{g\in \calF^{d-1,\netwidth}_{\mu_{1,\infty}\leq \mu}} \left\lvert  \sum_{i=1}^m \xi_i \sigma(g(x_i))\right\rvert  \right]\\\notag
&= \E{\xi }
\left[ \frac{1}{m} \sup_{g\in \calF^{d-1,\netwidth}_{\mu_{1,\infty}\leq \mu}} \sum_{i=1}^m \xi_i \sigma(g(x_i)) \right]\\ \label{eq:anti-contraction}
&\leq \E{\xi }
\left[ \frac{1}{m} \sup_{g\in \calF^{d-1,\netwidth}_{\mu_{1,\infty}\leq \mu}} \sum_{i=1}^m \xi_i g(x_i) \right]\\\notag
&=\E{\xi }
\left[ \frac{1}{m} \sup_{g\in \calF^{d-1,\netwidth}_{\mu_{1,\infty}\leq \mu}} \left\lvert \sum_{i=1}^m \xi_i g(x_i) \right\rvert  \right]\\\notag
&=\calR_m(\calF^{d-1,\netwidth}_{\mu_{1,\infty}\leq \mu}) \\ \notag
\end{align}
where the equality \eqref{eq:anti-sym} is by anti-symmetric property of $\sigma$ and inequality~\eqref{eq:anti-contraction} is by the version of Contraction Lemma without the absolute value. This will give us the bound on Rademacher complexity of $\calF^{d,\netwidth}_{\mu_{1,\infty}\leq \mu}$ based on the Rademacher complexity of  $\calF^{d-1,\netwidth}_{\mu_{1,\infty}\leq \mu}$. Applying the same argument on  all layers and using lemma~\ref{lem:layer1} to bound the  complexity of the first layer completes the proof.
\end{proof}

\subsection{Proof that $\psi^d_{p,q}(\vecw)$ is a semi-norm in $\calF^d$}
\label{sec:proof-cvx}
We repeat the statement here for convenience.
\newtheorem*{thm:cvx}{Theorem \ref{thm:cvx}}
\begin{thm:cvx}
  For any $d,p,q \geq 1$ such that $\frac{1}{q}\leq
  \frac{1}{d-1}\big(1-\frac{1}{p}\big)$, $\psi^d_{p,q}(f)$ is a semi-norm
  in $\calF^d$.
\end{thm:cvx}
\begin{proof}
The proof consists of three parts. First we show
 that the level set $\calF^d_{\psi_{p,q}^d\leq
 \psi}=\{f_\vecw\in \calF^d: \psi_{p,q}^d(f)\leq \psi\}$ is a convex set
 if the condition on $d,p,q$ is satisfied. Next, we
 establish the non-negative homogeneity of $\psi_{p,q}^d(f)$. Finally,
 we show that if a function $\alpha:\calF^d\rightarrow\R$  is non-negative
 homogeneous and every sublevel set $\{f_\vecw\in \calF^d : \alpha(f)\leq
 \psi\}$ is convex, then $\alpha$ satisfies the triangular inequality.

\paragraph{Convexity of the level sets}
First we show that for any two functions $f_1,f_2\in
 \calF^d_{\psi_{p,q}\leq  \psi}$ and  $0\leq \alpha \leq 1$, the
 function $g=\alpha f_1 + (1-\alpha)f_2$ is in the model class
 $\calF^{d}_{\psi_{p,q}\leq \psi}$. We prove this by constructing
 weights $\vecw$ that realizes $g$. Let $U$ and $V$ be the weights of two
 neural networks such that $\psi_{p,q}(U) = \psi_{p,q}^{d}(f_1)\leq
 \psi$ and $\psi_{p,q}(V)=\psi_{p,q}^d(f_2) \leq \psi$. 
 For every layer $i=1,\ldots,d$ let
  \begin{equation*}
    \label{eq:W1}
    \tilde{U}_i = \sqrt[d]{\psi_{p,q}(U)}U_i/{\norm{U_i}_{p,q}},
  \quad\quad \tilde{V}_i = \sqrt[d]{\psi_{p,q}(V)}V_i/{\norm{V_i}_{p,q}}.
  \end{equation*}
  and set $W_1 = \begin{bmatrix} \tilde{U}_1 \\
    \tilde{V}_1 \end{bmatrix}$ for the first layer, $W_i =
  \begin{bmatrix} \tilde{U}_i & 0 \\ 0 & \tilde{V}_i\end{bmatrix}$ for
  the intermediate layers and $W_d = 
\begin{bmatrix}
\alpha\tilde{U}_d &  (1-\alpha)\tilde{V}_d
\end{bmatrix}$ for the output layer.

Then for the defined $\vecw$, we have $f_W=\alpha f_1 + (1-\alpha)f_2$ for
rectified linear and any other non-negative homogeneous activation
function. Moreover, for any $i<d$, the norm of each layer is
\begin{equation}\label{eq:conv-i}
\norm{W_i}_{p,q} = \left(\psi_{p,q}(U)^{\frac{q}{d}} + \psi_{p,q}(V)^{\frac{q}{d}}\right)^{\frac{1}{q}} \leq 2^\frac{1}{q}\psi^{\frac{1}{d}}
\end{equation}
and in layer $d$ we have:
\begin{equation}\label{eq:conv-d}
\norm{W_d}_p = \left(\alpha^p\psi_{p,q}(U)^{\frac{p}{d}} + (1-\alpha)^p\psi_{p,q}(V)^{\frac{p}{d}}\right)^{\frac{1}{p}} \leq 2^{1/p-1}\psi^{1/d}
\end{equation}
Combining inequalities \eqref{eq:conv-i} and \eqref{eq:conv-d}, we get
$
\psi^{d}_{p,q}(f_W) \leq 2^{\frac{d-1}{q} + \frac{1}{p}} \psi \leq \psi,
$
where the last inequality holds because we assume that
$\frac{1}{q}\leq \frac{1}{d-1}\big(1-\frac{1}{p}\big)$. Thus for every
$\psi\geq 0$, $\calF^d_{\psi_{p,q}\leq \psi}$ is a convex set.

\paragraph{Non-negative homogeneity}
  For any function $f_\vecw\in\calF^d$ and any $\alpha \geq 0$, let
  $U$ be the weights realizing $f$ with
  $\psi^d_{p,q}(f)=\psi_{p,q}(U)$. Then $\sqrt[d]{\alpha}U$
  realizes $\alpha f$ establishing $\psi^d_{p,q}(\alpha f) \leq
  \psi_{p,q}(\sqrt[d]{\alpha}U) = \alpha \psi_{p,q}(U)=\alpha
  \psi^d_{p,q}(U)=\alpha \psi_{p,q}^d(f)$. This establishes the
 non-negative homogeneity of $\psi_{p,q}^d$.
\paragraph{Convex sublevel sets and homogeneity imply triangular inequality}
Let $\alpha(f)$ be non-negative homogeneous and assume that every
sublevel set $\{f_\vecw\in\calF^d:\alpha(f)\leq \psi\}$ is convex. Then
for  $f_1,f_2\in\calF^d$, defining $\psi_1\defeq\alpha(f_1)$,
 $\psi_2\defeq \alpha(f_2)$, $\tilde{f}_1\defeq(\psi_1+\psi_2)f_1/\psi_1$, and
 $\tilde{f}_2\defeq(\psi_1+\psi_2)f_2/\psi_2$, we have
\begin{align*}
\alpha(f_1+f_2)=\alpha\left(\frac{\psi_1}{\psi_1+\psi_2}\tilde{f}_1+\frac{\psi_2}{\psi_1+\psi_2}\tilde{f}_2\right)\leq
 \psi_1+\psi_2=\alpha(f_1)+\alpha(f_2).
\end{align*}
Here the inequality is due to the convexity of the
level set  and the fact that
$\alpha(\tilde{f}_1)=\alpha(\tilde{f}_2)=\psi_1+\psi_2$, because of
the homogeneity. Therefore $\alpha$ satisfies the triangular inequality
and thus it is a seminorm.

\end{proof}


\subsection{Path Regularization}\label{sec:path-appendix}
\subsubsection{Theorem \ref{thm:path-layer}}
\label{sec:path-layer}
\begin{lemma}\label{lem:unit-norm}
For any function $f_\vecw\in \calF^{d,\netwidth}_{\psi_{p,\infty} \leq \psi}$ there is a layered network with weights $w$ such that $\psi_{p,\infty}(\vecw) = \psi^{d,\netwidth}_{p,\infty}(f)$ and for any internal unit $v$, $\sum_{(u\rightarrow v)\in E} |w_{u\rightarrow v}|^p = 1$.
\end{lemma}
\begin{proof}
Let $w$ be the weights of a network such that $\psi_{p,\infty}(\vecw) = \psi^{d,\netwidth}_{p,\infty}(f)$. We now construct a network with weights $\widetilde{w}$ such that $\psi_{p,\infty}(\vecw) = \psi^{d,\netwidth}_{p,\infty}(f)$ and for any internal unit $v$, $\sum_{(u\rightarrow v)\in E} |\widetilde{w}(u\rightarrow v)|^p = 1$. We do this by an incremental algorithm. Let $w_0=w$. At each step $i$, we do the following. 

Consider the first layer, Set $V_k$ to be the set of neurons in the layer $k$.
Let $x$ be the maximum of $\ell_p$ norms of input weights to each neuron in set $V_1$ and let $U_x\subseteq V_1$ be the set of neurons whose $\ell_p$ norms of their input weight is exactly $x$. Now let $y$ be the maximum of $\ell_p$ norms of input weights to each neuron in the set $V_1\setminus U_x$ and let $U_y$ be the set of the neurons such that the $\ell_p$ norms of their input weights is exactly $y$. Clearly $y<x$. We now scale down the input weights of neurons in set $U_x$ by $y/x$ and scale up all the outgoing edges of vertices in $U_x$ by $x/y$ ($y$ cannot be zero for internal neurons based on the definition). It is straightforward that the new network realizes the same function and the $\ell_{p,\infty}$ norm of the first layer has changed by a factor $y/x$. Now for every neuron $v\in V_2$, let $r(v)$ be the $\ell_{p}$ norm of the new incoming weights divided by $\ell_p$ norm of the original incoming weights. We know that $r(v)\leq x/y$. We again scaly down the input weights of every$v\in V_2$ by $1/r(v)$ and scale up all the outgoing edges of $v$ by $r(v)$. Continuing this operation to on each layer, each time we propagate the ratio to the next layer while the network always realizes the same function and for each layer $k$, we know that for every $v\in V_k$, $r(v)\leq x/y$. After this operation, in the network, the $\ell_{p,\infty}$ norm of the first layer is scaled down by $y/x$ while the $\ell_{p,\infty}$ norm of the last layer is scaled up by at most $x/y$ and the $\ell_{p,\infty}$ norm of the rest of the layers has remained the same. Therefore, if $w_i$ is the new weight setting, we have $\psi_{p,\infty}(w_i) \leq \psi_{p,\infty}(w_{i-1})$.

After continuing the above step at most $|V_1|-1$ times, the $\ell_p$ norm of input weights is the same for all neurons in $V_1$. We can then run the same algorithm on other layers and at the end we have a network with weight setting $\widetilde{w}$ such that the for each $k<d$, $\ell_p$ norm of input weight to each of the neurons in layer $k$ is equal to each other and $\psi_{p,\infty}(\widetilde{w})\leq \psi_{p,\infty}(\vecw)$. This is in fact an equality because weight setting $w'$ realizes function $f$ and we know that $\psi_{p,\infty}(\vecw) = \psi^{d,\netwidth}_{p,\infty}(f)$. A simple scaling of weights in layers gives completes the proof.
\end{proof}

\begin{reptheorem}{thm:path-layer}
For $p\geq 1$, any $d$ and (finite or infinite) $\netwidth$, for any $f_\vecw\in\calF^{d,\netwidth}$:  $\pathr_p^{d,\netwidth}(f) = \psi^{d,\netwidth}_{p,\infty}$.
\end{reptheorem}
\begin{proof}
By the Lemma \ref{lem:unit-norm}, there is a layered network with weights $\widetilde{w}$ such that $\psi_{p,\infty}(\widetilde{w}) = \psi^{d,\netwidth}_{p,\infty}(f)$ and for any internal unit $v$, $\sum_{(u\rightarrow v)\in E} |\widetilde{w}(u\rightarrow v)|^p = 1$. Let $\vecw$ be the weights of the layered network that corresponds to the function $\widetilde{w}$. Then we have:
\begin{align}
v_p(\tilde{w}) &= \left(
    \sum_{\vin[i]\overset{e_1}{\rightarrow}v_1\overset{e_2}{\rightarrow}v_2\cdots\overset{e_k}{\rightarrow}\vout} \prod_{i=1}^k \abs{\widetilde{w}(e_i)}^p \right)^\frac{1}{p}\\ \label{eq:uniteq1}
    &= \left(\sum_{i_{d-1}=1}^\netwidth \dots \sum_{i_1=1}^\netwidth \sum_{i_0=1}^{\nin}  \lvert W_d[i_{d-1}]\rvert^p \prod_{k=1}^{d-1} \lvert W^k[i_{k},i_{k-1}]\rvert^p\right)^{\frac{1}{p}}\\
&= \left(\sum_{i_{d-1}=1}^\netwidth \lvert W_d[i_{d-1}]\rvert^p \dots \sum_{i_1=1}^\netwidth  \lvert W^k[i_2,i_{1}]\rvert^p \sum_{i_0=1}^{\nin} \lvert W^k[i_1,i_{0}]\rvert^p\right)^{\frac{1}{p}}\\
&= \left(\sum_{i_{d-1}=1}^\netwidth \lvert W_d[i_{d-1}]\rvert^p \dots \sum_{i_1=1}^\netwidth  \lvert W^k[i_2,i_{1}]\rvert^p \right)^{\frac{1}{p}}\\
&= \left(\sum_{i_{d-1}=1}^\netwidth \lvert W_d[i_{d-1}]\rvert^p \dots \sum_{i_2=1}^\netwidth  \lvert W^k[i_3,i_{2}]\rvert^p \right)^{\frac{1}{p}}\\ \label{eq:uniteq2}
&= \left(\sum_{i_{d-1}=1}^\netwidth \lvert W_d[i_{d-1}]\rvert^p \right)^{\frac{1}{p}} = \ell_p(W_d) =\psi_{p,\infty}(\vecw)\\
\end{align}
where inequalities~\ref{eq:uniteq1} to \ref{eq:uniteq2} are due to the fact that the $\ell_p$ norm of input weights to each internal neuron is exactly 1 and the last equality is again because $\ell_{p,\infty}$ of all layers is exactly 1 except the layer $d$.
\end{proof}
\subsubsection{Proof of Theorem \ref{thm:path-dag}}
\label{sec:proof-path-dag}
In this section, without loss of generality, we assume that all the internal nodes in a DAG have incoming edges and outgoing edges because otherwise we can just discard them. Let $\nout(v)$ be the longest directed path from vertex $v$ to $\vout$ and $\nin(v)$ be the longest directed path from any input vertex $\vin[i]$ to $v$. We say graph $G$ is a sublayered graph if $G$ is a subgraph of a layered graph.

We first show the necessary and sufficient conditions under which a DAG is a sublayered graph.

\begin{lemma}\label{lem:layer-cond}
The graph $G(E,V)$ is a sublayered graph if and only if any path from input nodes to the output nodes has length $d$ where $d$ is the length of the longest path in $G$
\end{lemma}
\begin{proof}
Since the internal nodes have incoming edges and outgoing edges; hence if $G$ is a sublayered graph it is straightforward by induction on the layers that for every vertex $v$ in layer $i$, there is a vertex $u$ in layer $i+1$ such that $(v\rightarrow u)\in E$ and this proves the necessary condition for being sublayered graph.

To show the sufficient condition, for any internal node $u$, $u$ has $\nin(v)$ distance from the input node in every path that includes $u$ (otherwise we can build a path that is longer than $d$). Therefore, for each vertex $v\in V$, we can place vertex $v$ in layer $\nin(v)$ and all the outgoing edges from $v$ will be to layer $\nin(v)+1$.
\end{proof}

\begin{lemma}\label{lem:path-edge}
If the graph $G(E,V)$ is not a sublayered graph then there exists a directed edge $(u\rightarrow v)$ such that  
$\nin(u)+\nout(v)<d-1$ where $d$ the length of the longest path in $G$.
\end{lemma}
\begin{proof}
We prove the lemma by an inductive argument. If $G$ is not sublayered, by lemma~\ref{lem:layer-cond}, we know that there exists a path $v_0\rightarrow\dots v_i\dots \rightarrow v_{d'}$ where $v_0$ is an input node ($\nin(v_0)=0$), $v_{d'}=\vout$ ($\nout(v_{d'}=0$) and $d'<d$. Now consider the vertex $v_1$. We need to have $\nout(v_1)=d-1$ otherwise if $\nout(v_1)<d-1$ we get $\nin(u)+\nout(v)<d-1$ and if  $\nout(v_1)>d-1$ there will be path in $G$ that is longer than $d$. Also, since $\nout(v_1)=d-1$ and the longest path in $G$ has length $d$, we have $\nin(v_1)=1$. 

By applying the same inductive argument on each vertex $v_i$ in the path we get $\nin(v_i)=i$ and $\nout(v_i)=d-i$. Note that if the condition $\nin(u)+\nout(v)<d-1$ is not satisfied in one of the steps of the inductive argument, the lemma is proved. Otherwise, we have $\nin(v_{d'-1})=d'-1$ and $\nout(v_{d'-1})=d-d'+1$ and therefore $\nin(v_{d'-1})+\nout(\vout) = d'-1<d-1$ that proves the lemma.
\end{proof}
\begin{reptheorem}{thm:path-dag}
  For any $p\geq 1$ and any $d$: $\displaystyle \psi^d_{p,\infty}(f) =
  \min_{\textrm{$G\in \DAG(d)$}} \pathr^G_p(f)$.
\end{reptheorem}
\begin{proof}
Consider any $f_{\vecw} \in \calF^{\DAG(d)}$ where the graph $G(E,V)$ is not sublayered. Let $\rho$ be the total number of paths from input nodes to the output nodes. Let $T$ be sum over paths of the length of the path. We indicate an algorithm to change $G$ into a sublayered graph $\tilde{G}$ of depth $d$ with weights $\tilde{w}$ such that $f_{\vecw}=f_{\tilde{G},\tilde{w}}$ and $\pathr(\vecw)=\pathr(\tilde{w})$. Let $G_0=G$ and $w_0=w$. 

At each step $i$, we consider the graph $G_{i-1}$. If $G_{i-1}$ is sublayered, we are done otherwise by lemma \ref{lem:path-edge}, there exists an edge $(u\rightarrow v)$ such that $\nin(u)+\nout(v)<d-1$. Now we add a new vertex $\tilde{v}_i$ to graph $G_{i-1}$, remove the edge $(u\rightarrow v)$, add two edges $(u\rightarrow \tilde{v}_i)$ and $(\tilde{v}_i\rightarrow v)$ and return the graph as $G_{i}$ and since we had $\nin(u)+\nout(v)<d-1$ in $G_{i-1}$, the longest path in $G_{i}$ still has length $d$. We also set $w(u\rightarrow \tilde{v}_i) = \sqrt{|w_{u\rightarrow v}|}$ and $w(\tilde{v}_i\rightarrow v) = \sign(w_{u\rightarrow v})\sqrt{|w_{u\rightarrow v}|}$. Since we are using rectified linear units activations, for any $x>0$, we have $[x]_+=x$ and therefore:
\begin{align*}
w(\tilde{v}_i\rightarrow v)\left[w(u\rightarrow \tilde{v}_i) o(u)\right]_+ &=\sign(w_{u\rightarrow v})\sqrt{|w_{u\rightarrow v}|}\left[\sqrt{|w_{u\rightarrow v}|} o(u)\right]_+\\
&=\sign(w_{u\rightarrow v})\sqrt{|w_{u\rightarrow v}|}\sqrt{|w_{u\rightarrow v}|} o(u)\\
&=w_{u\rightarrow v}o(u)
\end{align*}
So we conclude that $f_{G_{i},w_i}=f_{G_{i-1},w_{i-1}}$. Clearly, since we didn't change the length of any path from input vertices to the output vertex, we have $\pathr(\vecw)=\pathr(\tilde{w})$. Let $T_i$ be sum over paths of the length of the path in $G_i$. It is clear that $T_{i-1} \leq T_{i}$ because we add a new edge into a path at each step. We also know by lemma~\ref{lem:layer-cond} that if $T_{i}=\rho d$, then $G_i$ is a sublayered graph. Therefore, after at most $\rho d - T_0$ steps, we return a sublayered graph $\tilde{G}$ and weights $\tilde{w}$ such that $f_{\vecw}=f_{\tilde{G},\tilde{w}}$. We can easily turn the sublayered graph $\tilde{G}$ a layered graph by adding edges with zero weights and this together with Theorem \ref{thm:path-layer} completes the proof.
\end{proof}

\subsection{Hardness of Learning Neural Networks}\label{sec:hardness}

 \citet{Daniely14} show in Theorem 5.4 and in Section 7.2 that subject to the strong random CSP assumption, for
  any $k=\omega(1)$ the model class of intersection of
  homogeneous halfspaces over $\{\pm 1\}^n$ with normals in $\{\pm
  1\}$ is not efficiently PAC learnable (even
  improperly)\footnote{Their Theorem 5.4 talks about unrestricted
    halfspaces, but the construction in Section 7.2 uses only data in
    $\{ \pm 1 \}^{\nin}$ and halfspaces specified by $\langle w,x\rangle
    >0$ with $w\in\{\pm 1\}^{\nin}$}. Furthermore, for any $\epsilon>0$, \cite{klivans2006} prove this hardness result subject to intractability of $\tilde{Q}(\nin^{1.5})$-unique shortest
vector problem for $k=\nin^\epsilon$.

  If it is not possible to efficiently PAC learn intersection of
halfspaces (even improperly), we can conclude it is also not possible
to efficiently PAC learn any model class which can represent such
intersection. In Theorem~\ref{thm:hardness} we show that intersection of homogeneous half spaces can be realized with unit margin by neural networks with bound norm.
\begin{theorem}\label{thm:hardness}
For any $k>0$, the intersection of $k$ homogeneous half spaces is realizable with unit margin by $\calF^2_{\psi_{p,q} \leq \psi}$ where $\psi=4\nin^{\frac{1}{p}}k^2$.
\end{theorem}

\begin{proof}
The proof is by a construction that is similar to the one in \cite{livni14}.
For each hyperplane $\inner{w_i}{x}>0$, where
  $w_i\in \{\pm 1\}^{\nin}$, we include two units in the first layer:
  $g^+_{i}(x) = [\inner{w_i}{x}]_+$ and $g^-_i(x) =
  [\inner{w_i}{x}-1]_+$.  We set all incoming weights of the
  output node to be $1$.  Therefore, this network is realizing the
  following function:
$$
f(x) = \sum_{i=1}^k \left([\inner{w_i}{x}]_+ - [\inner{w_i}{x}-1]_+\right)
$$
Since all inputs and all weights are integer, the outputs of the first
layer will be integer, $\left([\inner{w_i}{x}]_+ -
  [\inner{w_i}{x}-1]_+\right)$ will be zero or one, and $f$ realizes
the intersection of the $k$ halfspaces with unit margin. Now, we just
need to make sure that $\psi^2_{p,q}(f)$ is bounded by $\psi=4\nin^{\frac{1}{p}}k^2$:
\begin{align*}
\psi^2_{p,q}(f) &= \nin^{\frac{1}{p}} (2k)^{\frac{1}{q}}(2k)^{\frac{1}{p}} \\
&\leq \nin^{\frac{1}{p}}(2k)^{2} = \psi.
\end{align*}
\end{proof}

\section{Discussion}

We presented a general framework for norm-based capacity control for
feed-forward networks, and analyzed when the norm-based control is
sufficient and to what extent capacity still depends on other
parameters.  In particular, we showed that in depth $d>2$ networks,
per-unit control with $p>1$ and overall regularization with $p>2$ is
not sufficient for capacity control without also controlling the
network size.  This is in contrast with linear models, where with any
$p<\infty$ we have only a weak dependence on dimensionality, and
two-layer networks where per-unit $p=2$ is also sufficient for
capacity control.  We also obtained generalization guarantees for
perhaps the most natural form of regularization, namely $\ell_2$
regularization, and showed that even with such control we still
necessarily have an exponential dependence on the depth.  

Although the additive $\mu$-measure and multiplication
$\psi$-measure are equivalent at the optimum, they behave rather
differently in terms of optimization dynamics (based on anecdotal
empirical experience) and understanding the relationship between them,
as well as the novel path-based regularizer can be helpful in practical
regularization of neural networks.  

Although we obtained a tight characterization of when size-independent
capacity control is possible, the precise polynomial dependence of
margin-based classification (and other tasks) on the norm in might not
be tight and can likely be improved, though this would require going
beyond bounding the Rademacher complexity of the real-valued class.
In particular, Theorem \ref{thm:l-norm} gives the same bound for
per-unit $\ell_1$ regularization and overall $\ell_1$ regularization,
although we would expect the later to have lower capacity.

Beyond the open issue regarding depth-independent $\psi$-based
capacity control, another interesting open question is understanding
the expressive power of $\calF^d_{\psi_{p,q}\leq\psi}$,
particularly as a function of the depth $d$.  Clearly going from depth
$d=1$ to depth $d=2$ provides additional expressive power, but it is
not clear how much additional depth helps.  The class $\calF^2$
already includes all binary functions over $\{\pm 1\}^{\nin}$ and is dense
among continuous real-valued functions.  But can the $\psi$-measure
be reduced by increasing the depth?  Viewed differently:
$\psi^d_{p,q}(f)$ is monotonically non-increasing in $d$, but are
there functions for it continues decreasing?  Although it seems
obvious there are functions that require high depth for efficient
representation, these questions are related to decade-old problems in
circuit complexity and might not be easy to resolve.

\chapter{Sharpness/PAC-Bayes Generalization Bounds} \label{chap:sharpness}

So far we discussed norm based and sharpness based complexity measures to understand capacity. We also have discussed how to combine these two notions and the tradeoff in scaling between them under the PAC-Bayes framework. We next show how to utilize the general PAC-Bayes bound in Lemma~\ref{lem:general-bound} to prove generalization guarantees for feedforward networks based on the spectral norm of its layers.

\section{Spectrally-Normalized Margin Bounds}\label{sec:margin}
As we discussed in Section~\ref{sec:pac-bayes-general}, understanding the sharpness of the network is the key step to obtain a generalization bound using PAC-Bayes framework. The following lemma shows that the sharpness can be bounded by the product of spectral norm of the layers.
\begin{lemma}[Perturbation Bound]\label{lem:worstcase-sharpness} For any $B,d>0$, let $f_\vecw:\calX_{B,n}\rightarrow \R^k$ be a $d$-layer network. Then for any $\vecx\in \calX_{B,n}$ and any perturbation $\vecu$ such that $\norm{U_i}_2\leq \frac{1}{d}\norm{W_i}_2$, the sharpness of $f_\vecw$ can be bounded as follows:
\begin{small}
\begin{equation}
\abs{f_{\vecw+\vecu}(\vecx) - f_{\vecw}(\vecx)}_2\leq eB\left(\prod_{i=1}^d\norm{W_i}_2\right)\sum_{i=1}^d \frac{\norm{U_i}_2}{\norm{W_i}_2}.
\end{equation}
\end{small}
\end{lemma}
Next, we derive a generalization guarantee using Lemmas~\ref{lem:general-bound} and \ref{lem:worstcase-sharpness}.
\begin{theorem}[Generalization Bound]\label{thm:pac-bayes}
For any $B,d,h>0$, let $f_\vecw:\calX_{B,n}\rightarrow \R^k$ be a $d$-layer feedforward network with ReLU activations. Then for any probability $\delta$, margin $\gamma>0$, the following generalization bound holds with probability $1-\delta$ over the training set:
\begin{small}
\begin{equation}
L_0(f_{\vecw})\leq \hatl_\gamma(f_\vecw) + \calO\left(\sqrt{\frac{B^2 d^2 h\ln(dh) \Pi_{i=1}^d\norm{W_i}^2\sum_{i=1}^d\left(\norm{W_i}_F^2/\norm{W_i}_2^2\right) + \ln\frac{dm}{\delta}}{\gamma^2 m}}\right).
\end{equation}
\end{small}
\end{theorem}
\begin{proof} The proof involves mainly two steps. In the first step we calculate what is the maximum allowed perturbation of parameters to satisfy a given margin condition $\gamma$, using Lemma~\ref{lem:worstcase-sharpness}. In the second step we calculate the KL term in the PAC-Bayes bound in Lemma~\ref{lem:general-bound}, for this value of perturbation.

Let $\beta = \left(\Pi_{i=1}^d \norm{W_i}_2\right)^{1/d}$ and consider the reparametrization $\widetilde{W_i}=\frac{\beta}{\norm{W_i}_2}W_i$. Since for feedforward network with ReLU activations $f_{\widetilde{\vecw}}=f_{\vecw}$, the bound in the theorem statement is invariant to this reparametrization. W.l.o.g. we assume that for any layer $i$, $\norm{W_i}_2=\beta$. Choose the prior $P$ to be $\mathcal{N}(0,\sigma_p^2 I)$ and consider the random perturbation $\vecu \sim \mathcal{N}(0,\sigma_{q}^2 I)$. The following inequality holds on the spectral norm of $U_i$~\cite{tropp2012user}:
\begin{equation}
\mathbb{P}_{U_i\sim N(0,\sigma_{q})}\left[\norm{U_i}_2>t\right] \leq 2h e^{-t^2/2h\sigma^2_{q}}.
\end{equation}
Taking a union bond over the layers, we get spectral norm of perturbation in each layer is bounded by $\sigma_{q} \sqrt{2h\ln(4dh)}$. Define set $\calS$ as $\big\{\vecu\;\big|\norm{U_i}_2\leq \sigma_{q} \sqrt{2h\ln(4dh)}\big\}$. Given the bound on spectral norm of each layer, $\vecu \in \calS$ with probability at least $\frac{1}{2}$. Let $\hat{\beta}$ be an estimate of $\beta$ that is picked before observing data. If $|{\hat{\beta}}-\beta|\leq \frac{1}{d}\beta$, then $\frac{1}{e}\beta^{d-1}\leq {\hat{\beta}}^{d-1}\leq e\beta^{d-1}$. Using Lemma~\ref{lem:worstcase-sharpness} with probability at least $\frac{1}{2}$:

\begin{equation*}
\max_{\vecx\in \calX_{B,n}} \abs{f_{\vecw+\vecu}(\vecx)-f_{\vecw}(\vecx)} \leq edB\beta^{d-1}\norm{U_i}_2 \leq e^2dB{\hat{\beta}}^{d-1}\sigma_q\sqrt{2h\ln(4dh)} \leq \frac{\gamma}{4},
\end{equation*}

where we choose $\sigma_q=\frac{\gamma}{42dB{\hat{\beta}}^{d-1}\sqrt{h\ln(4hd)}}$ to get the last inequality. 

Let $q(\vecz)$ be the density function of the posterior. We now calculate the KL-term in Lemma~\ref{lem:general-bound} for $\sigma_p=\sigma_q$ on the set $\calS$:
\begin{small}
\begin{equation*}
KL_{\calS}(\vecw+\vecu||P) = \int_{\calS}q(\vecz)\frac{2\inner{\vecz}{ \vecw}-\abs{ \vecw}^2}{2\sigma_q^2}d\vecz \leq \frac{\abs{ \vecw}^2}{2\sigma_q^2} \leq \calO\left(B^2 d^2 h\ln(dh)\frac{\Pi_{i=1}^d\norm{W_i}_2^2}{\gamma^2}  \sum_{i=1}^d \frac{\norm{W_i}_F^2}{\norm{W_i}_2^2}\right)
\end{equation*}
\end{small}

Finally it remains to show how to find the estimates ${\hat{\beta}}$. We only need to consider values of $\beta$ in the range $\left(\frac{\gamma}{2B}\right)^{1/d}\leq \beta \leq \left(\frac{\gamma\sqrt{m}}{2B}\right)^{1/d}$. For $\beta$ outside this range  the theorem statement holds trivially. Recall that LHS of the theorem statement, $L_0(f_{\vecw})$ is always bounded by  $1$.  If $\beta^d<\frac{\gamma}{2B}$ then for any $\vecx$, $\abs{f_\vecw(\vecx)}\leq \beta^dB\leq \gamma/2$ and therefore $L_{\gamma}=1$. Alternately, if $\beta^d>\frac{\gamma\sqrt{m}}{2B}$, then the second term in equation~\ref{eq:pacbayes} is greater than one. Hence, we only need to consider values of $\beta$ in the range discussed above.  Since we need $|{\hat{\beta}} - \beta|\leq \frac{1}{d}\beta\leq \frac{1}{d}\left(\frac{\gamma}{2B}\right)^{1/d}$, the size of this cover is $dm^{\frac{1}{2d}}$. Taking a union bound over this cover and using Lemma~\ref{lem:general-bound} gives us the theorem statement.
\end{proof}

\section{Generalization Bound based on Expected Sharpness}
We showed how bounding the sharpness could give us a generalization bound. We now establish sufficient conditions to bound the expected sharpness of a feedforward network with ReLU activations.  Such conditions serve as a useful guideline in studying what helps an optimization method to converge to less sharp optima. Unlike existing generalization bounds~\cite{bartlett2002rademacher,NeyTomSre15,luxburg2004distance,xu2012robustness,sokolic2016generalization}, our sharpness based bound does not suffer from exponential dependence on depth.

Now we discuss the conditions that affect the sharpness of a network. As discussed earlier, weak interactions between layers can cause the network to have high sharpness value. Condition $C1$ below prevents such weak interactions (cancellations). A network can also have high sharpness if the changes in the number of activations is exponential in the perturbations to its weights, even for small perturbations. Condition $C2$ avoids such extreme situations on activations. Finally, if a non-active node with large weights becomes active because of the perturbations in lower layers, that can lead to huge changes to the output of the network. Condition $C3$ prevents having such spiky (in magnitude) hidden units. This leads us to the following three conditions, that help in avoiding such pathological cases.

\begin{itemize}
\item[$(C1):$] Given $x$, let $x=W_0$ and $D_0 =I$. Then, for all $0 \leq a < c < b \leq d, \| \left(\Pi_{i=a}^{b} D_{i}W_i \right)\|_F \geq  \frac{\mu}{\sqrt{h_c}}  \| \Pi_{i=c+1}^{b} D_{i}W_i  \|_F \|  \left(\Pi_{i=a}^{c} D_{i}W_i \right)\|_F $.
\item[$(C2):$] Given $x$, for any level $k$, $\frac{1}{h_k} \sum_{i \in [h_k]} 1_{W_{k,i} \Pi_{j=1}^{k-1} D_j W_j x \leq \delta} \leq C_2 \delta$.
\item[$(C3):$] For all $i$, $\| W_{i}\|_{2,\infty}^2  h_i \leq C_3^2 \|D_i W_i\|_F^2$.
\end{itemize}

Here, $W_{k, i}$ denotes the weights of the $i^{th}$ output node in layer $k$. $\| W_{i}\|_{2,\infty}$ denotes the maximum $L2$ norm of a hidden unit in layer $i$. Now we state our result on the generalization error of a ReLU network, in terms of average sharpness and its norm. Let $\|x\| = 1$ and $h=\max_{i=1}^d h_i$. 

\begin{theorem}\label{thm:relu}
Let $U_i$ be a random $h_i \times h_{i-1}$ matrix with each entry distributed according to $\mathcal{N}(0,\sigma_i^2)$. Then, under the conditions $C1, C2, C3$,  with probability $\geq 1-\delta$, 
\begin{small}
\begin{align*}
&\Ep{\vecu \sim \mathcal{N}(0,\sigma)^n}{L(f_{\vecw+\vecu})}- \hatl(f_\vecw) \leq O \left( \left[\Pi_{i=1}^d \left(1+\gamma_i \right) -1 \right. \right. \\ & \left. \left. + \Pi_{i=1}^d \left( 1+\gamma_i C_2 C_3 \right)\left(\Pi_{i=1}^d (1+ \gamma_i  C_{\delta} C_2)- 1\right) \right] C_{L} \sum_x \frac{\|f_\vecw(x)\|_F}{m} \right)+ \sqrt{\frac{1}{m}\left(\sum_{i=1}^d \frac{\|W_i\|_F^2 }{\sigma_i^2 } +  \ln \frac{2m}{\delta} \right)}.
\end{align*}
\end{small}
where $\gamma_i = \frac{\sigma_i \sqrt{h_i} \sqrt{h_{i-1}}}{\mu^2 \|W_i\|_F}$ and $C_{\delta}=2\sqrt{\ln(dh/\delta)} $.
\end{theorem}

To understand the above generalization error bound,  consider choosing $\gamma_i =\frac{\sigma}{C_{\delta}d}$, and we get a bound that simplifies as follows:
\begin{align*}
 \Ep{\vecu \sim \mathcal{N}(0,\sigma)^n}{L(f_{\vecw+\vecu}} - \hatl(f_\vecw) & \leq O \left(  \sigma \left( 1 + (1+ \sigma C_2 C_3)  C_2 \right) C_L \frac{\sum_x\|f_\vecw(x)\|_F}{m} \right)  \\ & \quad \quad+  \sqrt{\frac{1}{m}\left(\frac{d^2}{ \mu^4}\sum_{i=1}^d  \frac{h_i h_{i-1} }{\sigma^2}+  \ln \frac{2m}{\delta} \right)}
\end{align*}

If we choose large $\sigma$, then the network will have higher expected sharpness but smaller 'norm' and vice versa. Now one can optimize over the choice of $\sigma$ to balance between the terms on the right hand side and get a better capacity bound. For any reasonable choice of $\sigma$, the generalization error above, depends only linearly on depth and does not have any exponential dependence, unlike other notions of generalization. Also the error gets worse with decreasing $\mu$ and increasing $C_2, C_3$ as the sharpness of the network increases which is in accordance with our discussion of the conditions above.

\begin{figure}[t]
\centering
\includegraphics[width=.32\textwidth]{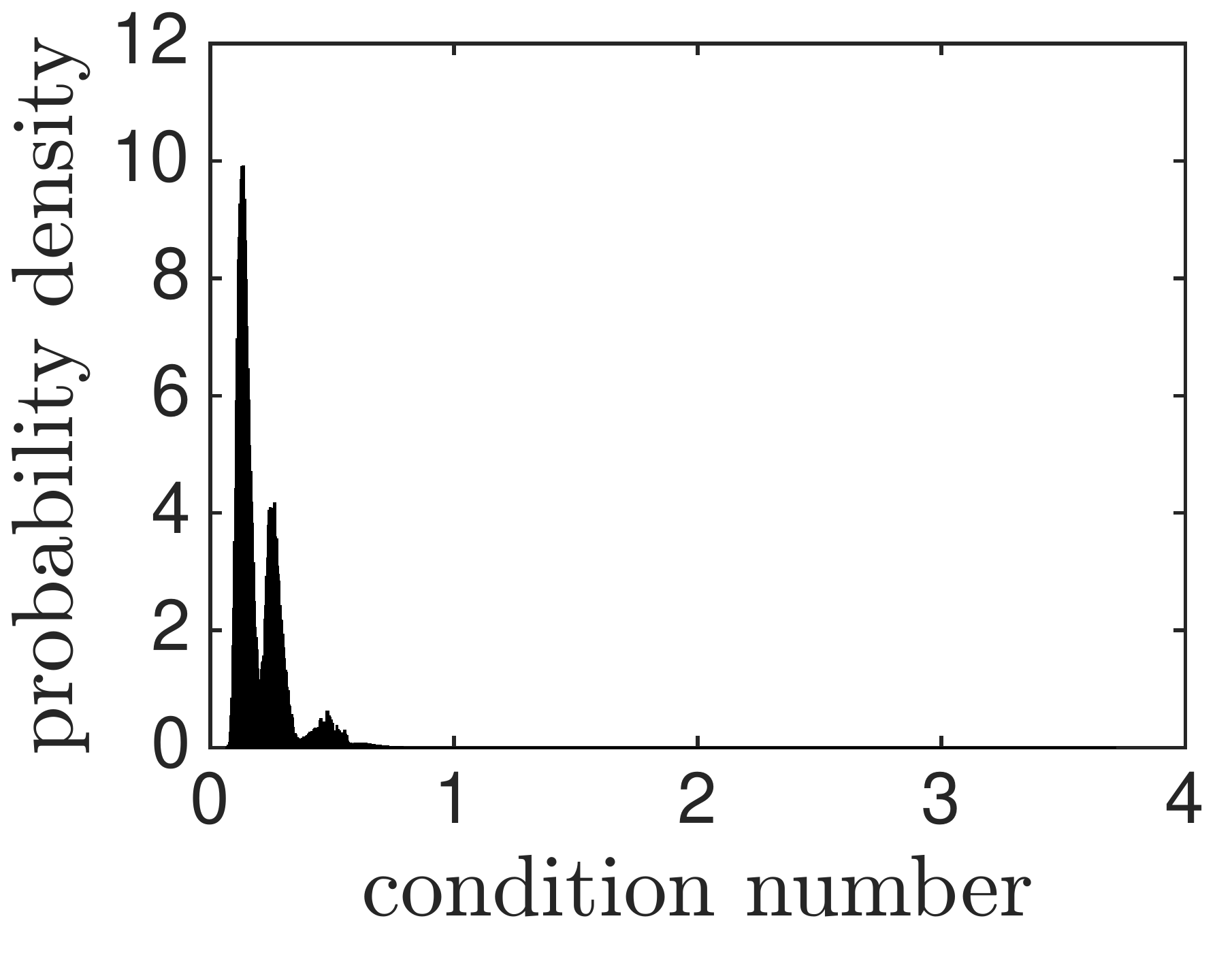}
\includegraphics[width=.32\textwidth]{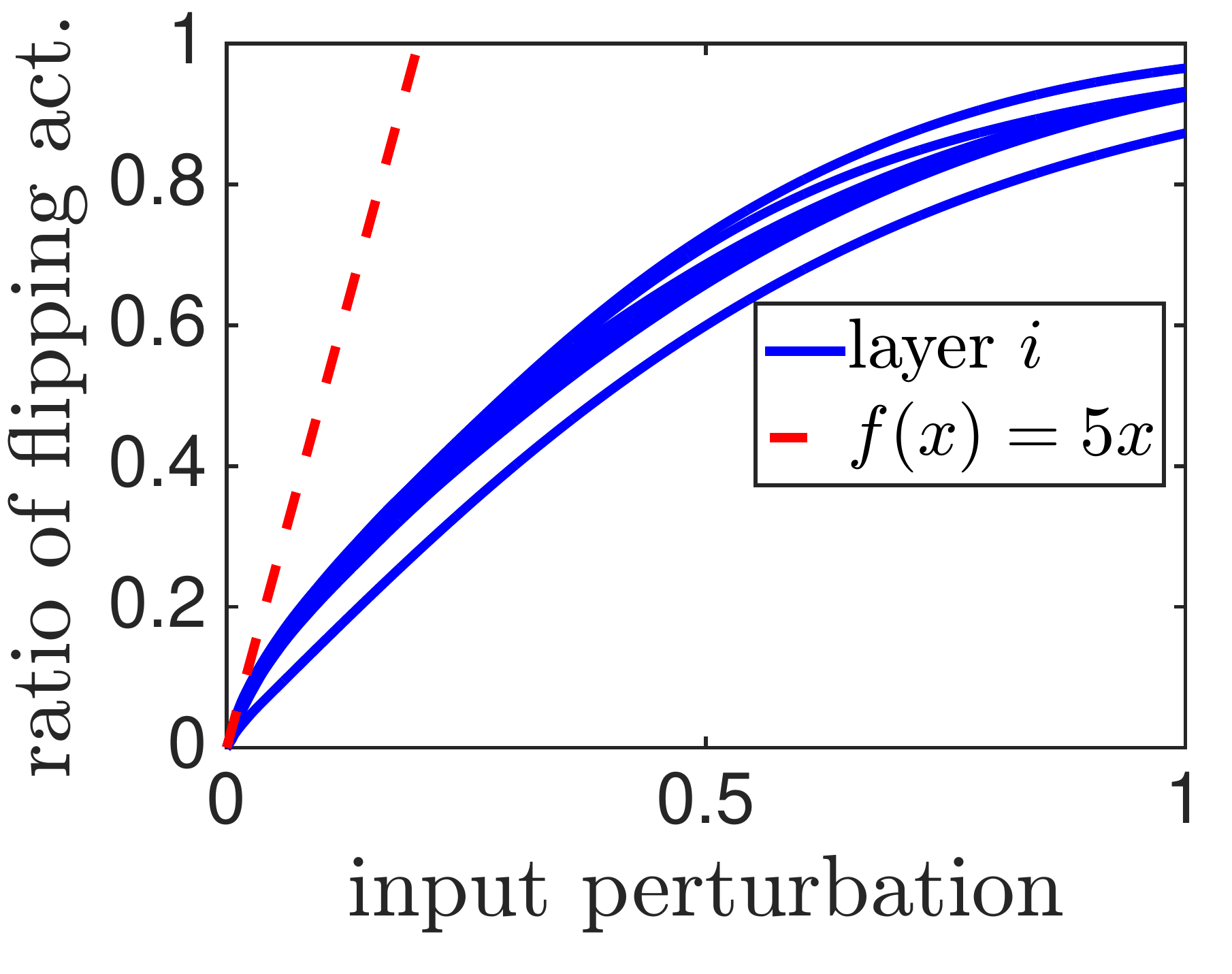}
\includegraphics[width=.32\textwidth]{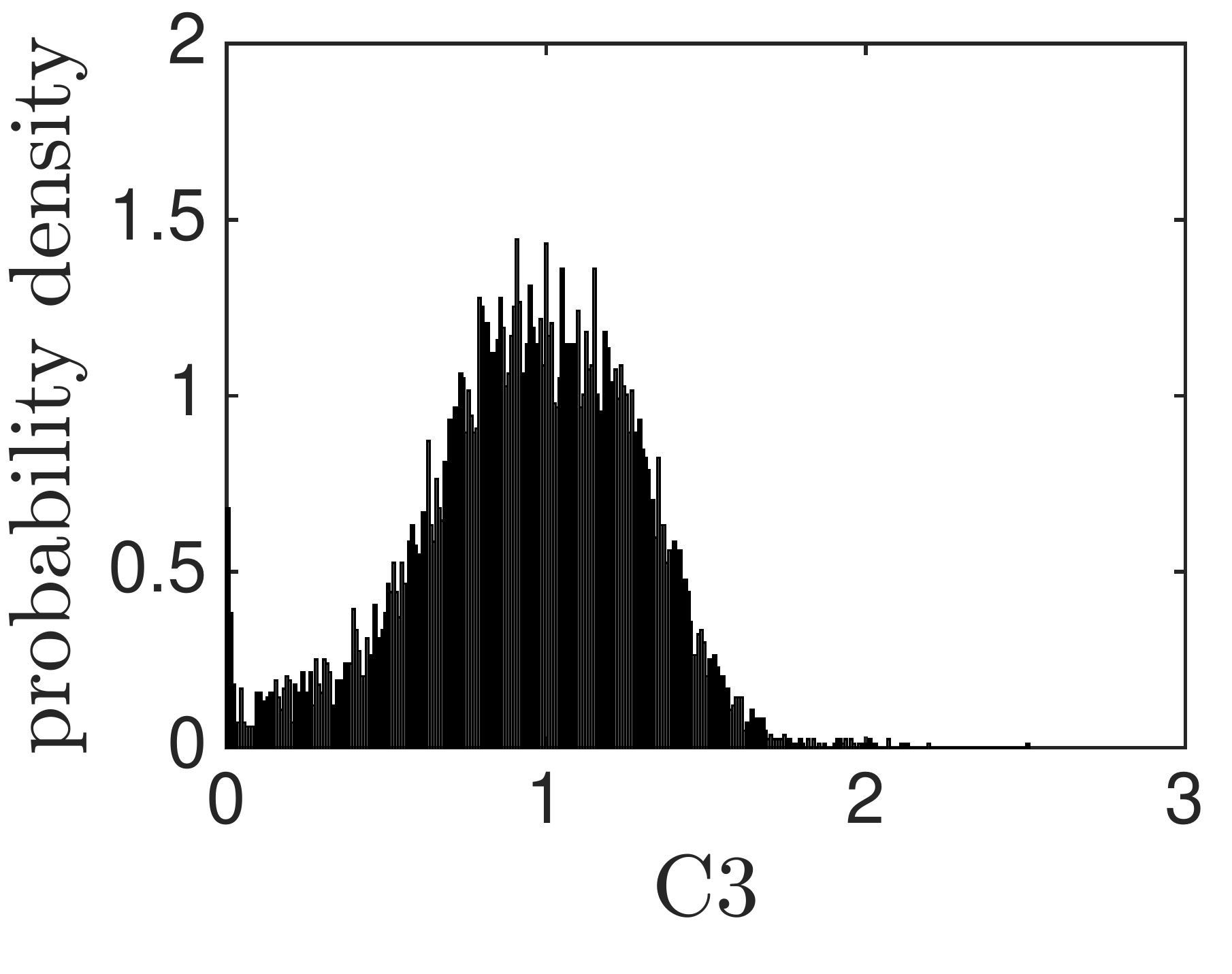}\\
\caption[\small Verifying the conditions of Theorem \ref{thm:relu} on a multi-layer perceptron]{\small Verifying the conditions of Theorem \ref{thm:relu} on a 10 layer perceptron with 1000 hidden units in each layer, i.e. more than 10,000,000 parameters on MNIST. We have numerically checked that all values are within the displayed range. \textbf{Left}: $C1$: condition number of the network, i.e. $\frac{1}{\mu}$. \textbf{Middle}:  $C2$: the ratio of activations that flip based on magnitude of perturbation. \textbf{Right}:  $C3:$ the ratio of norm of incoming weights to each hidden units with respect to average of the same quantity over hidden units in the layer.}
\label{fig:conditions_verify1}
\end{figure}

Additionally the conditions $C1-C3$ actually hold for networks trained in practice as we verify in Figure~\ref{fig:conditions_verify1}, and our experiments suggest that, $\mu \geq 1/4, C2 \leq 5$ and $C3 \leq 3$. More details on the verification and comparing the conditions on learned network with those of random weights, are presented in Section \ref{sec:sharpness-proofs}.

\paragraph{Proof of Theorem~\ref{thm:relu}}
We bound the expectation as follows:
\begin{align}
&\mathE\abs{ \hatl(f_{\vecw+\vecu}(x)) -\hatl(f_{\vecw}(x))} \nonumber \\ 
&\quad \quad  \leq C_{L} \mathE\| f_{\vecw+\vecu}(x) -f_\vecw(x)\|_F \nonumber \\
&\quad \quad \stackrel{(i)}{=}  C_{L}\mathE\| (W+\vecu)_d \left(\Pi_{i=1}^{d-1} \hatD_{i} (W+\vecu)_i \right)*x - W_d \left(\Pi_{i=1}^{d-1} D_{i}W_i \right)*x \|_F  \nonumber \\
& \quad \quad \leq  C_{L} \mathE\| (W+\vecu)_d \left(\Pi_{i=1}^{d-1} D_{i} (W+\vecu)_i \right)*x - W_d \left(\Pi_{i=1}^{d-1} D_{i}W_i \right)*x \|_F \nonumber \\ & \quad \quad \quad \quad + C_{L} \mathE\| (W+\vecu)_d \left(\Pi_{i=1}^{d-1} \hatD_{i} (W+\vecu)_i \right)*x - (W+\vecu)_d \left(\Pi_{i=1}^{d-1} D_{i} (W+\vecu)_i \right)*x \|_F \nonumber \\
&\quad \quad \leq  C_{L} \mathE\| (W+\vecu)_d \left(\Pi_{i=1}^{d-1} D_{i} (W+\vecu)_i \right)*x - W_d \left(\Pi_{i=1}^{d-1} D_{i}W_i \right)*x \|_F +C_{L} \mathE\|\err_d\|_F,  \label{eq:thm_relu1}
\end{align}
where $\err_d =  \| (W+\vecu)_d \left(\Pi_{i=1}^{d-1} \hatD_{i} (W+\vecu)_i \right)*x - (W+\vecu)_d \left(\Pi_{i=1}^{d-1} D_{i} (W+\vecu)_i \right)*x \|_F$. $(i)$ $\hatD_i$ is the diagonal matrix with 0's and 1's corresponding to the activation pattern of the perturbed network $f_{\vecw+\vecu}(x)$.

The first term in the equation~\eqref{eq:thm_relu1} corresponds to error due to perturbation of a network with unchanged activations (linear network). Intuitively this is small when any subset of successive layers of the network do no interact weakly with each other (not orthogonal to each other). Condition $C1$ captures this intuition and we bound this error in Lemma~\ref{eq:lem_linear1}. 

\begin{lemma}\label{lem:linear}
Let $U_i$ be a random $h_i \times h_{i-1}$ matrix with each entry distributed according to $\mathcal{N}(0,\sigma_i^2)$. Then, under the condition $C1$,
\begin{multline*}
\mathE\| (W+\vecu)_d \left(\Pi_{i=1}^{d-1} D_{i}(W+\vecu)_i \right)*x - W_d \left(\Pi_{i=1}^{d-1} D_{i}W_i \right)*x\|_F \\ \leq  \left(\Pi_{i=1}^d \left(1+ \frac{\sigma_i \sqrt{h_ih_{i-1}}}{\mu^2 \|D_i W_i\|_F} \right) -1 \right)   \|f_\vecw(x)\|_F. 
\end{multline*}
\end{lemma}

The second term in the equation~\eqref{eq:thm_relu1} captures the perturbation error due to change in activations. If a tiny perturbation can cause exponentially many changes in number of active nodes, then that network will have huge sharpness. Condition $C2$ and $C3$ essentially characterize the behavior of sensitivity of activation patterns to perturbations, leading to a bound on this term in Lemma~\ref{lem:recursion}.

\begin{lemma} \label{lem:recursion}
Let $U_i$ be a random $h_i \times h_{i-1}$ matrix with each entry distributed according to $\mathcal{N}(0,\sigma_i^2)$. Then, under the conditions $C1$, $C2$ and $C3$, with probability $\geq 1-\delta$, for all $ 1 \leq k \leq d$, 
$$\|\hatD_{k} -D_{k}\|_1 \leq  O \left( C_2 h_k  C_{\delta} \sigma_k  \|f^{k-1}_{\vecw}\|_F \right) $$ and 
$$\E\| \err_k\|_F \leq  O \left( \Pi_{i=1}^k \left( 1+\gamma_iC_2 C_3\right)\left(\Pi_{i=1}^k (1+ \gamma_i C_{\delta} C_2) - 1\right)\|f^k_\vecw\|_F \right).$$
where $\gamma_i = \frac{\sigma_i \sqrt{h_i} \sqrt{h_{i-1}}}{\mu^2 \|D_iW_i\|_F}$ and $C_\delta=2\sqrt{\ln(dh/\delta)}$.
\end{lemma}

Hence, from Lemma~\ref{eq:lem_linear1} and Lemma~\ref{lem:recursion} we get, 
\begin{small}
\begin{align*}
&\mathE\abs{ \hatl(f_{\vecw+\vecu}(x)) -\hatl(f_{\vecw}(x))} \nonumber \\ 
&\leq \left[\Pi_{i=1}^d \left(1+ \gamma_i \right) -1  + \Pi_{i=1}^d \left( 1+\gamma_i C_2 C_3 \right)\left(\Pi_{i=1}^d (1+ \gamma_i C_{\delta} C_2  )- 1\right) \right] C_{L} \|f_\vecw(x)\|_F.
\end{align*}

Here $\gamma_i =\frac{\sigma_i\sqrt{h_i}\sqrt{h_{i-1}} }{\mu^2 \|D_iW_i\|_F } $. Substituting the above bound on expected sharpness in the PAC-Bayes result (equation~\eqref{eq:pacbayes}), gives the result.
\end{small}

\section{Supporting Results}

\subsection{Supporting Lemma}
\begin{lemma}\label{lem:gauss_product}
Let $A$ ,$B$ be $n_1 \times n_2$ and $n_3 \times n_4$ matrices and $\vecu$ be a $n_2\times n_3$ entrywise random Gaussian matrix with $\vecu_{ij} \sim \mathcal{N}(0,\sigma)$. Then,
$$\mathE\left[\| A*\vecu*B\|_F \right] \leq \sigma \|A\|_F  \|B\|_F .$$
\end{lemma}
\begin{proof}
By Jensen's inequality, 
\begin{align*}
\mathE\left[\| A*\vecu*B\|_F \right]^2 &\leq \mathE\left[\| A*\vecu*B\|_F^2 \right] \\
&= \mathE\left[\left(\sum_{ij} \sum_{kl} A_{ik} \vecu_{kl} B_{lj} \right)^2 \right]\\
&= \sum_{ij} \sum_{kl} A_{ik}^2 \mathE\left[\vecu_{kl}^2\right] B_{lj}^2 \\
&= \sigma^2 \|A\|_F^2 \| B\|_F^2.
\end{align*}

\end{proof}


\subsection{Conditions in Theorem ~\ref{thm:relu}}
In this section, we compare the conditions in Theorem ~\ref{thm:relu} of a learned network with that of its random initialization. We trained a 10-layer feedforward network with 1000 hidden units
in each layer on MNIST dataset. Figures~\ref{fig:conditions_verify2}, \ref{fig:conditions_verify3} and \ref{fig:conditions_verify4} compare condition $C1$, $C2$ and $C3$ on learned weights to that of random initialization respectively. Interestingly,
we observe that the network with learned weights is very similar to its random initialization in terms of these conditions.
\begin{figure}[t]
\centering
\includegraphics[width=.32\textwidth]{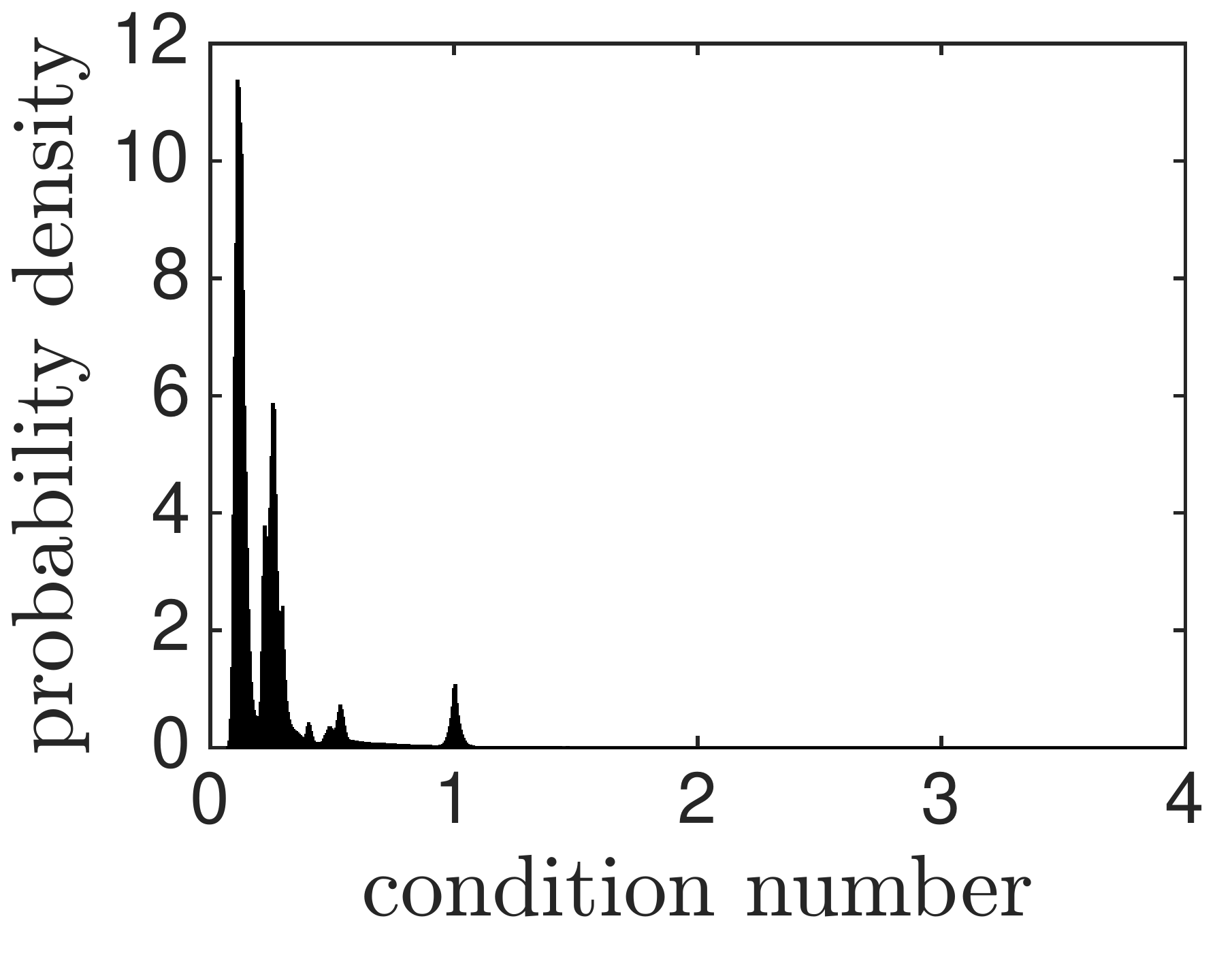}
\includegraphics[width=.32\textwidth]{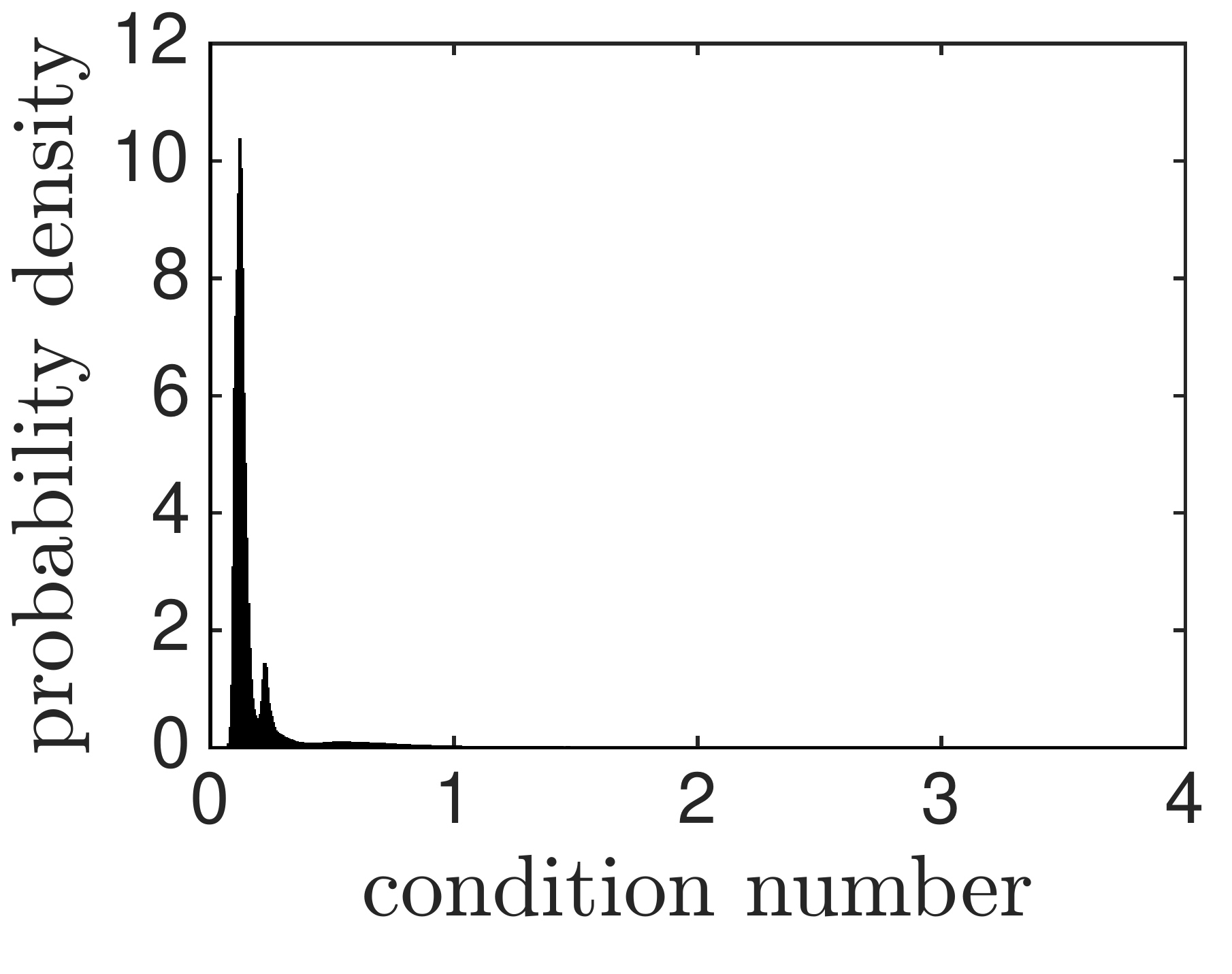}
\includegraphics[width=.32\textwidth]{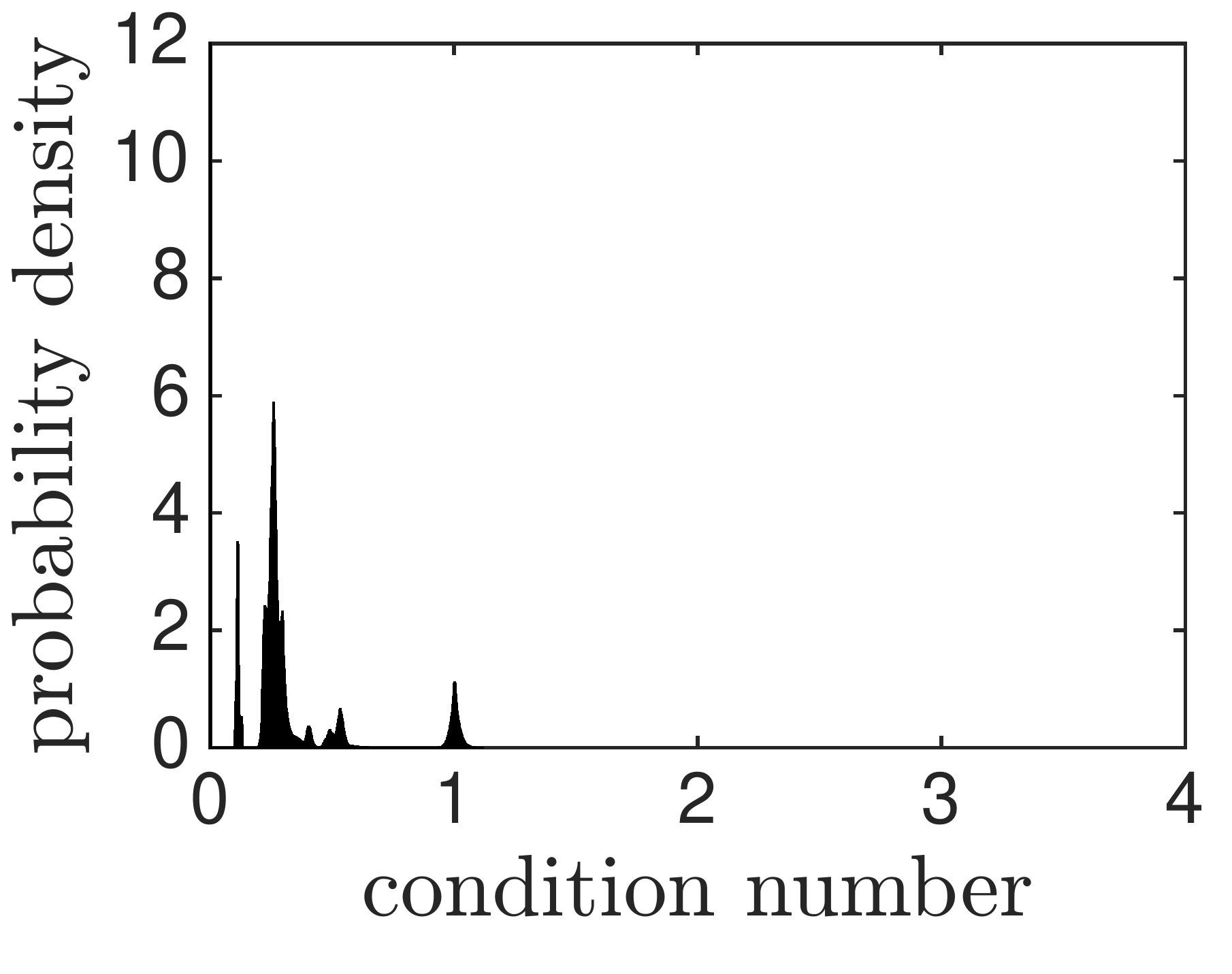}\\
\includegraphics[width=.32\textwidth]{hist.pdf}
\includegraphics[width=.32\textwidth]{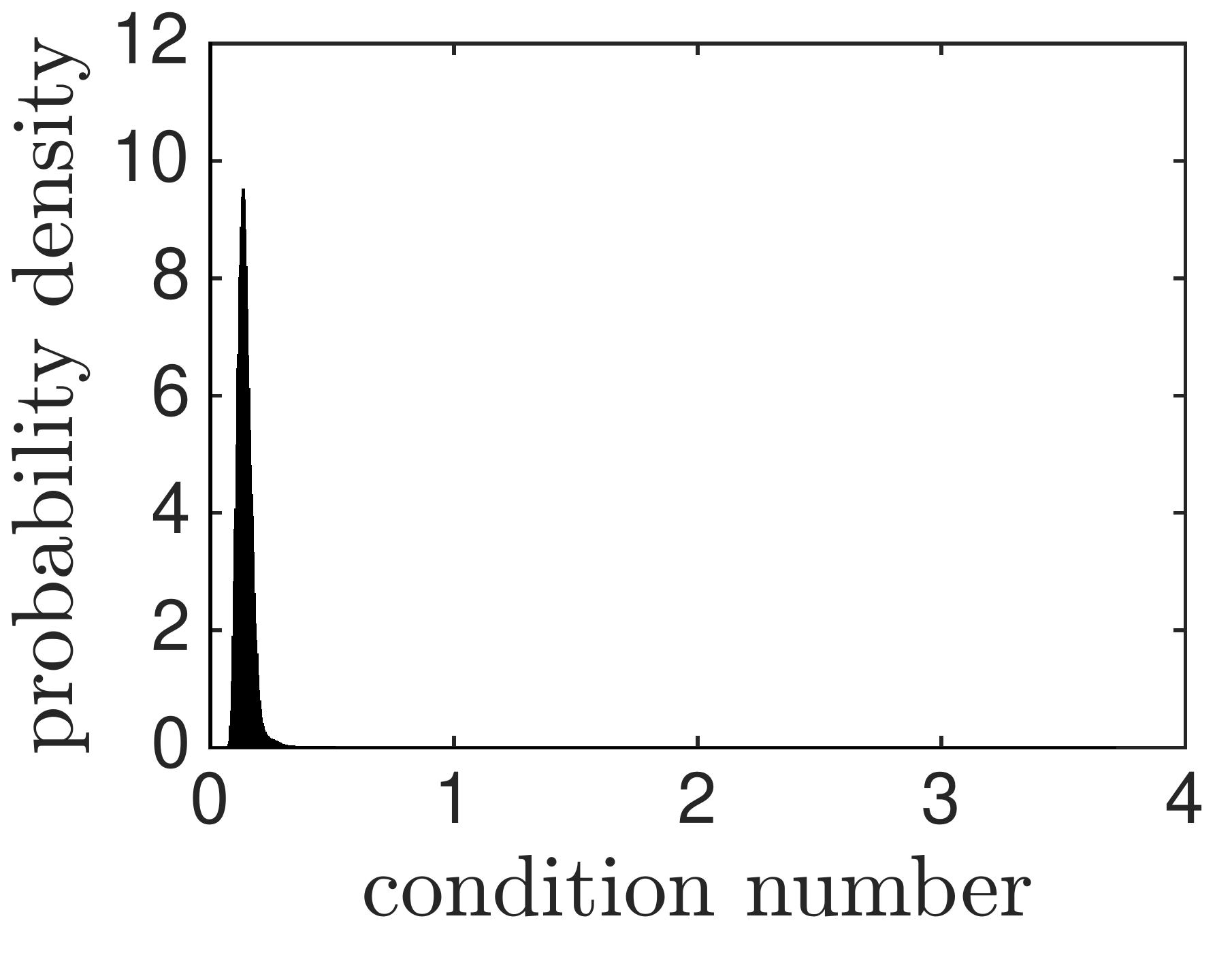}
\includegraphics[width=.32\textwidth]{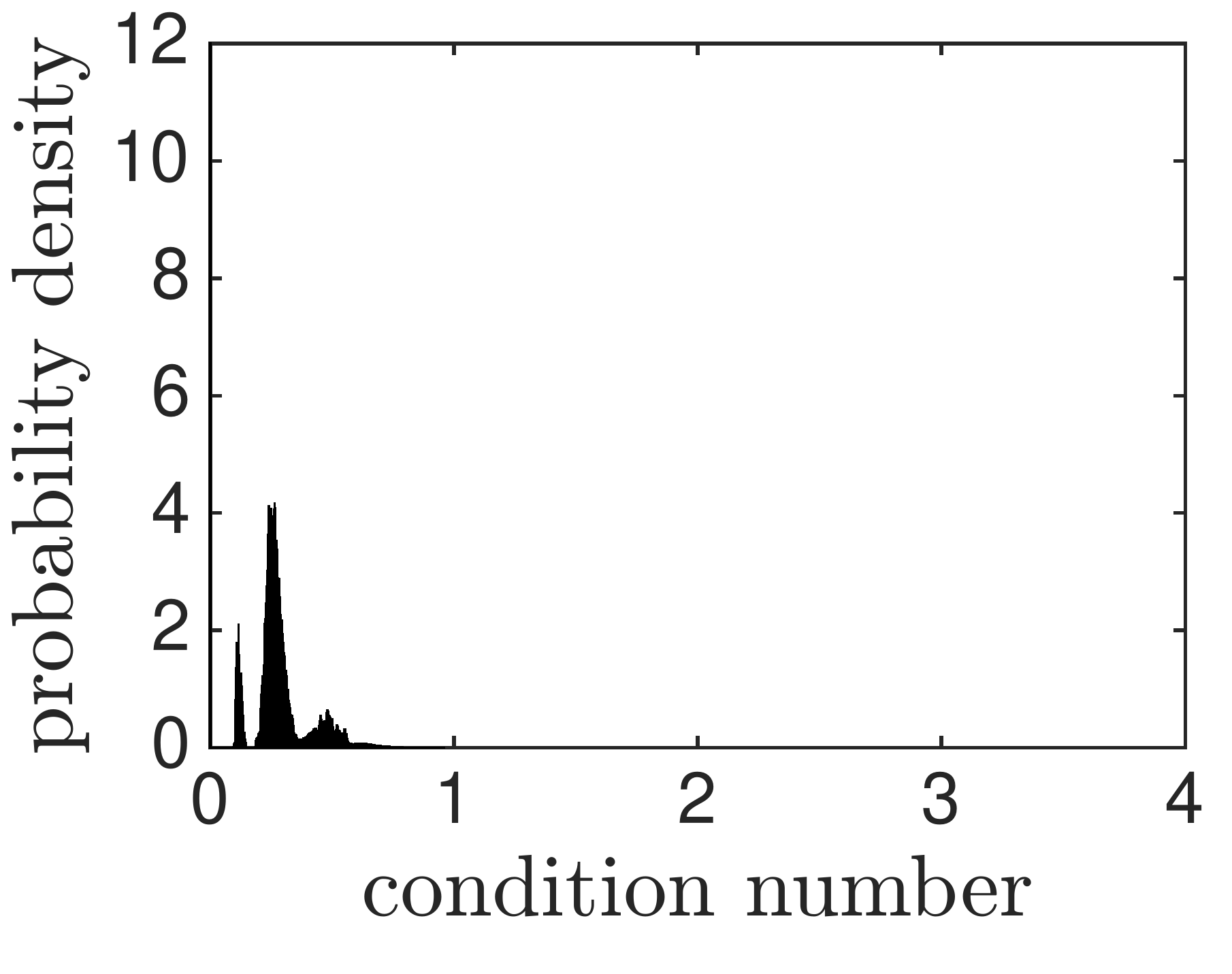}\\
\caption[\small Condition $C1$ in Theorem~\ref{thm:relu} on a learned network.]{\small Condition $C1$: condition number $\frac{1}{\mu}$ of the network and its decomposition to two cases for random initialization and learned weights. \textbf{Top}: random initialization \textbf{Bottom}: learned weights. \textbf{Left}: distribution of all combinations of $a\leq c\leq b-1$. \textbf{Middle}: when $a<c<b-1$. \textbf{Right}: when $c=a$ or $c=b-1$.}
\label{fig:conditions_verify2}
\end{figure}

\begin{figure}[t]
\centering
\includegraphics[width=.32\textwidth]{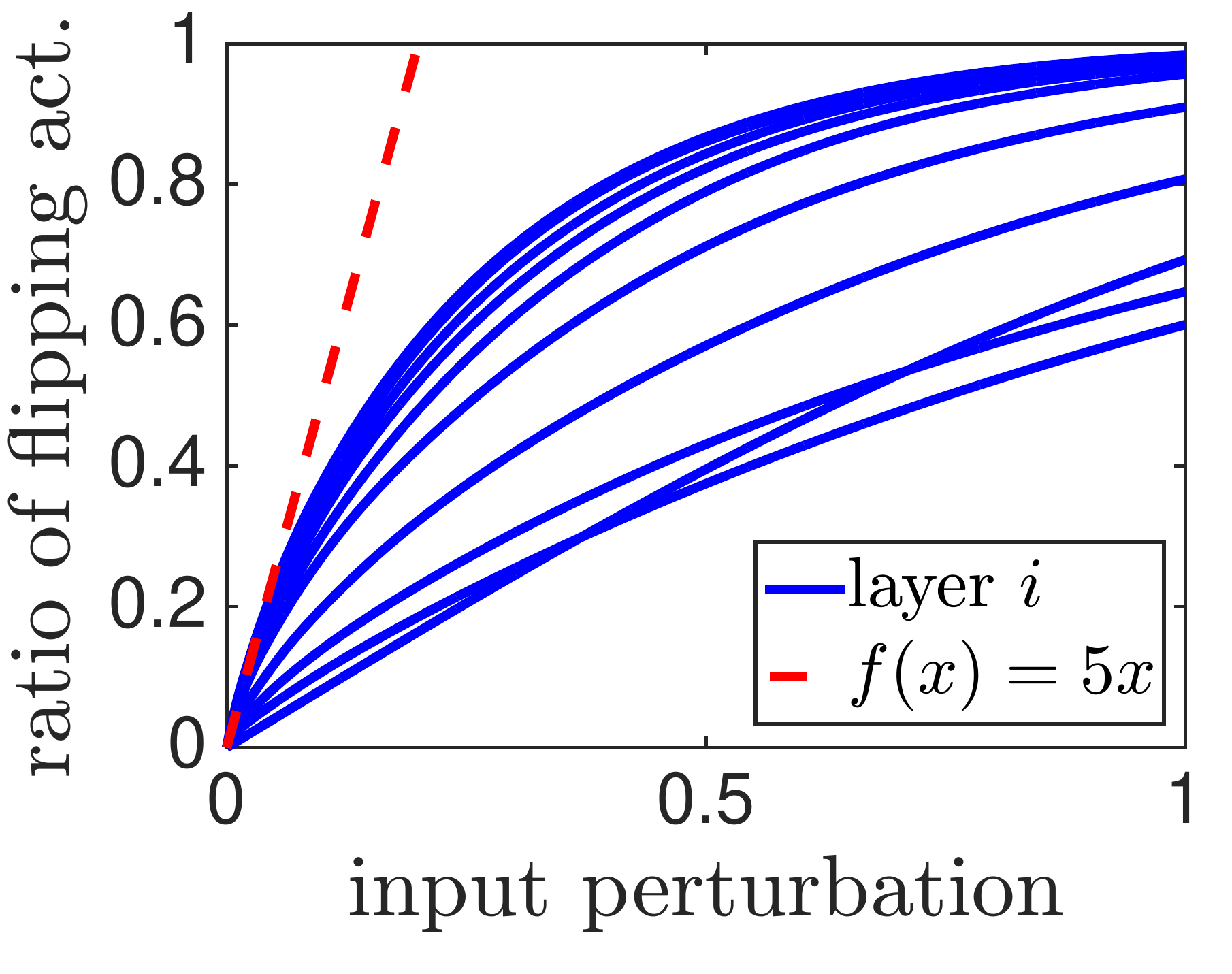}
\includegraphics[width=.32\textwidth]{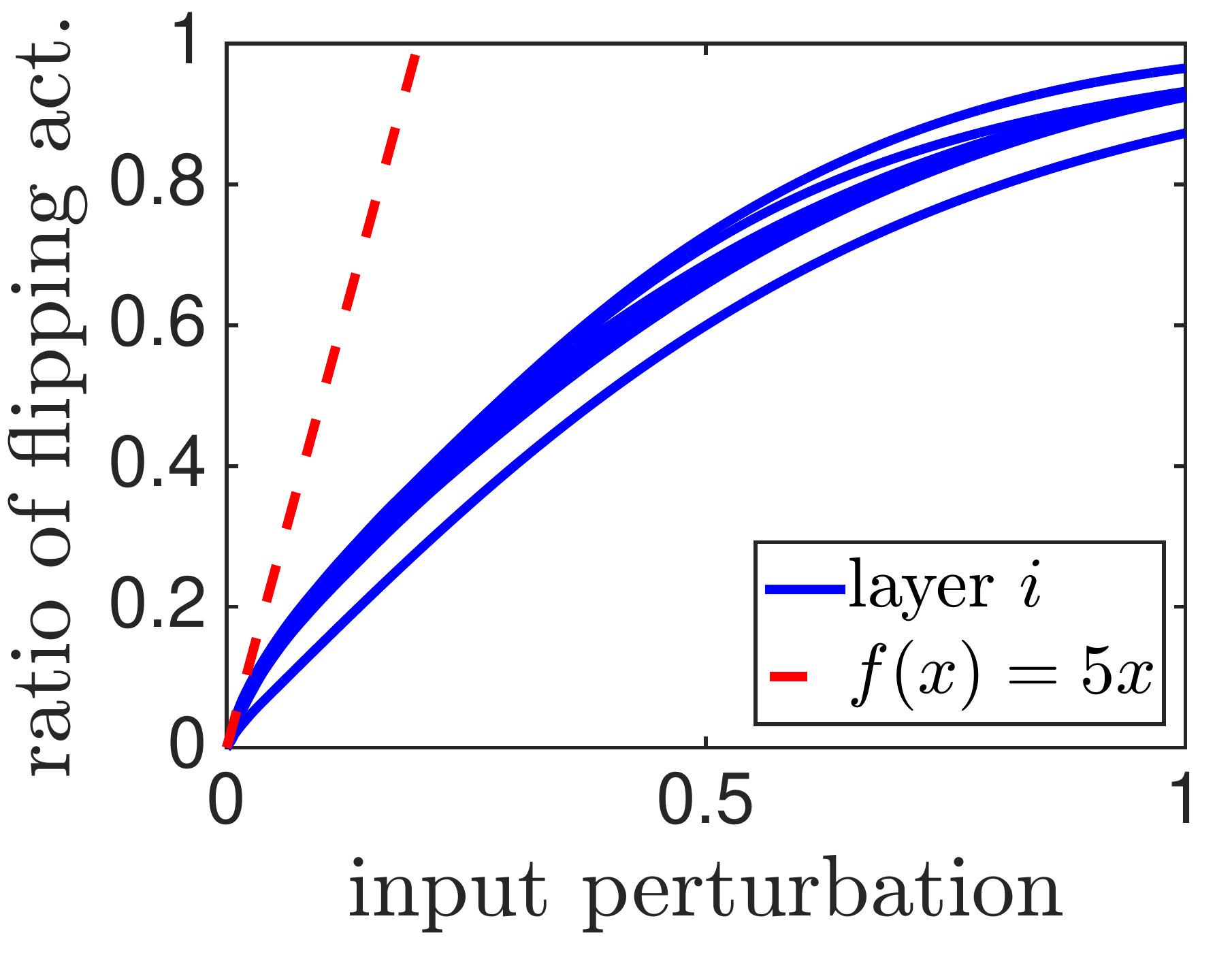}
\includegraphics[width=.33\textwidth]{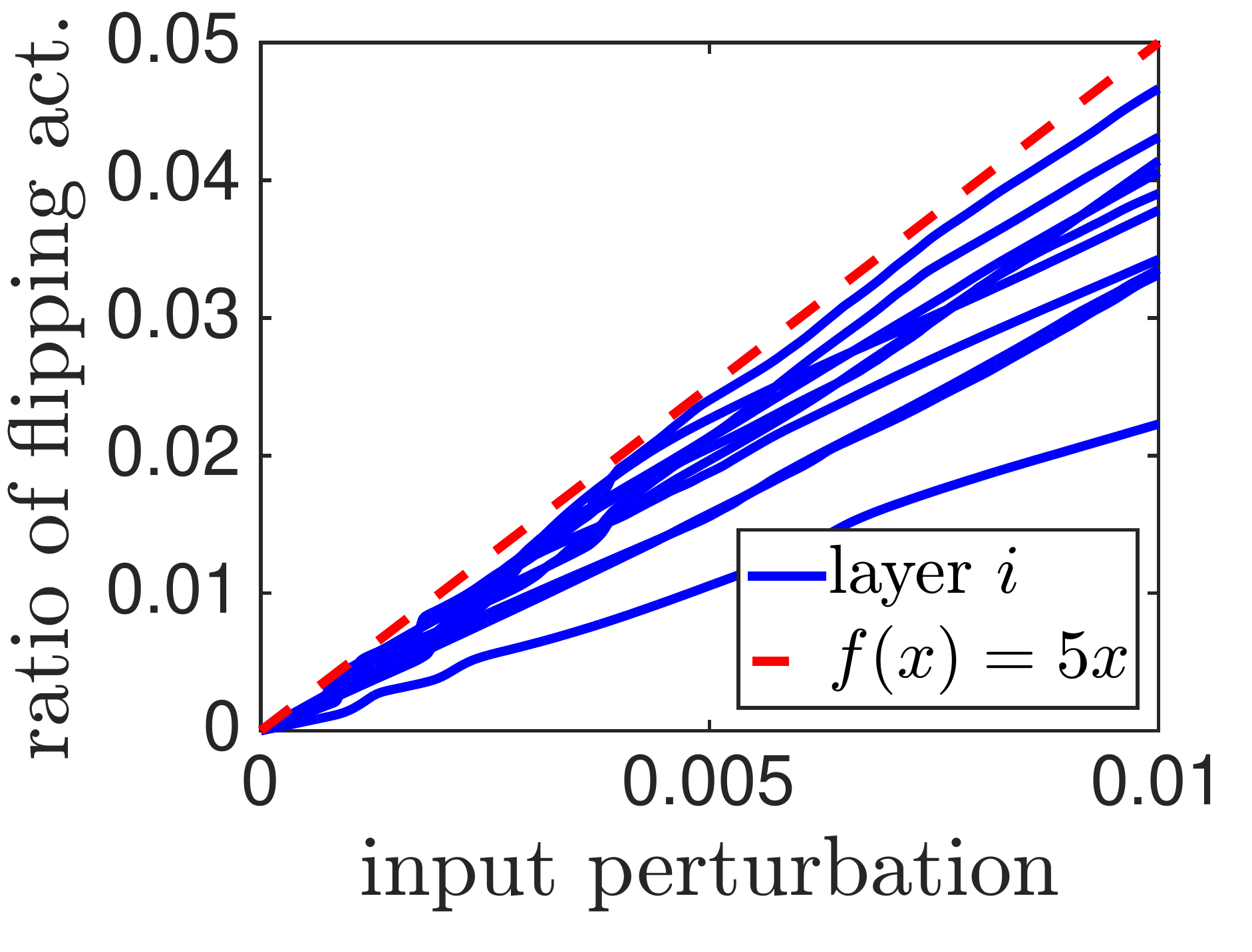}
\caption[\small Ratio of activations that flip based on the magnitude of perturbation.]{\small Ratio of activations that flip based on the magnitude of perturbation. \textbf{Left}: random initialization. \textbf{Middle}: learned weights. \textbf{Right}: learned weights (zoomed in).}
\label{fig:conditions_verify3}
\end{figure}

\begin{figure}[t]
\centering
\includegraphics[width=.24\textwidth]{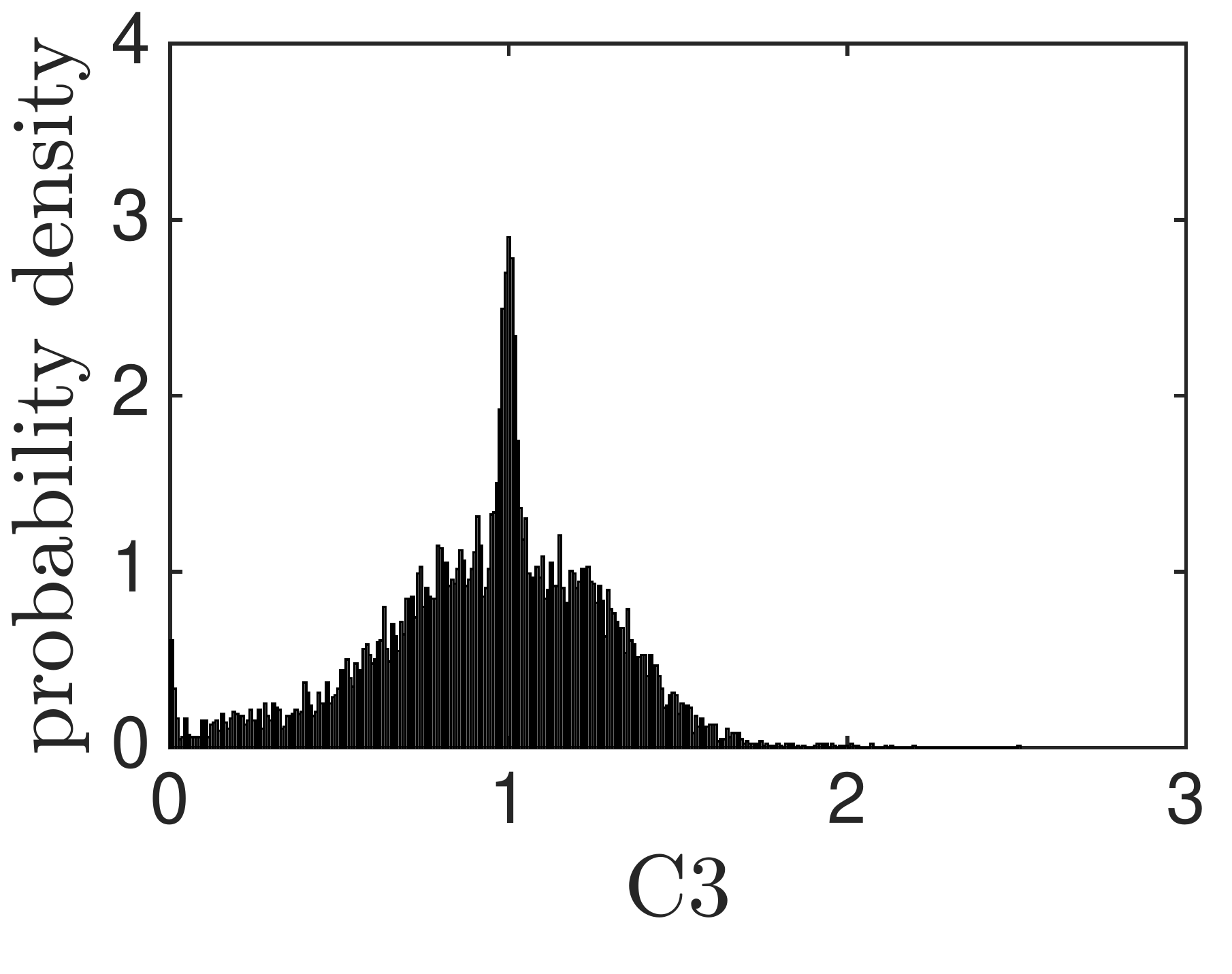}
\includegraphics[width=.24\textwidth]{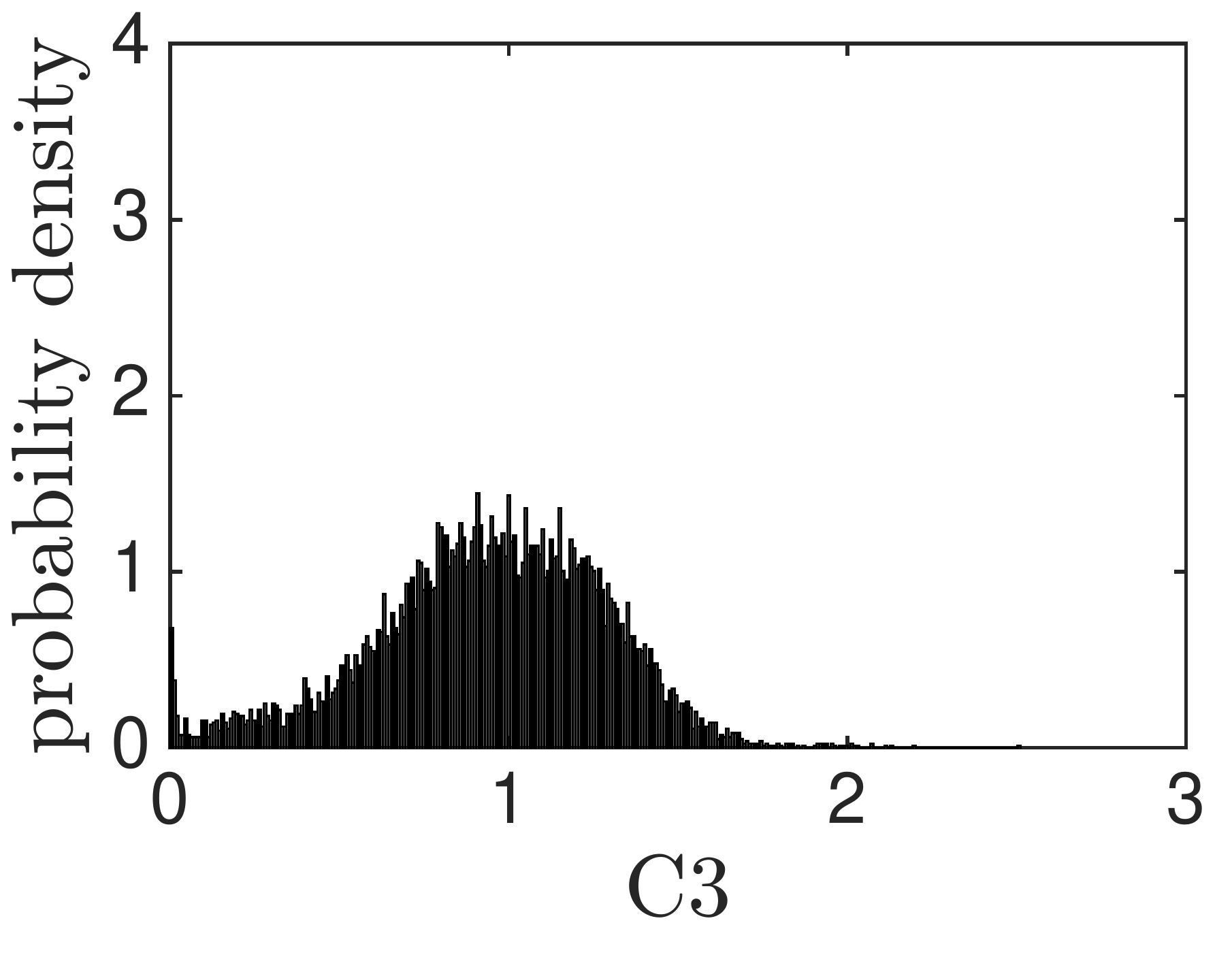}
\includegraphics[width=.25\textwidth]{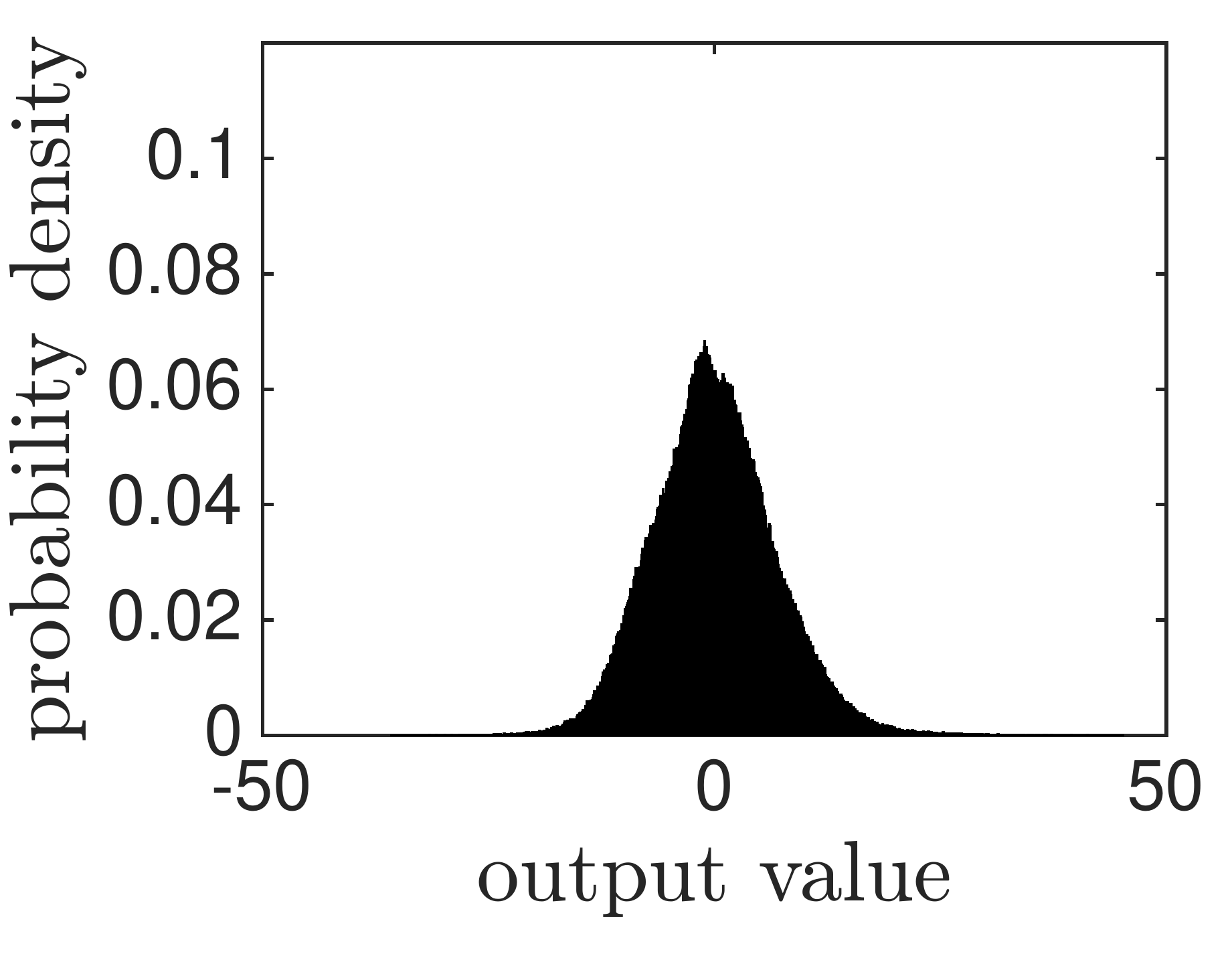}
\includegraphics[width=.25\textwidth]{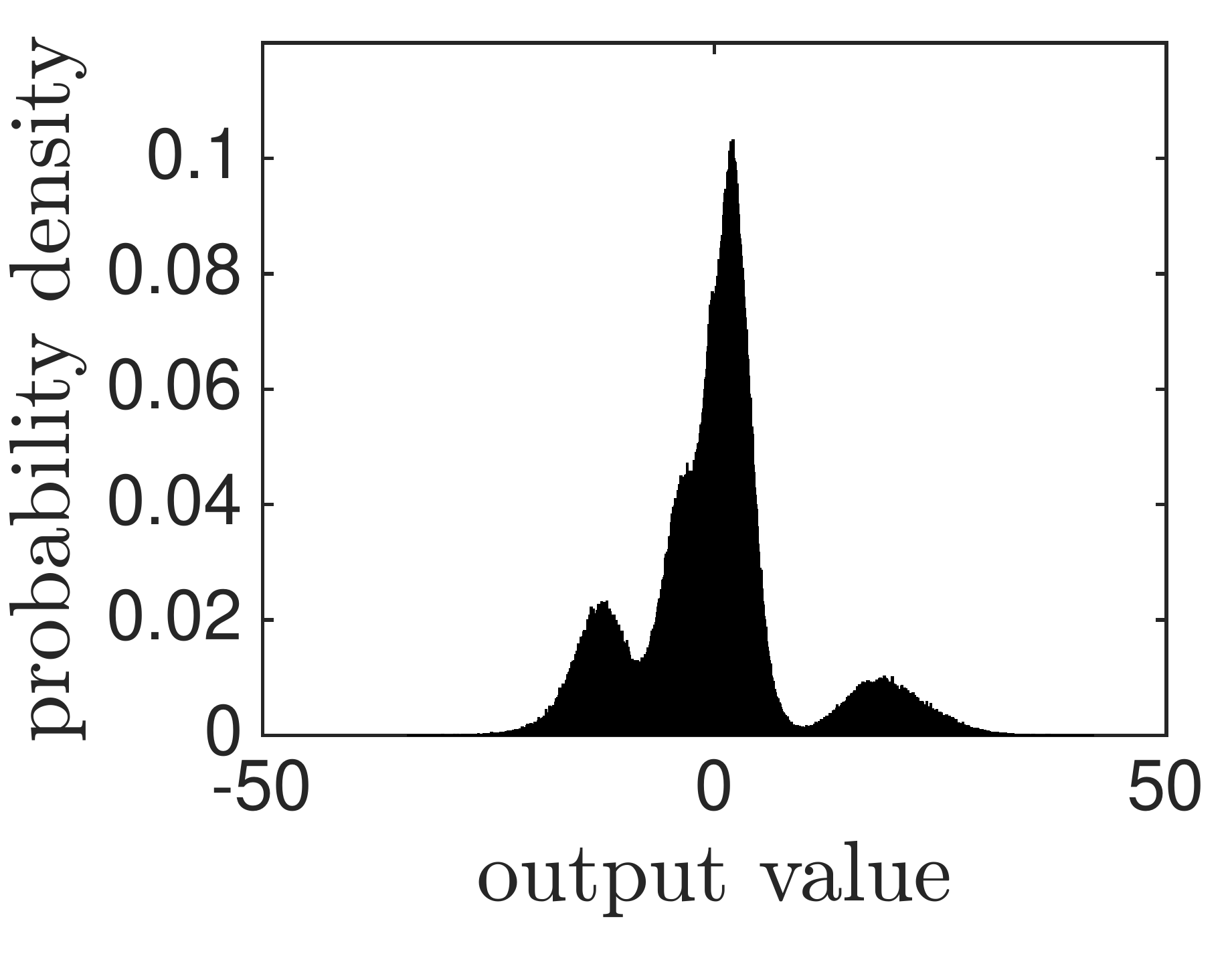}
\caption[\small Condition $C3$ in Theorem~\ref{thm:relu} for random initialization and learned network]{\small From left to right: Condition $C3$ for random initialization and learned network, output values for random and learned network}
\label{fig:conditions_verify4}
\end{figure}

\section{Proofs}\label{sec:sharpness-proofs}

\subsection{Proof of Lemma~\ref{lem:worstcase-sharpness}}
\begin{proof}
Let $\Delta_i= \abs{f^i_{\vecw+\vecu}(\vecx)-f^i_{\vecw}(\vecx)}_2$ be the sharpness of layer $i$. We will prove using induction that for any $i\geq 0$:
\begin{small}
\begin{equation*}
\Delta_i \leq \left(1+\frac{1}{d}\right)^i \left(\prod_{j=1}^i \norm{W_j}_2\right)\abs{\vecx}_2\sum_{j=1}^i \frac{\norm{U_j}_2}{\norm{W_j}_2}.
\end{equation*}
\end{small}
The above inequality together with $\left(1+\frac{1}{d}\right)^{d}\leq e$ proves the lemma statement. The induction base clearly holds since $\Delta_0 =\abs{\vecx-\vecx}_2=0$. For any $i\geq 1$, we have the following:

\begin{align*}
\Delta_{i+1} &= \abs{\left(W_{i+1}+U_{i+1}\right)\phi_i(f^i_{\vecw+\vecu}(\vecx))- W_{i+1}\phi_i(f^i_{\vecw}(\vecx))}_2\nonumber\\\nonumber
&= \abs{\left(W_{i+1}+U_{i+1}\right)\left(\phi_i(f^i_{\vecw+\vecu}(\vecx))- \phi_i(f^i_{\vecw}(\vecx))\right)+ U_{i+1}\phi_i( f^i_{\vecw}(\vecx))}_2\\\nonumber
&\leq \left(\norm{W_{i+1}}_2+\norm{U_{i+1}}_2\right)\abs{\phi_i(f^i_{\vecw+\vecu}(\vecx))- \phi_i(f^i_{\vecw}(\vecx))}_2+ \norm{U_{i+1}}_2\abs{\phi_i(f^i_{\vecw}(\vecx))}_2\\\nonumber
&\leq \left(\norm{W_{i+1}}_2+\norm{U_{i+1}}_2\right)\abs{f^i_{\vecw+\vecu}(\vecx)-f^i_{\vecw}(\vecx)}_2+ \norm{U_{i+1}}_2\abs{f^i_{\vecw}(\vecx)}_2\\
&=\Delta_i\left(\norm{W_{i+1}}_2+\norm{U_{i+1}}_2\right)+ \norm{U_{i+1}}_2\abs{f^i_{\vecw}(\vecx)}_2,
\end{align*}

where the last inequality is by the Lipschitz property of the activation function and using $\phi(0)=0$. The $\ell_2$ norm of outputs of layer $i$ is bounded by $\abs{\vecx}_2\Pi_{j=1}^i \norm{W_j}_2$ and by the lemma assumption we have $\norm{U_{i+1}}_2\leq \frac{1}{d}\norm{W_{i+1}}_2$. Therefore, using the induction step, we get the following bound: \begin{small}
\begin{align*}
\Delta_{i+1}&\leq \Delta_i\left(1+\frac{1}{d}\right)\norm{W_{i+1}}_2+ \norm{U_{i+1}}_2\abs{\vecx}_2\prod_{j=1}^i \norm{W_j}_2\nonumber\\\nonumber
&\leq \left(1+\frac{1}{d}\right)^{i+1} \left(\prod_{j=1}^{i+1}\norm{W_j}_2\right)\abs{\vecx}_2\sum_{j=1}^i \frac{\norm{U_j}_2}{\norm{W_j}_2}+ \frac{\norm{U_{i+1}}_2}{\norm{W_{i+1}}_2}\abs{\vecx}_2\prod_{j=1}^{i+1}\norm{W_i}_2\\
&\leq \left(1+\frac{1}{d}\right)^{i+1} \left(\prod_{j=1}^{i+1}\norm{W_j}_2\right)\abs{\vecx}_2\sum_{j=1}^{i+1} \frac{\norm{U_j}_2}{\norm{W_j}_2}.
\end{align*}
\end{small}
This completes the proof.
\end{proof}

\subsection{Proof of Lemma~\ref{lem:linear}}
\begin{proof}
Define $g_{\{W_{-i -j}, \vecu_{i,j} \}}(x)$ as the network $f_W$ with weights in layers $i, j,$, $W_i, W_j$ replaced by $U_i, \vecu_j$. Hence,
\begin{align}
 \| (W+\vecu)_d &\left(\Pi_{i=1}^{d-1} D_{i}(W+\vecu)_i \right)*x - W_d \left(\Pi_{i=1}^{d-1} D_{i}W_i \right)*x\|_F \nonumber \\
& \leq  \|\sum_i g(\{W_{-i}, U_i \},x)\|_F +\|\sum_{i,j} g(\{W_{-i -j}, \vecu_{i,j} \},x)\|_F + \cdots + \| f_{\vecu}(x)\|_F  \label{eq:lem_linear1}
\end{align}

\noindent {\bf Base case:} First we show the bound for terms with one noisy layer. Let $g( \{ W_{-k}, \vecu_k \},x)$ denote $f_W(x)$ with weights in layer $k$, $W_k$ replaced by $\vecu_k$. Now notice that,
\begin{align*}
\mathE\|g( \{ W_{-k}, \vecu_k \},x)\|_F &= \mathE\|W_d \Pi_{i=k+1}^{d-1} D_{i}W_i* D_k \vecu_k * \left(\Pi_{i=1}^{k-1} D_{i}W_i \right)*x \|_F \\
&\stackrel{(i)}{\leq}  \sigma_k   \|W_d \Pi_{i=k+1}^{d-1} D_{i}W_i\|_F  \|  \|\left(\Pi_{i=1}^{k-1} D_{i}W_i \right)*x \|_F \\
&\stackrel{(ii)}{\leq}  \sigma_k \frac{\sqrt{h_k h_{k-1}}}{\mu^2 \|D_k W_k\|_F} \|W_d \left(\Pi_{i=1}^{d-1} D_{i}W_i \right)*x \|_F \\
&=\sigma_k \frac{\sqrt{h_k h_{k-1} }}{\mu^2 \|D_k W_k\|_F} \| f_W(x)\|_F.
\end{align*}
$(i)$ follows from Lemma~\ref{lem:gauss_product}. $(ii)$ follows from condition $C1$.

\noindent {\bf Induction step:} Let for any set $s \subset [d], |s| =k$, the following holds:   $$\mathE\|g(\{W_{-i}, \vecu_{i} \}_{i \in s},x)\|_F \leq  \|f_W(x)\|_F \Pi_{i \in s} \sigma_i \frac{\sqrt{h_i h_{i-1}}}{\mu^2\| D_i W_i\|_F} . $$

We will prove this now for terms with $k+1$ noisy layers. 
\begin{align*}
\mathE\|g(\{W_{-i}, \vecu_{i} \}_{i \in s \cup \{j\}},x)\|_F &\leq  \sigma_j \frac{\sqrt{h_j h_{j-1}}}{\mu^2 \|D_j W_j\|} \mathE\|g(\{W_{-i}, \vecu_{i} \}_{i \in s},x)\|_F \\
&\leq  \sigma_j \frac{\sqrt{h_j h_{j-1}}}{\mu^2\|D_j W_j\|}    \|f_W(x)\|_F \Pi_{i \in s} \sigma_i \frac{\sqrt{h_i h_{i-1}}}{\mu^2 \| D_i W_i\|_F} \\
&= \|f_W(x)\|_F \Pi_{i \in s \cup \{j\}} \sigma_i \frac{\sqrt{h_i h_{i-1}}}{\mu^2 \| D_i W_i\|_F}
\end{align*}

Substituting the above expression in equation~\eqref{eq:lem_linear1} gives,
 \begin{multline*}
 \| (W+\vecu)_d \left(\Pi_{i=1}^{d-1} D_{i}(W+\vecu)_i \right)*x - W_d \left(\Pi_{i=1}^{d-1} D_{i}W_i \right)*x\|_F  \\ \leq  \left(\Pi_{i=1}^d \left(1+ \frac{\sigma_i \sqrt{h_i}\sqrt{h_{i-1}}}{\mu^2 \|D_i W_i\|_F} \right) -1 \right)  \|f_W(x)\|_F.
\end{multline*}

\end{proof}

\subsection{Proof of Lemma~\ref{lem:recursion}}
\begin{proof}
We prove this lemma by induction on $k$. Recall that $\hatD_i$ is the diagonal matrix with 0's and 1's corresponding to the activation pattern of the perturbed network $f_{W+\vecu}(x)$. Let $1_{E}$ denote the indicator function, that is $1$ if the event $E$ is true, $0$ else. We also use $f_W^k(x)$ to denote the network truncated to level $k$, in particular $f_W^k(x) =\Pi_{i=1}^k D_k W_k x$.

\noindent {\bf Base case:}

\begin{multline*}
\|\hatD_1 - D_1\|_1 = \sum_i 1_{\inner{(W+\vecu)_{1,i}}{x}*\inner{W_{1,i}}{x} <0} =\sum_i 1_{\inner{(W)_{1,i}}{x}^2 < -\inner{(\vecu)_{1,i}}{x}*\inner{(W)_{1,i}}{x}} \\ \leq \sum_i 1_{\abs{\inner{(W)_{1,i}}{x}} < \abs{\inner{(\vecu)_{1,i}}{x}}}.
\end{multline*}
Since $\vecu_1$ is a random Gaussian matrix, and $\|x\| \leq 1$, for any $i$, $\abs{\inner{(\vecu)_{1,i}}{x}} \leq \sigma_1 (1+\delta_1)\sqrt{2\ln(h_1)}$ with probability greater than $1-\delta_1$. Hence, with probability $\geq 1-\delta_1$,

\begin{align*}
\|\hatD_1 - D_1\|_1 \leq \sum_i 1_{\abs{\inner{(W)_{1,i}}{x}} \leq \sigma_1 \sqrt{20 \ln(h_1)}}  \leq C_2 h_1 \sigma_1 (1+\delta_1)\sqrt{2\ln(h_1)} =C_2 h_1 \sigma_1 C_{\delta_1}.
\end{align*}

This completes the base case for $k=1$. $\hatD_1$ is a random variable that depends on $\vecu_1$. Hence, in the remainder of the proof, to avoid this dependence, we separately bound $\hatD_1 -D$ using the expression above and compute expectation only with respect to $\vecu_1$. With probability $\geq 1-\delta_1$,

\begin{align*}
\mathE\| \err_1 \|_F &= \mathE\|\hatD_1*(W+\vecu)_1 x - D_1*(W+\vecu)_1 x \|_F \\
&\leq \mathE\|(\hatD_1- D_1)*W_1 x \|_F + \mathE\|(\hatD_1- D_1)*\vecu_1 x \|_F \\
& \stackrel{(i)}{\leq} \sqrt{C_2 h_1 \sigma_1 C_{\delta_1}} \sigma_1 + \sqrt{C_2 h_1 \sigma_1 C_{\delta_1}} \sigma_1\\
& = 2 \sqrt{C_2 h_1 \sigma_1 C_{\delta_1}} \sigma_1.
\end{align*}
$(i)$ follows because, each hidden node in  $\mathE\|(\hatD_1- D_1)*W_1 x \|_F$ has norm less than $\sigma_1 C_{\delta_1}$ (as it changed its activation), number of such units is less than $C_2 h_1 \sigma_1 C_{\delta_1}$.

$k=1$ case does not capture all the intricacies and dependencies of higher layer networks. Hence we also evaluate the bounds for $k=2$.

\begin{align*}
\|\hatD_2 - D_2\|_1 \leq \sum_i 1_{\inner{(W+\vecu)_{2,i}}{f^1_{W+\vecu}}*\inner{W_{2,i}}{f^1_W} \leq 0 } \leq \sum_i 1_{\abs{\inner{W_{2,i}}{f^1_{W}}} \leq \abs{\inner{\vecu_{2,i}}{f^1_{W+\vecu}}} + \abs{\inner{W_{2,i}}{f^1_{W+\vecu}-f^1_{W}}} }
\end{align*}

Let $C_{\delta_2}= (1+\delta_2)\sqrt{2\ln(h_2)}$. Then, with probability $\geq 1-\delta_1- \delta_2$,

\begin{align*}
&\abs{\inner{\vecu_{2,i}}{f^1_{W+\vecu}}} + \abs{\inner{W_{2,i}}{f^1_{W+\vecu}-f^1_{W}}} \\ &\leq C_{\delta_2} \sigma_2 \left( \|f^1_W\|_F+ 2 \sqrt{C_2 h_1 \sigma_1 C_{\delta_1}} \sigma_1 \right) + \|W_{2, i}\| 2 \sqrt{C_2 h_1 \sigma_1 C_{\delta_1}} \sigma_1 \\
&\leq  C_{\delta_2} \sigma_2 \left( \|f^1_W\|_F + 2 \sqrt{C_2 h_1 \sigma_1 C_{\delta_1}} \sigma_1   \right) + C_3 \frac{\|D_2 W_2\|_F}{\sqrt{h_2}} 2 \sqrt{C_2 h_1 \sigma_1 C_{\delta_1}} \sigma_1   \\
& \stackrel{(i)}{\leq} C_{\delta_2} \sigma_2 \left(\|f^1_W\|_F  + 2 \sqrt{\frac{\hat{\sigma_1}}{\sqrt{h_i + h_{i-1}}}} \hat{\sigma_1} \right) + 2 \hat{\sigma_1}  \frac{C_3  \|f_W(x)\|_F^{\nicefrac{1}{d}}}{\mu} \sqrt{\frac{\hat{\sigma_1}}{\sqrt{h_i + h_{i-1}}}}\\
&=C_{\delta_2} \sigma_2 \left( \|f^1_W\|_F + \gamma_1 \hat{\sigma_1} \right) +  \frac{C_3  \|f_W(x)\|_F^{\nicefrac{1}{d}}}{\mu} \gamma_1 \hat{\sigma_1} 
\end{align*}
where, $\gamma_i = 2 \sqrt{\frac{\hat{\sigma_1}}{\sqrt{h_i + h_{i-1}}}}$. $(i)$ follows from condition $C1$, which results in $\Pi_{i=2}^d \frac{\mu \|D_i W_i\|_F}{\sqrt{h_i}} \frac{\mu \|D_1 W_1 x\|_F}{\sqrt{h_1}} \leq \|f_W(x)\|_F$. Hence, if we consider the rebalanced network\footnote{The parameters of ReLu networks can be scaled between layers without changing the function} where all layers have same values for $\frac{\mu \|D_i W_i\|_F}{\sqrt{h_i}} $, we get, $\frac{\mu \|D_i W_i\|_F}{\sqrt{h_i}} \leq \|f_W(x)\|_F^{\nicefrac{1}{d}}$. Also the above equations follow from setting, $\sigma_i =\frac{\hat{\sigma}_i}{C_2 C_{\delta_i}\sqrt{h_i + h_{i-1}}}$. 

Hence, with probability $\geq 1-\delta_1-  \delta_2$,
\begin{align*}
\|\hatD_2 - D_2\|_1 &\leq C_2*h_2 \left( C_{\delta_2} \sigma_2 \left( \|f^1_W\|_F + \gamma_1 \hat{\sigma_1} \right) +  \frac{C_3  \|f_W(x)\|_F^{\nicefrac{1}{d}}}{\mu}\gamma_1 \hat{\sigma_1} \right).
\end{align*}

Since, we choose $\sigma_i$ to scale as some small number $O(\sigma)$, in the above expression the first term scales as $O(\sigma)$ and the last two terms decay at least as $O(\sigma^{3}{2})$. Hence we do not include them in the computation of $\err$.
 
\begin{align*}
\mathE\| \err_2 \|_F &= \mathE\|\hatD_2(W+\vecu)_2 *\hatD_1*(W+\vecu)_1 x - D_2(W+\vecu)_2 *D_1*(W+\vecu)_1 x \|_F \\
&\leq \mathE\| (\hatD_2-D_2)(W+\vecu)_2 *(\hatD_1-D_1)*(W+\vecu)_1 x\|_F  + \mathE\| D_2(W+\vecu)_2 *(\hatD_1-D_1)*(W+\vecu)_1 x\|_F \\ & \quad \quad+ \mathE\| (\hatD_2-D_2)(W+\vecu)_2 *D_1*(W+\vecu)_1 x\|_F.
\end{align*}

We will bound now the first term in the above expression. With probability $\geq 1-\delta_1-  \delta_2$,

\begin{align*}
\mathE& \| (\hatD_2-D_2)(W+\vecu)_2 *(\hatD_1-D_1)*(W+\vecu)_1 x\|_F \\
&\leq \mathE\| (\hatD_2-D_2)W_2 *(\hatD_1-D_1)*W_1 x\|_F + \mathE\| (\hatD_2-D_2)W_2 *(\hatD_1-D_1)*\vecu_1 x\|_F \\ & \quad \quad + \mathE\| (\hatD_2-D_2)\vecu_2 *(\hatD_1-D_1)*W_1 x\|_F + \mathE\| (\hatD_2-D_2)\vecu_2 *(\hatD_1-D_1)*\vecu_1 x\|_F \\
&\leq  2 \sqrt{C_2*h_2 C_{\delta_2} \sigma_2  \|f^1_W\|_F }C_{\delta_2} \sigma_2  \|f^1_W\|_F \sqrt{C_2*h_1 *C_{\delta_1} \sigma_1}C_{\delta_1} \sigma_1 \\ & \quad \quad + 2 \sqrt{C_2*h_2 C_{\delta_2} \sigma_2  \|f^1_W\|_F }C_{\delta_2} \sigma_2  \sqrt{h_1} \sqrt{C_2*h_1 *C_{\delta_1} \sigma_1}C_{\delta_1} \sigma_1 +O(\sigma^2)\\
&\leq  4 \|f_W^2\|_F \frac{C_{\delta_2} \sigma_2 C_{\delta_1} \sigma_1 \sqrt{h_1}}{\mu \|D_2 W_2\|_F} \Pi_{i=1}^2 \sqrt{C_2 h_i C_{\delta_i} \sigma_i}.
\end{align*}


\noindent {\bf Induction step:}

Now we assume the statement for all $i \leq k$ and prove it for $k+1$.
$\|\hatD_{k} -D_{k}\|_1 \leq C_2 h_k  C_{\delta_k} \sigma_k  \|f^{k-1}_W\|_F$ and $\E\| \err_k\|_F \leq  \Pi_{i=1}^k \left( 1+\frac{\sigma_i\sqrt{h_{i-1}}}{\mu^2 C^i_{2,\infty}C_2}\right)\left(\Pi_{i=1}^k (1+\frac{\hat{\sigma_i}^{3/2}}{C_2})- 1\right)\|f^k_w\|_F$. Now we prove the statement for $k+1$.
 
\begin{align*}
\|\hatD_{k+1} - D_{k+1}\|_1 &= \sum_i 1_{\inner{(W+\vecu)_{k+1 , i}}{ \Pi_{i=1}^{k } \hatD_i (W+\vecu)_i *x } *\inner{W_{2,i}}{D_1 W_1 x } \leq 0} \\
&\leq \sum_i 1_{ \abs{ \inner{W_{k+1 , i}}{ \Pi_{i=1}^{k } \hatD_i (W+\vecu)_i *x }} \leq \abs{ \inner{\vecu_{k+1 , i}}{ \Pi_{i=1}^{k } \hatD_i (W+\vecu)_i *x }} } \\
&=  \sum_i 1_{ \abs{ \inner{W_{k+1 , i}}{ f^k_{W+\vecu} }} \leq \abs{ \inner{\vecu_{k+1 , i}}{ f^k_{W+\vecu} }}} \\
&\leq   \sum_i 1_{ \abs{ \inner{W_{k+1 , i}}{ f^k_{W} }} \leq \abs{ \inner{\vecu_{k+1 , i}}{ f^k_{W} }} + \abs{ \inner{\vecu_{k+1 , i}}{ f^k_{W +\vecu} -f^k_W }} + \abs{ \inner{W_{k+1 , i}}{ f^k_{W+\vecu} -f^k_{W} }} } 
\end{align*}

Hence, with  probability $\geq 1-\sum_{i=1}^k\delta_i$, 
\begin{align*}
\|\hatD_{k+1} - D_{k+1}\|_1 &\leq  C_2 h_{k+1} \left[ C_{\delta_k} \sigma_{k+1} ( \|f^k_{W} \|_F +\|f^k_{W+\vecu} -f^k_W \|_F ) + \| W_{k+1, i}\|  \|f^k_{W+\vecu} -f^k_W \|_F \right] \\
&\leq C_2 h_{k+1}C_{\delta_k} \sigma_{k+1} \|f^k_{W} \|_F +C_2 h_{k+1}C_{\delta_k} \sigma_{k+1}\|f^k_{W+\vecu} -f^k_W \|_F  +C_2 h_{k+1} \| W_{k+1, i}\|  \|f^k_{W+\vecu} -f^k_W \|_F . 
\end{align*}
Now we will show that the last two terms in the above expression scale as $O(\sigma^2)$. For that, first notice that $\|f^k_{W+\vecu} -f^k_W \|_F \leq  \left(\Pi_{i=1}^k \left(1+ \frac{\sigma_i \sqrt{h_i h_{i-1}}}{\mu^2 \|D_i W_i\|_F} \right) -1 \right)  \|f_W(x)\|_F + \err_k$, from lemma~\ref{lem:linear}. Note that the second term in the above expression clearly scale as $O(\sigma^2)$.

Hence,
\begin{align*}
\|\hatD_{k+1} - D_{k+1}\|_1 \leq C_2 h_{k+1}C_{\delta_k} \sigma_{k+1} \|f^k_{W} \|_F +O(\sigma^2)
\end{align*}


\begin{align*}
\|\err_{k+1} \| &= \|f^{k+1}_{W+\vecu} -\tilde{f}^{k+1}_{W+\vecu}\|_F \\ &= \| \hatD_{k+1}(W+\vecu)_{k+1} \Pi_{i=1}^{k+1} \hatD_i (W+\vecu)_i x - D_{k+1}(W+\vecu)_{k+1} \Pi_{i=1}^{k+1} D_i (W+\vecu)_i x\|_F \\
&\leq \| (\hatD_{k+1}-D_{k+1})(W+\vecu)_{k+1} \Pi_{i=1}^{k+1} D_i (W+\vecu)_i x\|_F + \| \hatD_{k+1} (W+\vecu)_{k+1} \err_k\|_F \\
&\leq \| (\hatD_{k+1}-D_{k+1})(W+\vecu)_{k+1} \Pi_{i=1}^{k+1} D_i (W+\vecu)_i x\|_F + \| (\hatD_{k+1}-D_{k+1}) (W+\vecu)_{k+1} \err_k\|_F \\ & \quad \quad +\|D_{k+1} (W+\vecu)_{k+1} \err_k\|_F 
\end{align*}

Substituting the bounds for $\hatD_{k+1} -D_{k+1}$ and $\err_k$ gives us, with probability $\geq 1-\sum_{i=1}^{k+1} \delta_i$.
\begin{align*}
\mathE\|\err_{k+1} \| &\leq  \sqrt{C_2 h_{k+1}C_{\delta_k} \sigma_{k+1} \|f^k_{W} \|_F } C_{\delta_k} \sigma_{k+1} \|f^k_{W} \|_F \mathE\| \Pi_{i=1}^{k+1} D_i (W+\vecu)_i x\|_F \\
&\quad \quad + \mathE\|\err_k\|_F \left( \sqrt{C_2 h_{k+1}C_{\delta_k} \sigma_{k+1} \|f^k_{W} \|_F } C_{\delta_k} \sigma_{k+1} \|f^k_{W} \|_F + \|D_{k+1} W_{k+1}\|_F +\sigma_{k+1}\sqrt{h_{k+1}} \right)
\end{align*}
Now we bound the above terms following the same approach as in proof of Lemma~\ref{lem:linear}, by considering all possible replacements of $W_i$ with $U_i$. That gives us the result.


%

\end{proof}

%

\chapter{Empirical Investigation} \label{chap:empirical}

In this chapter we investigate the ability of the discussed measures
to explain the different generalization phenomenon.

\section{Complexity Measures}\label{sec:complexity-measures}
Capacity control in terms of norm, when using a zero/one loss
(i.e.~counting errors) requires us in addition to account for scaling
of the output of the neural networks, as the loss is insensitive to
this scaling but the norm only makes sense in the context of such
scaling.  For example, dividing all the weights by the same number will
scale down the output of the network but does not change the $0/1$
loss, and hence it is possible to get a network with arbitrary small
norm and the same $0/1$ loss.  Using a scale sensitive losses, such as the
cross entropy loss, does address this issue (if the outputs are scaled
down toward zero, the loss becomes trivially bad), and one can obtain
generalization guarantees in terms of norm and the cross entropy loss.

However, we should be careful when comparing the norms of different
models learned by minimizing the cross entropy loss, in particular
when the training error goes to zero.  When the training error goes to
zero, in order to push the cross entropy loss (or any other positive
loss that diminish at infinity) to zero, the outputs of the network
must go to infinity, and thus the norm of the weights (under any
norm) should also go to infinity.  This means that minimizing the
cross entropy loss will drive the norm toward infinity.  In practice,
the search is terminated at some finite time, resulting in large, but
finite norm.  But the value of this norm is mostly an indication of
how far the optimization is allowed to progress---using a stricter
stopping criteria (or higher allowed number of iterations) would yield
higher norm.  In particular, comparing the norms of models found using
different optimization approaches is meaningless, as they would all go
toward infinity.

Instead, to meaningfully compare norms of the network, we should
explicitly take into account the scaling of the outputs of the
network. One way this can be done, when the training error is indeed
zero, is to consider the ``margin'' of the predictions in addition to
the norms of the parameters.  We refer to the margin for a single data
point $x$ as the difference between the score of the correct label and
the maximum score of other labels, i.e.
\begin{equation}
  \label{eq:margin}
 f_\vecw(\vecx)[y_\text{true}] -
\max_{y\neq y_\text{true}} f_\vecw(\vecx)[y] 
\end{equation}
In order to measure scale over an entire training set, one simple
approach is to consider the ``hard margin'', which is the minimum
margin among all training points.  However, this definition is very
sensitive to extreme points as well as to the size of the training
set.  We consider instead a more robust notion that allows a small
portion of data points to violate the margin. For a given training set
and small value $\epsilon>0$, we define the margin $\gamma_\margin$ as
the lowest value of $\gamma$ such that $\lceil \epsilon m \rceil$ data
point have margin lower than $\gamma$ where $m$ is the size of the
training set.  We found empirically that the qualitative and relative
nature of our empirical results is almost unaffected by reasonable
choices of $\epsilon$ (e.g.~between $0.001$ and $0.1$).

The norm-based measures we investigate in this work and their
corresponding capacity bounds are as follows \footnote{We have dropped the term that only depend on the norm of the input. The bounds based on $\ell_2$-path norm and spectral norm can be derived directly from the those based on $\ell_1$-path norm and $\ell_2$ norm respectively. Without further conditions on weights, exponential dependence on depth is tight but the $4^d$ dependence might be loose~\cite{NeyTomSre15}. We will also discuss a rather loose bound on the capacity based on the spectral norm in Section \ref{subsec:lipschitz}.}:

\begin{itemize}
\item $\ell_2$ norm with capacity proportional to $\frac{1}{\gamma_{\margin}^2}\prod_{i=1}^d 4\norm{W_i}^2_F$~\cite{NeyTomSre15}.
\item $\ell_1$-path norm with capacity proportional to $\frac{1}{\gamma_{\margin}^2}\left(\sum_{j \in \prod_{k=0}^d[h_k]}\abs{\prod_{i=1}^d 2W_i[j_i,j_{i-1}]}\right)^2$\cite{bartlett2002rademacher,NeyTomSre15}.
\item $\ell_2$-path norm with capacity proportional to $\frac{1}{\gamma_{\margin}^2}\sum_{j \in \prod_{k=0}^d[h_k]}\prod_{i=1}^d 4h_iW_i^2[j_i,j_{i-1}]$.
\item spectral norm with capacity proportional to $\frac{1}{\gamma_{\margin}^2}\prod_{i=1}^d h_i\norm{W_i}^2_2$.
\end{itemize}
where $\prod_{k=0}^d[h_k]$ is the Cartesian product over sets $[h_k]$. The above bounds indicate that capacity can be bounded in terms of either $\ell_2$-norm or $\ell_1$-path norm independent of number of parameters. The $\ell_2$-path norm dependence on the number of hidden units in each layer is unavoidable. However, it is not clear that the dependence on the number of parameters is needed for the bound based on the spectral norm. 

\paragraph{PAC-Bayes Bound}A simple way to instantiate the PAC-Based bound discussed in Section~\ref{sec:pac-bayes-general}
is to set $P$ to be a zero mean, $\sigma^2$ variance Gaussian distribution.
Choosing the perturbation $\eps$ to also be a zero mean spherical
Gaussian with variance $\sigma^2$ in every direction, yields the
following guarantee (w.p.~$1-\delta$ over the training set):
\begin{small}
\begin{equation}\label{eq:pacbayes2}
\Ep{\eps \sim \mathcal{N}(0,\sigma)^n}{L(f_{\vecw+\eps})} \leq \hatL(f_\vecw) + \underbrace{\Ep{\eps \sim \mathcal{N}(0,\sigma)^n}{\hatL(f_{\vecw+\eps})} -\hatL(f_\vecw) }_{\text{expected sharpness}}+ 4\sqrt{\frac{1}{m} \bigg( \underbrace{\frac{\|\vecw\|_2^2}{2\sigma^2} }_{\text{KL}}+ \ln \frac{2m}{\delta} \bigg) },
\end{equation}
\end{small}
Another interesting approach is to set the variance of the perturbation to each parameter with respect to the magnitude of the parameter. For example if $\sigma_i=\alpha \abs{w_i}+\beta$, then the KL term in the above expression changes to $\sum_i\frac{w_i^2}{2\sigma_i^2} $. 

The above generalization guarantees give a clear way to think about capacity control jointly in terms of both the expected sharpness and the norm, and as we discussed earlier indicates that sharpness by itself cannot control the capacity without considering the scaling. In the above generalization bound, norms and sharpness interact in a direct way depending on $\sigma$, as increasing the norm by decreasing $\sigma$ causes decrease in sharpness and vice versa. It is therefore important to find the right balance between the norm and sharpness by choosing $\sigma$ appropriately in order to get a reasonable bound on the capacity.

\section{Experiments Settings}
In experiment with different network sizes, we train a two layer perceptron with ReLU activation and varying number of hidden units without Batch Normalization or dropout. In the rest of the experiments, we train a modified version of the VGG architecture \cite{simonyan2014very} with the configuration $2\times [64,3,3,1]$, $2\times [128,3,3,1]$, $2\times [256,3,3,1]$, $2\times [512,3,3,1]$ where we add Batch Normalization before ReLU activations and apply $2\times 2$ max-pooling with window size 2 and dropout after each stack. Convolutional layers are followed by $4\times 4$ average pooling, a fully connected layer with 512 hidden units and finally a linear layer is added for prediction. 

In all experiments we train the networks using stochastic gradient descent (SGD) with mini-batch size 64, fixed learning rate 0.01 and momentum 0.9 without weight decay. In all experiments where achieving zero training error is possible, we continue training until the cross-entropy loss is less than $10^{-4}$.

When calculating norms on a network with a Batch Normalization layer, we reparametrize the network to one that represents the exact same function without Batch Normalization as suggested in \cite{neyshabur16}. In all our figures we plot norm divided by margin to avoid scaling issues (see Section~\ref{sec:complexity-measures}), where we set the margin over training set $S$ to be $5^{th}$-percentile of the margins of the data points in $S$, i.e. $\text{Prc}_5\left\{f_\vecw(x_i)[y_i] - \max_{y\neq y_i} f_\vecw(x)[y] | (x_i,y_i)\in S\right\}.$ We have also investigated other versions of the margin and observed similar behavior to this notion. 

We calculate the sharpness, as suggested in \cite{keskar2016large} - for each parameter $w_i$ we bound the magnitude of perturbation by $\alpha(\abs{w_i}+1)$ for $\alpha=5.10^{-4}$. In order to compute the maximum perturbation (maximize the loss), we perform 2000 updates of stochastic gradient ascent starting from the minimum, with mini-batch size 64, fixed step size 0.01 and momentum 0.9.

To compute the expected sharpness, we perturb each parameter $w_i$ of the model with noise generated from Gaussian distribution with zero mean and standard deviation, $\alpha(10\abs{w_i}+1)$. The expected sharpness is average over 1000 random perturbations each of which are averaged over a mini-batch of size 64. We compute the expected sharpness for different choices of $\alpha$. For each value of $\alpha$ the KL divergence can be calculated as $\frac{1}{\alpha^2}\sum_i \left(\frac{w_i}{(10\abs{w_i}+1)}\right)^2$.

\section{True Labels Vs. Random Labels}
As an initial empirical investigation of the appropriateness of the
different complexity measures, we compared the complexity (under each
of the above measures) of models trained on true versus random labels.
We would expect to see two phenomena: first, the complexity of models
trained on true labels should be substantially lower than those
trained on random labels, corresponding to their better generalization
ability.  Second, when training on random labels, we expect capacity
to increase almost linearly with the number of training examples, since
every extra example requires new capacity in order to fit it's random
label.  However, when training on true labels we expect the model to
capture the true functional dependence between input and output and
thus fitting more training examples should only require small
increases in the capacity of the network.  The results are reported in
Figure \ref{fig:norm-true-random}.  We indeed observe a gap between
the complexity of models learned on real and random labels for all
four norms, with the difference in increase in capacity between true
and random labels being most pronounced for the $\ell_2$ norm and
$\ell_2$-path norm.

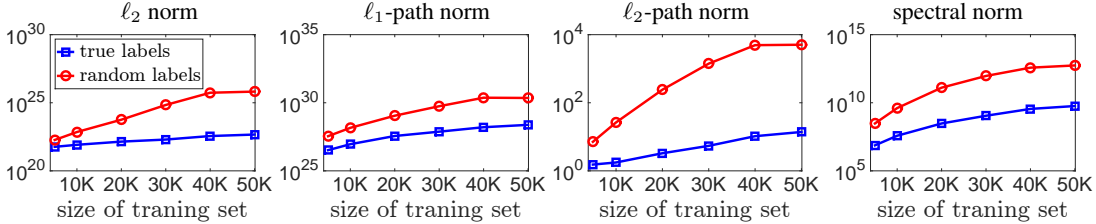
\begin{figure}[t]
\centering
\includegraphics[width=.245\textwidth]{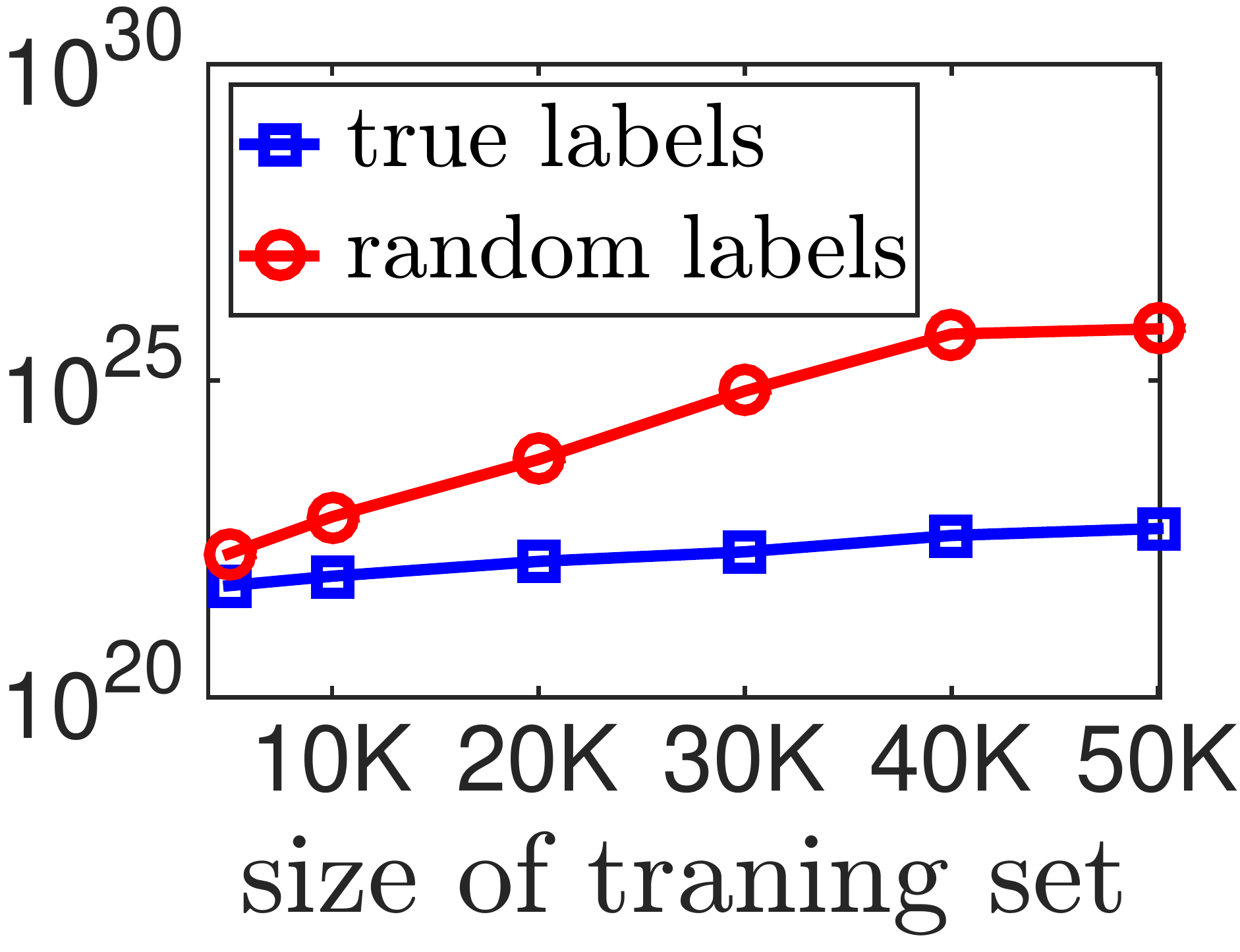}
\includegraphics[width=.245\textwidth]{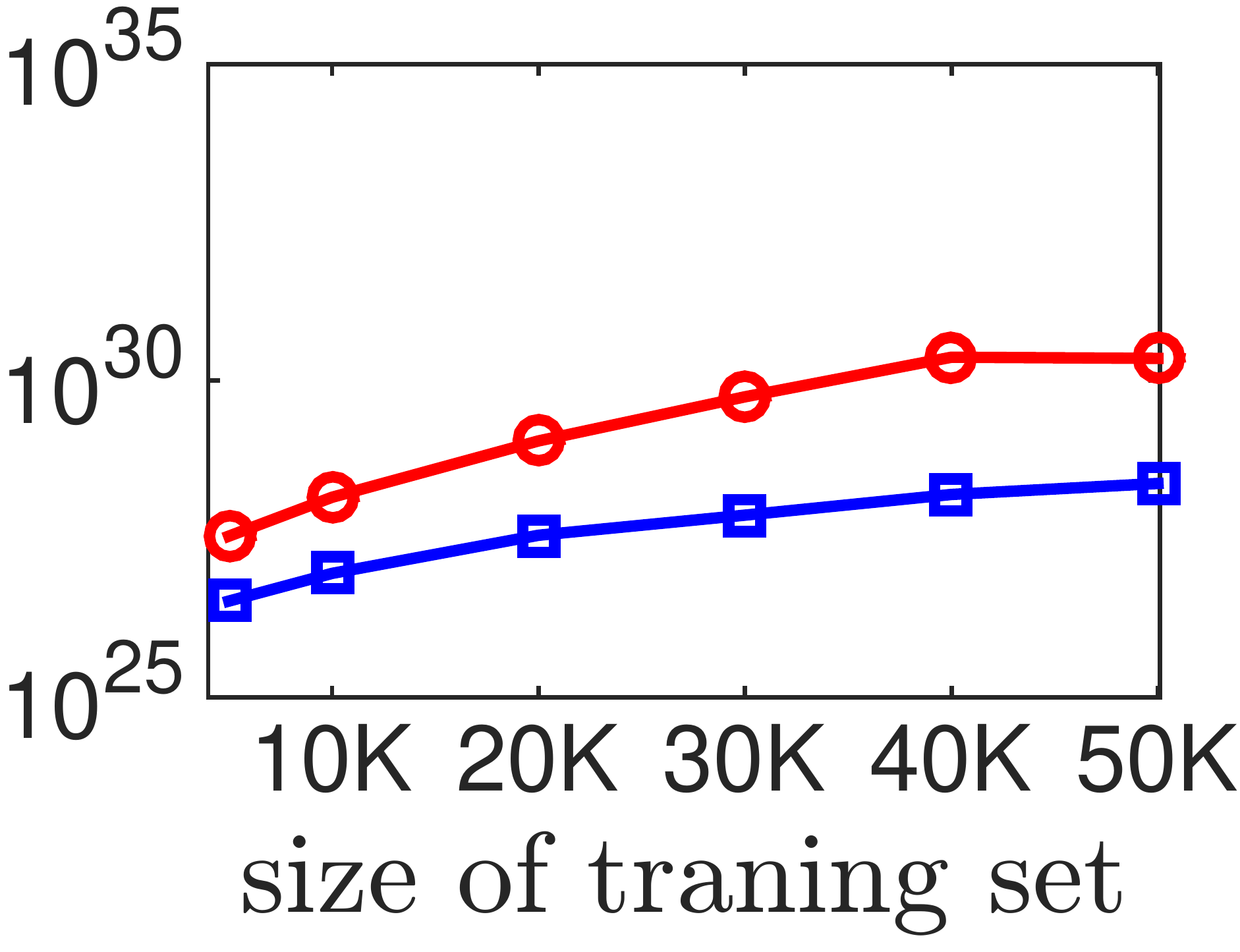}
\includegraphics[width=.245\textwidth]{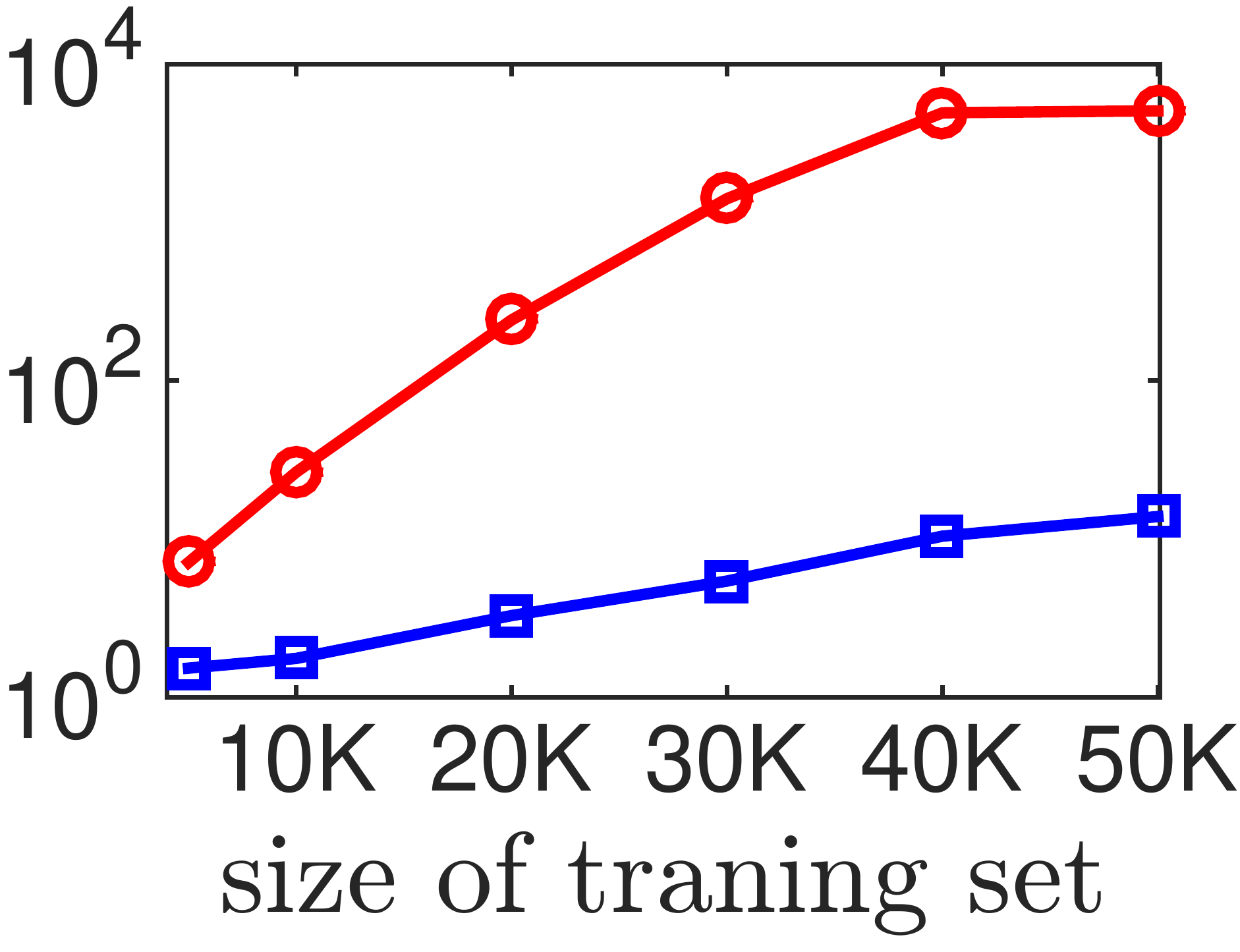}
\includegraphics[width=.245\textwidth]{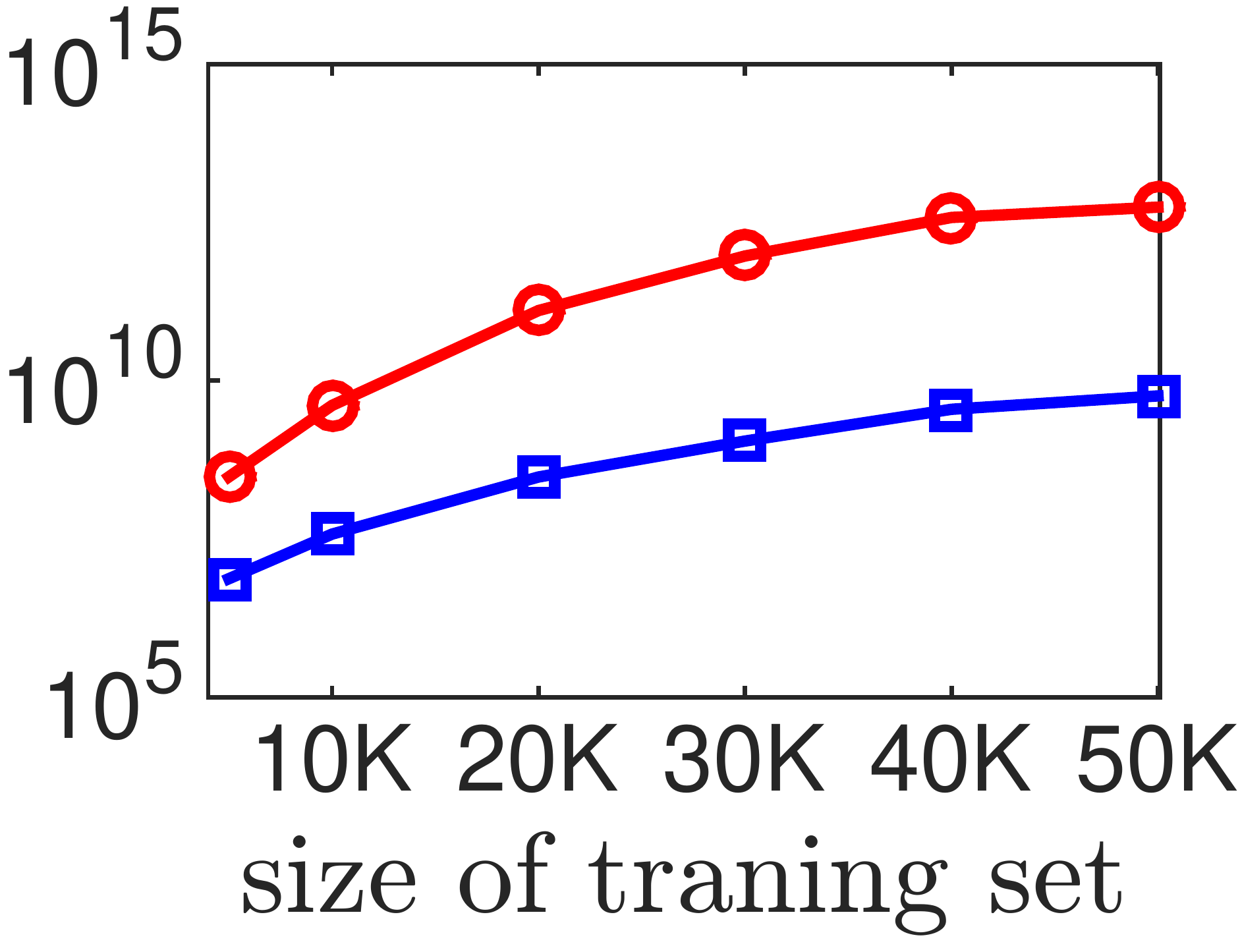}
\begin{picture}(0,0)(0,0)
{\small \put(-162, 90){$\ell_2$ norm}\put(-70, 90){$\ell_1$-path norm}\put(28, 90){$\ell_2$-path norm}\put(130, 90){spectral norm}}
\end{picture}
\caption[\small Comparing complexity measures on a VGG network trained with true or random labels.]{\small Comparing different complexity measures on a VGG network trained on subsets of CIFAR10 dataset with true (blue line) or random (red line) labels. We plot norm divided by margin to avoid scaling issues (see Section~\ref{sec:complexity-measures}), where for each complexity measure, we drop the terms that only depend on depth or number of hidden units; e.g. for $\ell_2$-path norm we plot $\gamma_{\margin}^{-2}\sum_{j \in \prod_{k=0}^d[h_k]}\prod_{i=1}^d W_i^2[j_i,j_{i-1}]$.We also set the margin over training set $S$ to be $5^{th}$-percentile of the margins of the data points in $S$, i.e. $\text{Prc}_5\left\{f_\vecw(x_i)[y_i] - \max_{y\neq y_i} f_\vecw(\vecx)[y] | (x_i,y_i)\in S\right\}$. In all experiments, the training error of the learned network is zero. The plots indicate that these measures can explain the generalization as the complexity of model learned with random labels is always higher than the one learned with true labels. Furthermore, the gap between the complexity of models learned with true and random labels increases as we increase the size of the training set.}
\label{fig:norm-true-random}
\end{figure}

In our experiments on PAC-Bayes bound, we observe that looking at both sharpness and norm in Equation~\ref{eq:margin}
jointly indeed makes a better predictor for the generalization error.
As discussed earlier, \citet{dziugaite2017computing} numerically
optimize the overall PAC-Bayes generalization bound over a family of
multivariate Gaussian distributions (different choices of
perturbations and priors). Since the precise way the sharpness and
KL-divergence are combined is not tight, certainly not in
\eqref{eq:pacbayes2}, nor in the more refined bound used by
\citet{dziugaite2017computing}, we prefer shying away from numerically
optimizing the balance between sharpness and the KL-divergence.
Instead, we propose using bi-criteria plots, where sharpness and
KL-divergence are plotted against each other, as we vary the
perturbation variance.  For example, in the center and right panels of
Figure \ref{fig:sharpness-true-random} we show such plots for networks
trained on true and random labels respectively.  We see that although
sharpness by itself is not sufficient for explaining generalization in
this setting (as we saw in the left panel), the bi-criteria plots are
significantly lower for the true labels.  Even more so, the change in
the bi-criteria plot as we increase the number of samples is
significantly larger with random labels, correctly capturing the
required increase in capacity.  For example, to get a fixed value of
expected sharpness such as $\error=0.05$, networks trained with random
labels require higher norm compared to those trained with true labels.
This behavior is in agreement with our earlier discussion, that
sharpness is sensitive to scaling of the parameters and is not a
capacity control measure as it can be artificially changed by scaling
the network. However, combined with the norm, sharpness does seem to
provide a capacity measure.

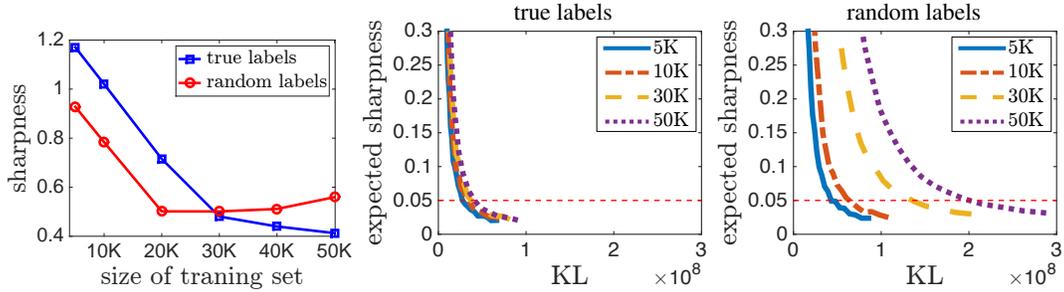
\begin{figure}[t]
\centering
\includegraphics[width=.32\textwidth]{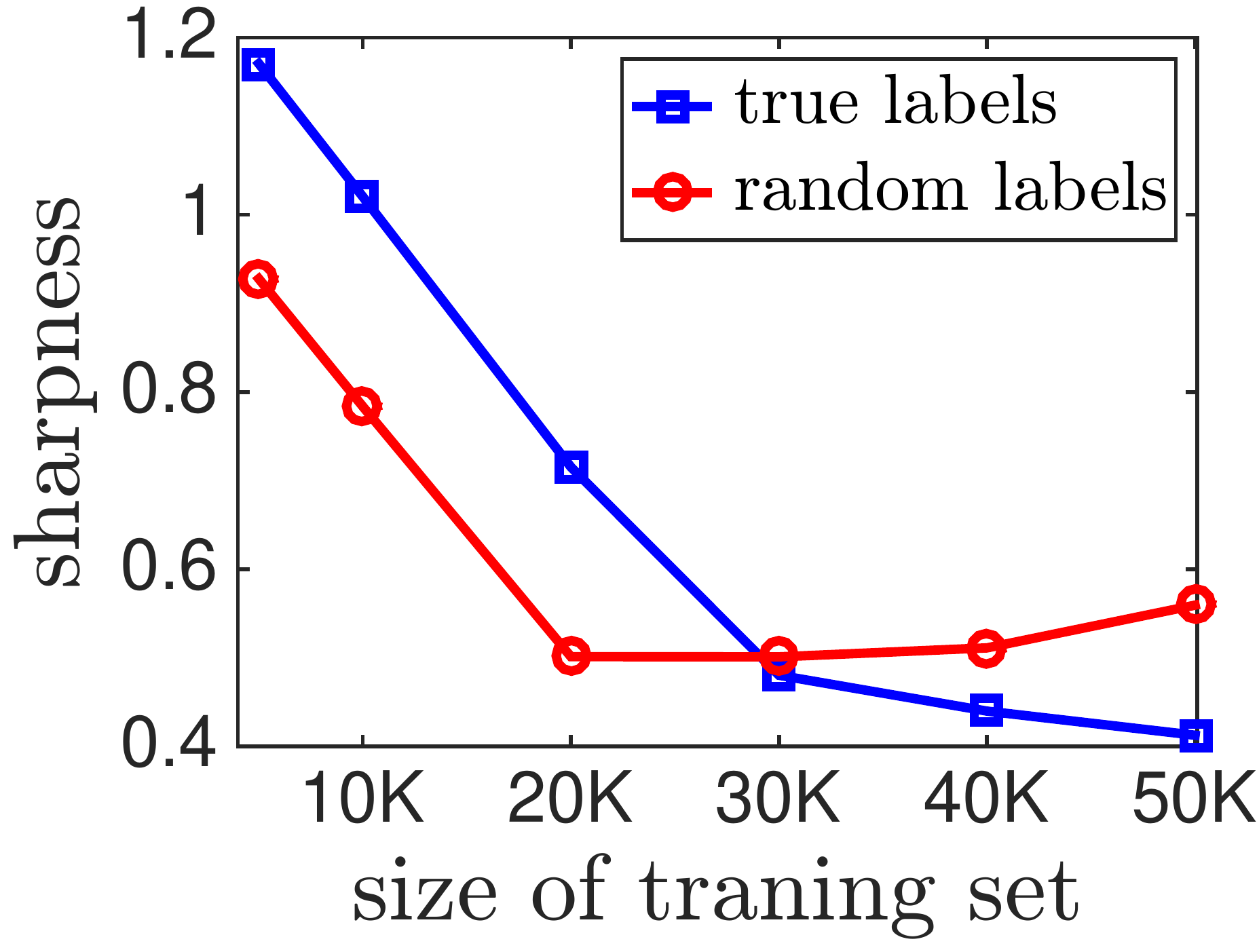}
\includegraphics[width=.32\textwidth]{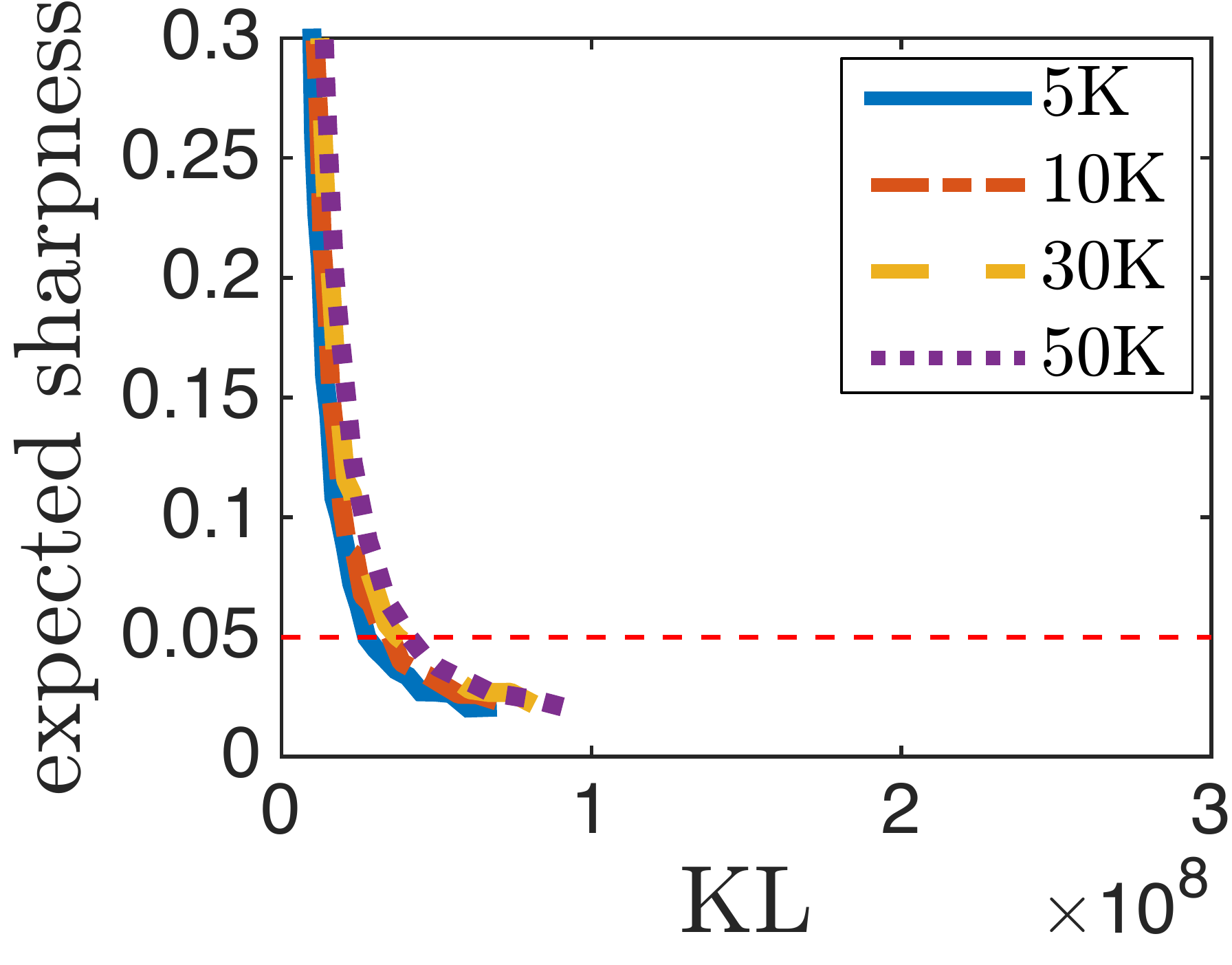}
\includegraphics[width=.32\textwidth]{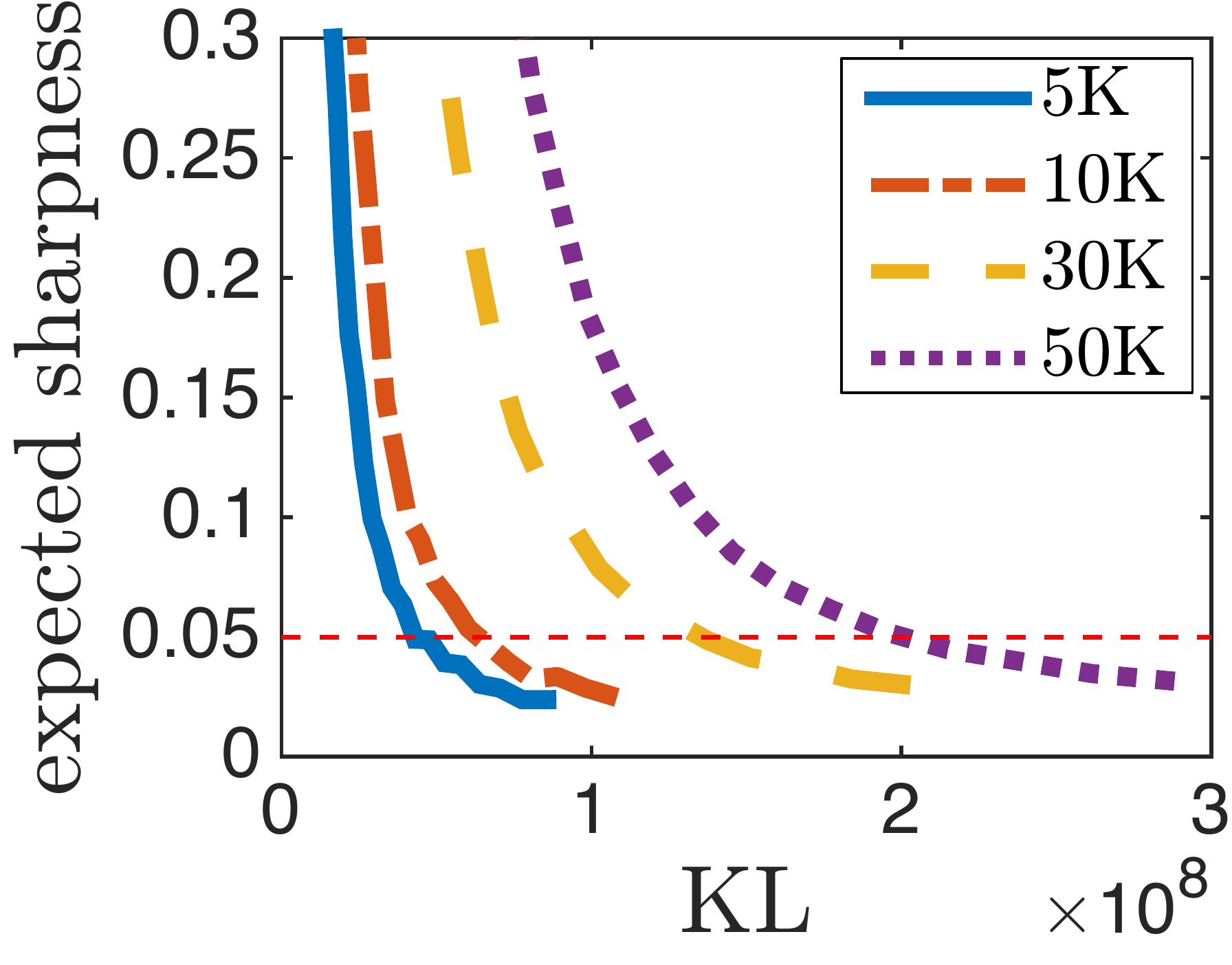}
\caption[ \small Sharpness and PAC-Bayes measures on a VGG network trained with true or random labels.]{\small Sharpness and PAC-Bayes measures on a VGG network
  trained on subsets of CIFAR10 dataset with true or random labels. In
  the left panel, we plot max sharpness, which we calculate as
  suggested by \citet{keskar2016large} where the perturbation for
  parameter $w_i$ has magnitude $5.10^{-4}(\abs{w_i}+1)$. The middle
  and right plots demonstrate the relationship between expected
  sharpness and KL divergence in PAC-Bayes analysis for true and
  random labels respectively. For PAC-Bayes plots, each point in the
  plot correspond to a choice of variable $\alpha$ where the standard
  deviation of the perturbation for the parameter $i$ is
  $\alpha(10\abs{w_i}+1)$. The corresponding $KL$ to each $\alpha$ is
  nothing but weighted $\ell_2$ norm where the weight for each
  parameter is the inverse of the standard deviation of the
  perturbation.}
\begin{picture}(0,0)(0,0)
{\small \put(-6, 200){true labels}\put(120, 200){random labels}}
\end{picture}
\label{fig:sharpness-true-random}
\end{figure}

\section{Different Global Minima}

Given different global minima of the training loss on the same
training set and with the same model class, can these measures
indicate which model is going to generalize better? In order to verify
this property, we can calculate each measure on several different
global minima and see if lower values of the measure imply lower
generalization error. In order to find different global minima for the
training loss, we design an experiment where we force the optimization
methods to converge to different global minima with varying
generalization abilities by forming a confusion set that includes
samples with random labels. The optimization is done on the loss that
includes examples from both the confusion set and the training set.
Since deep learning models have very high capacity, the optimization
over the union of confusion set and training set generally leads to a
point with zero error over both confusion and training sets which thus
is a global minima for the training set.

We randomly select a subset of CIFAR10 dataset with 10000 data points
as the training set and our goal is to find networks that have zero
error on this set but different generalization abilities on the test
set. In order to do that, we train networks on the union of the
training set with fixed size 10000 and confusion sets with varying
sizes that consists of CIFAR10 samples with random labels; and we
evaluate the learned model on an independent test set. The trained
network achieves zero training error but as shown in Figure
\ref{fig:cifar-core}, the test error of the model increases with
increasing size of the confusion set. The middle panel of this Figure
suggests that the norm of the learned networks can indeed be
predictive of their generalization behavior. However, we again observe
that sharpness has a poor behavior in these experiments. The right
panel of this figure also suggests that PAC-Bayes measure of joint
sharpness and KL divergence, has better behavior - for a fixed
expected sharpness, networks that have higher generalization error,
have higher norms.

\begin{figure}[t]
\centering
\includegraphics[width=.32\textwidth]{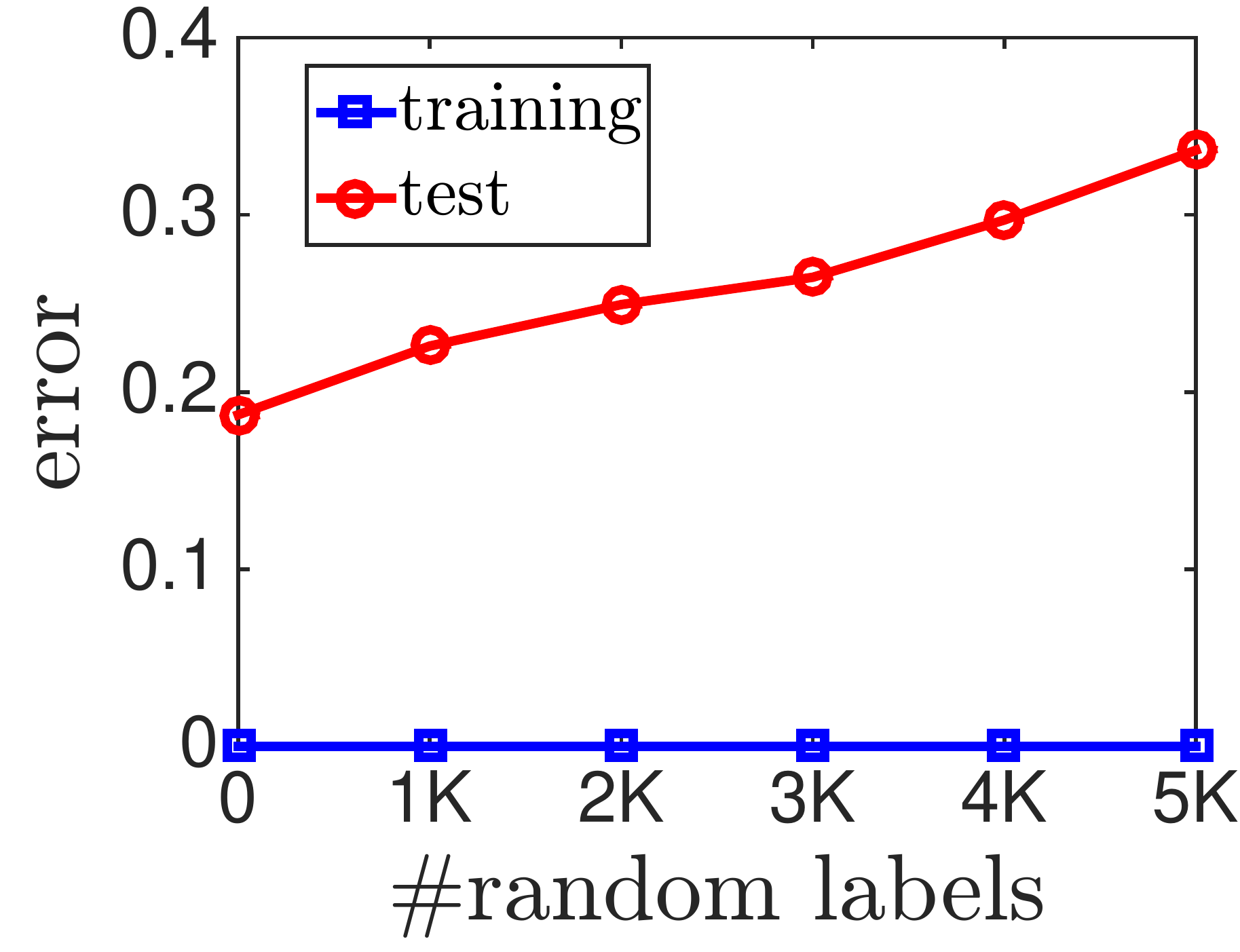}
\includegraphics[width=.32\textwidth]{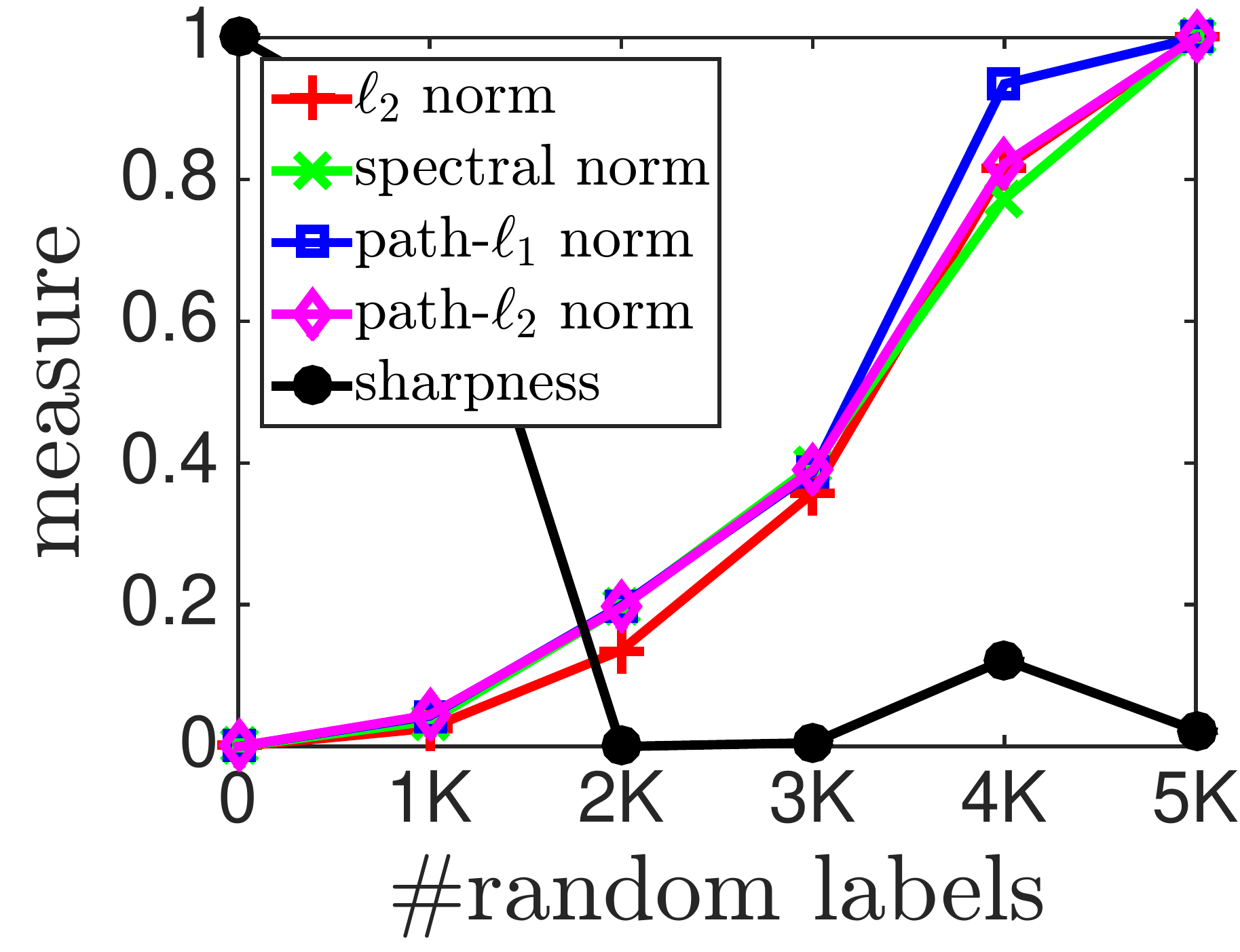}
\includegraphics[width=.32\textwidth]{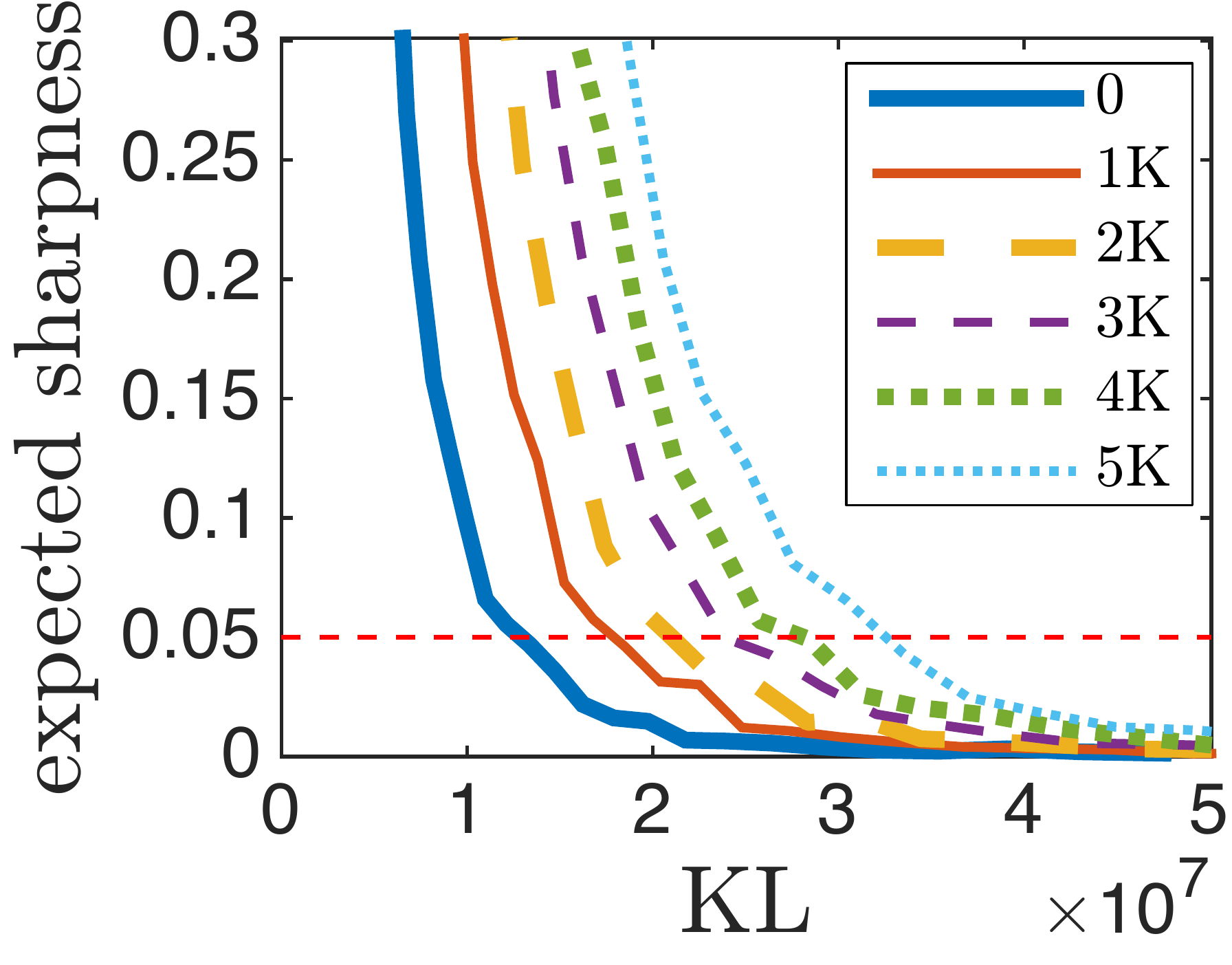}
\caption[\small Experiments on global minima with poor generalization.]{\small Experiments on global minima with poor generalization. For each experiment, a VGG network is trained on union of a subset of CIFAR10 dataset with size 10000 containing samples with true labels and another subset of CIFAR10 datasets with varying size containing random labels. The learned networks are all global minima for the objective function on the subset with true labels. The left plot indicates the training and test errors based on the size of the set with random labels. The plot in the middle shows change in different measures based on the size of the set with random labels. The plot on the right indicates the relationship between expected sharpness and KL in PAC-bayes for each of the experiments. Measures are calculated as explained in Figures \ref{fig:norm-true-random} and \ref{fig:sharpness-true-random}.}
\label{fig:cifar-core}
\end{figure}

\section{Increasing Network Size}

We also repeat the experiments conducted by \citet{neyshabur15b} where
a fully connected feedforward network is trained on MNIST dataset with
varying number of hidden units and we check the values of different
complexity measures on each of the learned networks.The left panel in
Figure \ref{fig:hidden} shows the training and test error for this
experiment. While 32 hidden units are enough to fit the training data,
we observe that networks with more hidden units generalize better.
Since the optimization is done without any explicit regularization,
the only possible explanation for this phenomenon is the implicit
regularization by the optimization algorithm. Therefore, we expect a
sensible complexity measure to decrease beyond 32 hidden units and
behave similar to the test error. Different measures are reported for
learned networks. The middle panel suggest that all margin/norm based
complexity measures decrease for larger networks up to 128 hidden
units. For networks with more hidden units, $\ell_2$ norm and
$\ell_1$-path norm increase with the size of the network. The middle
panel suggest that $\ell_2$-path norm can provide some explanation for
this phenomenon. However, as we discussed in
Section~\ref{sec:complexity-measures}, the actual complexity measure based on
$\ell_2$-path norm also depends on the number of hidden units and
taking this into account indicates that the measure based on
$\ell_2$-path norm cannot explain this phenomenon. This is also the
case for the margin based measure that depends on the spectral norm.
In subsection \ref{subsec:lipschitz} we discussed another complexity
measure that also depends the spectral norm through Lipschitz
continuity or robustness argument. Even though this bound is very
loose, it is monotonic with respect to the spectral norm that is
reported in the plots. Unfortunately, we do observe some increase in
spectral norm by increasing number of hidden units beyond 512. The
right panel shows that the joint PAC-Bayes measure decrease for larger
networks up to size 128 but fails to explain this generalization
behavior for larger networks. This suggests that the measures looked so
far are not sufficient to explain all the generalization phenomenon
observed in neural networks.

\begin{figure}[t]
\centering
\includegraphics[width=.32\textwidth]{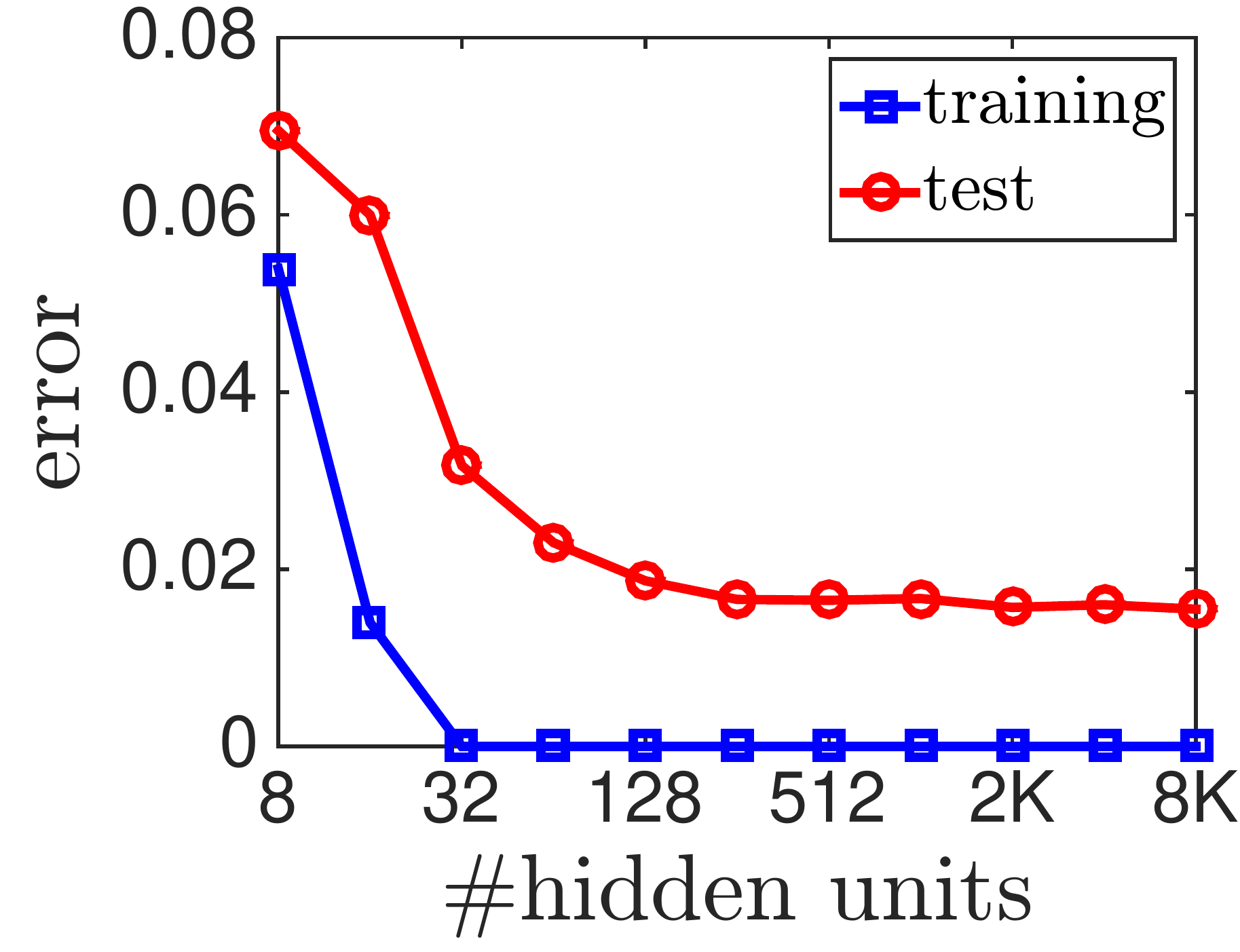}
\includegraphics[width=.32\textwidth]{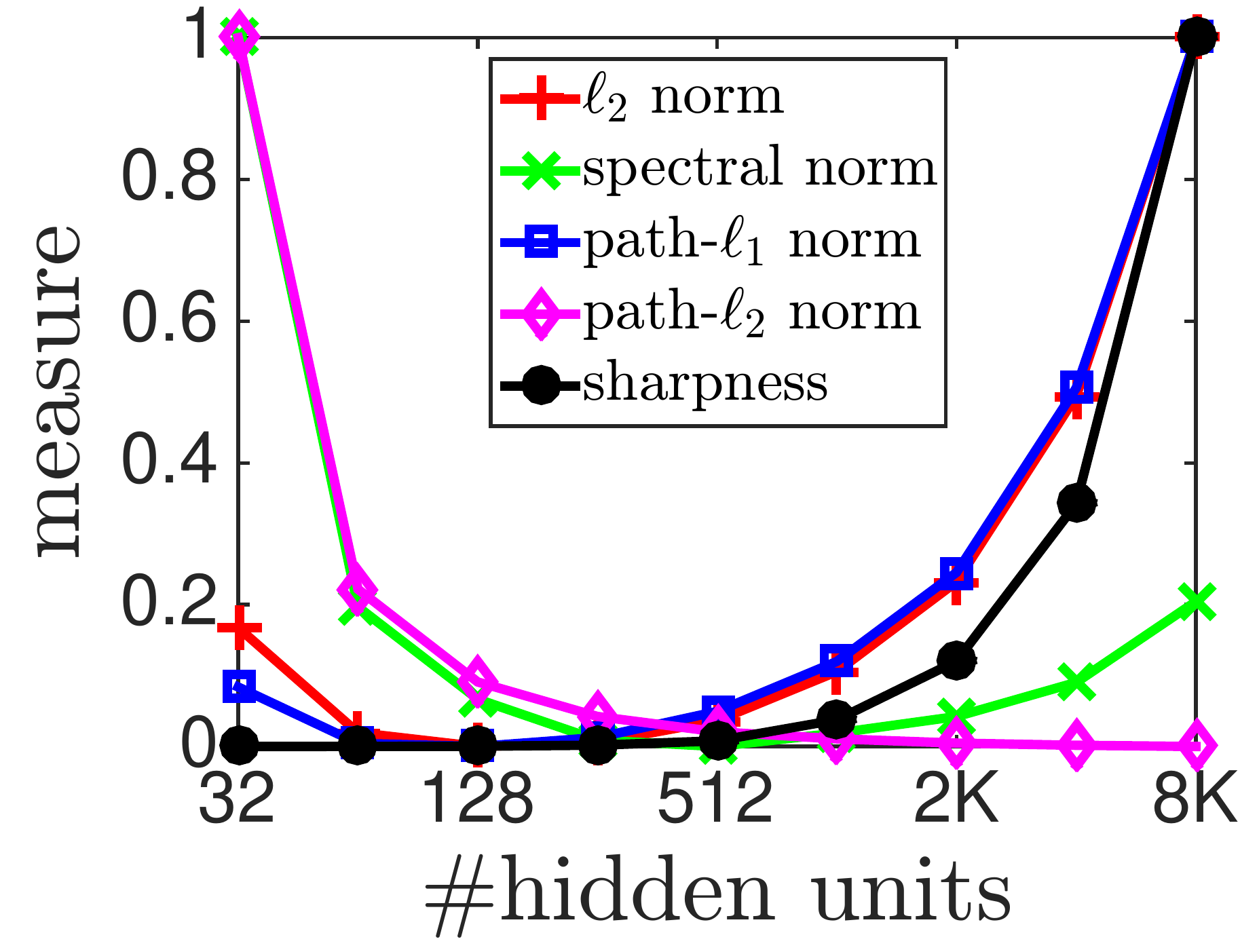}
\includegraphics[width=.32\textwidth]{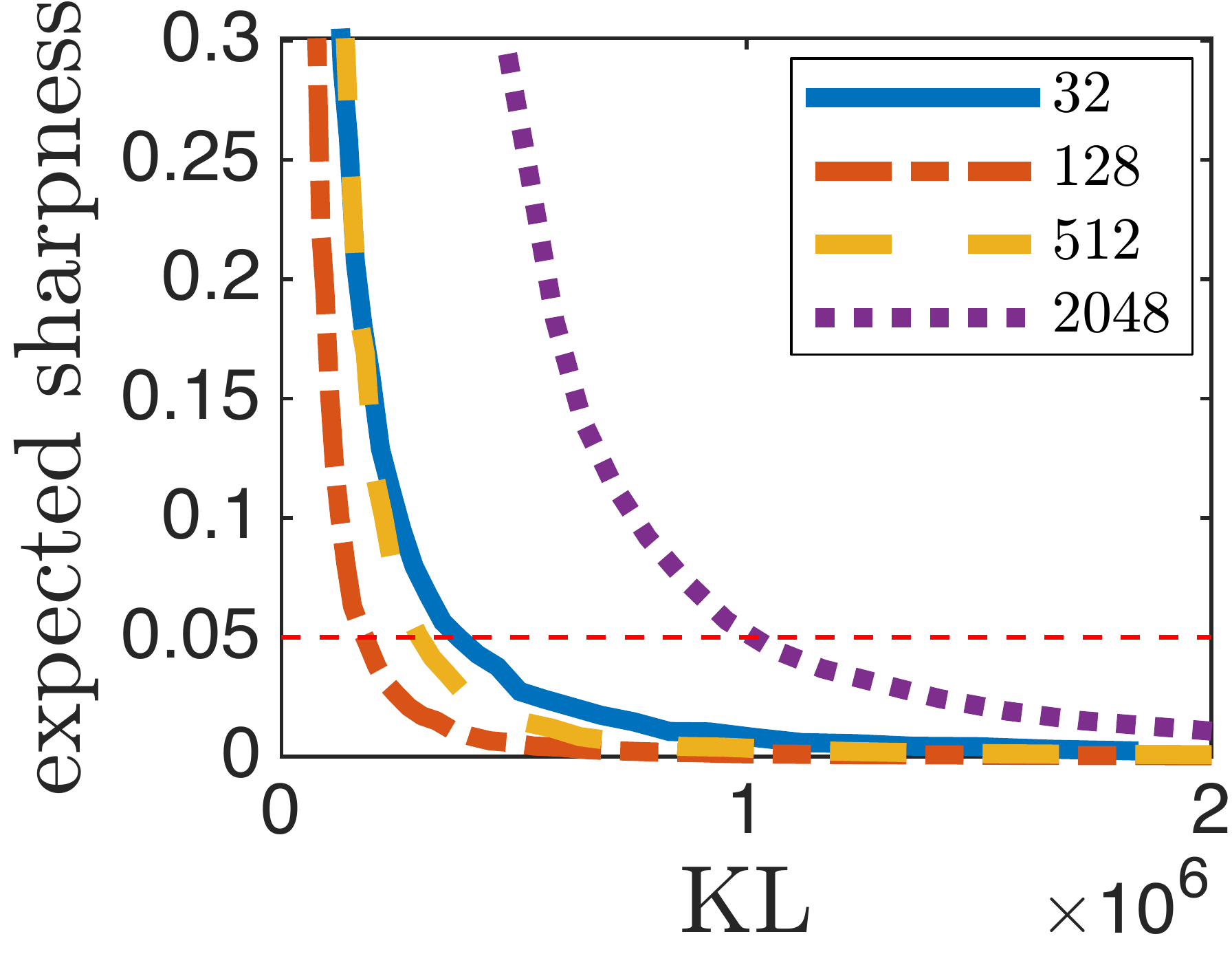}
\caption[\small The effect of increasing network size on generalization.]{\small The generalization of two layer perceptron trained on MNIST dataset with varying number of hidden units. The left plot indicates the training and test errors. The test error decreases as the size increases. The middle plot shows different measures for each of the trained networks. The plot on the right indicates the relationship between expected sharpness and KL in PAC-Bayes for each of the experiments. Measures are calculated as explained in Figures \ref{fig:norm-true-random} and \ref{fig:sharpness-true-random}.}
\label{fig:hidden}
\end{figure}

\part[\Large Geometry of Optimization and Generalization]{Geometry of Optimization and Generalization} \label{part:geometry}
\chapter{Invariances} \label{chap:invariances}

In Chapter~\ref{chap:implicit}, we discussed how optimization is related to generalization due to the implicit regularization. 
Revisiting the choice of gradient descent, we recall that optimization is also inherently tied to a choice of geometry
or measure of distance, norm or divergence. Gradient descent for example is tied to the $\ell_2$ norm as it
is the steepest descent with respect to $\ell_2$ norm in the parameter space, while coordinate descent corresponds
to steepest descent with respect to the $\ell_1$ norm and exp-gradient (multiplicative weight) updates is tied to
an entropic divergence. Moreover, at least when the objective function is convex, convergence behavior is
tied to the corresponding norms or potentials. For example, with gradient descent, or SGD, convergence
speeds depend on the $\ell_2$ norm of the optimum.  The norm or divergence can be viewed as a
regularizer for the updates.  There is therefore also a strong link
between regularization for optimization and regularization for
learning: optimization may provide implicit regularization in terms of
its corresponding geometry, and for ideal optimization performance the
optimization geometry should be aligned with inductive bias driving
the learning \cite{srebro11}.

Is the $\ell_2$ geometry on the weights the appropriate geometry for
the space of deep networks?  Or can we suggest a geometry with more
desirable properties that would enable faster optimization and perhaps
also better implicit regularization?  As suggested above, this
question is also linked to the choice of an appropriate regularizer
for deep networks.

Focusing on networks with RELU activations in this section, we observe that scaling down the incoming edges to a hidden unit and scaling up the outgoing edges by the same factor yields an equivalent network computing the same function. Since predictions are invariant to such rescalings, it is natural to seek a geometry, and corresponding optimization method, that is similarly invariant. In this chapter, we study invariances in feedforward networks with shared weights.

\section{Invariances in Feedforward and Recurrent Neural Networks}

Feedforward networks are highly
over-parameterized, i.e. there are many weight settings
${\vecw}$ that represent the same function $f_{\vecw}$.  Since our true object of
interest is the function $f$, and not the identity ${\vecw}$ of the
weghts, it would be beneficial if optimization would depend only
on $f_{\vecw}$ and not get ``distracted'' by difference in ${\vecw}$ that does not
affect $f_{\vecw}$.  It is therefore helpful to study the transformations on
the weights that will not change the function presented by the
network and come up with methods that their performance is not
affected by such transformations. 

\begin{definition}
We say a class of neural networks is \emph{invariant} to a transformation $\calT$ if for any parameter setting $\vecp$ and its corresponding weights $\vecw$, $f_{\vecw} =  f_{\calT(\vecw)}$. Similarly, we say an update rule $\calA$ is \emph{invariant} to $\calT$ if for any $\vecp$ and its corresponding $\vecw$, $f_{\calA(\vecw)} =  f_{\calA(\calT(\vecw))}$. 
\end{definition}

Invariances have also been studied as different mappings from the parameter space to the same function space~\cite{ollivier2015riemannian} while we define the transformation as a mapping inside a fixed parameter space. A very important invariance in feedforward networks is \emph{node-wise rescaling}~\cite{neyshabur2016data}. For any internal node $v$ and any scalar $\alpha>0$, we can multiply all incoming weights into $v$ (i.e. $w_{u\rightarrow v}$ for any $(u\rightarrow v)\in E$) by $\alpha$ and all the outgoing weights (i.e. $w_{v\rightarrow u}$ for any $(v\rightarrow u)\in E$) by $1/\alpha$ without changing the function computed by the network. Not all node-wise rescaling transformations can be applied in feedforward nets with shared weights. This is due to the fact that some weights are forced to be equal and therefore, we are only allowed to change them by the same scaling factor.

\begin{definition}
Given a class of neural networks, we say an invariant transformation $\calT$ that is defined over edge weights is \emph{feasible} for parameter mapping $\pi$ if the shared weights remain equal after the transformation, i.e. for any $i$ and for any $e,e'\in E_i$, $\calT(\vecw)_e =\calT(\vecw)_{e'}$. 
\end{definition}

We have discussed the complete characterize all feasible node-wise invariances of RNNs in \cite{neyshabur2016path}.

\begin{figure}[t!] \label{fig:unbalanced}
\hspace{0.5in}
\subfloat[Training on MNIST]{
  \includegraphics[width=0.23\textwidth]{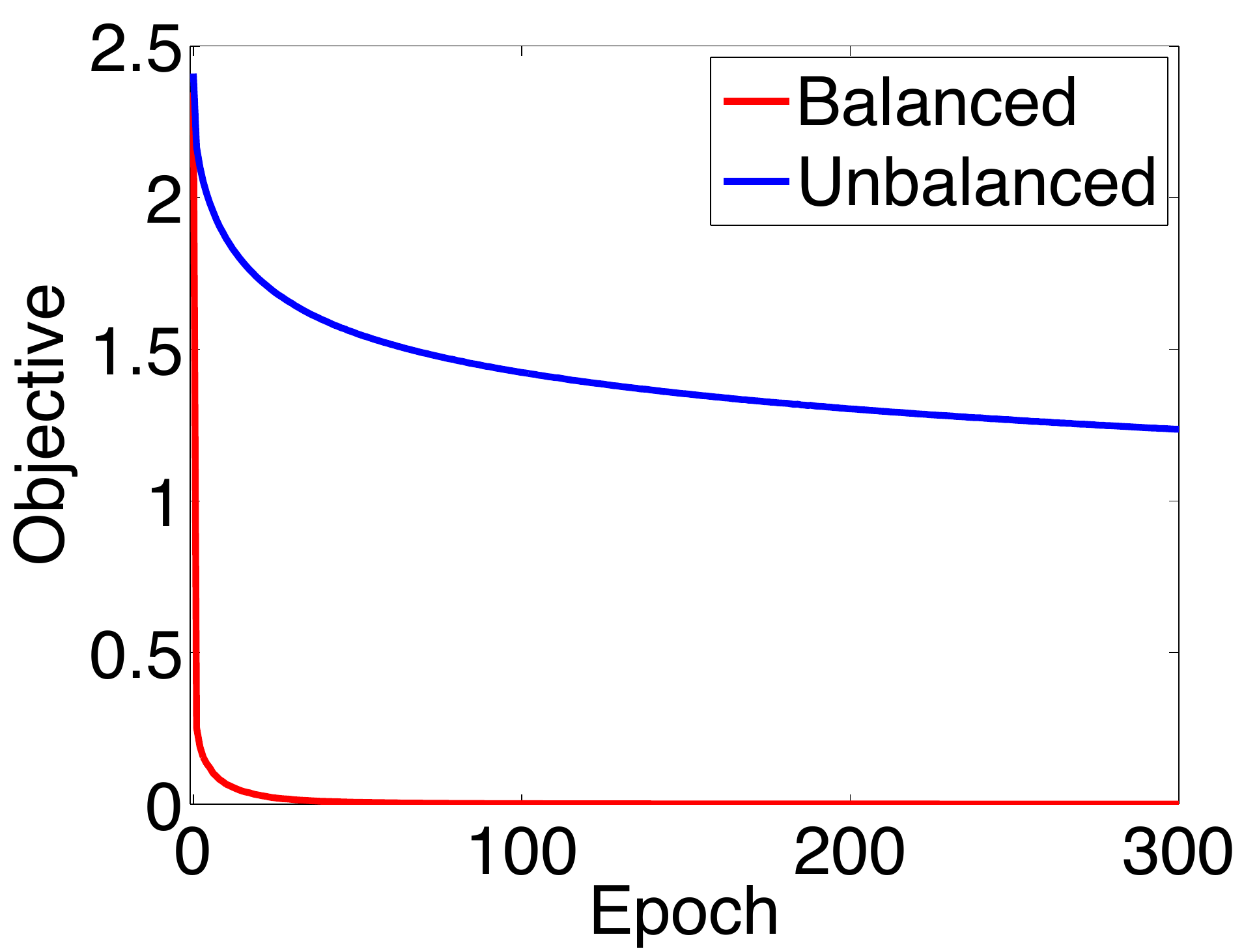}\label{fig:compare-a}
 }\hspace{1in}
 \subfloat[  \small Weight explosion in an unbalanced network]{
 \hspace{0.4in}
  \includegraphics[width=0.32\textwidth]{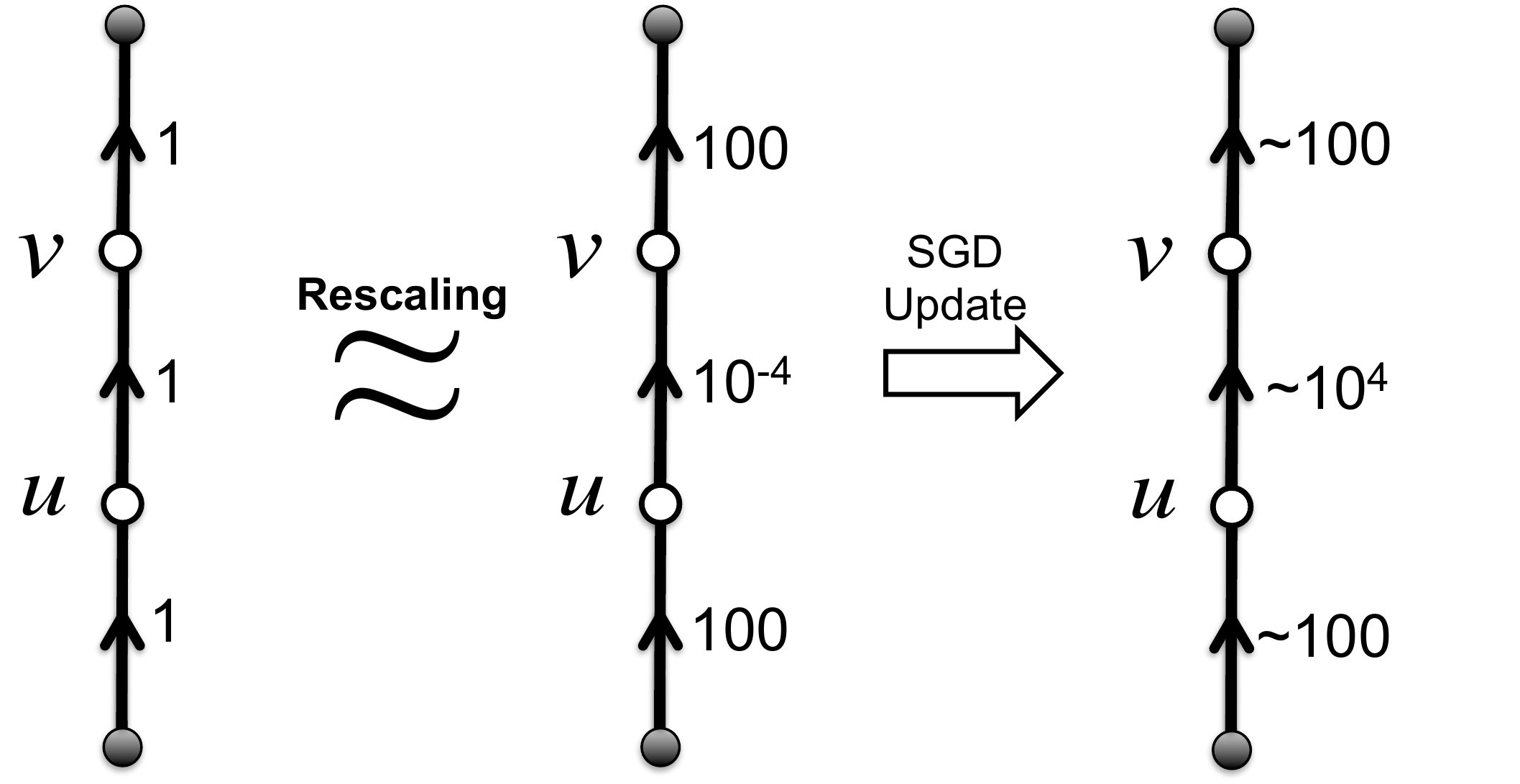}\label{fig:compare-b}
  \hspace{0.3in}
 }
 \newline
 \begin{center}
  \subfloat[ \small Poor updates in an unbalanced network]{
  \includegraphics[width=0.9\textwidth]{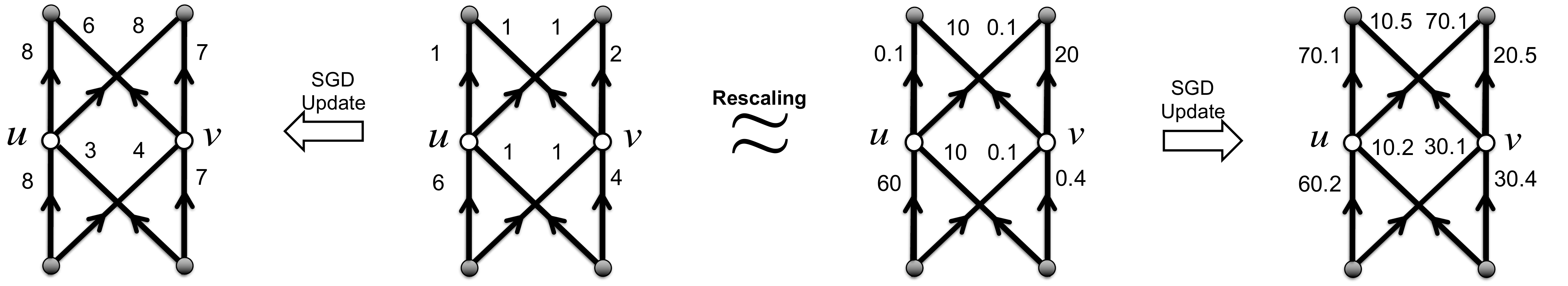}\label{fig:compare-c}
   }
   \end{center}
 \caption[\small Invariance in fully connected feedforward networks]{ \small (a): Evolution of the cross-entropy error function when training a 
feed-forward network on MNIST with two hidden layers, each containing 4000 hidden units.
The unbalanced initialization (blue curve) is generated by applying a sequence of rescaling functions on the balanced initializations (red curve). (b): Updates for a simple case where the input is $x=1$, 
thresholds are set to zero (constant), the stepsize is 1, and the gradient with respect to output is $\delta = -1$. (c): Updated network for the case where the input is $x=(1,1)$, thresholds are set to zero (constant), the stepsize is 1, and the gradient with respect to output is $\delta=(-1,-1)$. }
 \label{fig:unbalanced}
\end{figure}

Unfortunately, gradient descent is {\em not} rescaling invariant. The main problem with the gradient updates is that scaling down the weights of an edge will also scale up the gradient which, as we see later, is exactly the opposite of what is expected from a rescaling invariant update. 

Furthermore, gradient descent performs very poorly on ``unbalanced''
networks.  We say that a network is {\em balanced} if the norm of
incoming weights to different units are roughly the same or within a
small range. For example, Figure~\ref{fig:unbalanced}\subref{fig:compare-a} shows a huge
gap in the performance of SGD initialized with a randomly generated
balanced network, when training on MNIST, compared to a
network initialized with unbalanced weights. Here
the unbalanced weights are generated by applying a sequence of random
rescaling functions on the balanced weights to create a rescaling equivalent unbalanced network.

In an unbalanced network, gradient descent updates could blow up the smaller weights, while keeping the larger weights almost unchanged. This is illustrated in Figure~\ref{fig:unbalanced}\subref{fig:compare-b}. If this were the only issue, one could scale down all the weights after each update. However, in an unbalanced network,  the relative changes in the weights are also very different compared to a balanced network. For example, Figure \ref{fig:unbalanced}\subref{fig:compare-c} shows how two rescaling equivalent networks could end up computing a very different function after only a single update.

Therefore, it is helpful to understand what are the \emph{feasible} node-wise rescalings for RNNs. In the following theorem, we characterize all feasible node-wise invariances in RNNs.

\begin{SCfigure}
\vspace{-11pt}
\includegraphics[width=10cm]{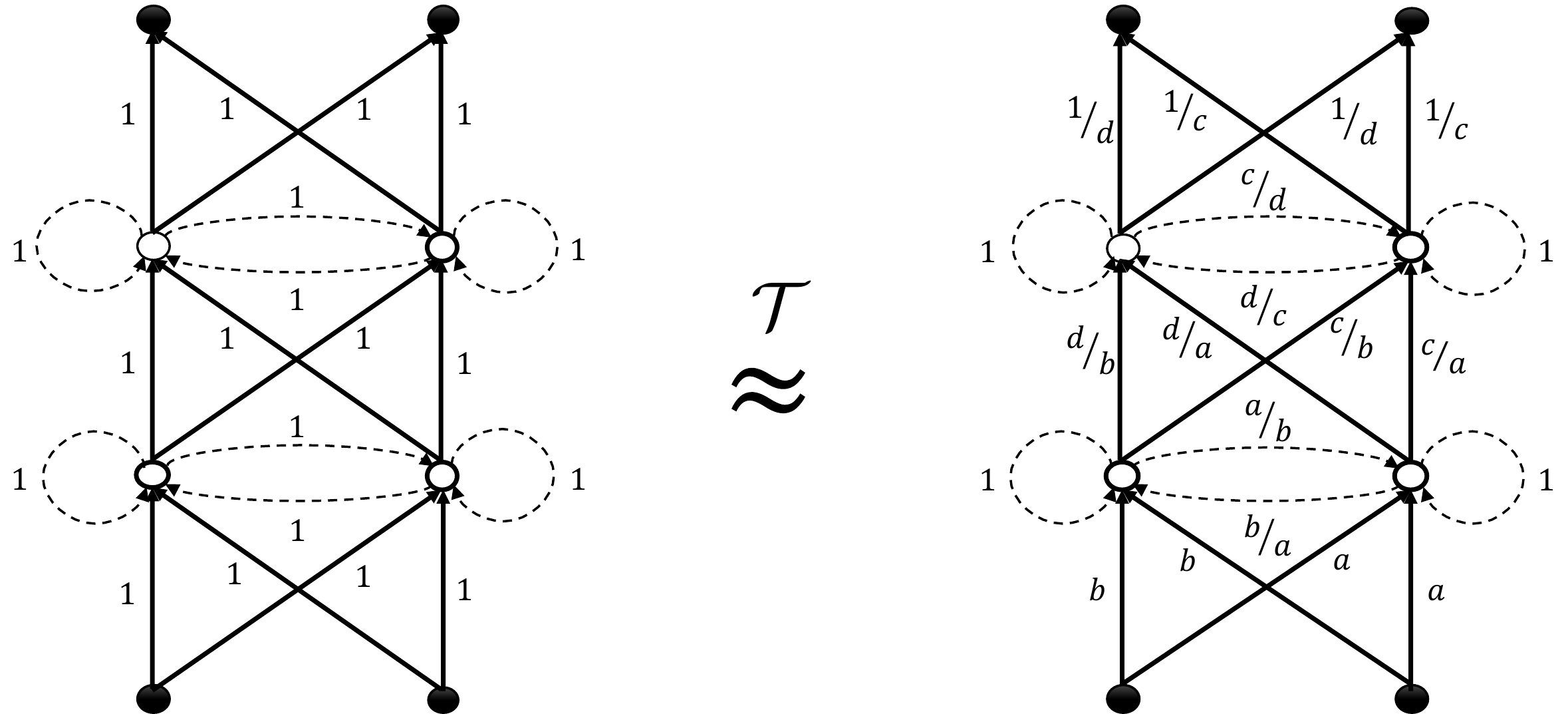}
\caption[\small Invariances in recurrent neural networks (RNNs).]{\small An example of invariances in an RNN with two hidden layers each of which has 2 hidden units. The dashed lines correspond to recurrent weights. The network on the left hand side is equivalent (i.e. represents the same function) to the network on the right for any nonzero $\alpha^1_1=a$, $\alpha^1_2=b$, $\alpha^2_1=c$, $\alpha^2_2=d$. }
\label{fig:invariance}
\end{SCfigure}

\begin{theorem}\label{thm:net-inv}
For any $\alpha$ such that $\alpha^i_j>0$, any Recurrent Neural Network with ReLU activation is invariant to the transformation
$\calT_\alpha\left(\left[\Win,\Wrec,\Wout\right]\right) = \left[\calT_{\In,\alpha}\left(\Win\right),\calT_{\text{rec},\alpha}\left(\Wrec\right),\calT_{\Out,\alpha}\left(\Wout\right)\right]$ where for any $i,j,k$:
\begin{align}
&\calT_{\In,{\alpha}}( \Win)^{i}[j,k]= 
\begin{cases}
\alpha^i_j \Win^{i}[j,k] & i=1, \\ 
\left(\alpha^i_j /\alpha^{i-1}_k \right) \Win^{i}[j,k] & 1< i <d,\\
\end{cases}\\
\calT_{\text{rec},\alpha}( \Wrec )^{i}[j,k]&=
\left(\alpha^i_j /\alpha^{i}_k \right) \Wrec^{i}[j,k],\qquad \calT_{\Out,{\alpha}}( \Wout )[j,k]=\left(1 /\alpha^{d-1}_k \right) \Wout[j,k]. \notag
\end{align}
Furthermore, any feasible node-wise rescaling transformation can be presented in the above form.
\end{theorem}

The proof is given in Section~\ref{sec:invariances-proofs}. The above theorem shows that there are many transformations under which RNNs represent the same function. An example of such invariances is shown in Fig.~\ref{fig:invariance}. Therefore, we would like to have optimization algorithms that are invariant to these transformations and in order to do so, we need to look at measures that are invariant to such mappings.

\section{Understanding Invariances} \label{sec:rescaling}
The goal of this section is to discuss whether being invariant to node-wise
rescaling transformations is sufficient or not.

Ideally we would like our algorithm to be at least invariant to all the transformations to which the model $G$ is invariant. Note that this is different than the invariances studied in \cite{ollivier2015riemannian}, in that they study algorithms that are invariant to reparametrizations of the same model but we look at transformations within the the parameter space that preserve the function in the model. This will eliminate the need for non-trivial initialization. Thus our goal is to characterize the whole variety of transformations to which the model is invariant and check if the algorithm is invariant to all of them.

We first need to note that invariance can be composed. If a 
network $G$ is invariant to transformations $T_1$ and
$T_2$, it is also invariant to their composition $T_1\circ T_2$. This is
also true for an algorithm. If an algorithm is invariant to
transformations $T_1$ and $T_2$, it is also invariant to their
composition. This is because $f_{T_2\circ T_1\circ
\mathcal{A}(\vecw)}=f_{T_2\circ \mathcal{A}(T_1\circ
\vecw)}=f_{\mathcal{A}(T_2\circ T_1(\vecw))}$.

Then it is natural to talk about the {\em basis} of invariances. The
intuition is that although there are infinitely many transformations to which the
model (or an algorithm) is invariant, they could be generated as  
compositions of finite number of transformations.

In fact, in the infinitesimal limit the directions of infinitesimal
changes in the parameters to which the function $f_{\vecw}$ is
insensitive form a subspace. This is because for a  fixed input $\vecx$,
we have
\begin{align}
\label{eq:taylor-expansion}   
 f_{\vecw+\vecDelta}(x) = f_{\vecw}(\vecx) + \sum\nolimits_{e\in E}\frac{\partial f_{\vecw}(\vecx)}{\partial w_{e}}\cdot \Delta_{e} + O(\|\vecDelta\|^2),
\end{align}
where $E$ is the set of edges,
due to a Taylor expansion around $\vecw$. Thus the function $f_{\vecw}$
is insensitive (up to $O(\|\vecDelta\|^2)$) to any change in
the direction $\vecDelta$ that lies in the (right) null space of the
Jacobian matrix $\partial
f_{\vecw}(\vecx)/\partial \vecw$ for all input $\vecx$
simultaneously. More formally, the subspace can be defined as
\begin{align}
\label{eq:defN}
 N(\vecw) = \bigcap\nolimits_{\vecx\in\mathR^{|V_{\rm in}|}}\textrm{Null}\left(\frac{\partial
 f_{\vecw}(\vecx)}{\partial \vecw}\right).
\end{align}
Again, any change to $\vecw$ in the direction $\vecDelta$ that lies in
$N(\vecw)$ leaves the function $f_{\vecw}$ unchanged (up to $O(\|\vecDelta\|^2)$) at {\em
every} input $x$.
Therefore, if we can calculate the dimension of $N(\vecw)$ and if we
have ${\rm dim}N(\vecw)= |V_{\rm internal}|$, where 
we denote the number of internal nodes by $|V_{\rm internal}|$, then we can conclude that all
infinitesimal transformations to which the model is invariant can be
spanned by infinitesimal node-wise rescaling transformations.

Note that the null space $N(\vecw)$ and its dimension is a function of
$\vecw$. Therefore, there are some points in the parameter space that
have more invariances than other points. For example,
suppose that $v$ is an internal node with ReLU activation that receives
connections only from other ReLU units (or any unit whose output is
nonnegative).
If all the incoming weights to 
$v$ are negative including the bias, the output of node $v$ will be zero
regardless of the input, and
the function $f_{\vecw}$ will be insensitive to any transformation to
the outgoing weights of $v$.
Nevertheless we conjecture that
 as the network size grows, the chance of
being in such a degenerate configuration during training will diminish exponentially.

When we study the dimension of $N(\vecw)$, it is convenient to analyze the
dimension of the span of the row vectors of
the Jacobian matrix $\partial
f_{\vecw}(\vecx)/\partial \vecw$ instead. We define the degrees of freedom of
model $G$ at $\vecw$ as
 \begin{align}
  \label{eq:dof}
 d_G(\vecw) = {\rm dim}\left(\bigcup\nolimits_{\vecx\in\mathR^{|V_{\rm in}|}}{\rm Span}\left(
\frac{\partial f_{\vecw}(\vecx)}{\partial \vecw}[v,:] : v\in V_{\rm out}
 \right)\right),
 \end{align}
where $\partial f_{\vecw}(\vecx)[v,:]/\partial \vecw $ denotes the $v$th
row vector of the Jacobian matrix and $\vecx$ runs over all possible input $\vecx$.
Intuitively, $d_{G}(\vecw)$ is the dimension of the set of directions that
changes $f_{\vecw}(x)$ for at least one input $x$.

Due to the rank nullity theorem $d_G(\vecw)$ and the dimension of $N(\vecw)$
are related as follows:
\begin{align*}
 d_G(\vecw) + {\rm dim}\left(N(\vecw)\right)=|E| ,
\end{align*}
where $|E|$ is the number of parameters. Therefore, again if $d_G(\vecw)=|E| -
|V_{\rm internal}|$, then we can conclude that infinitesimally speaking,
all transformations to which the
model is invariant can be spanned by node-wise rescaling
transformations.

Considering only invariances that hold uniformly over all input $\vecx$
could give an under-estimate of the class of invariances, i.e., there might be some invariances that
hold for many input $\vecx$ but not all. An alternative approach
for characterizing invariances is to define a measure of distance
between functions that the neural network model represents based on 
the input distribution, and 
infinitesimally study the subspace of directions to which the distance
is insensitive. We can define distance between two functions $f$
and $g$ as
\begin{align*}
 D(f,g) = \E{\vecx\sim \mathcal{D}}\left[m(f(\vecx),g(\vecx))\right],
\end{align*}
where $m:\mathR^{|V_{\rm out}|\times |V_{\rm out}|}\rightarrow \mathR$ is a
(possibly asymmetric) distance measure between two vectors
$\vecz,\vecz'\in\mathR^{|V_{\rm out}|}$, which we require that
$m(\vecz,\vecz)=0$ and $\partial m/\partial \vecz'_{\vecz=\vecz'}=0$. For example, $m(\vecz,\vecz')=\|\vecz-\vecz'\|^2$.

The second-order Taylor expansion of the distance $D$ can be written as
\begin{align*}
  D(f_{\vecw}\| f_{\vecw+\vecDelta}) &=\frac{1}{2}
\vecDelta^\top \cdot F(\vecw)\cdot
  \vecDelta
 +o(\|\vecDelta\|^2),
\end{align*}
where
\begin{align*}
 F(\vecw)&=\E{\vecx\sim \mathcal{D}}\left[\left(\frac{\partial
 f_{\vecw}(\vecx)}{\partial \vecw}\right)^\top
 \cdot\left.\frac{\partial^2
m(\vecz,\vecz')}{\partial \vecz'^2}\right|_{\vecz=\vecz'=f_{\vecw}(\vecx)} 
\cdot\left(\frac{\partial f_{\vecw}(\vecx)}{\partial \vecw}\right)\right]
 \end{align*}
and $\partial^2 m(\vecz,\vecz')/\partial \vecz'^2|_{\vecz=\vecz'=f_{\vecw}(\vecx)}$ is the Hessian of
the distance measure $m$ at $\vecz=\vecz'=f_{\vecw}(\vecx)$.

Using the above expression, we can define the input distribution
dependent version of $N(\vecw)$ and $d_{G}(\vecw)$ as
 \begin{align*}
  N_{\mathcal{D}}(\vecw) = {\rm Null} F(\vecw),\qquad
 d_{G,\mathcal{D}}(\vecw) = {\rm rank} F(\vecw).
 \end{align*}
Again due to the rank-nullity theorem we have
$d_{G,\mathcal{D}}(\vecw)+{\rm dim}(N_{\mathcal{D}}(\vecw))=|E|$.

As a special case, we obtain the Kullback-Leibler
divergence $D_{\rm KL}$, which is commonly considered as {\em the} way
to study invariances, by 
choosing $m$ as the
conditional Kullback-Leibler divergence of output $y$ given the network output
as
\begin{align*}
 m(\vecz,\vecz') = \E{y\sim q(y|\vecz)}\left[\log\frac{q(y|\vecz)}{q(y|\vecz')}\right],
\end{align*}
where $q(y|\vecz)$ is a link function, which can be, e.g., the soft-max
$q(y|\vecz)=e^{z_y}/\sum_{y'=1}^{|V_{\rm out}|}e^{z_{y'}}$.
However, note that the
invariances in terms of $D_{\rm KL}$ depends not only on the input
distribution but also on the choice of the link function $q(y|\vecz)$.

%
%

\subsection{Path-based characterization of the network}
\label{sec:path-network}
A major challenge in studying the degrees of freedom \eqref{eq:dof}
 is the fact that the
Jacobian $\partial f_{\vecw}(x)/\partial \vecw$ depends on both parameter $w$
and input $x$. In this section, we first tease apart the two
dependencies by rewriting $f_{\vecw}(x)$ as 
the sum over all
directed paths from every input node to each output node as follows:
 \begin{align}
 \label{eq:f-as-sum-over-paths}
 f_{\vecw}(\vecx)[v] &= \sum\nolimits_{p\in\Pi(v)}g_{p}(\vecx)\cdot\pi_p(\vecw)\cdot x[{\rm 
head}(p)],
 \end{align}
where $\Pi(v)$ is the set of all directed path from any 
input
 node to $v$, ${\rm head}(p)$ is the first node of path $p$,
  $g_{p}(\vecx)$ takes 1 if all the rectified
 linear units along path $p$ is active and zero otherwise, and
$\pi_p(\vecw)=\prod_{e\in E(p)} w(e)$
is the product of the weights along path $p$; $E(p)$ denotes the set 
of edges that appear along path $p$.

Let $\Pi=\cup_{v\in V_{\rm out}}\Pi(v)$ be the set of all directed paths.
We define the path-Jacobian matrix $J(\vecw)\in\mathR^{|\Pi|\times |E|}$
as $J(\vecw)=(\partial \pi_p(\vecw)/\partial w_e)_{p\in \Pi, e\in E}$.
In addition, we define
$\vecphi(\vecx)$ as a $|\Pi|$ dimensional vector  with $g_p(\vecx)\cdot x[{\rm
head}(p)]$ in the corresponding entry.
The Jacobian of the network $f_{\vecw}(\vecx)$ can now be expressed as
 \begin{align}
  \label{eq:jacobian}
\frac{\partial f_{\vecw}(\vecx)[v]}{\partial \vecw}&= J_v(\vecw)^\top \vecphi_v(\vecx),
 \end{align}
where where $J_v(\vecw)$ and $\vecphi_v(\vecx)$ are the submatrix (or subvector) of
 $J(\vecw)$ and $\vecphi(\vecx) $ that corresponds to output
node $v$, respectively\footnote{Note that although path activation $g_p(\vecx)$ is a function of $\vecw$,
it is insensitive to an infinitesimal change in the parameter, unless
the input to one of the rectified linear activation functions
along path $p$ is at exactly zero, which happens with probability
zero. Thus we treat $g_p(\vecx)$ as constant here.}. Expression \eqref{eq:jacobian} 
clearly separates the dependence to the parameters $\vecw$ and input $\vecx$.

Now we have the following statement (the proof is given in Section~\ref{sec:invariances-proofs}).
 \begin{theorem}
  \label{thm:dof}
The degrees-of-freedom $d_{G}(\vecw)$ of neural network model $G$ is
at most the rank of the path Jacobian matrix $J(\vecw)$.
 The equality holds if ${\rm dim}\left({\rm
 Span}(\vecphi(\vecx):\vecx\in\mathR^{|V_{\rm in}|})\right)=|\Pi|$; i.e. when
 the dimension of the space
 spanned by  $\vecphi(\vecx)$
equals the total  number of paths $|\Pi|$.
 \end{theorem}

An analogous statement holds for the input distribution dependent
degrees of freedom $d_{G,\mathcal{D}}(\vecw)$, namely,
$d_{G,\mathcal{D}}(\vecw)\leq {\rm rank} J(\vecw)$ and the equality holds if
the rank of the $|\Pi|\times |\Pi|$ path covariance matrix
$(\E{\vecx\sim\mathcal{D}}\left[
\partial^2 m(\vecz,\vecz')/\partial z'_v\partial z'_{v'}\phi_{p}(\vecx)\phi_{p'}(\vecx)
\right])_{p,p'\in \Pi}$ is full, where $v$ and $v'$ are the end nodes of
paths $p$ and $p'$, respectively.

It remains to be understood when the dimension of the span of the path
vectors $\vecphi(\vecx)$ become full. The answer depends on
$\vecw$. Unfortunately, there is no typical behavior as we know from the
example of an internal ReLU unit connected to ReLU units by negative
weights. In fact, we can choose any number of internal units in the network to be
in this degenerate state creating different degrees of degeneracy.
Another way to introduce degeneracy is to insert a linear layer in the
network. This will superficially increase the number of paths but will
not increase the dimension of the span of $\vecphi(\vecx)$. 
For example,
consider a linear classifier $z_{\rm out}=\inner{\vecw}{\vecx}$ with $|V_{\rm in}|$ inputs. If
the whole input space is spanned by $\vecx$, the dimension of the span of
$\vecphi(\vecx)$ is $|V_{\rm in}|$, which agrees with the number of paths. Now let's insert a linear layer with units
$V_1$ in between the input and the output layers. The number of paths has increased from $|V_{\rm in}|$ to $|V_{\rm
in}|\cdot|V_1|$. However the dimension of the span of $\vecphi(\vecx)={\vec
1}_{|V_1|}\otimes \vecx$ is still $|V_{\rm in}|$, because the linear units
are always active.
Nevertheless we conjecture that 
there is a configuration $\vecw$ such that ${\rm dim}\left({\rm
 Span}(\vecphi(\vecx):\vecx\in\mathR^{|V_{\rm in}|})\right)=|\Pi|$ and the set of
 such $\vecw$ grows as the network becomes larger.

\subsection{Combinatorial characterization of the rank of path Jacobian} 
 Finally, we show that the rank of the path-Jacobian matrix $J(\vecw)$ is
determined purely combinatorially by the graph $G$ 
except a subset of the parameter space with zero Lebesgue measure. The
proof is given in Section~\ref{sec:invariances-proofs}.

 \begin{theorem}
  \label{thm:rank-path-jacobian}
The rank of the path Jacobian matrix $J(\vecw)$ is generically (excluding set
  of parameters with zero Lebesgue measure) equal to the number of
 parameters $|E|$ minus the number of internal nodes of the network.
 \end{theorem}

Note that the dimension of the space spanned by node-wise rescaling
equals the number of internal nodes. Therefore,
node-wise rescaling is the {\em only} type of invariance for a ReLU
network with fixed architecture $G$, if ${\rm dim}\left({\rm
Span}(\phi(\vecx):\vecx\in\mathR^{|V_{\rm in}|})\right)=|\Pi|$ at parameter $\vecw$.
 
As an example, let us consider a simple 3 layer network with 2 nodes in each layer
except for the output layer, which has only 1 node (see 
Figure~\ref{fig:net2221}). The network has 10
parameters (4, 4, and 2 in each layer respectively) and 8
paths. The Jacobian $(\partial f_{\vecw}(\vecx)/\partial \vecw)$ can be written as
$(\partial f_{\vecw}(\vecx)/\partial \vecw)
= J(\vecw)^\top\cdot \vecphi(\vecx)$, where
\begin{align}
 \label{eq:J-2221}
 J(\vecw) &=
 \left[
 \begin{array}{c|c|c}
\begin{array}{cccc}
 w_5 w_9 & & & \\
 & w_5w_9 & & \\
 & & w_6w_9 & \\
 & & & w_6w_9 \\
 \hline
 w_7w_{10} & & & \\
 & w_7w_{10} & & \\
 & & w_8w_{10} & \\
 & & & w_8w_{10} \\
\end{array}
&
\begin{array}{cccc}
w_9w_1 & & & \\
w_9w_2 & & & \\
 & w_9w_3 & & \\
 & w_9w_4 & &\\
 \hline
  & & w_{10}w_1 & \\
 & & w_{10}w_2 & \\
 & & & w_{10}w_3 \\
 & & & w_{10}w_4
\end{array}
& \begin{array}{cc}
 w_5w_1 & \\
 w_5w_2 & \\
 w_6w_3 & \\
 w_6w_4 & \\
\hline
 & w_7w_1 \\
 & w_7w_2 \\
 & w_8w_3 \\
 & w_8w_4
\end{array}
 \end{array}
\right]
 \intertext{and}
 \phi(\vecx)^\top&=
\begin{bmatrix}
 g_1(\vecx)x[1] & g_2(\vecx)x[2]  & g_3(\vecx)x[1] & g_4(\vecx)x[2] &  g_5(\vecx)x[1] & 
g_6(\vecx)x[2]&  g_7(\vecx)x[1] &  g_8(\vecx)x[2] 
\end{bmatrix}.\notag
\end{align}
The rank of $J(\vecw)$ in \eqref{eq:J-2221} is (generically) equal to $10-4=6$, which is smaller
than both the number of parameters and the number of paths.
 \begin{figure}[htb]
\begin{center}
 \includegraphics[clip,width=.28\textwidth]{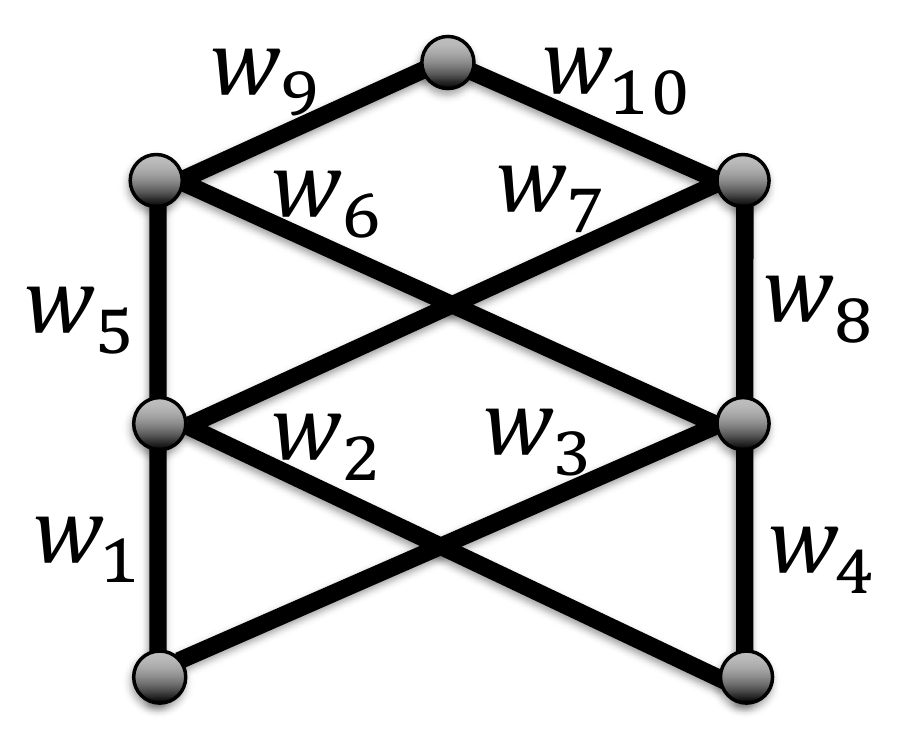}
\end{center}
\caption{\small A 3 layer network with 10 parameters and 8 paths.}
\label{fig:net2221}
 \end{figure}

\section{Proofs}\label{sec:invariances-proofs}
\subsection{Proof of Theorem~\ref{thm:net-inv}}
We first show that any RNN is invariant to $\calT_\alpha$ by induction on layers and time-steps. More specifically, we prove that for any $0\leq t \leq T$ and $1\leq i<d$, $\vec h_t^i\left(\calT_\alpha(\vec W)\right)[j] = \alpha^i_j\vec h_t^i(\vec W)[j]$. The statement is clearly true for $t=0$; because for any $i,j$, $\vec h_0^i\left(\calT_\alpha(\vec W)\right)[j] = \alpha^i_j\vec h_0^i(\vec W)[j]=0$.

Next, we show that for $i=1$, if we assume that the statement is true for $t=t'$, then it is also true for $t=t'+1$:
\begin{align*}
\vec h_{t'+1}^1\left(\calT_\alpha(\vec W)\right)[j] &= \left[\sum_{j'}\calT_{\In,\alpha}(\Win)^1[j,j']\vec x_{t'+1}[j']+ \calT_{\text{rec},\alpha}(\Wrec)^1[j,j'] \vec h_{t'}^1\left(\calT_\alpha(\vec W)\right)[j']\right]_+\\
&=\left[\sum_{j'}\alpha^1_j\Win^1[j,j']\vec x_{t'+1}[j'] +\left(\alpha^1_j /\alpha^1_{j'} \right) \Wrec^1[j,j'] \alpha^1_{j'}\vec h_{t'}^1(\vec W))[j']\right]_+\\
&= \alpha^1_j\vec h_t^i(\vec W)[j]
\end{align*}

We now need to prove the statement for $1<i<d$. Assuming that the statement is true for $t\leq t'$ and the layers before $i$, we have:
\begin{align*}
\vec h_{t'+1}^i\left(\calT_\alpha(\vec W)\right)[j] &= \left[\sum_{j'}\calT_{\In,\alpha}(\Win)^i[j,j']\vec h_{t'+1}^{i-1}\left(\calT_\alpha(\vec W)\right)[j'] +\calT_{\text{rec},\alpha}(\Wrec)^i[j,j'] \vec h_{t'}^i\left(\calT_\alpha(\vec W)\right)[j']\right]_+\\
&=\left[\sum_{j'}\frac{\alpha^i_j}{\alpha^{i-1}_{j'}}\Win^i[j,j']\alpha^{i-1}_{j'}\vec h_{t'+1}^{i-1}(\vec W))[j'] + \frac{\alpha^i_j }{ \alpha^i_{j'} }\Wrec^i[j,j'] \alpha^i_{j'}\vec h_{t'}^i(\vec W))[j']\right]_+\\
&= \alpha^i_j\vec h_t^i(\vec W)[j]
\end{align*}
Finally, we can show that the output is invariant for any $j$ at any time step $t$:
\begin{align*}
f_{\calT(\vec W),t}(\vec x_t)[j] &=  \sum_{j'} \calT_{\text{out},\alpha}(\Wout)[j,j'] \vec h_{t}^{d-1}(\calT_{\alpha}(\vec W)[j']=\sum_{j'} (1/\alpha^{d-1}_{j'})\Wout[j,j'] \alpha^{d-1}_{j'}\vec h_{t}^{d-1}(\vec W)[j'] \\
&=\sum_{j'}\Wout[j,j'] \vec h_{t}^{d-1}(\vec W)[j'] = f_{\vec W,t}(\vec x_t)[j]
\end{align*}

We now show that any feasible node-wise rescaling can be presented as $\calT_\alpha$. Recall that node-wise rescaling invariances for a general feedforward network can be written as $\widetilde{\calT_\beta}(\vec w)_{u\rightarrow v} = (\beta_v/\beta_u)w_{u\rightarrow v}$ for some $\beta$ where $\beta_v>0$ for internal nodes and $\beta_v=1$ for any input/output nodes. An RNN with $T=0$ has no weight sharing and for each node $v$ with index $j$ in layer $i$, we have $\beta_v=\alpha_j^i$. For any $T>0$ however, we there is no invariance that is not already counted. The reason is that by fixing the values of $\beta_v$ for the nodes in time step 0, due to the feasibility, the values of $\beta$ for nodes in other time-steps should be tied to the corresponding value in time step $0$. Therefore, all invariances are included and can be presented in form of $\calT_\alpha$.

\subsection{Proof of Theorem \ref{thm:dof}}
\begin{proof}
First we see that \eqref{eq:jacobian} is true because
 \begin{align*}
\frac{\partial f_{\vecw}(\vecx)[v]}{\partial \vecw} =\Bigl(\sum_{p\in\Pi(v)}
\frac{\partial \pi_p(\vecw)}{\partial w_e}
 \cdot g_{p}(\vecx)\cdot x[{\rm head}(p)]\Bigr)_{e\in E}
 = J_{v}(\vecw)\top\cdot \phi_{v}(\vecx).
\end{align*}
 Therefore, 
 \begin{align}
\bigcup_{\vecx\in\mathR^{|V_{\rm in}|}} {\rm Span}\left(
\frac{\partial f_{\vecw}(\vecx)[v]}{\partial \vecw}: v\in V_{\rm out}
 \right)
&= \bigcup_{\vecx\in\mathR^{|V_{\rm in}|}}{\rm Span}\left(
J_{v}(\vecw)\top\cdot \phi_{v}(\vecx): v\in V_{\rm out}
  \right)\notag\\
  \label{eq:span}
 &=J(\vecw)\top\cdot {\rm Span}\left(\phi(\vecx): \vecx\in\mathR^{|V_{\rm in}|}\right).
\end{align}
Consequently,  any vector of the form $(\frac{\partial
f_{\vecw}(\vecx)[v]}{\partial w_e})_{e\in E}$ for a fixed input $\vecx$ lies in the
 span of the row vectors of the path Jacobian $J(\vecx)$.

The second part says $d_{G}(\vecw)={\rm rank}J(\vecw)$ if ${\rm
dim}\left({\rm Span}(\phi(\vecx):\vecx\in\mathR^{|V_{\rm in}|})\right)=|\Pi|$, which is the number
of rows of $J(\vecw)$. We can see that this is true from expression \eqref{eq:span}.

\end{proof}

\subsection{Proof of Theorem \ref{thm:rank-path-jacobian}}
\begin{proof}
First, $J(\vecw)$  can be written as an Hadamard product between path 
incidence
 matrix $M$ and a rank-one matrix as follows:
\begin{align*}
  J(\vecw) &= M \circ \left(\vecw^{-1} \cdot {\boldsymbol{\pi}}^\top(\vecw)\right),
\end{align*}
 where $M$ is the path incidence matrix whose $i,j$ entry is one if
the $i$th edge is part of the $j$th path,  $\vecw^{-1}$ is an entry-wise
 inverse of the parameter vector $\vecw$, 
${\boldsymbol{\pi}}(\vecw)=(\pi_p(\vecw))$ is a vector containing the product along each 
path in
 each entry, and $\top$ denotes transpose.

 Since we can rewrite
\begin{align*}
  J(\vecw) &= {\rm diag}(\vecw^{-1})\cdot M \cdot{\rm diag}({\boldsymbol{\pi}}(\vecw)),
\end{align*}
 we see that (generically) the rank of $J(\vecw)$ is equal to the rank of
 zero-one matrix $M$.

 Note that the rank of $M$ is equal to the number of linearly
 independent columns of $M$, in other words, the number of linearly
 independent paths. In general, most paths are not independent. For
 example, in Figure \ref{fig:net2221}, we can see that the column
 corresponding to the path
 $w_2w_7w_{10}$ can be produced by combining 3 columns corresponding to
 paths $w_1w_5w_9$,  $w_1w_7w_{10}$, and $w_2w_5w_9$.

 In order to count the number of independent paths, we use 
mathematical
 induction. For simplicity, consider a layered graph with $d$
 layers. All the edges from
 the $(d-1)$th layer nodes to the output layer nodes are linearly 
independent,
 because they correspond to different parameters. So far we have
 $n_dn_{d-1}$ independent paths.

 Next, take one node $u_0$
 (e.g., the leftmost node) from the $(d-2)$th layer. All the paths 
starting
 from this node through the layers above are linearly
 independent. However, other nodes in this layer only contributes
 linearly to the number of independent paths. This is the case because
we can take an edge $(u,v)$, where $u$ is one of the remaining 
$n_{d-2}-1$
 vertices in the $(d-2)$th layer and $v$ is one of the $n_{d-1}$ 
nodes in
 the $(d-1)$th layer, and we can take any path (say $p_0$) from there 
to
 the top  layer. Then this is the only independent path that uses the
 edge $(u,v)$, because any other combination of edge $(u,v)$ and path
 $p$ from $v$ to the top layer can be
 produced as follows (see Figure \ref{fig:dependence}):
 \begin{align*}
  (u,v)\rightarrow p = (u,v)\rightarrow p_0 - (u_0,v)\rightarrow p_0 +
  (u_0,v)\rightarrow p.
 \end{align*}
 Therefore after considering all nodes in the $d-2$th layer, we have
 \begin{align*}
  n_{d}n_{d-1} + n_{d-1}(n_{d-2}-1)
 \end{align*}
 independent paths. Doing this calculation inductively, we have
 \begin{align*}
  n_{d}n_{d-1} + n_{d-1}(n_{d-2}-1) + \cdots + n_{1}(n_0-1)
 \end{align*}
 independent paths, where $n_0$ is the number of input units. This
 number is clearly equal to the number of parameters
 ($n_dn_{d-1}+\cdots+ n_{1}n_0$) minus the number of internal nodes 
($n_{d-1}+\cdots+n_1$).
\end{proof}
\begin{figure}[thb]
 \begin{center}
  \includegraphics[clip,width=0.7\textwidth]{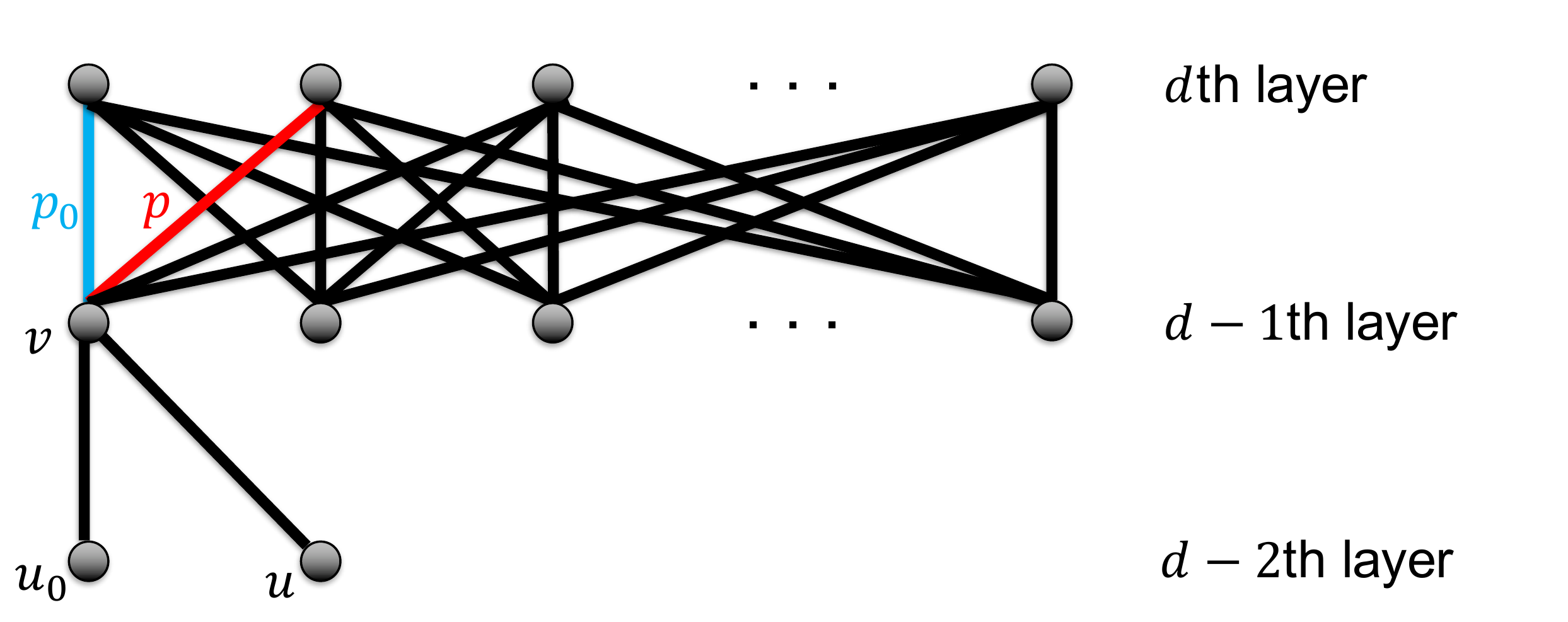}
 \end{center}
  \caption[\small Schematic illustration of the linear dependence of four
  paths.]{\small Schematic illustration of the linear dependence of the four
  paths $(u_0,v)\rightarrow p_0$, $(u_0,v)\rightarrow p$,
  $(u,v)\rightarrow p_0$, and $(u,v)\rightarrow p$. Because of this
  dependence, any additional edge $(u,v)$ only contributes one
  additional independent path.}
  \label{fig:dependence}
\end{figure}

\chapter{Path-Normalization for Feedforward and Recurrent Neural Networks} \label{chap:path-sgd}

As we discussed, optimization is inherently tied to a choice of
geometry, here represented by a choice of complexity measure or
``norm''\footnote{The path-norm which we define is a norm on
  functions, not on weights, but as we prefer not getting into this
  technical discussion here, we use the term ``norm'' very loosely to
  indicate some measure of magnitude~\cite{neyshabur2015norm}.}.
  In Chapter~\ref{chap:invariances}, we studies the invariances in neural networks.
  We would to have a complexity measure that has similar invariance properties as neural networks.
  In Section \ref{sec:path-reg} we introduce the
path-regularizer which is invariant to node-wise rescaling transformations explained in Chapter~\ref{chap:invariances}. In Section \ref{sec:path-forward}, we derive Path-SGD
optimization algorithm for standard feed-forward networks which is the steepest descent with respect to the path-regularizer. Finally, we extend Path-SGD to recurrent neural networks in Section~\ref{sec:path-shared}.

\section{Path-regularizer}\label{sec:path-reg}

We consider the generic group-norm type regularizer in equation \eqref{eq:mu}.
As we discussed before, two simple cases of above group-norm are $q_1=q_2=1$ and $q_1=q_2=2$ that
correspond to overall $\ell_1$ regularization and weight decay
respectively. Another form of regularization that is shown to be very
effective in RELU networks is the max-norm regularization, which is the
maximum over all units of norm of incoming edge to the
unit\footnote{This definition of max-norm is a bit different than the
  one used in the context of matrix factorization~\cite{srebro05}. The
  later is similar to the minimum upper bound over $\ell_2$ norm of
  both outgoing edges from the input units and incoming edges to the
  output units in a two layer feed-forward
  network.}~\cite{goodfellow13,srivastava14}. The max-norm correspond
to ``per-unit" regularization when we set $q_2=\infty$ in
equation~\eqref{eq:mu} and can be written in the following form (for $q_1=2$):
\begin{equation}
  \label{eq:mu}
  \mu_{2,\infty}(w) =\sup_{v \in V}\left(\sum_{(u\rightarrow v) \in E} \left\lvert w_{(u\rightarrow v)}\right\rvert ^2\right)^{1/2}
\end{equation}

Weight decay is probably the most commonly used regularizer. On the
other hand, per-unit regularization might not seem ideal as it is
very extreme in the sense that the value of regularizer corresponds to
the highest value among all nodes.  However, the situation is very
different for networks with RELU activations (and other activation
functions with non-negative homogeneity property).  In these cases,
per-unit $\ell_2$ regularization has shown to be very
effective~\cite{srivastava14}. The main reason could be because RELU
networks can be rebalanced in such a way that all hidden units haveneyshabur2015norm
the same norm. Hence, per-unit regularization will not be a crude
measure anymore.

Since $\mu_{p,\infty}$ is not rescaling invariant and the values of the
scale measure are different for rescaling equivalent networks, it is
desirable to look for the minimum value of a regularizer among all
rescaling equivalent networks. Surprisingly, for a feed-forward
network, the minimum $\ell_2$ per-unit regularizer among all rescaling
equivalent networks can be calculated in close form and we call it the path-regularizer \cite{neyshabur2015norm,neyshabur2015path}.

The path-regularizer is the sum over all paths from input nodes to
output nodes of the product of squared weights along the path. To define
it formally, let $\calP$ be the set of directed paths from input to
output units so that for any pathneyshabur2015path
$\zeta=\left(\zeta_0,\dots,\zeta_{\length(\zeta)}\right)\in \calP$ of
length $\length(\zeta)$, we have that $\zeta_0\in V_\In$,
$\zeta_{\length(\zeta)}\in V_\Out$ and for any $0\leq i \leq
\length(\zeta)-1$, $(\zeta_{i}\rightarrow \zeta_{i+1})\in E$. We also
abuse the notation and denote $e\in \zeta$ if for some $i$,
$e=(\zeta_i,\zeta_{i+1})$. Then the path regularizer can be written
as:
\begin{equation}
\gamma_\net^2(\vecw) = \sum_{\zeta \in \calP} \prod_{i=0}^{\length(\zeta)-1} w_{\zeta_{i}\rightarrow \zeta_{i+1}}^2
\end{equation}
The above formulation of the path-regularizer
involves an exponential number of terms.  However, it can be computed
efficiently by dynamic programming in a single forward step using the
following equivalent recursive definition:
\begin{equation}
\gamma_v^2(\vecw) = \sum_{(u\rightarrow v)\in E} \gamma^2_u(\vecw) w_{u \rightarrow v}^2\;, \qquad \gamma_\net^2(\vecw) = \sum_{u\in V_\Out} \gamma^2_u(\vecw)
\end{equation}

\section{Path-SGD for Feedforward Networks} \label{sec:path-forward}

We consider an approximate steepest descent step with respect to the
path-norm.  More formally, for a network without shared weights, where
the parameters are the weights themselves, consider the diagonal
quadratic approximation of the path-regularizer about the current
iterate $\vecw^{(t)}$:
\begin{equation}
\hat{\gamma}^2_{\rm net}(\vecw^{(t)}+\Delta \vecw) = \gamma^2_{\rm net}(\vecw^{(t)})+
\inner{\nabla \gamma^2_{\rm net}(\vecw^{(t)})}{\Delta \vecw} +
\frac{1}{2} \Delta \vecw^\top
\diag\left(\nabla^2  \gamma^2_{\rm net}(\vecw^{(t)})\right) \Delta \vecw
\end{equation}
Using the corresponding quadratic norm $\norm{\vecw-\vecw'}_{\hat{\gamma}^2_{\rm net}(\vecw^{(t)}+\Delta \vecw)}^2=\frac{1}{2}\sum_{e\in E}\frac{\partial^2 \gamma^2_{\rm 
net}}{\partial w^2_{e}}\left(w_e-w'_e\right)^2$, we can define an
approximate steepest descent step as:
\begin{equation}\label{eq:stp}
\vecw^{(t+1)}=\min_{\vecw} \eta \inner{\nabla L(\vecw)}{\vecw-\vecw^{(t)}} + \norm{\vecw-\vecw^{(t)}}_{\hat{\gamma}^2_{\rm net}(\vecw^{(t)}+\Delta \vecw)}^2.
\end{equation}
Solving \eqref{eq:stp} yields the
update:
\begin{equation}\label{eq:update-forward}
w^{(t+1)}_{e} = w^{(t)}_{e} - \frac{\eta}{\kappa_{e}(\vecw^{(t)})} \frac{\partial 
L}{\partial w_{e}}(\vecw^{(t)}) \quad\quad\textrm{where: } \kappa_{e}(\vecw)=\frac{1}{2}\frac{\partial^2 \gamma^2_{\rm 
net}(\vecw)}{\partial w^2_{e}}.
\end{equation}
The stochastic version that uses a subset
of training examples to estimate $\frac{\partial
  L}{\partial w_{u\rightarrow v}}(\vecw^{(t)})$ is called Path-SGD
\cite{neyshabur2015path}. We now show how Path-SGD can be extended to networks with
shared weights.

\section{Extending to Networks with Shared Weights} \label{sec:path-shared}
When the networks has shared weights, the path-regularizer is a
function of parameters $\vecp$ and therefore the quadratic
approximation should also be with respect to the iterate $\vecp^{(t)}$
instead of $\vecw^{(t)}$ which results in the following
update rule:
\begin{equation}\label{eq:stp-sh}
\vecp^{(t+1)}=\min_{\vecp} \eta \inner{\nabla L(\vecp)}{\vecp-\vecp^{(t)}}+ \norm{\vecp-\vecp^{(t)}}_{\hat{\gamma}^2_{\rm net}(\vecp^{(t)}+\Delta \vecp)}.
\end{equation}
where $\norm{\vecp-\vecp'}_{\hat{\gamma}^2_{\rm net}(\vecp^{(t)}+\Delta \vecp)}^2=\frac{1}{2}\sum_{i=1}^m\frac{\partial^2 \gamma^2_{\rm 
net}}{\partial p^2_{i}}\left(p_i-p'_i\right)^2$. Solving \eqref{eq:stp-sh} gives the following update:
\begin{equation}\label{eq:update-shared}
p^{(t+1)}_{i} = p^{(t)}_{i} - \frac{\eta}{\kappa_{i}(\vecp^{(t)})} \frac{\partial 
L}{\partial p_{i}}(\vecp^{(t)}) \quad\quad\textrm{where: } \kappa_{i}(\vecp)=\frac{1}{2}\frac{\partial^2 \gamma^2_{\rm 
net}(\vecp)}{\partial p^2_{i}}.
\end{equation}
The second derivative terms $\kappa_i$ are specified in terms of their
path structure as follows:
\begin{lemma}\label{lem:path} $\kappa_{i}(\vecp) = \kappa^{(1)}_{i}(\vecp) + \kappa^{(2)}_{i}(\vecp)$
where
\begin{align}
\kappa^{(1)}_{i}(\vecp) &= \sum_{e\in E_i}\sum_{\zeta \in
    \calP} \mathbf{1}_{e\in \zeta} \prod_{j=0 \atop e\neq
    (\zeta_j\rightarrow \zeta_{j+1})}^{\length(\zeta)-1}
  p^2_{\pi(\zeta_j\rightarrow \zeta_{j+1})}= \sum_{e\in E_i} \kappa_{e}(\vecw), \label{eq:k1}\\
 \kappa^{(2)}_{i}(\vecp)  &= p_i^2\sum_{e1,e2\in E_i \atop e_1\neq e_2}\sum_{\zeta \in \calP} \mathbf{1}_{e_1,e_2\in \zeta} \prod_{j=0 \atop { e_1\neq (\zeta_j\rightarrow \zeta_{j+1}) \atop e_2\neq (\zeta_j\rightarrow \zeta_{j+1}) }}^{\length(\zeta)-1} p^2_{\pi(\zeta_j\rightarrow \zeta_{j+1})}, 
 \label{eq:k2}
\end{align}
and $\kappa_e(\vecw)$ is defined in \eqref{eq:update-forward}.
\end{lemma}
The second term $\kappa^{(2)}_i(\vecp)$ measures the effect of
interactions between edges corresponding to the same parameter (edges
from the same $E_i$) on the same path from input to output.  In
particular, if for any path from an input unit to an output unit, no
two edges along the path share the same parameter, then $
\kappa^{(2)}(\vecp)=0$. For example, for any feedforward or
Convolutional neural network, $\kappa^{(2)}(\vecp)=0$.  But for RNNs,
there certainly are multiple edges sharing a single parameter on the
same path, and so we could have $\kappa^{(2)}(\vecp)\neq 0$.


The above lemma gives us a precise update rule for the approximate
steepest descent with respect to the path-regularizer. The following
theorem confirms that the steepest descent with respect to this
regularizer is also invariant to all feasible node-wise rescaling for
networks with shared weights.

\begin{theorem}\label{thm:pathsgd-inv}
  For any feedforward networks with shared weights, the update
  \eqref{eq:update-shared} is invariant to all feasible node-wise rescalings.
  Moreover, a simpler update rule that only uses $\kappa^{(1)}_i(\vecp)$
  in place of $\kappa_i(\vecp)$ is also invariant to all feasible
  node-wise rescalings.
\end{theorem}

Equations \eqref{eq:k1} and \eqref{eq:k2} involve a sum over all paths
in the network which is exponential in depth of the network.
We next show that both of these equations can be calculated
efficiently.

\subsection{Simple and Efficient Computations} \label{sec:path-compute}
We show how to calculate $\kappa^{(1)}_i(\vec p)$ and
$\kappa^{(2)}_i(\vec p)$ by considering a network with the
same architecture but with squared weights:
\begin{theorem}\label{thm:pathsgd-cal}
For any network $\calN(G,\pi,p)$, consider $\calN(G,\pi,\tilde{p})$ where for any $i$, $\tilde{p}_i=p_i^2$. Define the function $g:\R^{\abs{V_\In}}\rightarrow \R$ to be the sum of outputs of this network: $g(x)=\sum_{i=1}^{\abs{V_{\Out}}}f_{\tilde{\vec p}}(x)[i]$. Then $\kappa^{(1)}$ and $\kappa^{(2)}$ can be calculated as follows where $\vec 1$ is the all-ones input vector:
\begin{equation}
\kappa^{(1)}(\vec p) = \nabla_{\tilde{\vec p}} g(\vec 1),\qquad \kappa^{(2)}_i(\vec p)  = \sum_{(u\rightarrow v),(u'\rightarrow v')\in E_i \atop{(u\rightarrow v) \neq (u'\rightarrow v')} } \tilde{p}_i\frac{\partial g(\vec 1)}{\partial h_{v'}(\tilde{\vec p})}\frac{\partial h_{u'}(\tilde{\vec p})}{\partial h_v(\tilde{\vec p})} h_u(\tilde{\vec p}).
\end{equation}
\end{theorem}
In the process of calculating the gradient $\nabla_{\tilde{\vec p}}
g(\vec 1)$, we need to calculate $h_u(\tilde{\vec p})$ and $\partial g(\vec
1)/\partial h_{v}(\tilde{\vec p})$ for any $u,v$. Therefore, the only
remaining term to calculate (besides $\nabla_{\tilde{p}} g(\vec 1)$) is
$\partial h_{u'}(\tilde{\vec p})/\partial h_v(\tilde{\vec p})$.

Recall that $T$ is the length (maximum number of propagations through time) and $d$ is the number 
of layers in an RNN. Let $H$ be the number of hidden units in each layer and $B$
be the size of the mini-batch. Then calculating the gradient of the
loss at all points in the minibatch (the standard work required for
any mini-batch gradient approach) requires time $O(BdTH^2)$. In order
to calculate $\kappa^{(1)}_i(\vec p)$, we need to calculate the gradient
$\nabla_{\tilde{\vec p}}g(1)$ of a similar network at a {\em single}
input---so the time complexity is just an additional $O(dTH^2)$. The second term
$\kappa^{(2)}(\vec p)$ can also be calculated for RNNs in $O(dTH^2(T+H))$ \footnote{
For an RNN, $\kappa^{(2)}(\Win)=0$ and $\kappa^{(2)}(\Wout)=0$ because only
recurrent weights are can be shared multiple times along an input-output path. 
$\kappa^{(2)}(\Wrec)$ can be written and calculated in the matrix form:
{\tiny$\kappa^{(2)}(\Wrec^i) =\Wrec'^i \odot  \sum_{t_1=0}^{T-3}\left[\left(\left(\Wrec'^i\right)^{t_1}\right)^\top \odot \sum_{t_2=2}^{T-t_1-1} 
\frac{\partial g(\vec 1)}{\partial \vec h^i_{t_1+t_2+1}(\tilde{\vec p})} \left(\vec h^i_{t_2}(\tilde{\vec p})\right)^\top \right]
$}
where for any $i,j,k$ we have $\Wrec'^i[j,k] = \left(\Wrec^i[j,k]\right)^2$. 
The only terms that require extra computation are powers of
$\Wrec$ which can be done in $O(dTH^3)$ and the rest of the matrix
computations need $O(dT^2H^2)$.}.
Therefore, the ratio of time complexity of calculating the
first term and second term with respect to the gradient over
mini-batch is $O(1/B)$ and $O((T+H)/B)$ respectively.  Calculating
only $\kappa^{(1)}_i(\vec p)$ is therefore very cheap with minimal
per-minibatch cost, while calculating $\kappa^{(2)}_i(\vec p)$ might
be expensive for large networks.  Beyond the low computational cost,
calculating $\kappa^{(1)}_i(\vec p)$ is also very easy to implement as
it requires only taking the gradient with respect to a standard
feed-forward calculation in a network with slightly modified
weights---with most deep learning libraries it can be implemented very
easily with only a few lines of code.

\section{Proofs}

\subsection{Proof of Lemma~\ref{lem:path}}
We prove the statement simply by calculating the second derivative of the path-regularizer with respect to each parameter:
\begin{align*}
\kappa_{i}(\vec p)&=\frac{1}{2}\frac{\partial^2 \gamma^2_{\net}}{\partial p_i^2} =\frac{1}{2}\frac{\partial}{\partial p_i}\left(\frac{\partial }{\partial p_i}\sum_{\zeta \in \calP} \prod_{j=0}^{\length(\zeta)-1} w_{\zeta_{j}\rightarrow \zeta_{j+1}}^2\right)\\
&=\frac{1}{2}\frac{\partial}{\partial p_i}\left(\frac{\partial }{\partial p_i}\sum_{\zeta \in \calP} \prod_{j=0}^{\length(\zeta)-1} p_{\pi(\zeta_{j}\rightarrow \zeta_{j+1})}^2\right)
=\frac{1}{2}\sum_{\zeta \in \calP}\frac{\partial}{\partial p_i}\left(\frac{\partial }{\partial p_i} \prod_{j=0}^{\length(\zeta)-1} p_{\pi(\zeta_{j}\rightarrow \zeta_{j+1})}^2\right)\\
\end{align*}
Taking the second derivative then gives us both terms after a few calculations:
\begin{align*}
\kappa_{i}(\vec p)&=\frac{1}{2}\sum_{\zeta \in \calP}\frac{\partial}{\partial p_i}\left(\frac{\partial }{\partial p_i} \prod_{j=0}^{\length(\zeta)-1} p_{\pi(\zeta_{j}\rightarrow \zeta_{j+1})}^2\right)=\sum_{\zeta \in \calP}\frac{\partial}{\partial p_i}\left(p_i\sum_{e \in E_i} \vec 1_{e\in \zeta}\prod_{j=0 \atop e\neq (\zeta_j\rightarrow \zeta_{j+1}}^{\length(\zeta)-1}p_{\pi(\zeta_{j}\rightarrow \zeta_{j+1})}^2\right) \\
&= \sum_{\zeta \in \calP}\left[ p_i\frac{\partial}{\partial p_i}\left(\sum_{e \in E_i} \vec 1_{e\in \zeta}\prod_{j=0 \atop e\neq (\zeta_j\rightarrow \zeta_{j+1}}^{\length(\zeta)-1}p_{\pi(\zeta_{j}\rightarrow \zeta_{j+1})}^2\right) + \sum_{e \in E_i} \vec 1_{e\in \zeta}\prod_{j=0 \atop e\neq (\zeta_j\rightarrow \zeta_{j+1}}^{\length(\zeta)-1}p_{\pi(\zeta_{j}\rightarrow \zeta_{j+1})}^2 \right]\\
&=p_i^2\sum_{e1,e2\in E_i \atop e_1\neq e_2}\left[\sum_{\zeta \in \calP} \vec 1_{e_1,e_2\in \zeta} \prod_{j=0 \atop { e_1\neq (\zeta_j\rightarrow \zeta_{j+1}) \atop e_2\neq (\zeta_j\rightarrow \zeta_{j+1}) }}^{\length(\zeta)-1} p^2_{\pi(\zeta_j\rightarrow \zeta_{j+1})}\right]
+\sum_{e\in E_i}\left[\sum_{\zeta \in \calP} \vec 1_{e\in \zeta} \prod_{j=0 \atop e\neq (\zeta_j\rightarrow \zeta_{j+1})}^{\length(\zeta)-1} p^2_{\pi(\zeta_j\rightarrow \zeta_{j+1})}\right]\\
\end{align*}

\subsection{Proof of Theorem~\ref{thm:pathsgd-inv}}
Node-wise rescaling invariances for a feedforward network can be written as $\calT_\beta(\vec w)_{u\rightarrow v}=(\beta_v/\beta_u)w_{u\rightarrow v}$ for some $\beta$ where $\beta_v>0$ for internal nodes and $\beta_v=1$ for any input/output nodes. Any feasible invariance for a network with shared weights can also be written in the same form. The only difference is that some of $\beta_v$s are now tied to each other in a way that shared weights have the same value after transformation. First, note that since the network is invariant to the transformation, the following statement holds by an induction similar to Theorem~\ref{thm:net-inv} but in the backward direction:
\begin{equation}
\frac{\partial L}{\partial h_v}(\calT_\beta(\vec p)) = \frac{1}{\beta_v}\frac{\partial L}{\partial h_u}(\vec p) 
\end{equation}
for any $(u\rightarrow v)\in E$. Furthermore, by the proof of the Theorem~\ref{thm:net-inv} we have that for any $(u\rightarrow v)\in E$, $h_u(\calT_\beta(\vec p ) ) =  \beta_u h_u(\vec p)$. Therefore,
\begin{equation}
\frac{\partial L}{\partial \calT_\beta(\vec p)_i}(\calT_\beta(\vec p)) = \sum_{(u\rightarrow v)\in E_i}\frac{\partial L}{\partial h_v}(\calT_\beta(\vec p)) h_u(\calT_\beta(\vec p ) ) =  \frac{\beta_{u'}}{\beta_{v'}}\frac{\partial L}{\partial p_i}(\vec p) 
\end{equation}
where $(u'\rightarrow v')\in E_i$. In order to prove the theorem statement, it is enough to show that for any edge $(u\rightarrow v) \in E_i$, $\kappa_{i}(\calT_\beta(\vec p)) = (\beta_u/\beta_v)^2\kappa_{i}(\vec p)$ because this property gives us the following update:
\begin{equation*}
\calT_\beta(\vec p)_i - \frac{\eta}{\kappa_i(\calT_\beta(\vec p))} \frac{\partial L(\calT_\beta(\vec p))}{\partial \calT_\beta(\vec p)_i } = \frac{\beta_v}{\beta_u}p_i-
\frac{\eta}{(\beta_u/\beta_v)^2\kappa_{i}(\vec p)}\frac{\beta_{u}}{\beta_{v}}\frac{\partial L}{\partial p_i}(\vec p) = \calT_\beta(\vec p^+)_i
\end{equation*}
Therefore, it is remained to show that for any edge $(u\rightarrow v) \in E_i$ $v$, $\kappa_{i}(\calT_\beta(\vec p)) = (\beta_u/\beta_v)^2\kappa_{i}(\vec p)$. We show that this is indeed true for both terms $\kappa^{(1)}$ and $\kappa^{(2)}$ separately. 

We first prove the statement for $\kappa^{(1)}$. Consider each path $\zeta\in \calP$. By an inductive argument along the path, it is easy to see that multiplying squared weights along this path is invariant to the transformation:
\begin{equation*}
\prod_{j=0}^{\length(\zeta)-1} \calT_\beta(\vec p)^2_{\pi(\zeta_j\rightarrow \zeta_{j+1})} = \prod_{j=0}^{\length(\zeta)-1} p^2_{\pi(\zeta_j\rightarrow \zeta_{j+1})}
\end{equation*}
Therefore, we have that for any edge $e\in E$ and any $\zeta\in \calP$,
\begin{equation*}
\prod_{j=0 \atop e\neq (\zeta_j\rightarrow \zeta_{j+1})}^{\length(\zeta)-1} \calT_\beta(\vec p)^2_{\pi(\zeta_j\rightarrow \zeta_{j+1})}
= \left(\frac{\beta_u}{\beta_v}\right)^2 \prod_{j=0  \atop e\neq (\zeta_j\rightarrow \zeta_{j+1})}^{\length(\zeta)-1} p^2_{\pi(\zeta_j\rightarrow \zeta_{j+1})}
\end{equation*}
Taking sum over all paths $\zeta\in \calP$ and all edges $e=(u\rightarrow v) \in E$ completes the proof for $\kappa^{(1)}$. Similarly for $\kappa^{(2)}$, considering any two edges $e_1\neq e_2$ and any path $\zeta_\calP$, we have that:
\begin{equation*}
\calT_\beta(\vec p)_i^2\prod_{j=0 \atop { e_1\neq (\zeta_j\rightarrow \zeta_{j+1}) \atop e_2\neq (\zeta_j\rightarrow \zeta_{j+1}) }}^{\length(\zeta)-1} \calT_\beta(\vec p)^2_{\pi(\zeta_j\rightarrow \zeta_{j+1})} =  \left(\frac{\beta_v}{\beta_u}\right)^2 p_i^2 \left(\frac{\beta_u}{\beta_v}\right)^4\prod_{j=0 \atop { e_1\neq (\zeta_j\rightarrow \zeta_{j+1}) \atop e_2\neq (\zeta_j\rightarrow \zeta_{j+1}) }}^{\length(\zeta)-1} p^2_{\pi(\zeta_j\rightarrow \zeta_{j+1})}
\end{equation*}
where $(u\rightarrow v)\in E_i$. Again, taking sum over all paths $\zeta$ and all edges $e_1\neq e_2$ proves the statement for $\kappa^{(2)}$ and consequently for $\kappa^{(1)}+\kappa^{(2)}$.

\subsection{Proof of Theorem~\ref{thm:pathsgd-cal}}
First, note that based on the definitions in the theorem statement, for any node $v$, $h_v(\tilde{\vec p})=\gamma^2_v(p)$ and therefore $g(\vec 1)=\gamma_\net^2(p)$. Using Lemma~\ref{lem:path}, main observation here is that for each edge $e\in E_i$ and each path $\zeta\in \calP$, the corresponding term in $\kappa^{(1)}$ is nothing but product of the squared weights along the path except the weights that correspond to the edge $e$:
$$
\vec 1_{e\in \zeta} \prod_{j=0 \atop e\neq (\zeta_j\rightarrow \zeta_{j+1})}^{\length(\zeta)-1} p^2_{\pi(\zeta_j\rightarrow \zeta_{j+1})}
$$
This path can therefore be decomposed into a path from input to edge $e$ and a path from edge $e$ to the output. Therefore, for any edge $e$, we can factor out the number corresponding to the paths that go through $e$ and rewrite $\kappa^{(1)}$ as follows:
\begin{equation}
\kappa^{(1)}(p)=\sum_{(u\rightarrow v)\in E_i} \left[\left(\sum_{\zeta\in \calP_{\In \rightarrow u} } \prod_{j=0}^{\length(\zeta)-1} p_{\pi(\zeta_{j}\rightarrow \zeta_{j+1})}^2\right)\left(\sum_{\zeta\in \calP_{v\rightarrow \Out} } \prod_{j=0}^{\length(\zeta)-1}p_{\pi(\zeta_{j}\rightarrow \zeta_{j+1})}^2 \right)\right]
\end{equation}
where $\calP_{\In\rightarrow u}$ is the set of paths from input nodes to node $v$ and $\calP_{v\rightarrow \Out}$ is defined similarly for the output nodes.

By induction on layers of $\calN(G,\pi,\tilde{\vec p})$, we get the following:
\begin{align}
\sum_{\zeta\in \calP_{\In \rightarrow u} } \prod_{j=0}^{\length(\zeta)-1} p_{\pi(\zeta_{j}\rightarrow \zeta_{j+1})}^2 &= h_u(\tilde{\vec p})\\
\sum_{\zeta\in \calP_{v\rightarrow \Out} } \prod_{j=0}^{\length(\zeta)-1}p_{\pi(\zeta_{j}\rightarrow \zeta_{j+1})}^2 &= \frac{\partial g(1)}{\partial h_v(\tilde{\vec p})}
\end{align}
Therefore, $\kappa^{(1)}$ can be written as:
\begin{equation}
\kappa^{(1)}(p)= \sum_{(u\rightarrow v)\in E_i} \frac{\partial g(1)}{\partial h_v(\tilde{\vec p})}h_u(\tilde{\vec p}) 
= \sum_{(u\rightarrow v)\in E_i} \frac{\partial g(1)}{\partial w'_{u\rightarrow v}} =  \frac{\partial g(1)}{\partial \tilde{p}_{i}}
\end{equation}
Next, we show how to calculate the second term, i.e. $\kappa^{(2)}$. Each term in $\kappa^{(2)}$  corresponds to a path that goes through two edges. We can decompose such paths and rewrite $\kappa^{(2)}$ similar to the first term:
\begin{align*}
\kappa^{(2)}(p)&=p_i^2\sum_{(u\rightarrow v)\in E_i \atop{ (u'\rightarrow v')\in E_i \atop (u\rightarrow v) \neq (u'\rightarrow v')} } \left[\left(\sum_{\zeta\in \calP_{\In \rightarrow u}}\prod_{j=0}^{\length(\zeta)} p_{\pi(\zeta_{j}\rightarrow \zeta_{j+1})}^2\right)\right.\\
&\left.\left(\sum_{\zeta\in \calP_{v \rightarrow u'}}\prod_{j=0}^{\length(\zeta)-1}p_{\pi(\zeta_{j}\rightarrow \zeta_{j+1})}^2 \right)\left(\sum_{\zeta\in \calP_{v'\rightarrow \Out}}\prod_{j=0}^{\length(\zeta)-1}p_{\pi(\zeta_{j}\rightarrow \zeta_{j+1})}^2\right)\right]\\
&= \sum_{(u\rightarrow v)\in E_i \atop{ (u'\rightarrow v')\in E_i \atop (u\rightarrow v) \neq (u'\rightarrow v')} } \tilde{p}_i\frac{\partial g(\vec 1)}{\partial h_{v'}(\tilde{\vec p})}\frac{\partial h_{u'}(\tilde{\vec p})}{\partial h_v(\tilde{\vec p})} h_u(\tilde{\vec p})
\end{align*}
where $\calP_{u\rightarrow v}$ is the set of all directed paths from node $u$ to node $v$.

\chapter{Experiments on Path-SGD} \label{chap:path-experiments}
In this Chapter, we compare Path-SGD to other optimization methods on fully connected and recurrent neural networks.
\section{Experiments on Fully Connected Feedforward Networks}

We compare $\ell_2$-Path-SGD to two commonly used optimization methods in deep learning, SGD and AdaGrad. We conduct our experiments on 
four common benchmark datasets: the standard MNIST dataset of handwritten digits~\cite{lecun1998gradient}; 
CIFAR-10 and CIFAR-100 datasets of tiny images of natural scenes~~\cite{krizhevsky2009learning}; 
and Street View House Numbers (SVHN) dataset containing 
color images of house numbers collected by Google Street View~\cite{netzer2011reading}. 
Details of the datasets are shown in Table~\ref{table}.

\begin{table}[t]
\caption[\small General information on datasets used in the experiments on feedforward networks.]{\small General information on datasets used in the experiments on feedforward networks.}
\label{table}
\begin{center}
\begin{tabular}{c c c c c}
{\bf Data Set}  &{\bf Dimensionality}&{\bf Classes}&{\bf Training Set}&{\bf Test Set}
\\ \hline
CIFAR-10&3072 ($32 \times 32$ color)&10&50000&10000\\
CIFAR-100&3072 ($32 \times 32$ color)&100&50000&10000\\
MNIST&784 ($28 \times 28$ grayscale)&10&60000&10000\\
SVHN&3072 ($32 \times 32$ color)&10&73257&26032\\
\hline
\end{tabular}
\end{center}
\end{table}

In all of our experiments, we trained feed-forward networks with two hidden layers, each containing 4000 hidden units. We used mini-batches of size 100 and the step-size of $10^{-\alpha}$, where $\alpha$ is an integer between 0 and 10. To choose $\alpha$, for each dataset, we considered the validation errors over the validation set (10000 randomly chosen points that are kept out during the initial training) and picked the one that reaches the minimum error faster. We then trained the network over the entire training set. All the networks were trained both with and without dropout. When training with dropout, at each update step, we retained each unit with probability 0.5.

We tried both balanced and unbalanced initializations. In balanced initialization, incoming weights to each unit $v$ are initialized to i.i.d samples from a Gaussian distribution with standard deviation $1/\sqrt{\text{fan-in}(v)}$. In the unbalanced setting, we first initialized the weights to be the same as the balanced weights. We then picked 2000 hidden units randomly with replacement. For each unit, we multiplied its incoming edge and divided its outgoing edge by $10c$, where $c$ was chosen randomly from log-normal distribution.

The optimization results are shown in Figure~\ref{fig:nodrop}. For each of the four datasets, the plots for
objective function (cross-entropy), the training error and the test
error are shown from left to right where in each plot the values are
reported on different epochs during the optimization. Although we
proved that Path-SGD updates are the same for balanced and unbalanced
initializations, to verify that despite numerical issues they are
indeed identical, we trained Path-SGD with both balanced and unbalanced initializations. 
Since the curves were exactly the same we only show a single curve. The dropout is used for the 
experiments on CIFAR-100 and SVHN. Please see \cite{neyshabur2015path} for a more complete set of experimental results.

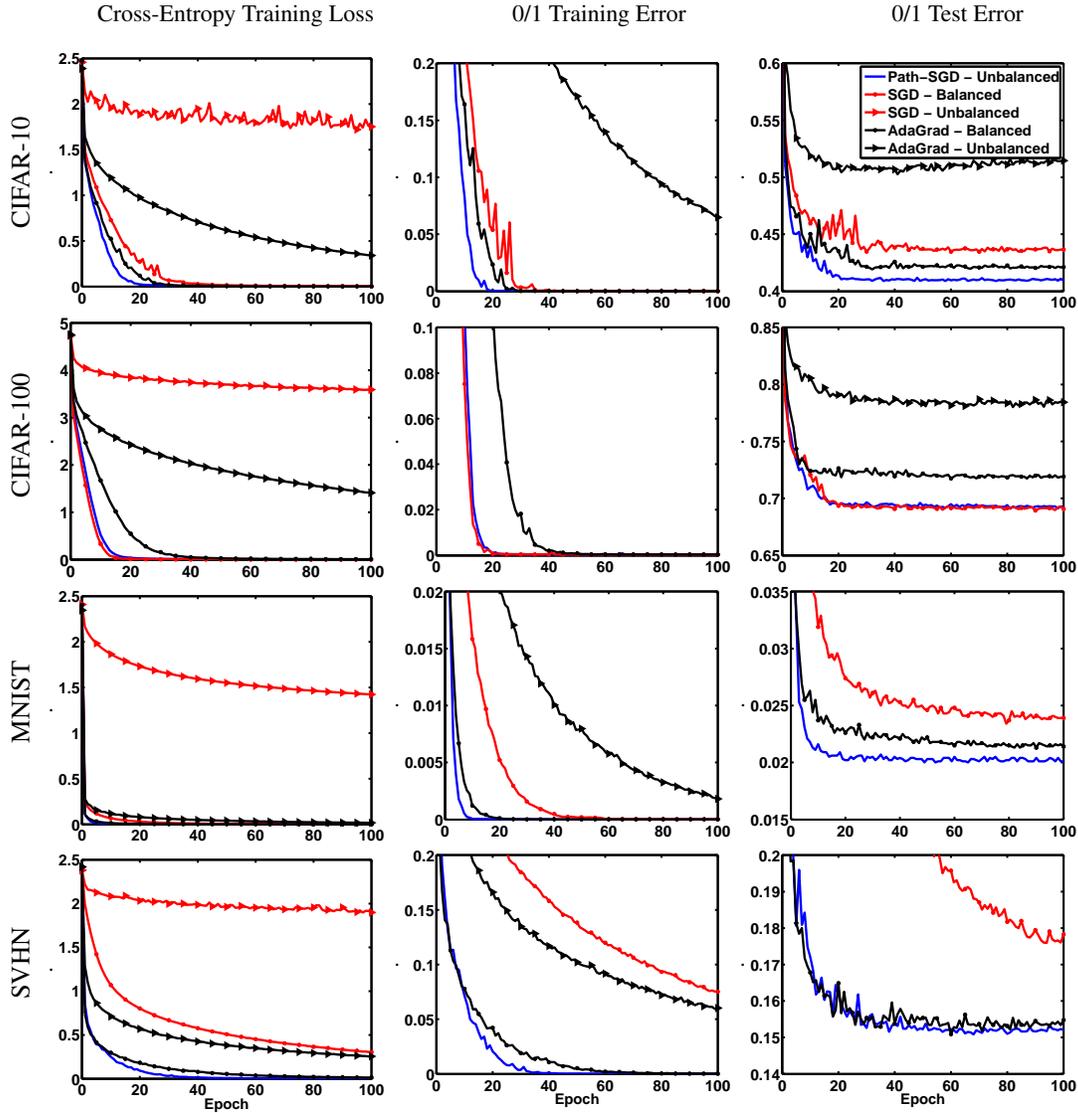
\begin{figure}[t!]
 \subfloat{
  \begin{tabular}{r}
   \includegraphics[width=\picwidth]{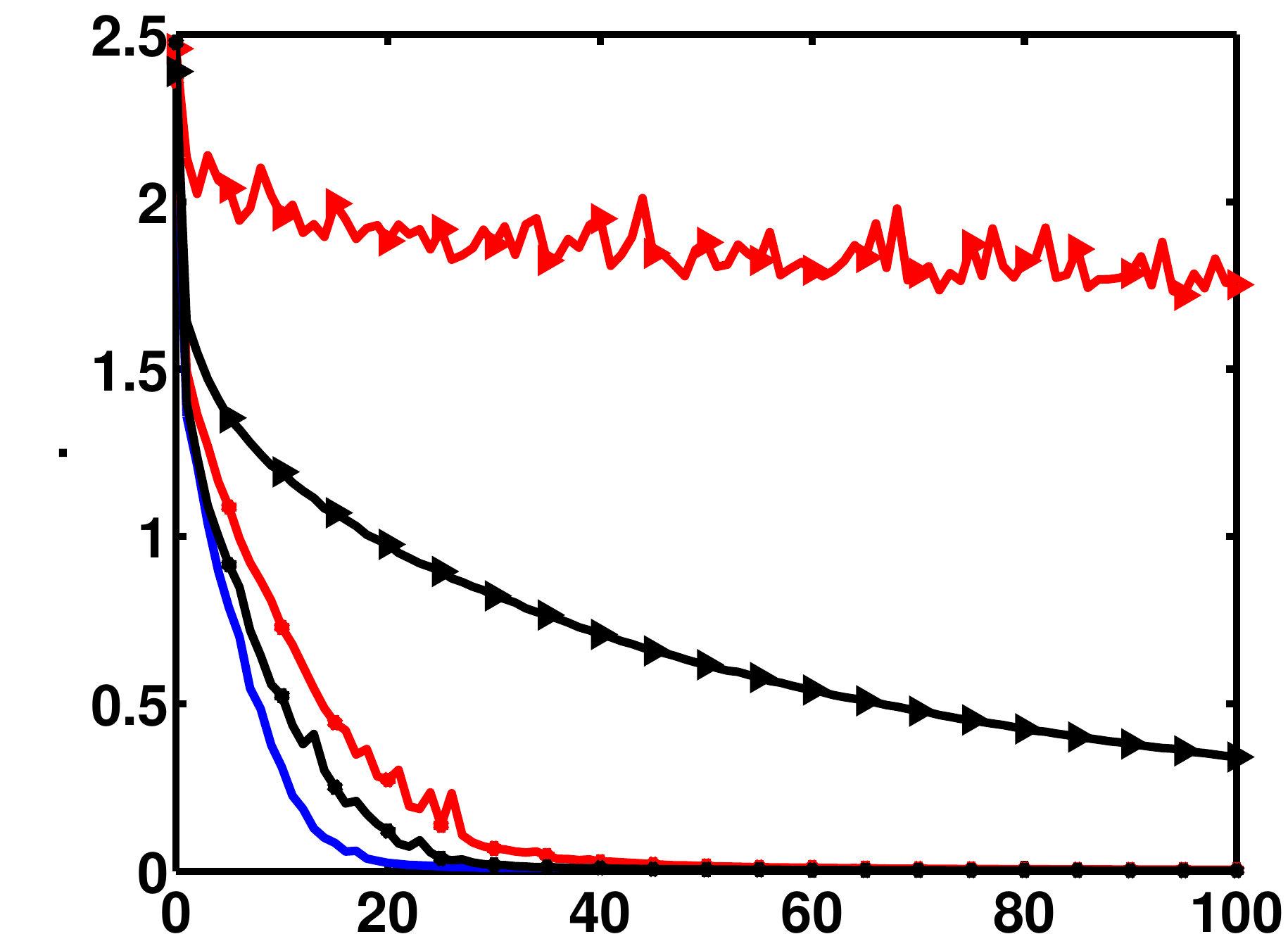} \\
   \includegraphics[width=\picwidth]{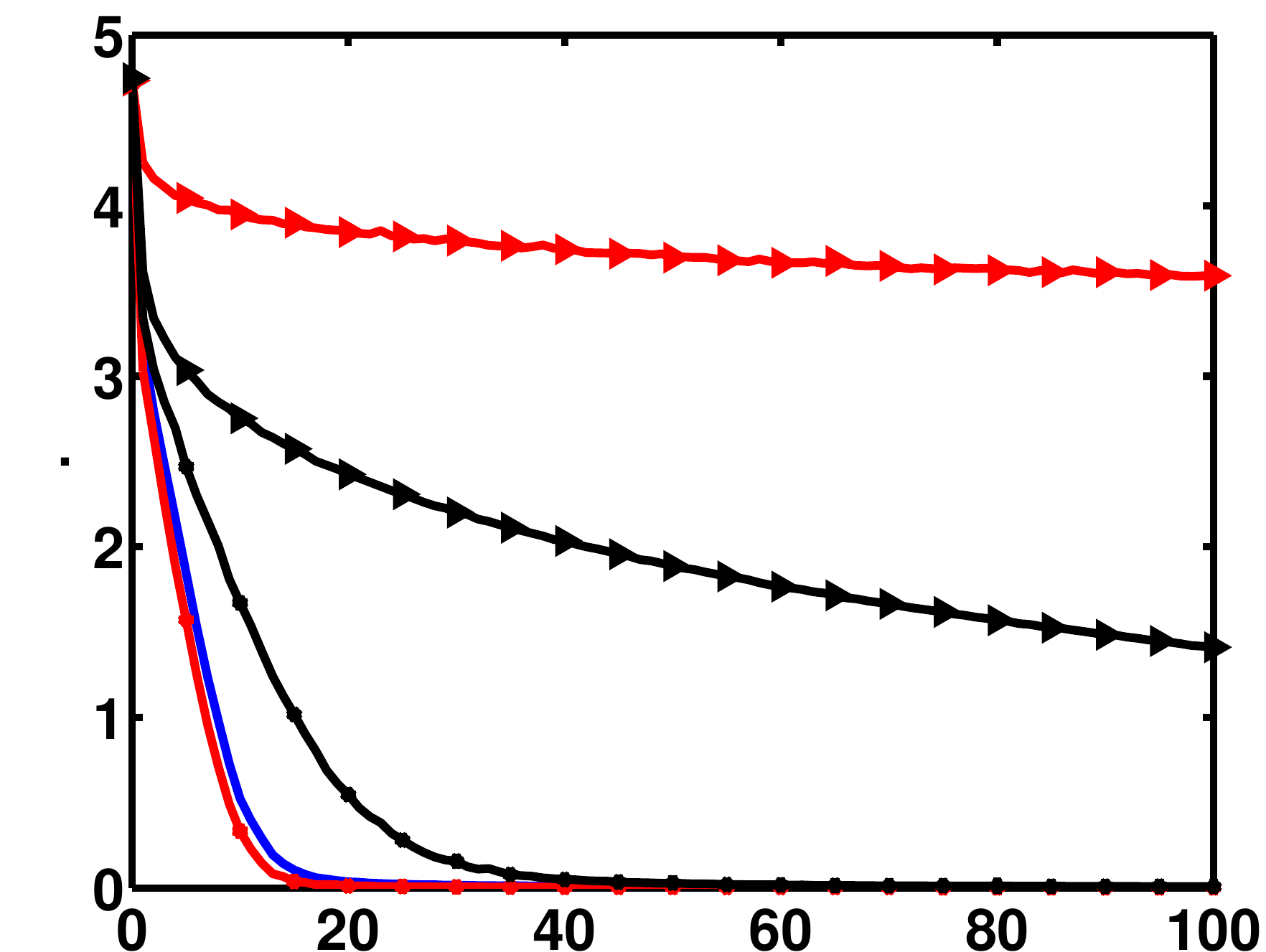} \\
   \includegraphics[width=\picwidth]{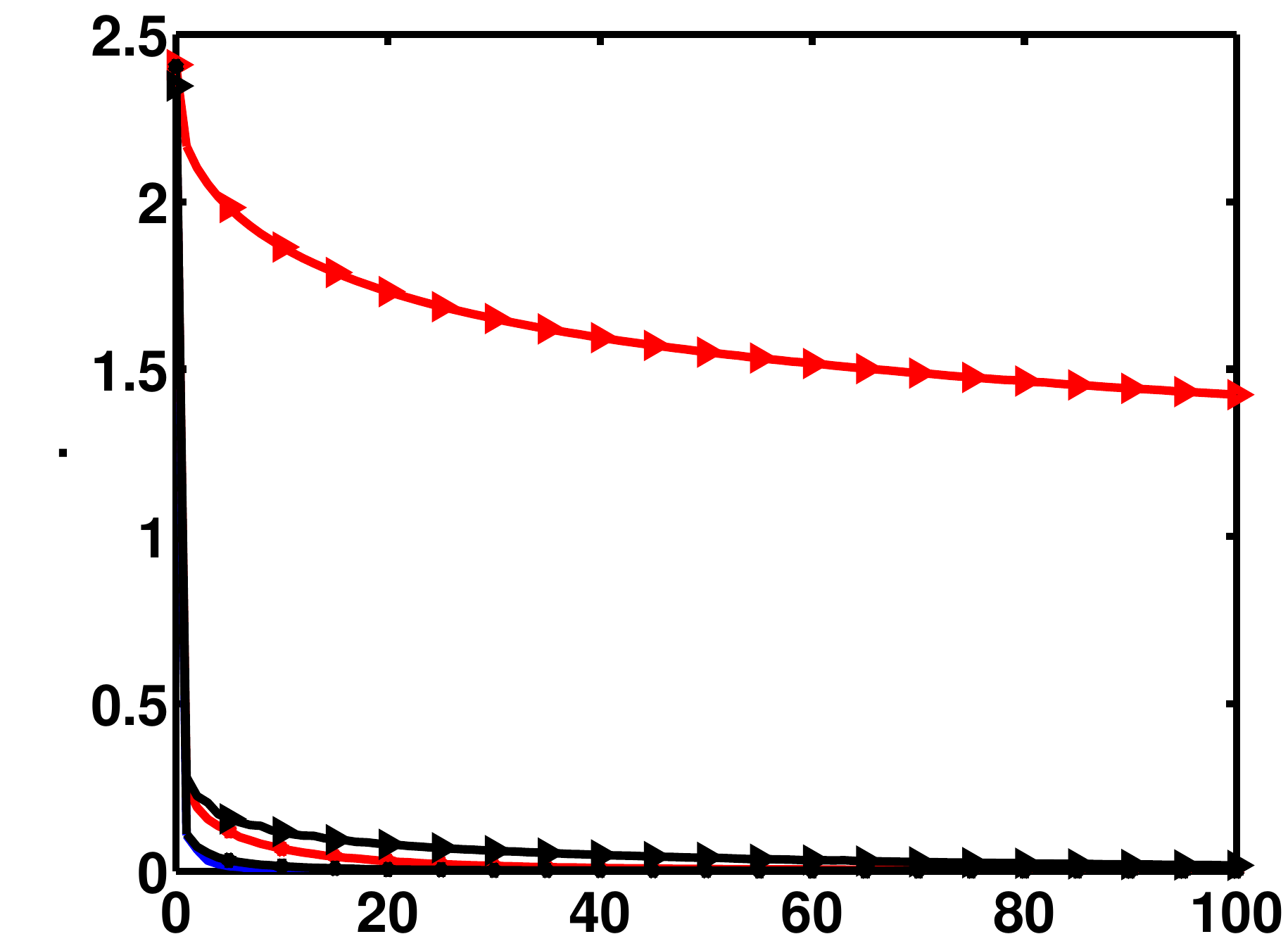} \\
   \includegraphics[width=\picwidth]{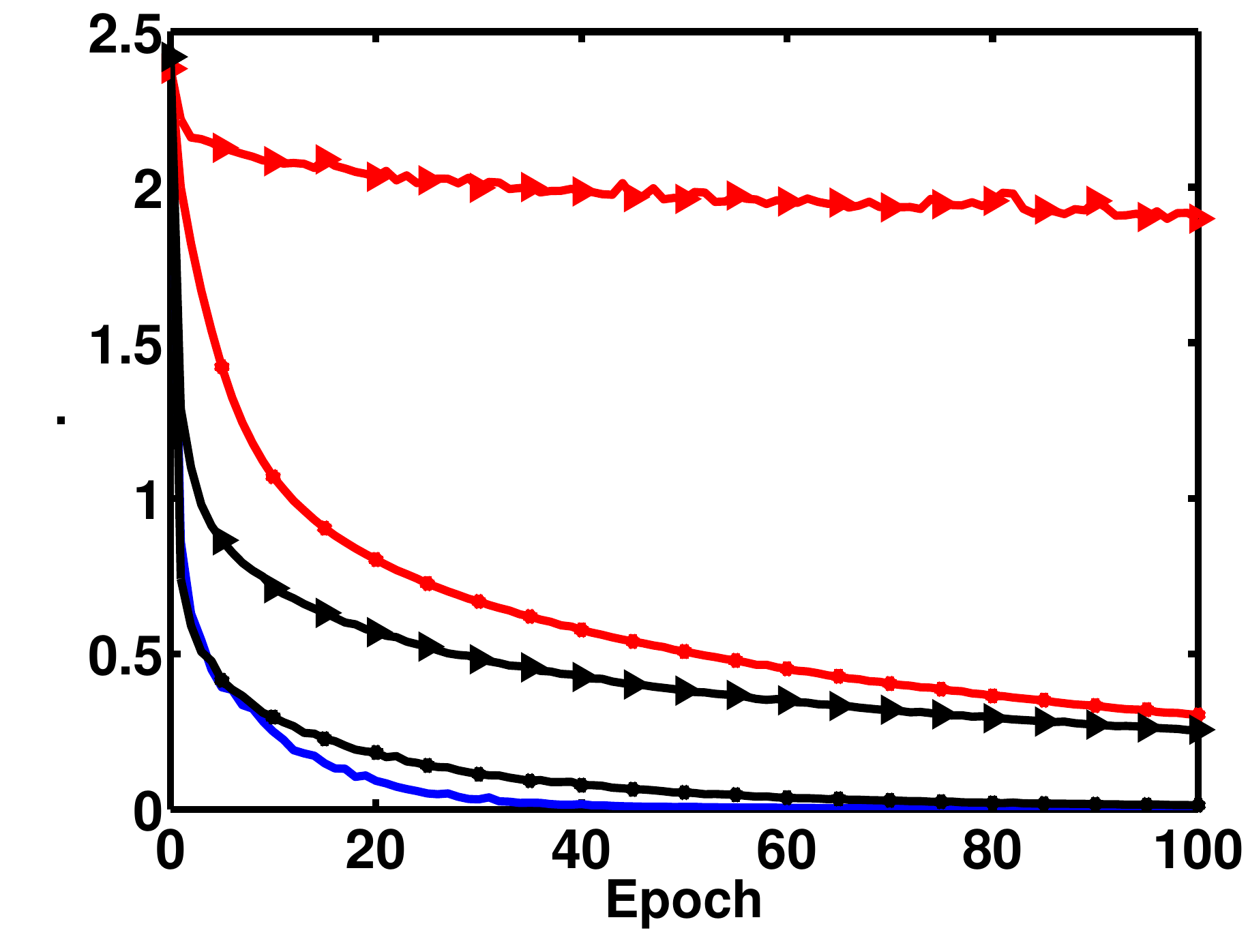}
  \end{tabular}
 }
 \hspace{-0.3in}
 \subfloat{
  \begin{tabular}{r}
   \includegraphics[width=\picwidth]{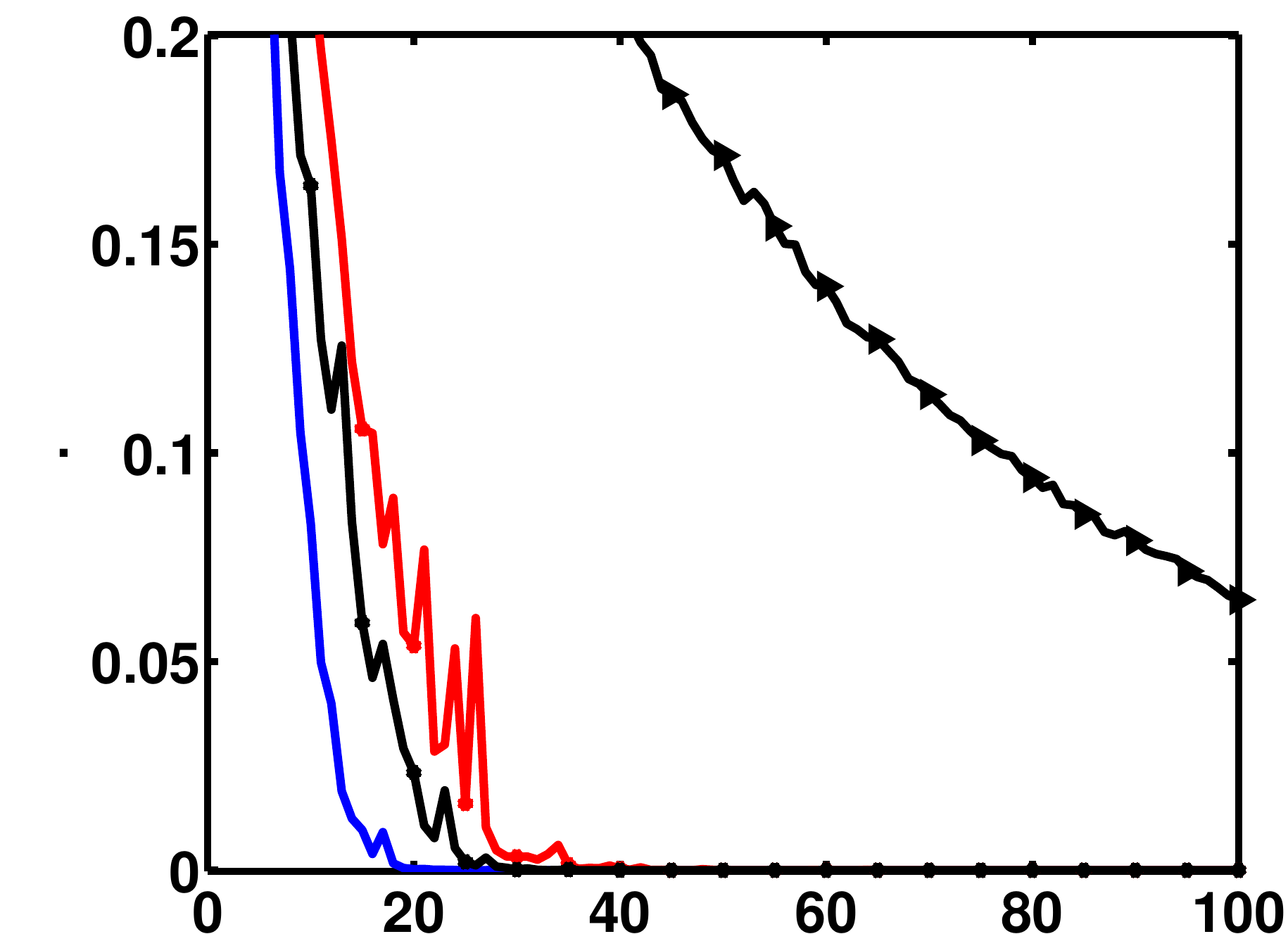} \\
   \includegraphics[width=\picwidth]{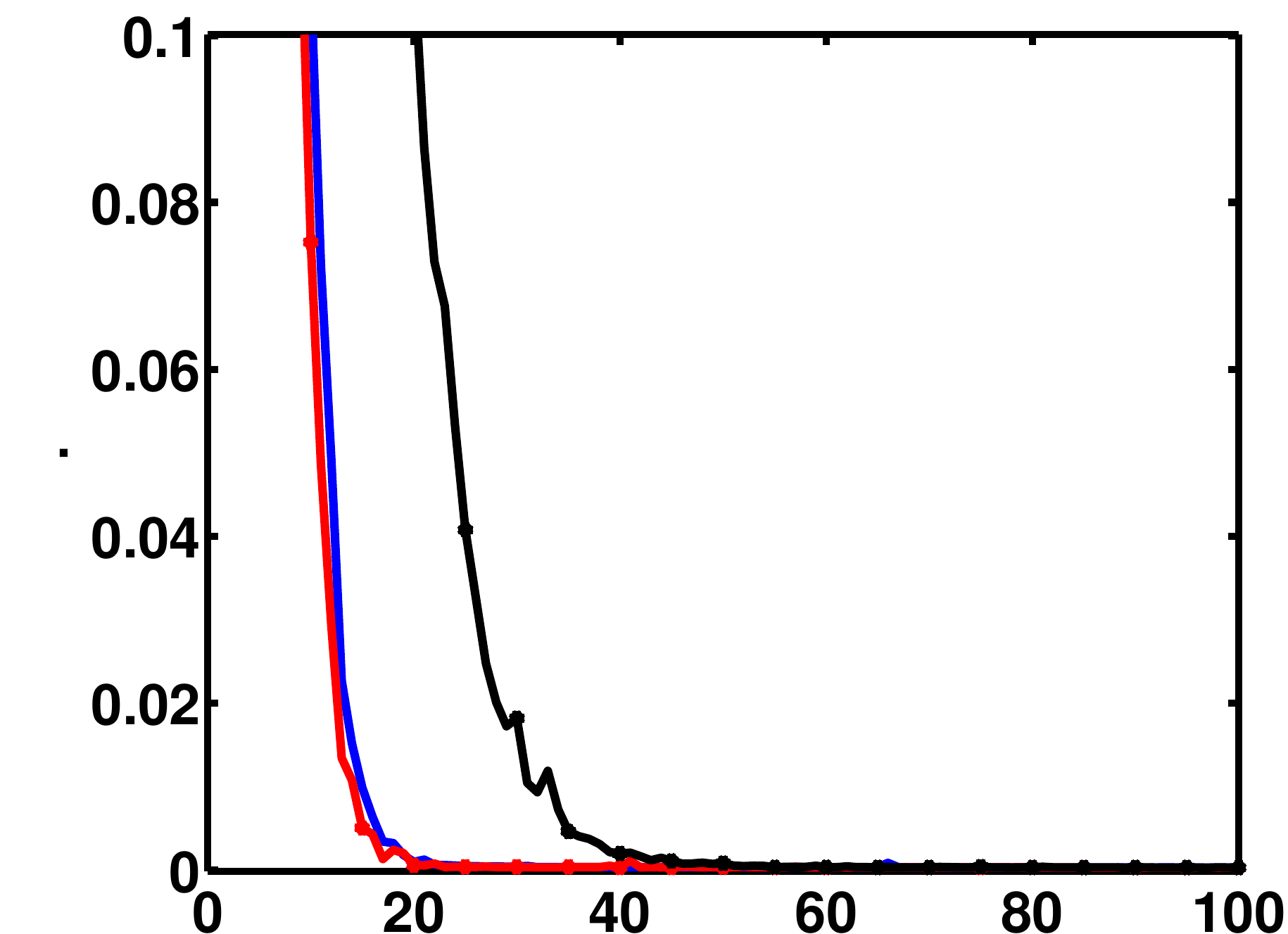} \\
   \includegraphics[width=\picwidth]{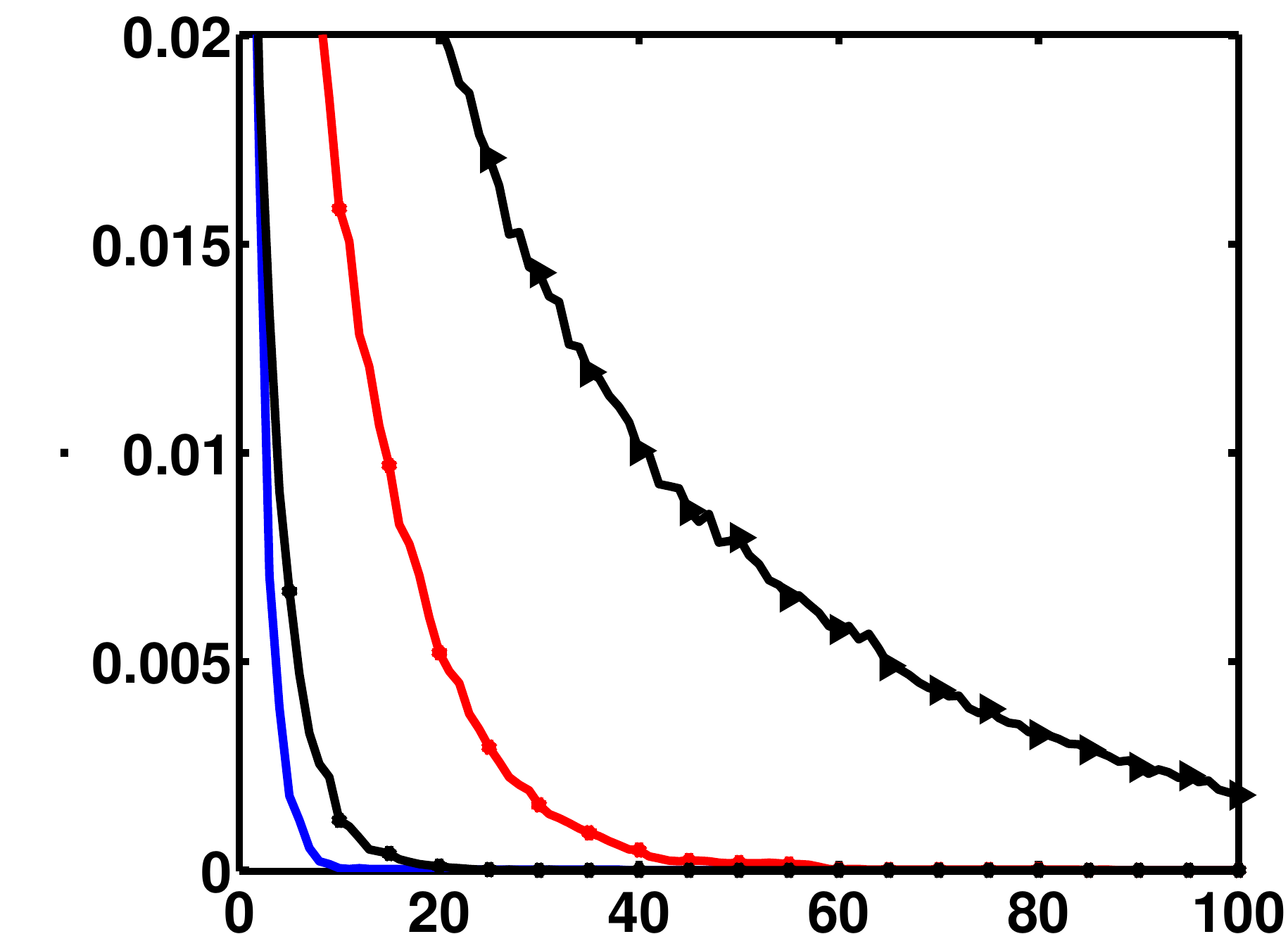} \\
   \includegraphics[width=\picwidth]{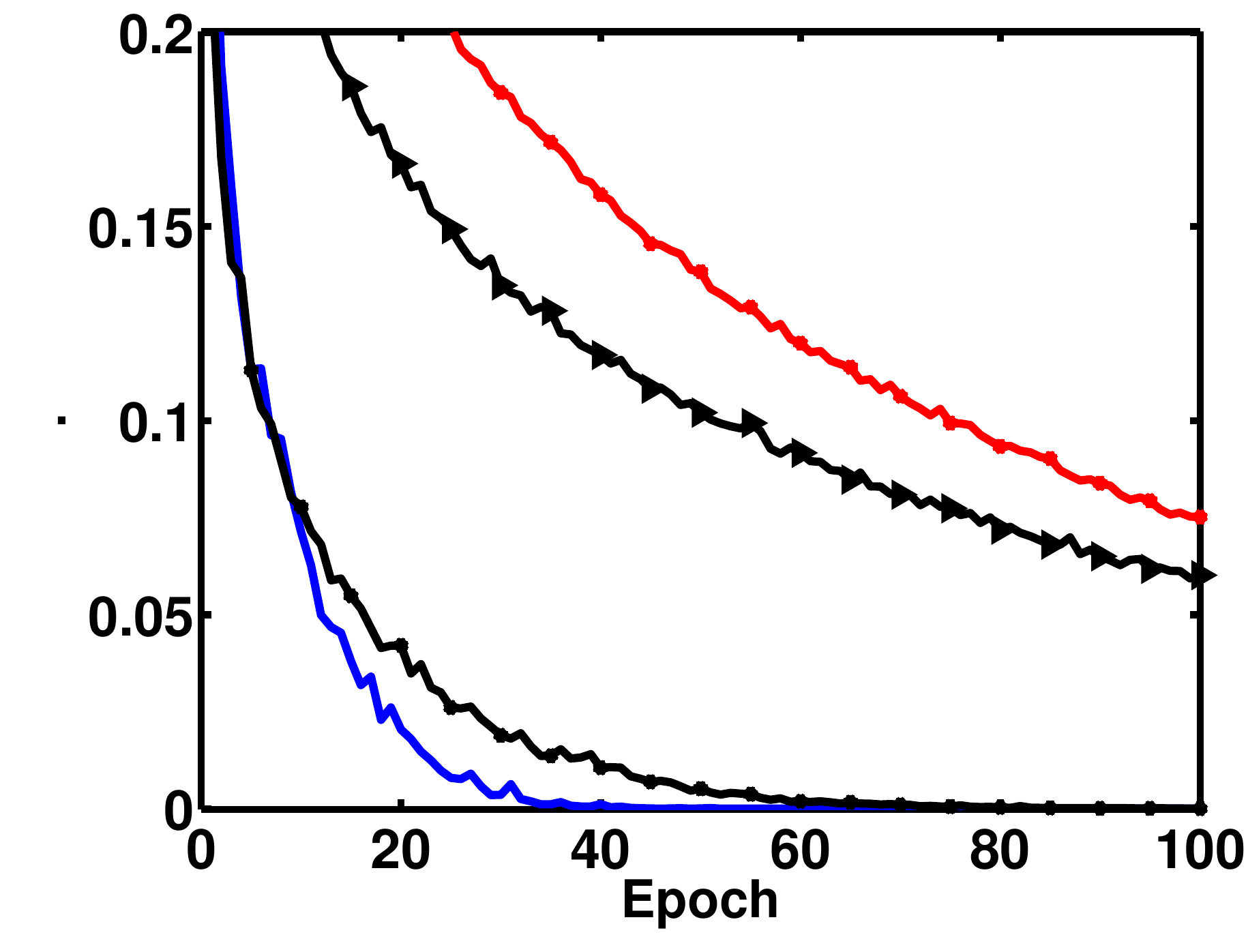}
  \end{tabular}
 }
 \hspace{-0.3in}
 \subfloat{
  \begin{tabular}{r}
   \includegraphics[width=\picwidth]{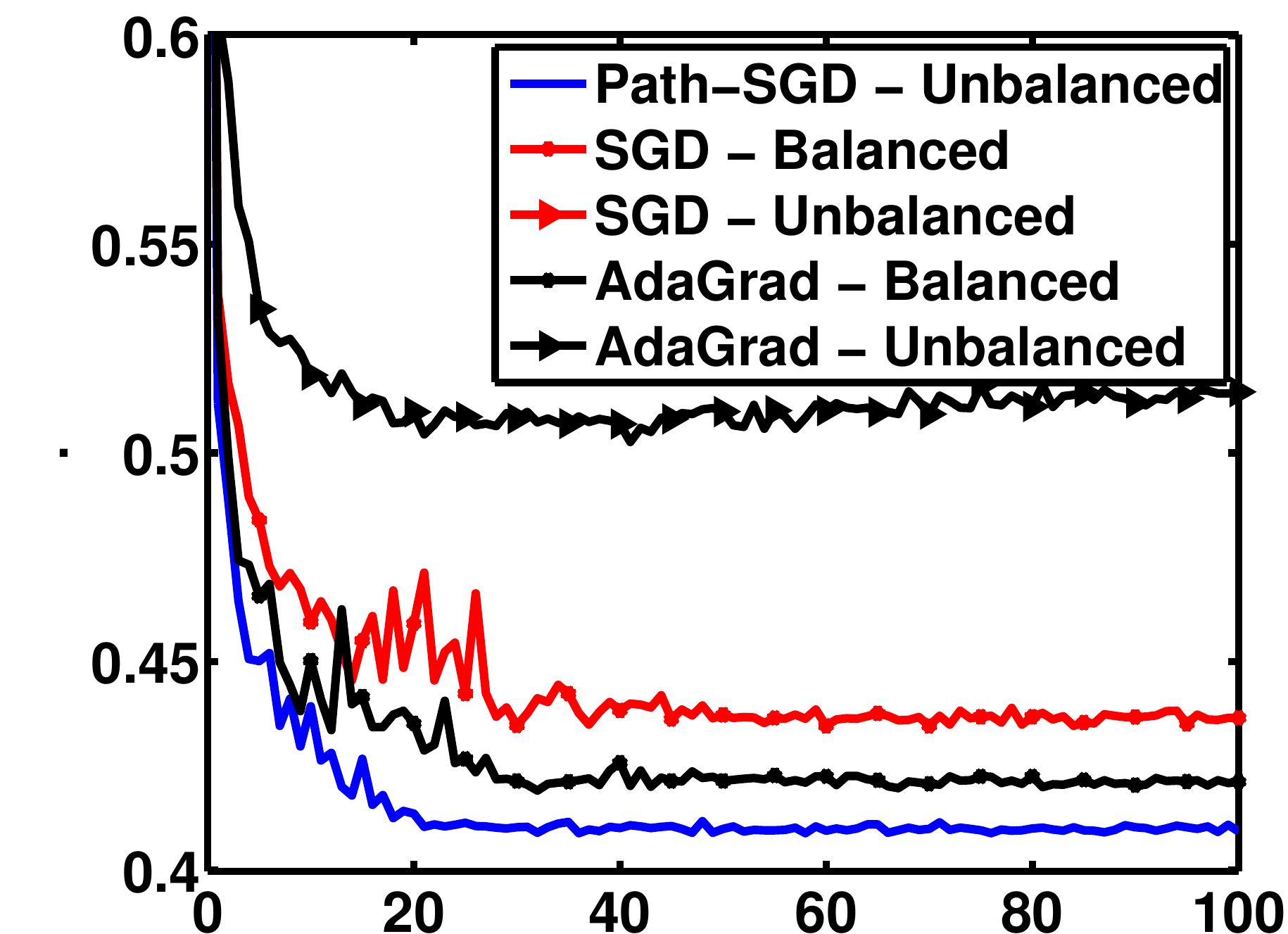} \\
   \includegraphics[width=\picwidth]{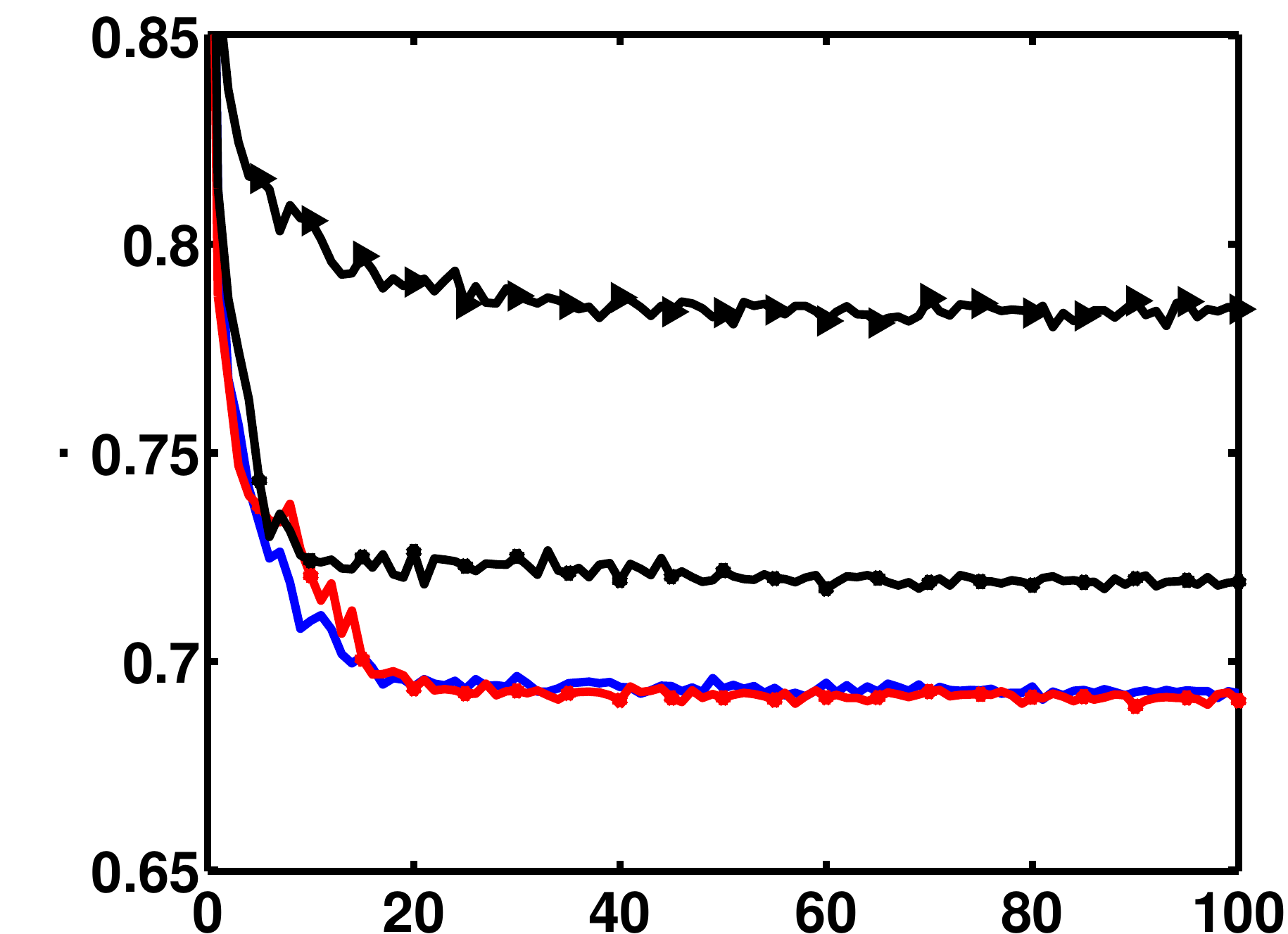} \\
   \includegraphics[width=\picwidth]{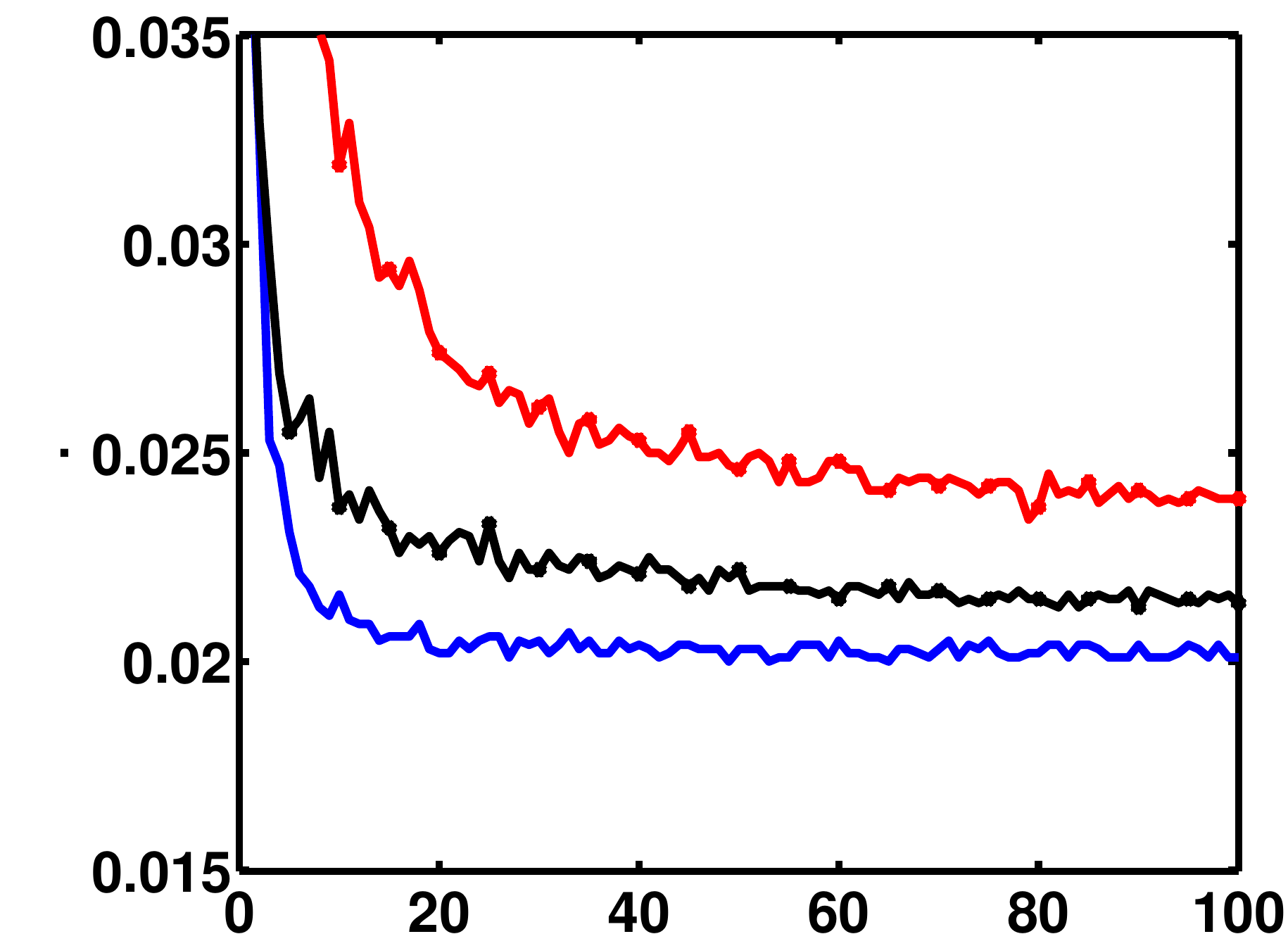} \\
   \includegraphics[width=\picwidth]{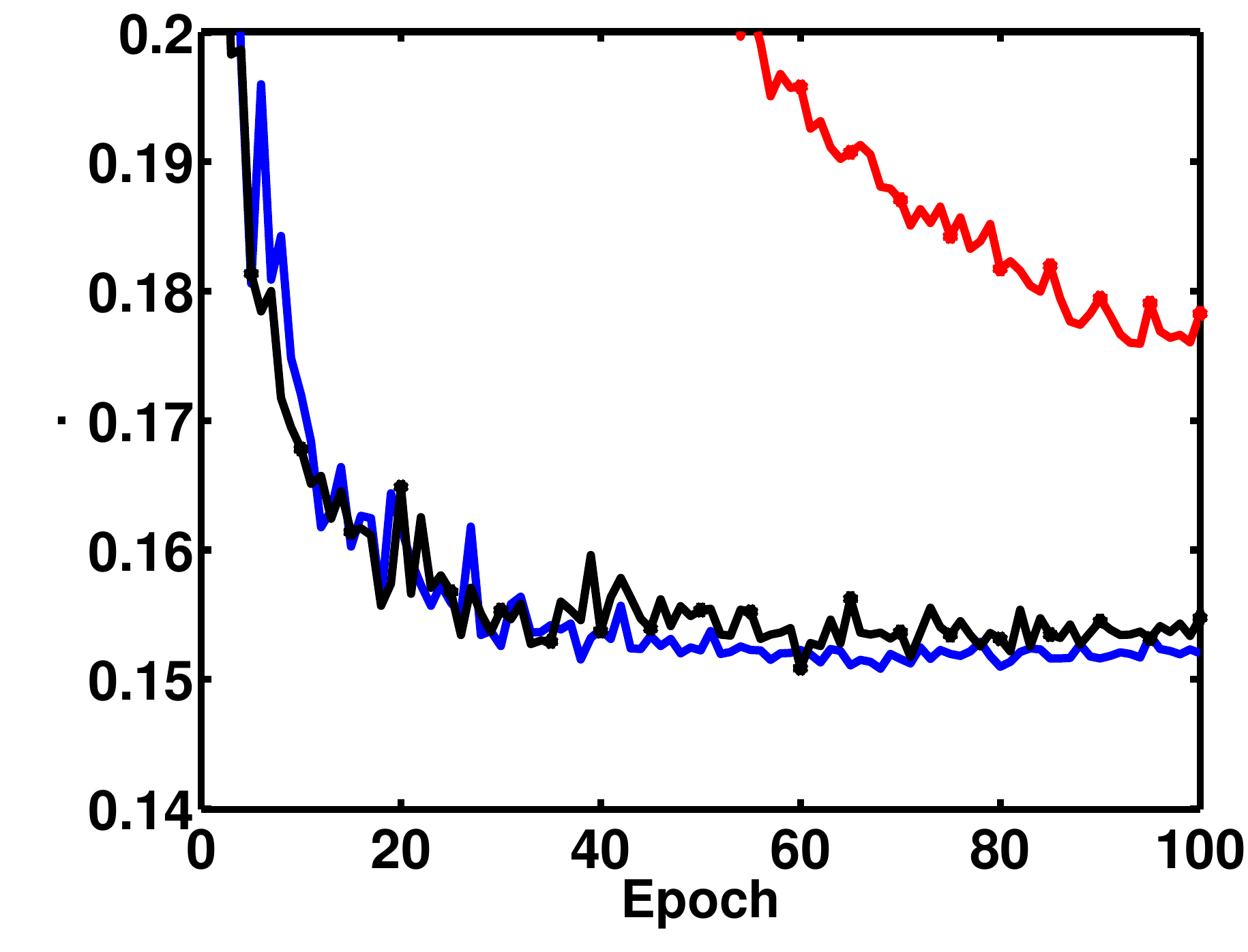}
  \end{tabular}
 }
 
  \begin{picture}(0,0)(0,0)
\rotatebox{90}{\put(342, 0){CIFAR-10}\put(240, 0){CIFAR-100}\put(147, 0){MNIST}\put(50, 0){SVHN}}
\end{picture}
  \begin{picture}(0,0)(0,0)
{\put(30, 420){\small Cross-Entropy Training Loss}\put(187, 420){\small 0/1 Training Error}\put(328, 420){ \small 0/1 Test Error}}
\end{picture}
 \caption[\small Comparing Path-SGD to other optimization methods on 4 dataset without dropout]{\small Learning curves using different optimization methods 
 for 4 datasets without dropout. Left panel displays the cross-entropy objective function; 
middle and right panels show the corresponding values of the training and test errors, where the values are reported on
different epochs during the course of optimization. Best viewed in color.}
 \label{fig:nodrop}
\vspace{-0.1in}
\end{figure}

\begin{figure}[t!]
 \subfloat{
  \begin{tabular}{r}
   \includegraphics[width=\picwidth]{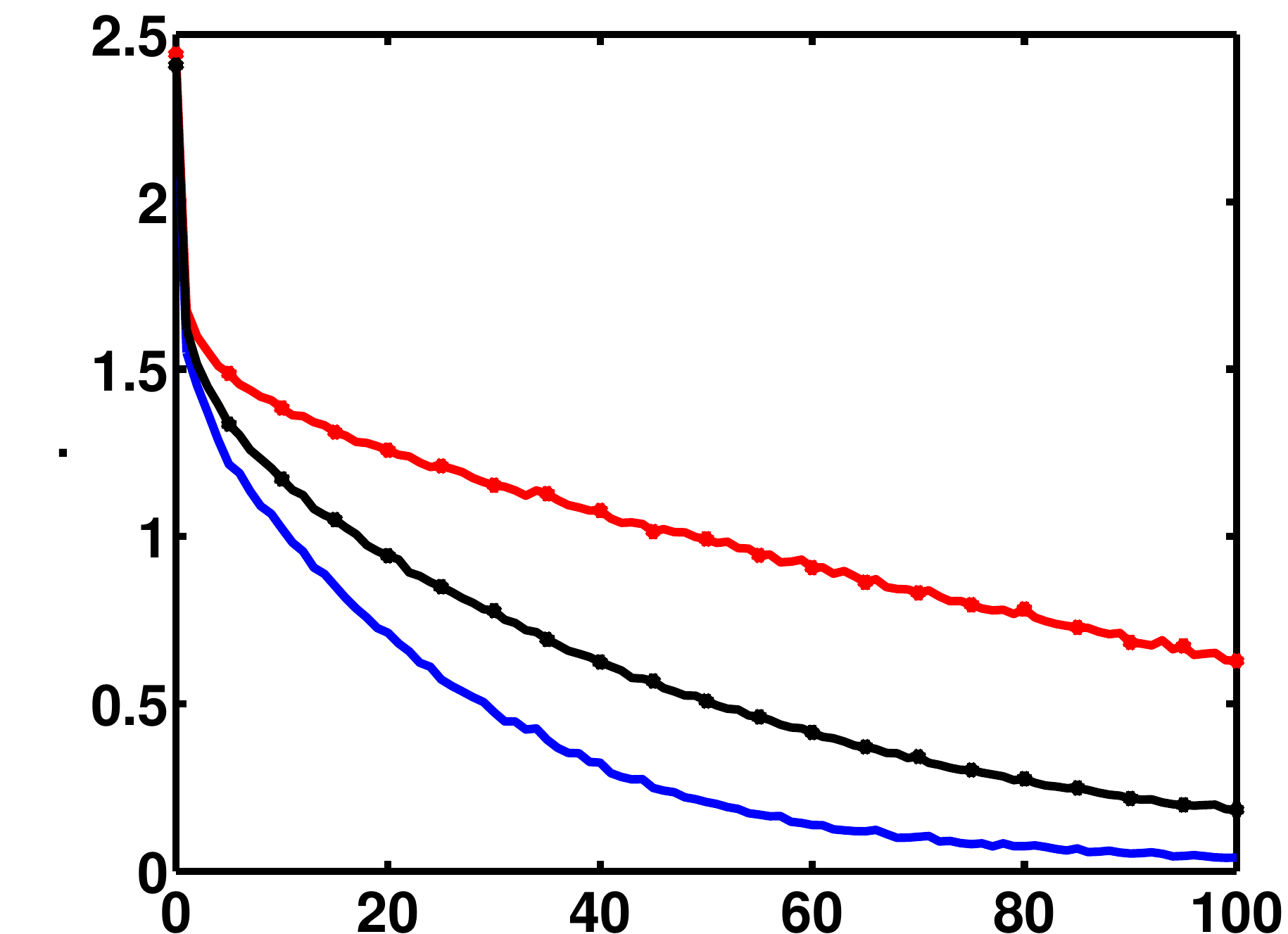} \\
   \includegraphics[width=\picwidth]{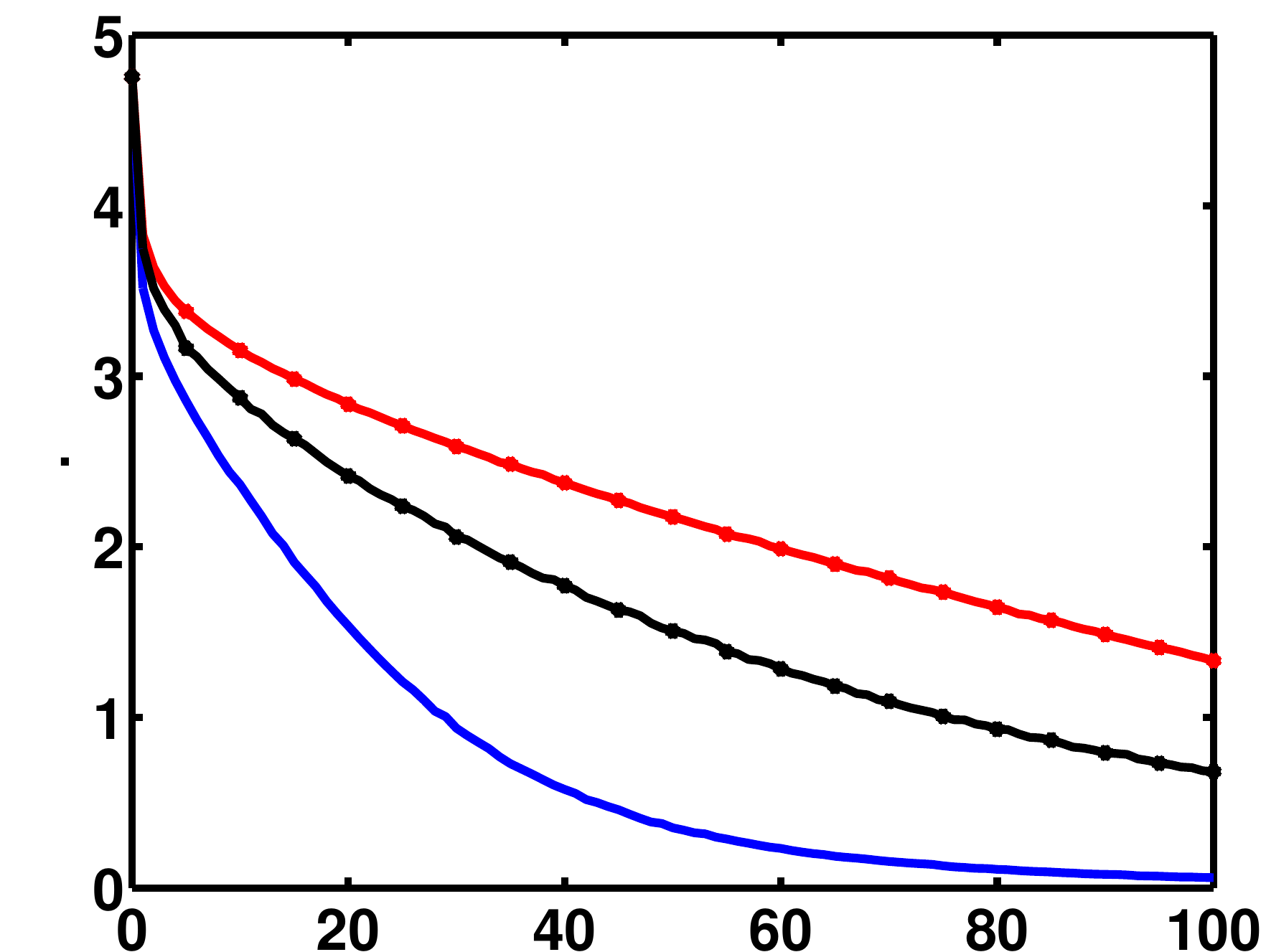} \\
   \includegraphics[width=\picwidth]{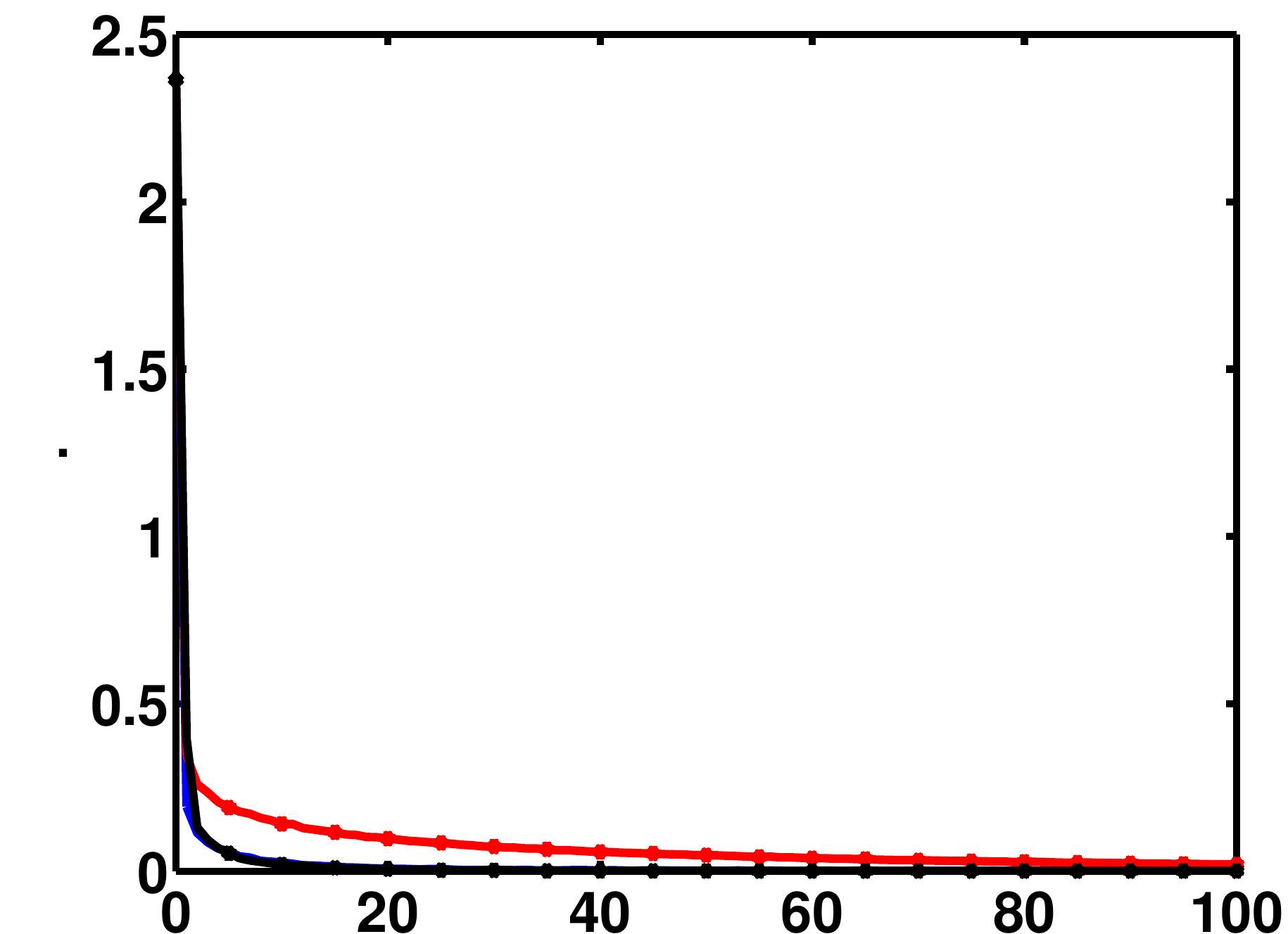} \\
   \includegraphics[width=\picwidth]{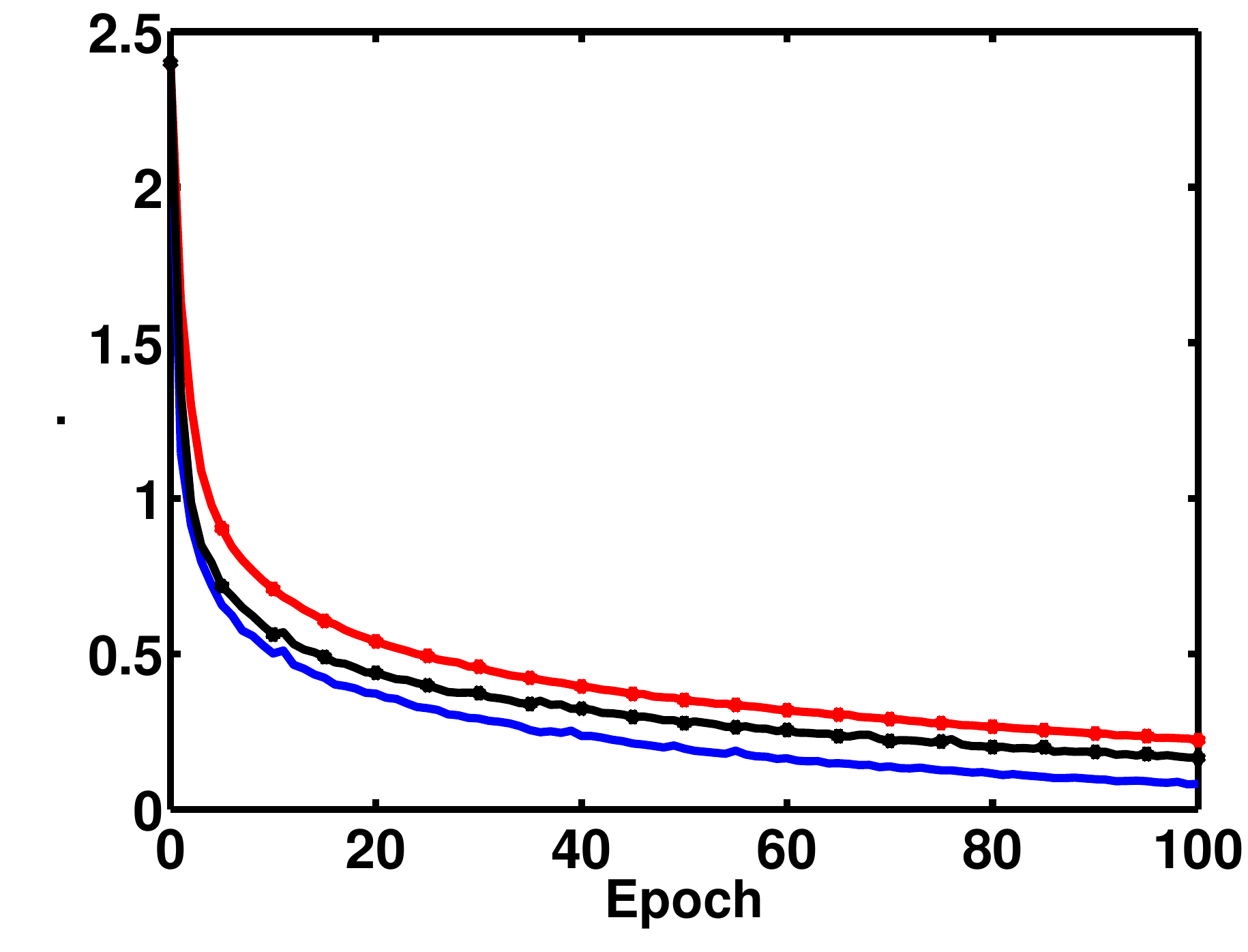}
  \end{tabular}
 }\hspace{-0.3in}
 \subfloat{
  \begin{tabular}{r}
   \includegraphics[width=\picwidth]{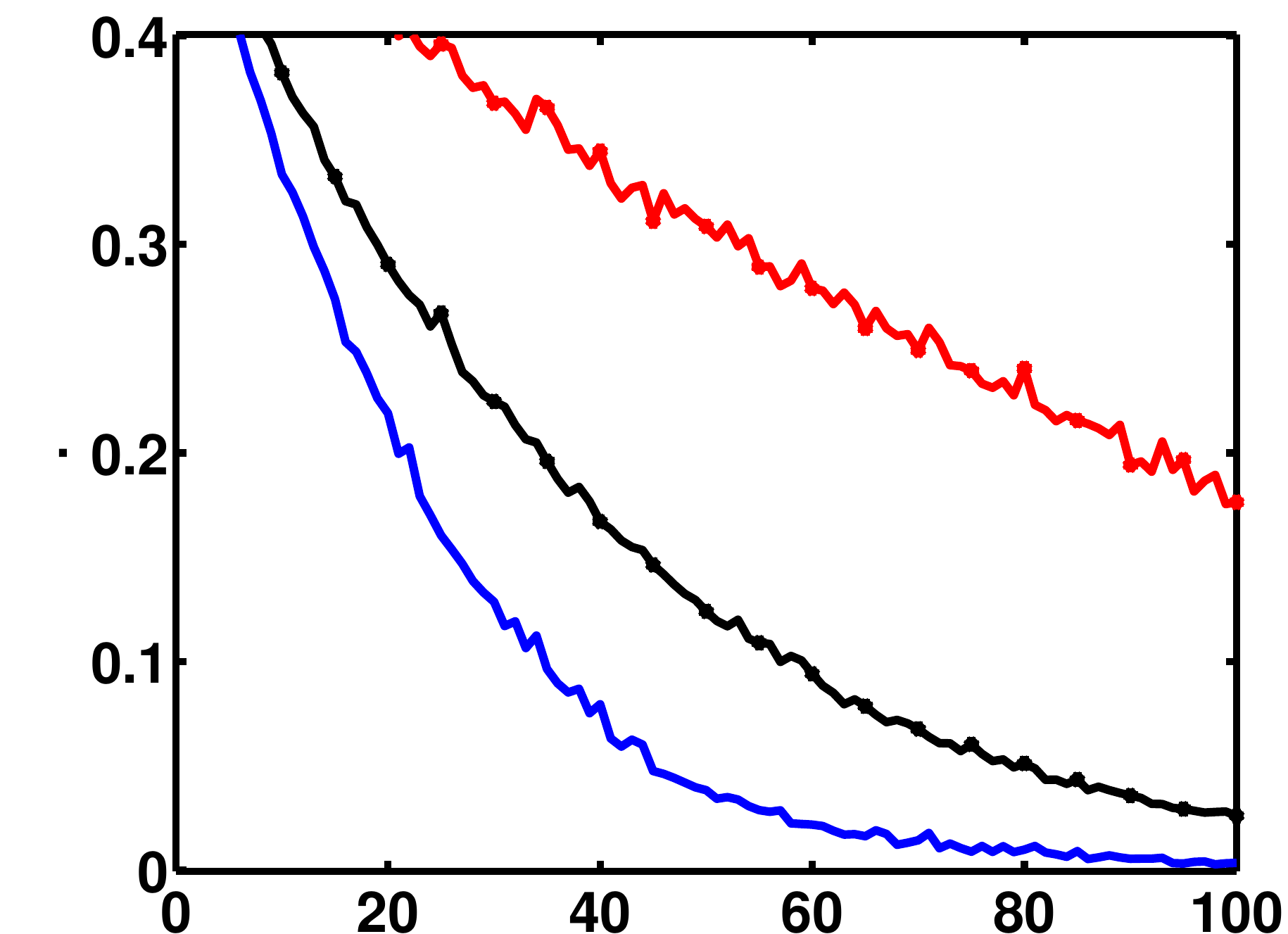} \\
   \includegraphics[width=\picwidth]{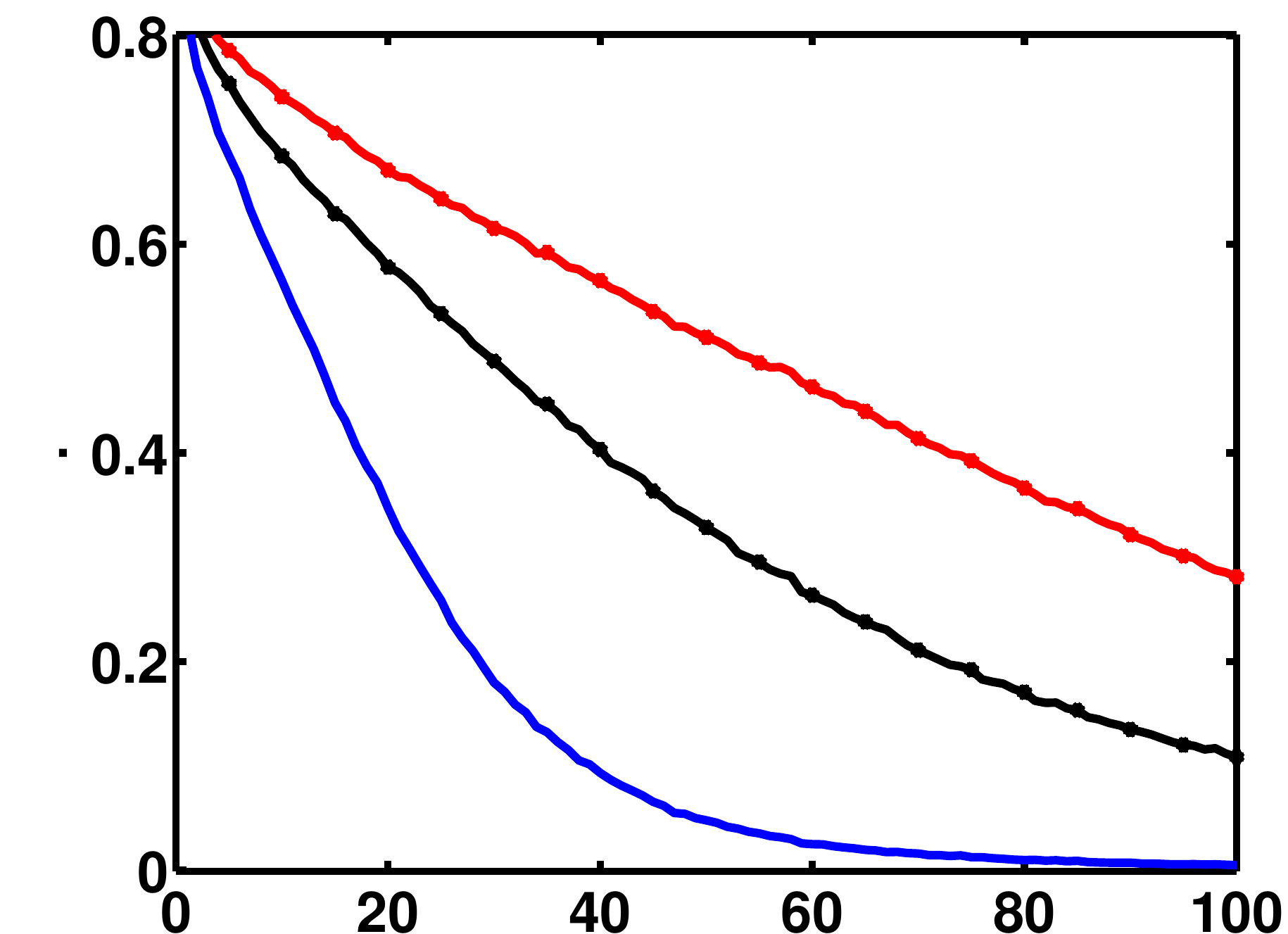} \\
   \includegraphics[width=\picwidth]{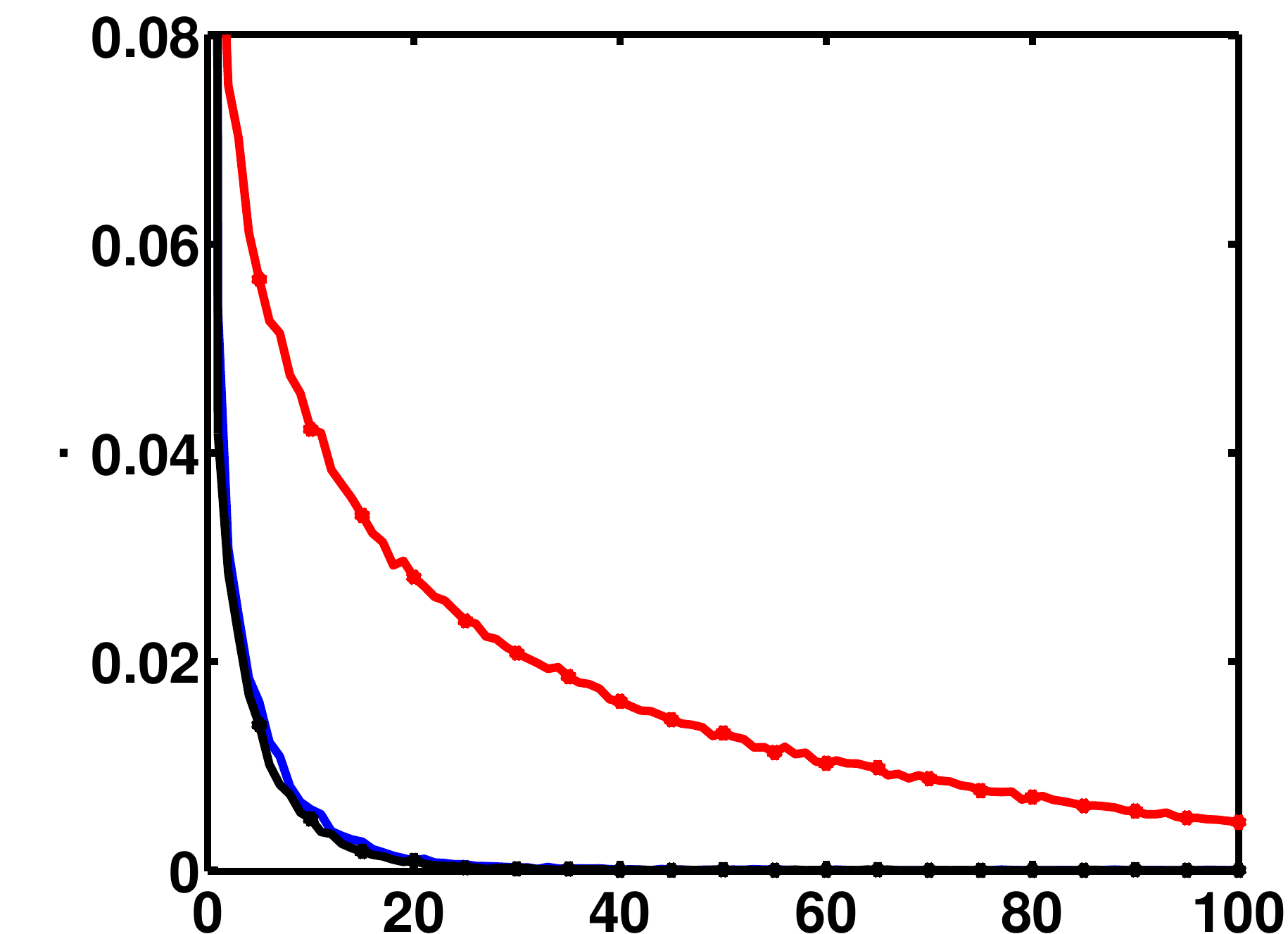} \\
   \includegraphics[width=\picwidth]{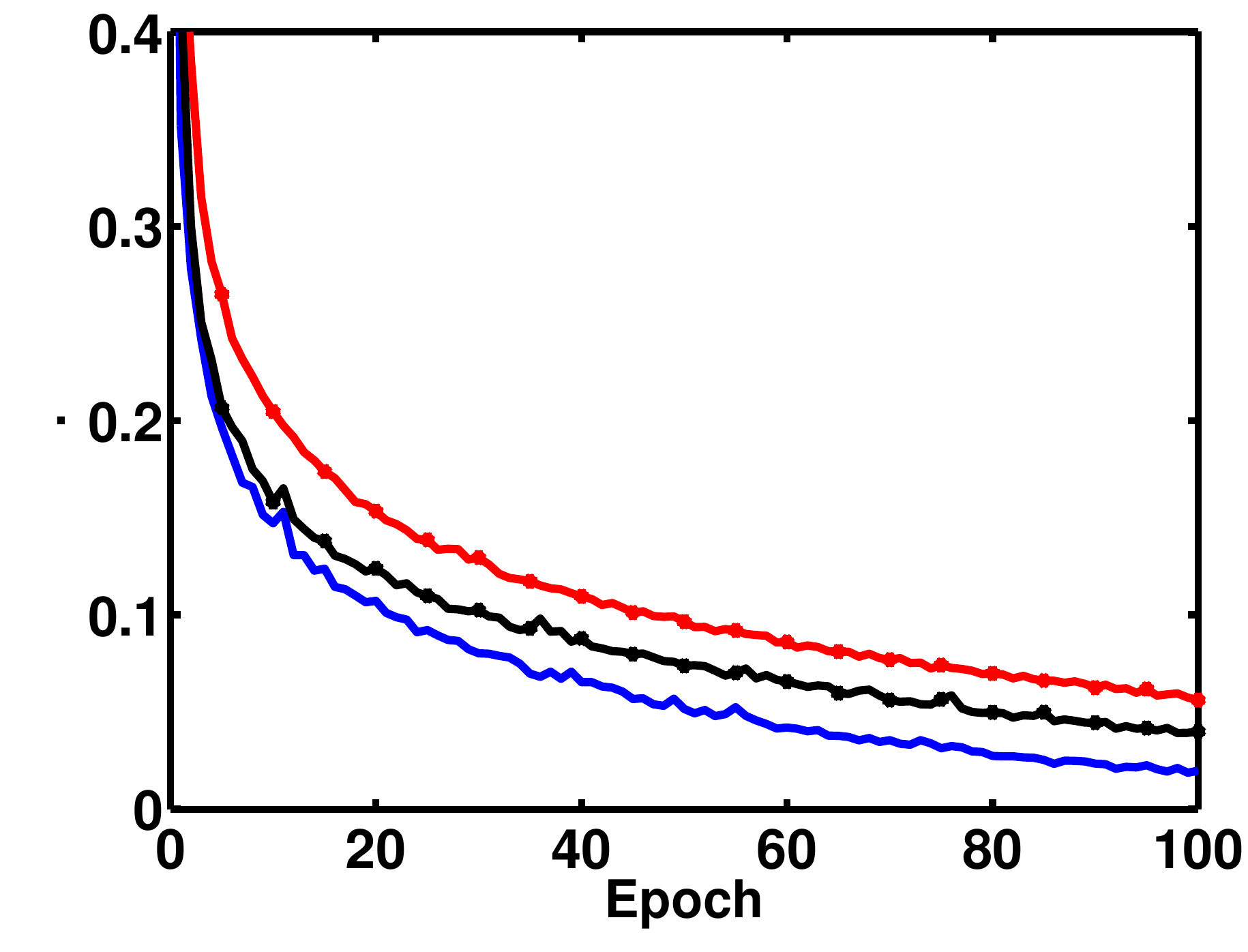}
  \end{tabular}
 }\hspace{-0.3in}
 \subfloat{
  \begin{tabular}{r}
   \includegraphics[width=\picwidth]{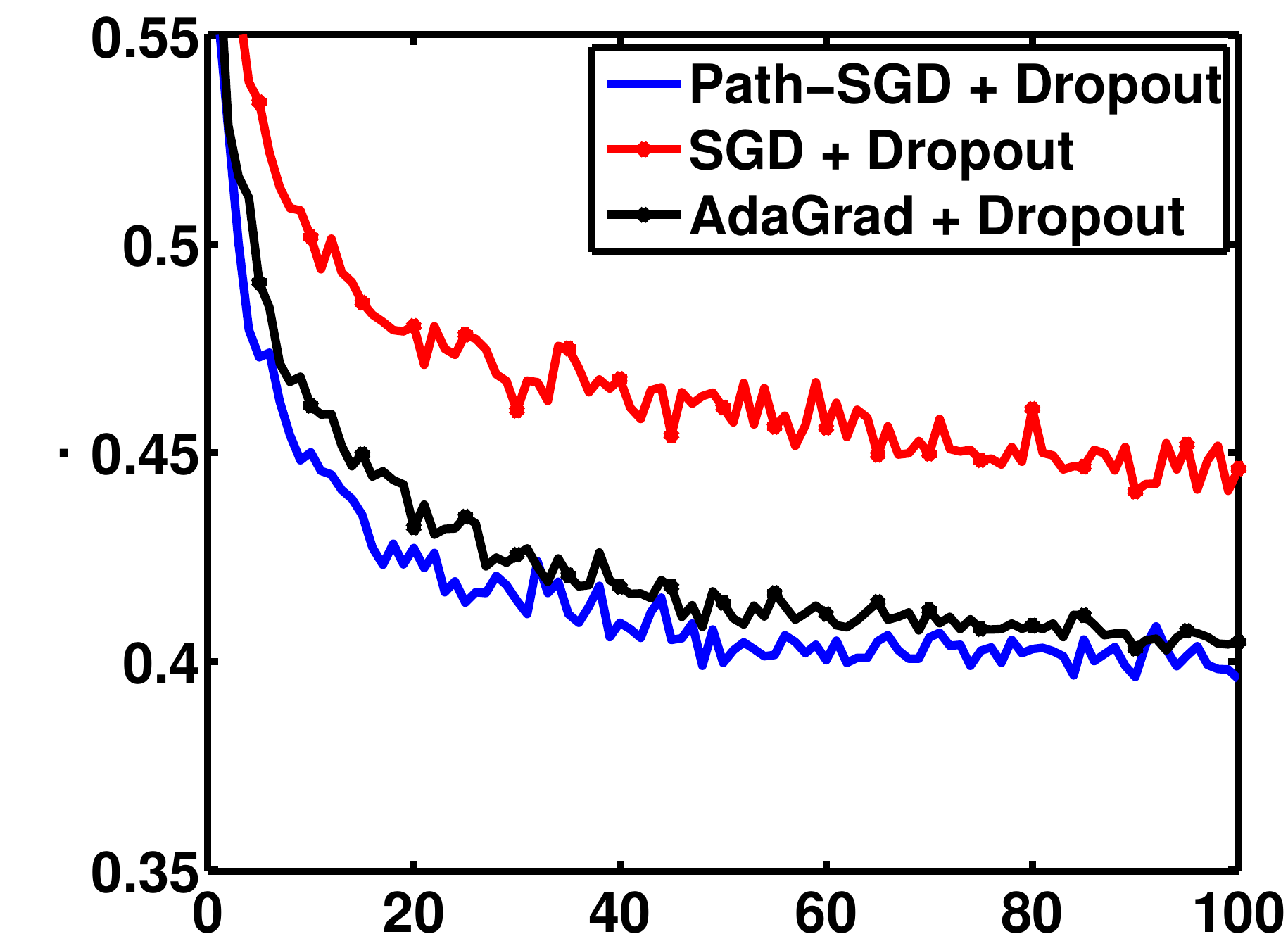} \\
   \includegraphics[width=\picwidth]{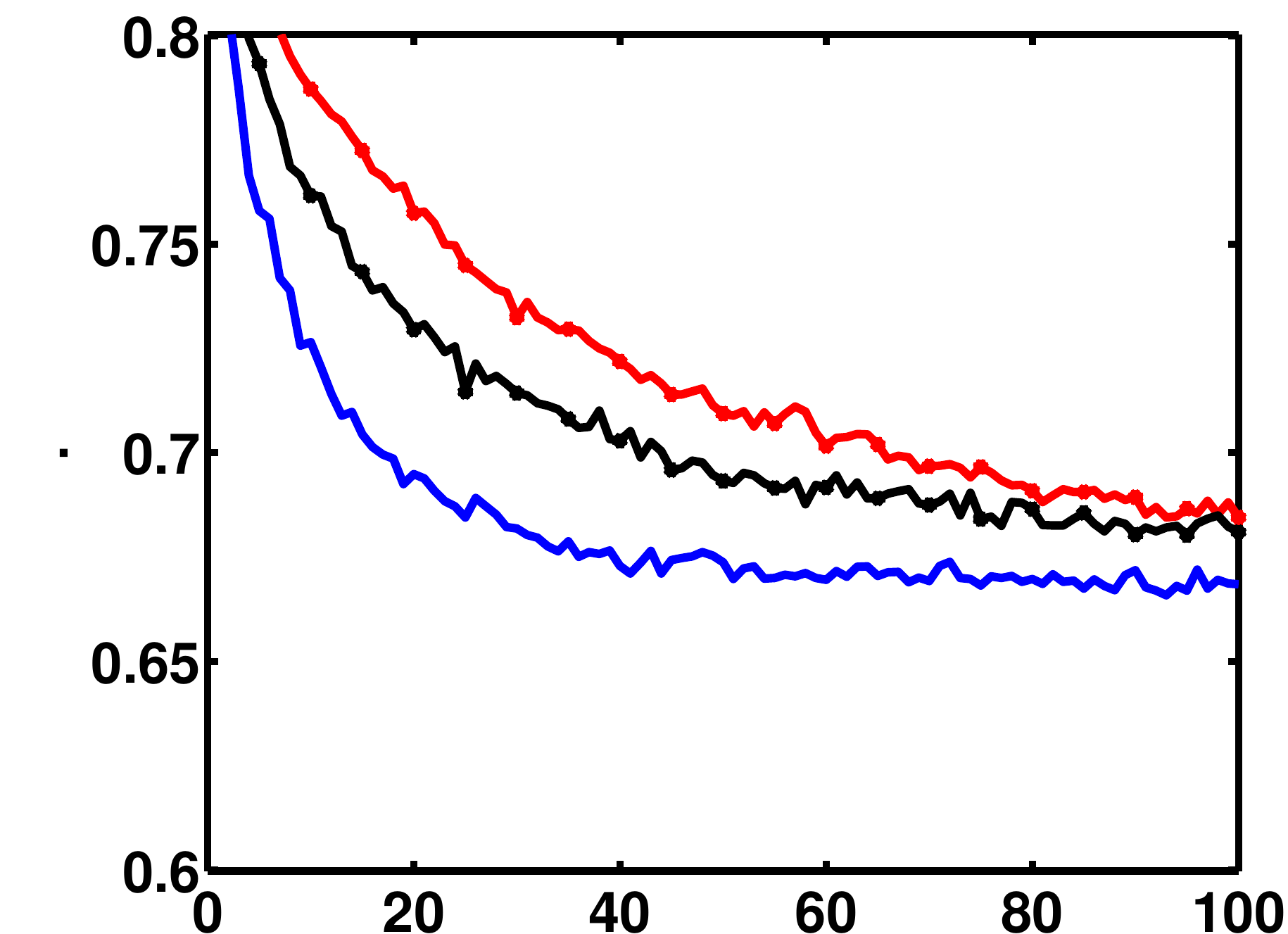} \\
   \includegraphics[width=\picwidth]{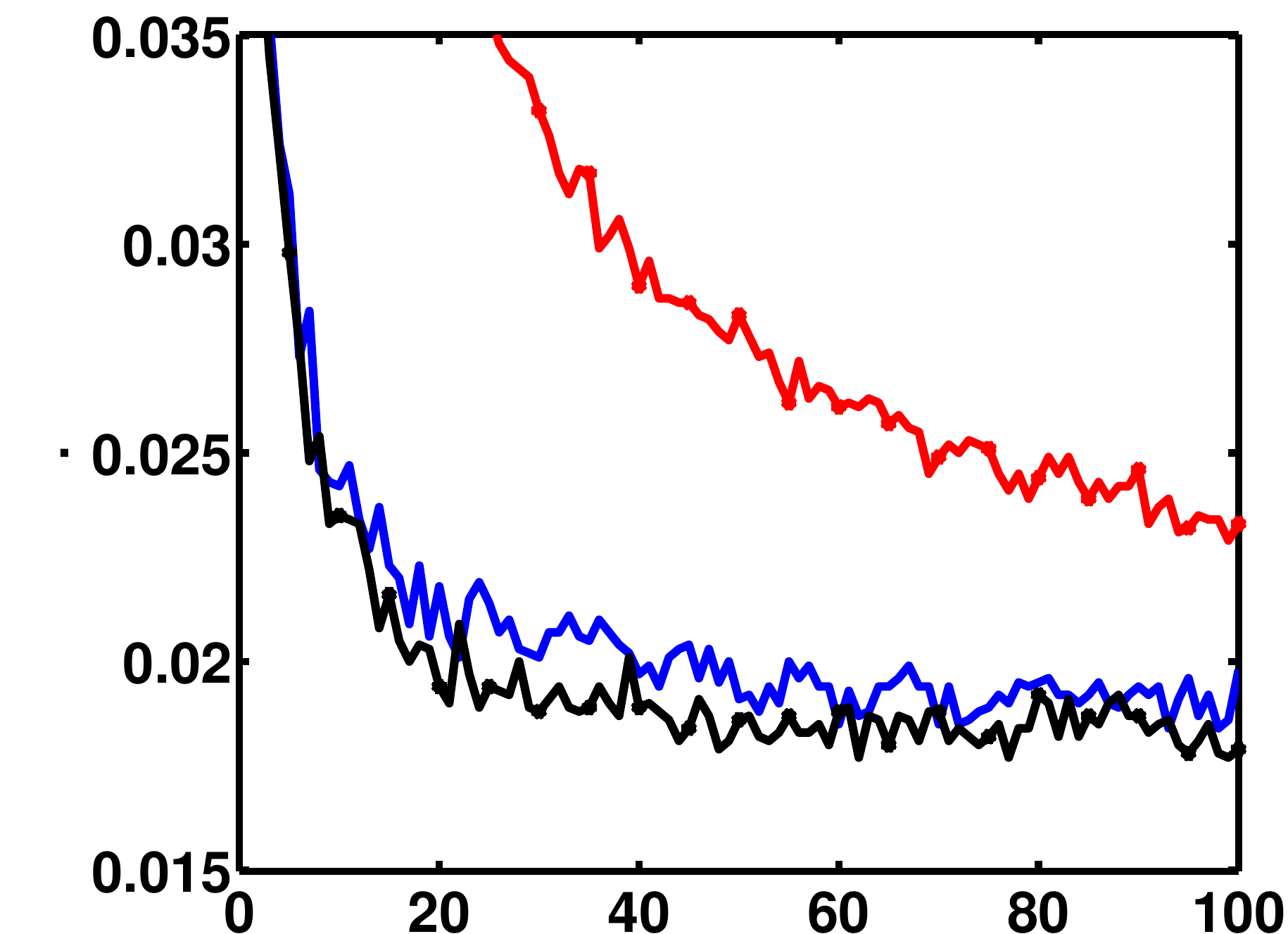} \\
   \includegraphics[width=\picwidth]{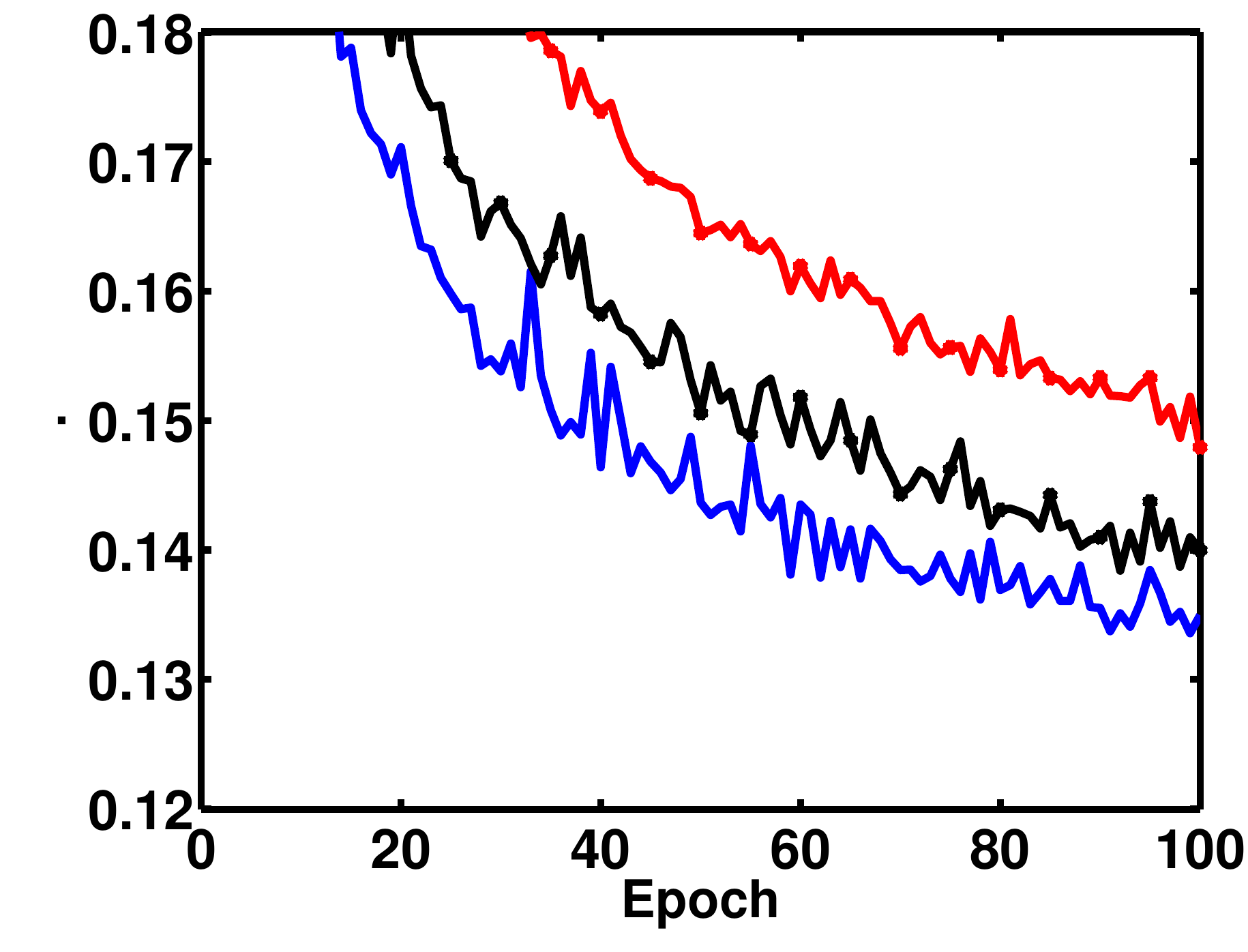}
  \end{tabular}
 }
 
 \begin{picture}(0,0)(0,0)
\rotatebox{90}{\put(342, 0){CIFAR-10}\put(240, 0){CIFAR-100}\put(147, 0){MNIST}\put(50, 0){SVHN}}
\end{picture}
  \begin{picture}(0,0)(0,0)
{\put(30, 420){\small Cross-Entropy Training Loss}\put(187, 420){\small 0/1 Training Error}\put(328, 420){ \small 0/1 Test Error}}
\end{picture}
\vspace{-0.1in}
 \caption[\small Comparing Path-SGD to other optimization methods on 4 dataset with dropout]{\small Learning curves using different optimization methods
 for 4 datasets with dropout. Left panel displays the cross-entropy objective function;     
middle and right panels show the corresponding values of the training and test errors. Best viewed in color.}
 \label{fig:dropout}
\vspace{-0.1in}
\end{figure}

We can see in Figure~\ref{fig:dropout} that as expected, the unbalanced initialization
considerably hurts the performance of SGD and AdaGrad (in many cases
their training and test errors are not even in the range of the plot
to be displayed), while Path-SGD performs essentially the same. Another
interesting observation is that even in the balanced settings, not
only does Path-SGD often get to the same value of objective function, training and test error faster, but also the final generalization error for Path-SGD is sometimes considerably lower than SGD and AdaGrad). The plots for test errors could also imply that implicit regularization due to steepest descent with respect to path-regularizer leads to a solution that generalizes better. This view is similar to observations in \cite{neyshabur2015search} on the role of implicit regularization in deep learning.

The results suggest that Path-SGD outperforms SGD and AdaGrad in two
different ways. First, it can achieve the same accuracy much faster
and second, the implicit regularization by Path-SGD leads to a local
minima that can generalize better even when the training error is
zero. This can be better analyzed
by looking at the plots for more number of epochs which we have
provided in  \cite{neyshabur2015path}. We should also point that
Path-SGD can be easily combined with AdaGrad or Adam to take advantage of the
adaptive stepsize or used together with a momentum term. This could
potentially perform even better compare to Path-SGD.

\section{Experiments on Recurrent Neural Networks}

\subsection{The Contribution of the Second Term}
As we discussed in section \ref{sec:path-compute}, the second term $\kappa^{(2)}$ in the update rule can be computationally expensive for large networks. In this section
we investigate the significance of the second term and show that at least in our experiments, the contribution of the second term is negligible.
To compare the two terms $\kappa^{(1)}$ and $\kappa^{(2)}$, we train a single layer RNN with $H=200$ hidden units for the task of word-level language modeling on Penn Treebank (PTB) Corpus~\cite{marcus1993}. Fig.~\ref{fig:path-vs-sgd} compares the performance of SGD vs. Path-SGD with/without  $\kappa^{(2)}$. We clearly see that both version of Path-SGD are performing very similarly and both of them outperform SGD significantly. This results in Fig.~\ref{fig:path-vs-sgd} suggest that the first term is more significant and therefore we can ignore the second term.

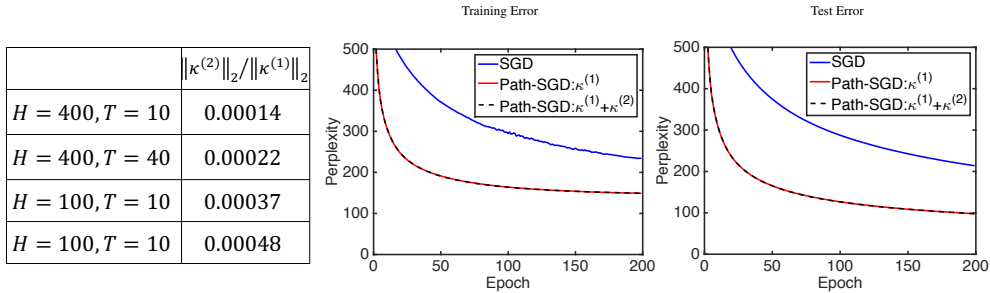
\begin{figure*}[t!]
        \centering
        \includegraphics[height=1.5in]{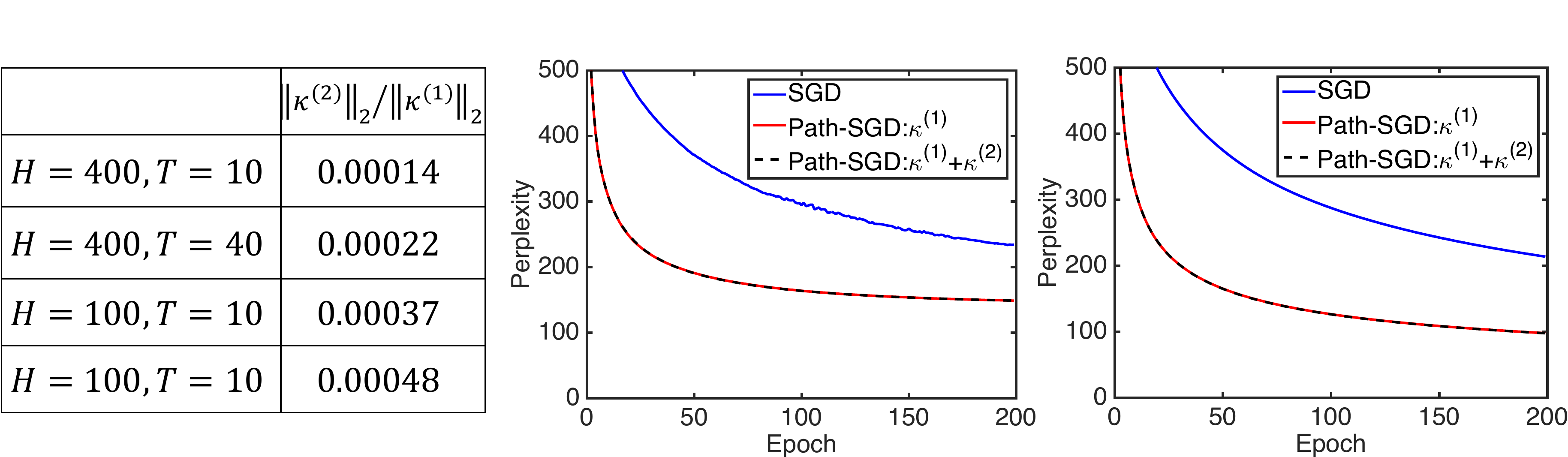}
    \caption[\small The contribution of the second term.]{\small Path-SGD with/without the second term in word-level language modeling on PTB. We use the standard split (929k training, 73k validation and 82k test) and the vocabulary size of 10k words. We initialize the weights by sampling from the uniform distribution with range $[-0.1,0.1]$. The table on the left shows the ratio of magnitude of first and second term for different lengths $T$ and number of hidden units $H$. The plots compare the training and test errors using a mini-batch of size 32 and backpropagating through $T=20$ time steps and using a mini-batch of size 32 where the step-size is chosen by a grid search.}
    \label{fig:path-vs-sgd}
    
\begin{picture}(0,0)(0,0)
{\put(120,190){\tiny Test Error}\put(-12,190){\tiny Training Error}}
\end{picture}
\vspace{-0.2in}
\end{figure*}

To better understand the importance of the two terms, we compared 
the ratio of the norms $\norm{\kappa^{(2)}}_2/\norm{\kappa^{(1)}}_2$ for different RNN lengths $T$ and number of hidden units $H$. 
The table in Fig.~\ref{fig:path-vs-sgd} shows that the contribution of the second term is bigger when the network has fewer number of hidden units and the length of the RNN is larger ($H$ is small and $T$ is large). However, in many cases, it appears that the first term has a much bigger contribution in the update step and hence the second term can be safely ignored. Therefore, in the rest of our experiments, we calculate the Path-SGD updates only using the first term $\kappa^{(1)}$.

\subsection{Addition problem} 
Training Recurrent Neural Networks is known to be hard for modeling long-term dependencies due to the 
gradient vanishing/exploding problem \cite{hochreiter98,bengio1994learning}. In this section, we consider synthetic problems that are specifically designed to test the ability of a model to capture the long-term dependency structure. Specifically, we consider the addition problem and the sequential MNIST problem. 

\textbf{Addition Problem}: The addition problem was introduced in \cite{hochreiter97}. Here,
each input consists of two sequences of length $T$, one of which includes numbers sampled from the uniform distribution with range $[0, 1]$ and the other sequence serves as a mask which is filled with zeros except for two entries.
These two entries indicate which of the two numbers in the first sequence we need to add and the task is to output the result of this addition.\\ 
\textbf{Sequential MNIST}:
In sequential MNIST, each digit image is reshaped into a sequence of length $784$,
turning the digit classification task into sequence classification 
with long-term dependencies~\cite{le15,arjovsky15}.

For both tasks, we closely follow the experimental protocol in \cite{le15}. 
We train a single-layer RNN consisting of 100 hidden units with path-SGD, referred to as \textbf{RNN-Path}. We also train an RNN of the same size with identity initialization, as was proposed in~\cite{le15}, using SGD as our baseline model, referred to as \textbf{IRNN}. We performed grid search for the
learning rates over $\{10^{-2},10^{-3},10^{-4}\}$ for both our model and the baseline.  Non-recurrent weights were initialized from the uniform distribution with range $[-0.01,0.01]$.
Similar to~\cite{arjovsky15}, we found the IRNN to be fairly unstable (with SGD optimization typically diverging). Therefore for IRNN, we ran 10 different initializations and picked the one that did not explode to show its performance. 

In our first experiment, we evaluate Path-SGD on the addition problem. 
The results are shown in Fig.~\ref{fig:adding} with increasing the length $T$ of the sequence: 
$\{100,400,750\}$. We note that this problem becomes much harder as $T$ increases because the dependency between
the output (the sum of two numbers) and the corresponding inputs becomes more distant. 
We also compare RNN-Path with the previously published results, 
including identity initialized RNN ~\cite{le15} (IRNN), unitary RNN \cite{arjovsky15} (uRNN), and np-RNN\footnote{The original paper does not include any result for 750, so we implemented np-RNN for comparison. However, in our implementation the np-RNN is not able to 
even learn sequences of length of 200. Thus we put ``>2'' for length of 750.} introduced by \cite{talathi16}. Table \ref{tb:adding_mnist} shows the effectiveness of using Path-SGD. Perhaps 
more surprisingly,
with the help of path-normalization, a simple RNN with the identity initialization is able to achieve a 0\% error on the sequences of length 750, 
whereas all the other methods, including LSTMs, fail. This shows that Path-SGD may help stabilize the training and alleviate the gradient problem, so as to perform well on longer sequence. 
We next tried to model the sequences length of 1000, but we found that for such very long 
sequences RNNs, even with Path-SGD, fail to learn. 

\begin{figure}[t]
\centering
\includegraphics[width=15cm]{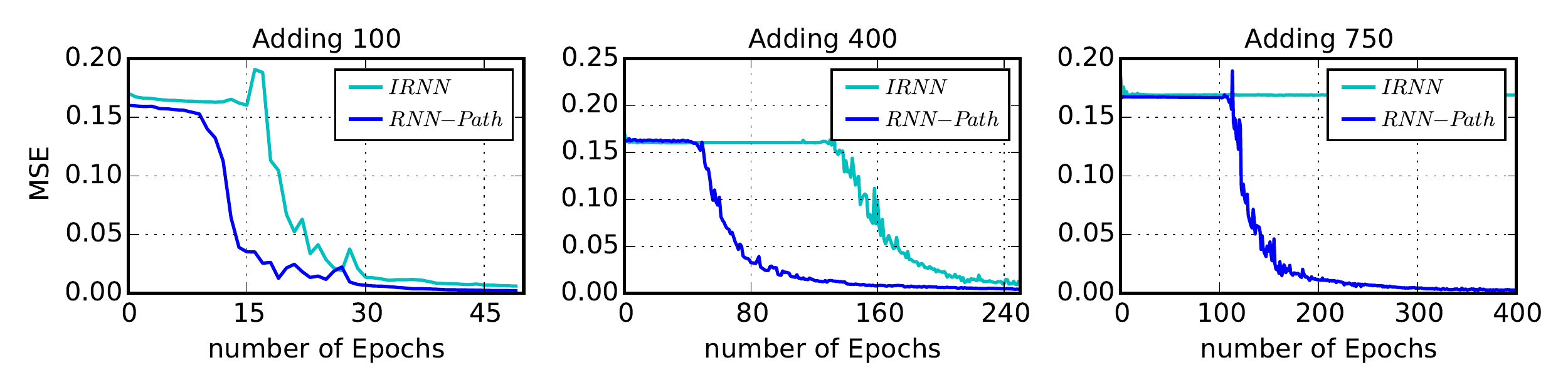}
\caption{\small Test errors for the addition problem of different lengths.}
\label{fig:adding}
\end{figure}

\begin{table}[t!]
    \begin{minipage}{.5\linewidth}
      \centering
      \footnotesize
        \begin{tabular}{c c c c c}
    \toprule
    &Adding & Adding& Adding  & \\
        & 100 & 400& 750  &   sMNIST\\
    \midrule
    IRNN~\cite{le15}& 0& 16.7& 16.7 & 5.0\\
    uRNN \cite{arjovsky15} & 0 & 3 & 16.7 & 4.9 \\
    LSTM \cite{arjovsky15}  & 0 & 2 & 16.7 & 1.8\\
    np-RNN\cite{talathi16}        & 0 & 2 & >2 & 3.1\\
    \midrule
    IRNN & 0 & 0 & 16.7  &7.1\\
    RNN-Path & 0& 0 & 0 &  3.1\\
    \bottomrule
  \end{tabular}
        \caption[\small Test error (MSE) for the adding problem and test classification error for the sequential MNIST.]{\small Test error (MSE) for the adding problem with different input sequence lengths and test classification error for the sequential MNIST.}
        \label{tb:adding_mnist}
      
    \end{minipage}%
    \;\;
    \begin{minipage}{.5\linewidth}
    \footnotesize
      \centering
       \begin{tabular}{lcc}
    \toprule
        & PTB & text8 \\
    \midrule
    RNN+smoothReLU \cite{pachitariu2013regularization} & - & 1.55\\
    HF-MRNN \cite{mikolov2012subword}    &  1.42 & 1.54\\
    RNN-ReLU\cite{KruegerM15} & 1.65 & -\\
    RNN-tanh\cite{KruegerM15} & 1.55 & -\\
    TRec,$\beta=500$\cite{KruegerM15} & 1.48 & -\\
    \midrule
    RNN-ReLU & 1.55 & 1.65\\
    RNN-tanh & 1.58 & 1.70\\
    RNN-Path & 1.47 & 1.58\\
    LSTM  & 1.41 & 1.52\\
    \bottomrule
  \end{tabular}
        \caption{\small Test BPC for PTB and text8.}
        \label{tb:bpc_penn_text8}
    \end{minipage} 
\end{table}

Next, we evaluate Path-SGD on the Sequential MNIST problem. 
Table \ref{tb:adding_mnist}, right column, reports test error rates achieved by RNN-Path 
compared to the previously published results. Clearly,
using Path-SGD helps RNNs achieve better generalization. In many cases,
RNN-Path outperforms other RNN methods (except for LSTMs), even for such a long-term dependency problem.

\subsection{Language Modeling Tasks}
In this section we evaluate Path-SGD on a language modeling task. 
We consider two datasets, Penn Treebank (PTB-c) and text8~\footnote{http://mattmahoney.net/dc/textdata}.
\textbf{PTB-c}: We performed experiments on a tokenized Penn Treebank Corpus, following 
the experimental protocol of~\cite{KruegerM15}. The training, validations and test data contain 5017k, 393k and 442k characters respectively.
The alphabet size is 50, and each training sequence is of length 50. \textbf{text8}: The text8 dataset contains 100M characters from Wikipedia with an alphabet size of 27. We follow the data partition of~\cite{mikolov2012subword}, 
where each training sequence has a length of 180.
Performance is evaluated using bits-per-character (BPC) metric, which is $\log_2$ of perplexity. 

Similar to the experiments on the synthetic datasets, 
for both tasks, we train a single-layer RNN consisting of 2048 hidden units with path-SGD (RNN-Path).
Due to the large dimension of hidden space, SGD can take a fairly long time to converge. 
Instead, we use Adam optimizer~\cite{kingma2014adam} to help speed up the training, where 
we simply use the path-SGD gradient as input to the Adam optimizer. 

We also train three additional baseline models: a ReLU RNN with 2048 hidden units, a tanh RNN with 2048 hidden units, 
and an LSTM with 1024 hidden units, 
all trained using Adam. We performed grid search for learning rate over $\{10^{-3},5\cdot10^{-4},10^{-4}\}$ for all of our models.  For ReLU RNNs, we initialize the recurrent matrices 
from uniform$[-0.01,0.01]$, and uniform$[-0.2,0.2]$ for non-recurrent weights. 
For LSTMs, we use orthogonal initialization~\cite{saxe2013exact} for the recurrent matrices and 
uniform$[-0.01,0.01]$ for non-recurrent weights. The results are summarized in Table \ref{tb:bpc_penn_text8}.

\removed{
\begin{table}[t]
  \footnotesize
  \centering
  \begin{tabular}{ll}
    \toprule
    Penn-Treebank     & BPC     \\
    \midrule
    RNN+smoothReLU \cite{pachitariu2013regularization} & - & 1.55\\
    HF-MRNN \cite{mikolov2012subword}    &  1.42 & 1.54\\
    RNN-ReLU\cite{KruegerM15} & 1.65 & -\\
    RNN-tanh\cite{KruegerM15} & 1.55 & -\\
    TRec,$\beta=500$\cite{KruegerM15} & 1.48 & -\\
    \midrule
    RNN-ReLU & 1.55 & 1.65\\
    RNN-tanh & 1.58 & 1.70\\
    RNN-Path & 1.47 & 1.58\\
    LSTM  & 1.41 & 1.52\\
    \bottomrule
  \end{tabular}
  \quad
  \begin{tabular}{ll}
    \toprule
    text8     & BPC     \\
    \midrule
    RNN+smoothReLU \cite{pachitariu2013regularization}  &1.55  \\
    HF-MRNN \cite{mikolov2012subword}    &  1.54\\
    \midrule
    RNN-ReLU & 1.65 \\
    RNN-tanh & 1.70 \\
    RNN-Path & 1.58 \\
    LSTM  & 1.52\\
    \bottomrule
  \end{tabular}
\caption{\small Left: test Bits Per characters on PTB.
Right: test Bits Per characters on text8.
}
\label{tb:bpc_penn_text8}
\end{table}

}

We also compare our results to an RNN that uses 
hidden activation regularizer~\cite{KruegerM15} (TRec,$\beta=500$), 
Multiplicative RNNs trained by Hessian Free methods~\cite{mikolov2012subword}
(HF-MRNN), and an RNN with smooth version of ReLU \cite{pachitariu2013regularization}.
Table~\ref{tb:bpc_penn_text8} shows that path-normalization is able to 
outperform RNN-ReLU and RNN-tanh,
while at the same time 
shortening the performance gap between plain RNN and other more complicated models (e.g. LSTM by 57\% on PTB and 54\% on text8 datasets).
This demonstrates the efficacy of path-normalized optimization for training RNNs with ReLU activation.

\chapter{Data-Dependent Path Normalization} \label{chap:data-dependent}

In this chapter, we focus on two efficient alternative optimization
approaches proposed recently for feed-forward neural networks that are
based on intuitions about parametrization, normalization and the
geometry of parameter space: {\bf Path-SGD} \citep{NeySalSre15} was
derived as steepest descent algorithm with respect to particular
regularizer (the $\ell_2$-path regularizer, i.e.~the sum over all
paths in the network of the squared product over all weights in the
path ~\citep{NeyTomSre15}) and is invariant to weight
reparametrization.  {\bf Batch-normalization} \citep{IofSze15} was
derived by adding normalization layers in the network as a way of
controlling the variance of the input each unit receives in a
data-dependent fashion.  In this chapter, we propose a unified framework
which includes both approaches, and allows us to obtain additional
methods which interpolate between them.  Using our unified framework,
we can also tease apart and combine two different aspects of these two
approaches: data-dependence and invariance to weight reparametrization.

Our unified framework is based on first choosing a per-node complexity
measure we refer to as $\gamma_v$ (defined in Section
\ref{sec:unified}).  The choice of complexity measure is parametrized
by a choice of ``normalization matrix'' $R$, and different choices for
this matrix incorporate different amounts of data dependencies: for
path-SGD, $R$ is a non-data-dependent diagonal matrix, while for batch
normalization it is a data-dependent covariance matrix, and we can
interpolate between the two extremes.  Once $\gamma_v$ is defined, and
for any choice of $R$, we identify two different optimization
approaches: one relying on a normalized re-parameterization at each
layer, as in batch normalization (Section \ref{sec:bn}), and the other
an approximate steepest descent as in path-SGD, which we refer to as
DDP-SGD (Data Dependent Path SGD) and can be implemented efficiently
via forward and backward propagation on the network (Section
\ref{sec:ddp-sgd}).  We can now mix and match between the choice of
$R$ (i.e.~the extent of data dependency) and the choice of
optimization approach.

One particular advantage of the approximate steepest descent approach
(DDP-SGD) over the normalization approach is that it is invariant to
weight rebalancing (discussed in Section \ref{sec:node-rescaling}).  
This is true
regardless of the amount of data-dependence used.  That is, it
operates more directly on the model (the function defined by the
weights) rather than the parametrization (the values of the weights
themselves).  This brings us to a more general discussion of
parametrization invariance in feedforward networks (Section 
\ref{sec:rescaling}).

Our unified framework and study of in invariances also allows us to
relate the different optimization approaches to Natural Gradients
\citep{Ama98}. In particular, we show that DDP-SGD with full
data-dependence can be seen as an efficient approximation of the
natural gradient using only the diagonal of the Fisher information
matrix (Section \ref{sec:natural}).

\paragraph{Notation}
This chapter requires more involved notation that is slightly different that the notation
of the rest of the dissertation. The notation is summarized in Figure~\ref{fig:notation}.
\begin{figure}[tb]
\begin{center}
\hspace{-0.3in}
\includegraphics[width=1\textwidth]{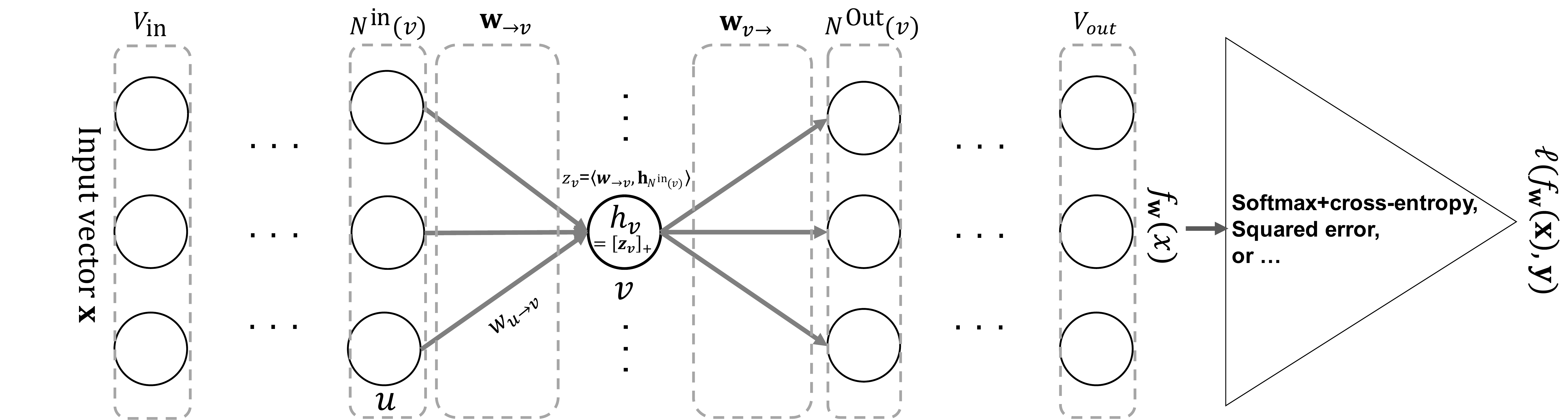}\\
\vspace{0.1in}
\small
\setlength\tabcolsep{4pt}
\begin{tabular}[tb]{|c|c|c|c|}
\hline
Symbol & Meaning &  Symbol
& Meaning\\
\hline
$\vecX$ / $\vec y$& input vector / label&  $V_\In$ / $V_\Out$& the set of input / output nodes\\
\hline
$\vecw$ & the parameter vector& $w_{u\rightarrow v}$
& the weight of the edge $(u\rightarrow v)$\\
\hline
$\vecw_{v\rightarrow}$ & the vector of incoming  weights to $v$ & $\vecw_{\rightarrow v}$
& the vector of outgoing  weights from $v$\\
\hline
$N^{\In}(v)$ & the set of nodes feeding into $v$ & $N^{\Out}(v)$
& the set of nodes that $v$ feeds into\\
\hline
$h_v$ & the output value of node $v$ & $z_v$
& the activation value of node $v$\\
\hline
\end{tabular}
  \caption{\small An example of layered feedforward network and notation used in the chapter}
  \label{fig:notation}
 \end{center}
\end{figure}

\paragraph{Related Works}
There has been an ongoing effort for better understanding of the 
optimization in 
deep networks and several heuristics have been suggested to improve 
the training
 \citep{lecun-98x, 
larochelle2009exploring,difficulty,sutskever2013importance}.
Natural gradient algorithm \citep{Ama98} is known to have a very strong
invariance property; it is not only invariant to reparametrization, but
also to the choice of network architecture. However it is known to
be computationally demanding and thus many approximations have been
proposed \citep{grosse2015scaling,martens2015optimizing,desjardins2015natural}.
However, such approximations make the algorithms less invariant than the
original natural gradient algorithm. \cite{pascanu2013revisiting} also discuss the connections between 
Natural
Gradients and some of the other proposed methods for training neural
networks, namely Hessian-Free Optimization \citep{martens2010deep}, 
Krylov
Subspace Descent \citep{vinyals2011krylov} and TONGA
\citep{roux2008topmoumoute}.

\cite{ollivier2015riemannian} also recently studied the issue of invariance and proposed computationally efficient approximations and alternatives to natural gradient. They study invariances as different mappings from parameter space to the same function space  while we look at the invariances as transformations (inside a fixed parameter space) to which the function is invariant in the model space (see Section~\ref{sec:rescaling}). Unit-wise algorithms suggested in Olivier's work are based on block-diagonal approximations of Natural Gradient in which blocks correspond to non-input units. The computational cost of the these unit-wise algorithms is quadratic in the number of incoming weights. To alleviate this cost, \cite{ollivier2015riemannian} also proposed quasi-diagonal approximations which avoid the quadratic dependence but they are only invariant to affine transformations of activation functions. The quasi-diagonal approximations are more similar to DDP-SGD in terms of computational complexity and invariances (see Section~\ref{sec:node-rescaling}). In particular, ignoring the non-diagonal terms related to the biases in quasi-diagonal natural gradient suggested in \cite{ollivier2015riemannian}, it is then equivalent to diagonal Natural Gradient which is itself equivalent to special case of DDP-SGD when $R_v$ is the second moment (see Table~\ref{tab:framework} and the discussion on relation to the Natural Gradient in Section~\ref{sec:ddp-sgd}).  

\begin{table}[tb]
\small
 \begin{center}
 \small
\setlength\tabcolsep{3pt}
\begin{tabular}[tb]{|c|c|c|c|}
\hline
$R_v$ & \pbox{20cm}{Measure} &  Normalized reparametrization
& Diagonal steepest descent\\
\hline
$D=\diag\left(\vecGamma^2_{N^{\In(v)}}\right)$ & Path-Norm & Unit-wise Path-Normalization
& {\bf Path-SGD}\\
\hline
$C=\cov\left(\vecH_{N^{\In(v)})}\right)$ & Variance
& {\bf Batch-Normalization} & \\
\hline
$M=\mathE\left[\vecH_{N^{\In(v)})} \vecH_{N^{\In(v)})}^\top\right]$ & Second Moment & 
& {\bf Diag. Natural Gradient}\\
\hline
\pbox{20cm}{$\alpha M+ (1-\alpha)D$ \\ $\alpha C+ (1-\alpha)D$}
& DDP-Norm & DDP-Normalization
& DDP-SGD\\
\hline
Node-wise Rescaling Invariant & Yes
& No
& Yes\\
\hline
\end{tabular}
 \end{center}
 \caption{\small Some of the choices for $R_v$ in the proposed unified framework.}
 \label{tab:framework}
\end{table}

\section{A Unified Framework}\label{sec:unified}
We define a complexity measure on each node as follows:
\begin{equation}\label{eq:R}
\gamma_v(\vecw) = \sqrt{\vecw_{\rightarrow v}^\top R_v\vecw_{\rightarrow v} }
\end{equation}
where $R_v$ is a positive semidefinite matrix that could depend on the
computations feeding into $v$, and captures both the complexity of the
nodes feeding into $v$ and possibly their interactions.  We consider
several possibilities for $R_v$, summarized also in Table 1.  

A first possibility is to set
$R_v=\diag\left(\vecGamma^2_{N^{\In(v)}}\right)$ to a diagonal matrix
consisting of the complexities of the incoming units.  This choice
does not depend on the source distribution (i.e.~the data), and also
ignores the effect of activations (since the activation pattern
depends on the input distribution) and of dependencies between
different paths in the network.  Intuitively, with this choice of
$R_v$, the measure $\gamma_v(\vecw)$ captures the ``potential'' (data
independent) variability or instability at the node.

Another possibility is to set $R_v$ to either the covariance
(centralized second moment) or to the (uncentralized) second moment
matrix of the outputs feeding into $v$.  In this case,
$\gamma^2_v(\vecw)$ would evaluate to the variance or (uncentralized)
second moment of $z_v$.  We could also linearly combined the data
independent measure, which measures inherent instability, with one of
these the data-dependent measure to obtain:
\begin{equation}\label{eq:gamma}
\gamma^2_v(\vecw) = \alpha S(z_v) + (1-\alpha)\sum_{u \in 
N^{\In} (v)}\gamma^2_u(\vecw) w^2_{u\rightarrow v}\quad (v\notin V_{\rm 
in}),
\end{equation}
where $S(z_v)$ is either the variance or uncentralized second moment,
and $\alpha$ is a parameter.

The complexity measure above is defined for each node of the network
separately, and propagates through the network.  To get an overall
measure of complexity we sum over the output units and define the
following complexity measure for the function $f_{\vecw}$ as
represented by the network:
\begin{equation}
\gamma^2_{\rm net}(\vecw) = \sum_{v\in V_{\rm out}}\gamma^2_v(\vecw).
\end{equation}

For $R_v=\diag\left(\vecGamma^2_{N^{\In(v)}}\right)$, this complexity
measure agrees with the $\ell_2$-Path-regularizer as introduced by
\cite{NeyTomSre15}.  This is the sum over all paths in the network of the squared
product of weights along the path.  The path-regularizer is also
equivalent to looking at the minimum over all ``node rescalings'' of
$\vecw$ (i.e.~all possibly rebalancing of weights yielding the same
function $f_\vecw$) of the $\max_v \norm{\vecw_{\rightarrow v}}$.
But, unlike this max-norm measure, the path-regularizer does {\em not} depend
on the rebalancing and is invariant to node rescalings \citep{NeyTomSre15}.

For data-dependent choices of $R_v$, we also get a similar invariance
property.  We refer to the resulting complexity measure,
$\gamma^2_{\rm net}(\vecw)$, as the Data-Dependent-Path (DDP) regularizer.

\removed{
\natinote{Do we want to explain that it is a norm when we consider it
  over functions?  If so, we should include back the definition of
  $\gamma(f)$ which I removed.   Otherwise, I don't think we need it}.
}

\removed{
Such a measure induces a complexity measure on functions $f$ that can 
be presented by network $G$.
The complexity of $f$ based on DDP-regularizer over $G$
can then be defined as the complexity of the simplest network that 
represent the function:
\begin{equation}
\gamma(f) =\inf_{f_{G,\vecw}=f} \gamma_{\rm net}(\vecw)
\end{equation}
}

After choosing $R_v$, we will think of $\gamma_v$ as specifying the
basic ``geometry'' and bias (for both optimization and learning) over
weights.  In terms of learning, we will (implicitly) prefer weights
with smaller $\gamma_v$ measure, and correspondingly in terms of
optimization we will bias toward smaller $\gamma_v$ ``balls''
(i.e.~search over the part of the space where $\gamma_v$ is smaller).
We will consider two basic ways of doing this: In Section \ref{sec:bn}
we will consider methods that explicitly try to keep $\gamma_v$ small
for all internal nodes in the network, that is explicitly search over
simpler weights.  Any scaling is pushed to the output units, and this
scaling hopefully does not grow too much due.  In Section \ref{sec:ddp-sgd}
we will consider (approximate) steepest descent methods with respect
to the overall $\gamma_{\rm net}$, i.e.~updates that aim at improving
the training objective while being small in terms of their effect on
$\gamma_{\rm net}$.

\section{DD-Path Normalization: A Batch Normalization Approach}\label{sec:bn}

In this Section, we discuss an optimization approach based on ensuring
$\gamma_v$ for all internal nodes $v$ are fixed and equal to
one---that is, the complexity of all internal nodes is ``normalized'',
and any scaling happens only at the output nodes.  We show that with a
choice of $R_v=\cov\left(\vech_{N^{\In(v)})}\right)$, this is
essentially equivalent to Batch Normalization \citep{IofSze15}.

Batch-Normalization \cite{IofSze15} was suggested as an alternate
architecture, with special ``normalization'' layers, that ensure the
variance of node outputs are normalized throughout training.
Considering a feed-forward network as a graph, for each node $v$, the
Batch-Normalization architecture has as parameters an (un-normalized)
incoming weight vector $\tilde{\vecw}$ and two additional scalars
$c_v,b_v\in\R$ specifying scaling and shift respectively.  The
function computed by the network is then given by a forward
propagation similar to standard feed-forward ReLU networks except that for each node an un-normalized
activation is first computed:
\begin{equation}
  \label{eq:ztilde}
  \tilde{z}_v = \inner{\tilde{\vecw}_{\rightarrow v}}{\vech_{N^{\In}(v)}}
\end{equation}
Then, this activation is normalized to obtain the normalized
activation, which is also scaled and shifted, and the output of the
unit is the output of the activation function for this activation value:
\begin{equation}\label{eq:bn}
  \begin{gathered}
z_v=c_v
\frac{\tilde{z}_v-\E{\tilde{z_v}}}{\sqrt{{\text{var}}(\tilde{z}_v)}}+b_v \\
    h_v=[z_v]_+
  \end{gathered}
\end{equation}
The variance and expectation are actually calculated on a
``mini-batch'' of training examples, giving the method its name.
Batch-normalization then proceeds by training the architecture
specified in \eqref{eq:ztilde} and \eqref{eq:bn} through mini-batch stochastic
gradient descent, with each gradient mini-batch also used for
estimating the variance and expectation in \eqref{eq:bn} for all
points in the mini-batch.

Instead of viewing batch-normalization as modifying the architecture,
or forward propagation, we can view it as a re-parameterization, or
change of variables, of the weights in standard feed-forward networks.  In particular, instead of specifying
the weights directly through $\vecw$, we specify them through
$\tilde{\vecw},\vecb$ and $\vecc$, with the mapping:
\begin{align}
  &\tilde{\gamma}^2_v = \tilde{\vecw}_{\rightarrow v}^\top
  R_v\tilde{\vecw}_{\rightarrow v} \quad\quad R_v=\cov(h_{N^{\In}(v)})
  \label{eq:tildegamma}\\
&w_{u \rightarrow v} = 
\begin{cases}
c \frac{\tilde{w}_{u \rightarrow v}}{\tilde{\gamma}_v} & u\neq 
v_{\text{bias}}\\
b-c \frac{\E{\inner{\tilde{\vecw}_{\rightarrow v}}{\vech_{N^{\In}(v)}}}}{\tilde{\gamma}_v}& u=v_{\text{bias}}\\
\end{cases}\label{eq:reparam}
\end{align}
The model class of functions used by Batch-Normalization is thus
exactly the same model class corresponding to standard feed-forward
network, just the parameterization is different.  However, the change
of variables from $\vecw$ to $\tilde{\vecw},\vecb,\vecc$ changes the
geometry implied by the parameter space, and consequently the
trajectory (in model space) of gradient updates---effectively
transforming the gradient direction by the Jacobian between the two
parameterizations.  Batch-Normalization can thus be viewed as an
alternate optimization on the same model class as standard
feed-forward networks, but with a different geometry.  The
reparametrization ensures that $\gamma_v(\vecw)=c_v$ for all
nodes---that is, the complexity is explicit in the parameterization
and thus gets implicitly regularized through the implicit
regularization inherent in stochastic gradient updates.

The re-parameterization \eqref{eq:reparam} is redundant and includes
more parameters than the original parameterization $\vecw$---in
addition to one parameter per edge, it includes also two additional
parameters per node, namely the shift $b_v$ and scaling $c_v$.  The
scaling parameters at internal nodes can be avoided and removed by
noting that in ReLU networks, due to the node-rescaling property, all
scaling can be done at the output nodes.  That is, fixing $c_v=1$ for
all internal $v$ does not actually change the model class (all
functions realizable by the model can be realized this way).
Similarly, we can also avoid the additional shift parameter $b_v$ and
rely only on bias units and bias weights $\tilde{w}_{v_{\rm
    bias}\rightarrow v}$ that get renormalized together with weights.
The bias term $\tilde{w}_{v_{\rm bias}\rightarrow v}$ does {\em not}
affect normalization (since it is deterministic and so has no effect
on the variance), it just gets rescaled with the other weights.

We thus propose using a simpler reparametrization (change of
variables), with the same number of parameters, using only
$\tilde{\vecw}$ and defining for each internal unit:
\begin{equation}\label{eq:ddp-reparam}
w_{u\rightarrow v}=\frac{\tilde{w}_{u\rightarrow v} }{\tilde{\gamma}_v}
\end{equation}
with $\tilde{\gamma}_v$ as in \eqref{eq:tildegamma}, and with the
output nodes un-normalized: $\vecw_{\rightarrow V_{\rm out}} =
\tilde{\vecw}_{\rightarrow V_{\rm out}}$.  This ensures that for all
internal nodes $\gamma_v(\vecw)=1$.


Going beyond Batch-Normalization, we can also use the same approach
with other choices of $R_v$, including all those in Table 1: We
work with a reparametrization $\tilde{\vecw}$, defined through
\eqref{eq:tildegamma} and \eqref{eq:ddp-reparam} but with different
choices of $R_v$, and take gradient (or stochastic gradient) steps
with respect to $\tilde{\vecw}$.  Expectations in the definition of
$R_v$ can be estimated on the stochastic gradient descent mini-batch
as in Batch-Normalization, or on independent samples of labeled or
unlabeled examples.  We refer to such methods as ``DDP-Normalized''
optimization.  Gradients in DDP-Normalization can be calculated
implemented very efficiently similar to Batch-Normalization (see
Section~\ref{sec:ddp-imp}).

When using this type of DDP-Normalization, we ensure that for any
internal node $\gamma_v(\vecw)=1$ (the value of $\tilde{\gamma}_v$ can
be very different from $1$, but what is fixed is the value of
$\gamma_v$ as defined in \eqref{eq:R} in terms of the weights $\vecw$,
which in turn can be derived from $\tilde{\vecw}$ through
\eqref{eq:reparam}), and so the overall complexity $\gamma_{\rm
  net}(\vecw)$ depends only on the scaling at the output layer.

Another interesting property of DDP-Normalization updates is that for
any internal node $v$, the updates direction of
$\tilde{w}_{\rightarrow v}$ is exactly orthogonal to the weights:
\begin{theorem}\label{thm:orthogonal}
For any weight $\tilde{\vecw}$ in DDP-Normalization and any non-input node $v\notin V_{\In}$
\begin{equation}
\inner{\tilde{\vecw}_{\rightarrow v}}{ \frac{ \partial L }{\partial \tilde{\vecw}_{\rightarrow v} } }=0
\end{equation} 
\end{theorem}
\begin{proof}
First we see that \eqref{eq:jacobian} is true because
 \begin{align*}
\frac{\partial f_{\vecw}(\vecx)[v]}{\partial \vecw} =\Bigl(\sum_{p\in\Pi(v)}
\frac{\partial \pi_p(\vecw)}{\partial w_e}
 \cdot g_{p}(\vecx)\cdot x[{\rm head}(p)]\Bigr)_{e\in E}
 = J_{v}(\vecw)\top\cdot \phi_{v}(\vecx).
\end{align*}
 Therefore, 
 \begin{align}
\bigcup_{\vecx\in\mathR^{|V_{\rm in}|}} {\rm Span}\left(
\frac{\partial f_{\vecw}(\vecx)[v]}{\partial \vecw}: v\in V_{\rm out}
 \right)
&= \bigcup_{\vecx\in\mathR^{|V_{\rm in}|}}{\rm Span}\left(
J_{v}(\vecw)\top\cdot \phi_{v}(\vecx): v\in V_{\rm out}
  \right)\notag\\
  \label{eq:span}
 &=J(\vecw)\top \cdot {\rm Span}\left(\phi(\vecx): \vecx\in\mathR^{|V_{\rm in}|}\right).
\end{align}
Consequently,  any vector of the form $(\frac{\partial
f_{\vecw}(\vecx)[v]}{\partial w_e})_{e\in E}$ for a fixed input $\vecx$ lies in the
 span of the row vectors of the path Jacobian $J(\vecx)$.

The second part says $d_{G}(\vecw)={\rm rank}J(\vecw)$ if ${\rm
dim}\left({\rm Span}(\phi(\vecx):\vecx\in\mathR^{|V_{\rm in}|})\right)=|\Pi|$, which is the number
of rows of $J(\vecw)$. We can see that this is true from expression \eqref{eq:span}.

\end{proof}
The fact that the gradient is orthogonal to the parameters
means weight updates in DDP-Normalization are done
in a way that it prevents the norm of weights to change considerably 
after each updates.

\section{DD-Path-SGD: A Steepest Descent Approach}\label{sec:ddp-sgd}

We now turn to a more direct approach of using our complexity measure
for optimization.  To do so, let us first recall the strong connection
between geometry, regularization and optimization through the specific
example of gradient descent.  

Gradient descent can be thought of as steepest descent with respect to
the Euclidean norm---that is, it takes a step in a direction that
maximizes improvement in the objective while also being small in terms
of the Euclidean norm of the step.  The step can also be viewed as a
regularized optimization of the linear approximation given by the
gradient, where the regularizer is squared Euclidean norm.  Gradient
Descent is then inherently linked to the Euclidean norm---runtime of
optimization is controlled by the Euclidean norm of the optimum and
stochastic gradient descent yields implicit Euclidean norm
regularization.  A change in norm or regularizer, which we think of as
a change of geometry, would then yield different optimization
procedure linked to that norm.

What we would like is to use the DDP-regularizer $\gamma_{\rm \net}(\vecw)$
to define our geometry, and for that we need a distance (or
divergence) measure corresponding to it by which we can measure the
``size'' of each step, and require steps to be small under this
measure.  We cannot quite do this, but instead we use a diagonal
quadratic approximation of $\gamma_{\rm net}(\vecw)$ about our current
iterate, and then take a steepest descent step w.r.t.~the quadratic
norm defined by this approximation.

Specifically, given a choice of $R_v$ and so complexity measure
$\gamma_{\rm net}(\vecw)$, for the current iterate $\vecw^{(t)}$ we define the
following quadratic approximation:
\begin{equation}
\hat{\gamma}^2_{\rm net}(\vecw^{(t)}+\Delta \vecw) = \gamma^2_{\rm net}(\vecw^{(t)})+
\inner{\nabla \gamma^2_{\rm net}(\vecw^{(t)})}{\Delta \vecw} +
\frac{1}{2} \Delta \vecw^\top
\diag\left(\nabla^2  \gamma^2_{\rm net}(\vecw^{(t)})\right) \Delta \vecw
\end{equation}
and the corresponding quadratic norm:
\begin{equation}
\norm{\vecw'-\vecw}^2_{\hat{\gamma}^2_{\rm net} } =
\norm{\vecw'-\vecw}^2_{\diag(\frac{1}{2}\nabla^2 \gamma^2_{\rm
    net}(\vecw^{(t)}))}=\sum_{(u\rightarrow v)\in G} 
\frac{1}{2}\frac{\partial^2 \gamma^2_{\rm net}}{\partial \vecw^2_{u\rightarrow v}} (\vecw'_{u\rightarrow v}-\vecw_{u\rightarrow v})^2.
\end{equation}
We can now define the DDP-update as:
\begin{equation}\label{eq:ddp}
\vecw^{(t+1)}=\min_{\vecw} \eta \inner{\nabla L(w)}{\vecw-\vecw^{(t)}} + 
\frac{1}{2} \norm{\vecw'-\vecw}^2_{\hat{\gamma}^2_{\rm net} }.
\end{equation}
Another way of viewing the above approximation is as taking a diagonal
quadratic approximation of the Bergman divergence of the regularizer.
Solving \eqref{eq:ddp} yields the update:
\begin{equation}
w^{(t+1)}_{u\rightarrow v} = w_{u\rightarrow v} - \frac{\eta}{\kappa_{u\rightarrow v}(\vecw)} \frac{\partial 
L}{\partial w_{u\rightarrow v}}(\vecw^{(t)}) \quad\quad\textrm{where: } \kappa_{u\rightarrow v}(\vecw)=\frac{1}{2}\frac{\partial^2 \gamma^2_{\rm 
net}}{\partial w^2_{u\rightarrow v}}.
\end{equation}
Instead of using the full gradient, we can also use a limited number
of training examples to obtain stochastic estimates of $\frac{\partial 
L}{\partial w_{u\rightarrow v}}(\vecw^{(t)})$---we refer to the
resulting updates as DDP-SGD.

For the choice $R_v=\diag(\gamma^2_{N^{\In}(v)})$, we have that
$\gamma^2_{\rm net}$ is the Path-norm and we recover Path-SGD
\cite{NeySalSre15}. As was shown there, the Path-SGD updates can be
calculated efficiently using a forward and backward propagation on the
network, similar to classical back-prop.  In Section~\ref{sec:ddp-imp} we show
how this type of computation can be done more generally also for other
choices of $R_v$ in Table 1.

\subsection*{Relation to the Natural Gradient}\label{sec:natural}

The DDP updates are similar in some ways to Natural Gradient updates,
and it is interesting to understand this connection.  Like the DDP,
the Natural Gradients direction is a steepest descent direction, but
it is based on a divergence measure calculated directly on the
function $f_{\vecw}$, and not the parameterization $\vecw$, and as
such is invariant to reparametrizations.  The natural gradient is
defined as a steepest descent direction with respect to the
KL-divergence between probability distributions, and so to refer to it
we must refer to some probabilistic model.  In our case, this will be
a conditional probability model for labels $\vecY$ conditioned on the
inputs $\vecX$, taking expectation with respect to the true marginal
data distribution over $\vecX$.

What we will show that for the choice
$R_v=\mathE[\vech_{N^{\In}(v)}\vech_{N^{\In}(v)}^\top]$, the DDP update
can also be viewed as an approximate Natural Gradient update.  More
specifically, it is a diagonal approximation of the Natural Gradient
for a conditional probability model $q(\vecY| \vecX;\vecw)$ (of the labels
$\vecY$ given an input $\vecX$) parametrized by $\vecw$ and specified
by adding spherical Gaussian noise to the outputs of the network:
$\vecY|\vecX\sim \mathcal{N}(f_\vecw(\vecX),I_{|V_{\rm out}|})$.

Given the conditional probability distribution $q(\vecY|\vecx;\vecw)$, we can
calculate the expected Fisher information matrix.  This is a matrix
indexed by parameters of the model, in our case edges $e=(u\rightarrow
v)$ on the graph
and their corresponding weights $w_e$, with entries defined as follows:
 \begin{align}
  \label{eq:fisher-information-m}
 F(\vecw)[e,e'] = \mathE_{\vecx\sim p(\vecX)}\mathE_{\vecY\sim q(\vecY|\vecx;\vecw)}\left[
\frac{\partial \log q(\vecY|\vecx;\vecw)}{\partial w_e}
\frac{\partial \log q(\vecY|\vecx;\vecw)}{\partial w_{e'}}
 \right],
 \end{align}
where $x\sim p(\vecX)$ refers to the marginal source distribution (the data
distribution).  That is, we use the true marginal distributing over
$\vecX$, and the model conditional distribution $\vecY|\vecX$, ignoring
the true labels. The Natural
Gradient updates can then be written as(see Section~\ref{sec:ddp-imp} for more 
information):
\begin{equation}\label{eq:ng-m}
  \vecw^{(t+1)} = \vecw^{(t)} - \eta F(\vecw^{(t)})^{-1} \nabla_\vecw L(\vecw^{(t)}).
\end{equation}

If we approximate the Fisher information matrix
with its diagonal elements, the update step normalizes each dimension 
of the gradient with the corresponding element on the diagonal of 
the Fisher information matrix:
\begin{equation}\label{eq:dng}
w^{(t+1)}_{e} = w^{(t)}_{e} - \frac{\eta}{F(\vecw)[e,e]} \frac{\partial L}{\partial 
w_{e}}(\vecw^{(t)}).
\end{equation}

Using diagonal approximation of Fisher information matrix to normalize the gradient values has been suggested before as a computationally tractable alternative to the full Natural Gradient \citep{lecun1998neural,schaul2013no}.  \cite{ollivier2015riemannian} also suggested a ``quasi-diagonal" approximations that includes, in addition to the diagonal, also some non-diagonal terms corresponding to the relationship between the bias term and every other incoming weight into a unit.

For our Gaussian probability model, where $\log
 q(\vecY|\vecX)=\frac{1}{2}\norm{\vecY-f_\vecw(\vecX)}^2+{\rm 
const}$, the diagonal can be calculated as:
\begin{equation}\label{eq:diag}
F(\vecw)[e,e] = \mathE_{\vecX\sim
 p(\vecX)}\left[\sum_{v'\in V_\Out} \left(\frac{\partial 
f_\vecw(\vecX)[v']}{\partial w_e}\right)^2\right],
\end{equation}
using \eqref{eq:partial-logq}. We next prove that this update is equivalent to
DDP-SGD for a specific choice of $R_v$, namely the second moment.
\begin{theorem}
The Diagonal Natural Gradient indicated in equations~\eqref{eq:dng} and ~\eqref{eq:diag} is equivalent to DDP-SGD for $R_v=\mathE\left[\vech_{N^{\In}(v)}\vech_{N^{\In}(v)}^\top\right]$.
\end{theorem}
\begin{proof}
We calculate the scaling factor $\kappa_{u\rightarrow v}(\vecw)$ for 
DDP-SGD as follows:
\begin{align*}
\kappa_{u\rightarrow v}(\vecw)&=\frac{1}{2}\frac{\partial^2 \gamma^2_{\rm net}}{\partial w_{u\rightarrow v}^2}=\frac{1}{2}\sum_{v'\in V_{\Out}}\frac{\partial^2 \mathE[z^2_{v'}]}{\partial w_{u\rightarrow v}^2} = \sum_{v'\in V_{\Out}} \frac{\partial}{\partial w_{u\rightarrow v}}\left(\frac{1}{2}\frac{\partial \mathE[z^2_{v'}]}{\partial w_{u\rightarrow v}}\right)\\
&= \sum_{v'\in V_{\Out}}\frac{\partial}{\partial w_{u\rightarrow v}}\left(\mathE\left[ z_{v'} \frac{\partial z_{v'}}{\partial w_{u\rightarrow v}}\right]\right)
= \sum_{v'\in V_{\Out}}\frac{\partial}{\partial w_{u\rightarrow v}}\left(\mathE\left[ z_{v'} h_u\frac{\partial z_{v'}}{\partial z_{v}}\right]\right)\\
&= \sum_{v'\in V_{\Out}}\mathE\left[ h^2_u\left(\frac{\partial z_{v'}}{\partial z_{v}}\right)^2\right] = \mathE\left[h^2_u \sum_{v'\in V_{\Out}} \left(\frac{\partial z_{v'}}{\partial z_{v}}\right)^2\right]\\
&=\mathE\left[\sum_{v'\in V_\Out} \left(\frac{\partial 
f_\vecw(\vecX)[v']}{\partial w_e}\right)^2\right] = F(\vecw)[u\rightarrow v,u\rightarrow v]\\
\end{align*}
Therefore, the scaling factors in DDP-SGD  with $R_v=\mathE\left[\vech_{N^{\In}(v)}\vech_{N^{\In}(v)}^\top\right]$
are exactly the diagonal elements of the Fisher Information matrix used 
in the Natural Gradient updates.
\end{proof}

\removed{
In classification tasks, we usually use softmax activations and it that case, the 
diagonal of the Fisher Information can be calculated as:
\begin{equation}
 F(\vecw)[u\rightarrow v,u\rightarrow v] = \mathE\left[h^2_{u}
  \sum_{v'\in V_{\rm out}} \sum_{v''\in V_{\rm out}}
  \left(1_{v'=v''}\cdot\sigma_{\text{soft}}(z_{v'})  
-\sigma_{\text{soft}}(z_{v'})\sigma_{\text{soft}}(z_{v''})\right)
 \frac{\partial  z_{v'}}{\partial z_v}\cdot
 \frac{\partial z_{v''}}{\partial z_v}
  \right]
\end{equation}
where $\sigma_{\text{soft}}$ is the softmax function and we have that 
$\sigma_{\text{soft}}(z_v) = \frac{e^{z_v}}{\sum_{v'\in V_\Out} 
e^{z_{v'}}}$. The above values can be calculated as fast as 
$\abs{V_\Out}$ backpropagation on the mini-batch. Moreover, 
considering the connection between DDP-SGD and the diagonal of outer 
product of Jacobians, we can consider adding the path-regularizer to 
the diagonal of the Fisher information.
}
\section{Node-wise invariance} \label{sec:node-rescaling} In this
section, we show that DDP-SGD is invariant to node-wise rescalings,
 while DDP-Normalization does not have favorable
invariance properties.

\subsection{DDP-SGD on feedforward networks}
In Chaper~\ref{chap:invariances}, we observed that feedforward ReLU networks are
invariant to node-wise rescaling.  To see if DDP-SGD is also invariant
to such rescaling, consider a rescaled $\vecw'=T(\vecw)$, where $T$ is
a rescaling by $\rho$ at node $v$.  Let
$\vecw^+$ denote the weights after a step of DDP-SGD.  To establish
invariance to node-rescaling we need to show that
$\vecw'^+=T(\vecw^+)$.  For the outgoing weights from $v$  we have:
\begin{align*}
w'^{+}_{v\rightarrow j} &= \rho w_{v\rightarrow j} - 
\frac{\rho^2 \eta}{\kappa_{v\rightarrow j}(\vecw)}\frac{\partial 
L}{\rho \partial w_{v\rightarrow j}}(\vecw)\\
&=\rho\left(w_{v\rightarrow j}- \frac{\eta}{\kappa_{v\rightarrow 
j}(\vecw)}\frac{\partial L}{\partial w_{v\rightarrow j}}(\vecw)\right)
 = \rho w^{+}_{v\rightarrow j}\\
\end{align*}
Similar calculations can be done for incoming weights to the node $v$.
The only difference is that $\rho$ will be substituted by $1/\rho$. Moreover,
note that due to non-negative homogeneity of ReLU activation function,
the updates for the rest of the weights remain exactly the same.
Therefore, DDP-SGD is node-wise rescaling invariant.

\subsection{SGD on DDP-Normalized networks}
Since DDP-Normalized networks are reparametrization of feedforward
networks, their invariances are different. Since the operations in DDP-Normalized
networks are based on $\tilde{w}$, we should study the invariances for
$\tilde{w}$.  The invariances in this case are given by rescaling of
incoming weights into a node, i.e. for an internal node $v$ and
scaling $\rho>0$:
\begin{equation}
T(\tilde{w})_{k\rightarrow v} =\rho \tilde{w}_{k\rightarrow v}\quad
 (\forall k\in N^\In(v))\notag
\end{equation}
while all other weights are unchanged.  The DDP-Normalized networks
are invariant to the above transformation because the output of each
node is normalized.  The SGD update rule is however not invariant to
this transformation:
\begin{align*}
T(\tilde{w})^{+}_{k\rightarrow v} &= \rho \tilde{w}_{k\rightarrow 
v}-\eta\frac{\partial L}{\rho \partial \tilde{w}_{k\rightarrow 
v}}(\tilde{\vecw})
\neq \rho\left(\tilde{w}_{k\rightarrow v}-\eta \frac{\partial L}{\partial 
\tilde{w}_{k\rightarrow v}}(\tilde{\vecw})\right) = \rho 
\tilde{w}^{+}_{k\rightarrow v}\\
\end{align*}

\section{Supporting Results}

\subsection{Implementation of DDP-Normalization}\label{sec:BN-imp}
Given any batch of $n$ data points to estimate mean, variance and the 
gradient, the stochastic gradients for the weight $\tilde{\vecw}$ (weights 
in the DDP-Normalized network) can then be calculated through the 
chain rule:
\begin{align}\label{eq:dl}
\frac{\partial L}{\partial \tilde{\vecw}_{\rightarrow v}} 
&=\frac{1}{n\tilde{\gamma}_v}\sum_{i=1}^{n} \frac{\partial L}{\partial 
z_{v}^{(i)}}\left[ {\vech}_{N^{\In}(v)}^{(i)} - \frac{1}{n} 
\sum_{j=1}^n {\vech}_{N^{\In}(v)}^{(j)} -\frac{
\hat{z}^{(i)}_v }{ 
2\tilde{\gamma}_v^2}\frac{\partial \tilde{\gamma}^2_v}{\partial \tilde{\vecw}_{\rightarrow 
v}} \right]\\
\frac{\partial L}{\partial z^{(i)}_{u}} &= \frac{1}{\tilde{\gamma}_v} 
\left[\sum_{v \in N^{\Out}(u)} \tilde{w}_{u\rightarrow 
v}\left(\frac{\partial L}{\partial z_{v}^{(i)}} - 
\frac{1}{n}\sum_{j=1}^n\frac{\partial L}{\partial 
z_{v}^{(j)}}\left(1-\alpha\frac{\hat{z}^{(i)}_v\hat{z}^{(j)}_v}{\tilde{\gamma}^2_v}\right)\right)\right]_{z^{(i)}_{u}\geq 0}\\ \notag
\end{align}
where $\hat{z}^{(i)}_{v} = \tilde{z}^{(i)}_{v} - \frac{1}{n} 
\sum_{j=1}^n \tilde{z}^{(j)}_{v}$ and we have:
\begin{equation}\label{eq:dgam}
\frac{\partial \tilde{\gamma}^2_v}{\partial \tilde{\vecw}_{\rightarrow 
v}}=2(1-\alpha)\tilde{\vecw}_{\rightarrow v} + 
\frac{2\alpha}{n}\sum_{i=1}^n \hat{z}^{(i)}_v \left({\vech}_{N^{\In}(v)}^{(i)} - \frac{1}{n} 
\sum_{j=1}^n {\vech}_{N^{\In}(v)}^{(j)}\right)
\end{equation}
Similar to Batch-Normalization, all the above calculations can be efficiently 
carried out as vector operations with negligible extra memory and computations.
\subsection{Implementation of DDP-SGD}\label{sec:ddp-imp}
In order to compute the second derivatives 
$\kappa_e(\vecw)=\frac{\partial^2 \gamma^2_{\rm net}}{\partial w^2_e}$, 
we first calculate the first derivative. The backpropagation can be 
done through $\gamma^2_u$ and  $z^{(i)}_{u}$ but this makes it 
difficult to find the second derivatives. Instead we propagate the 
loss through $\gamma^2_u$ and the second order terms of the form 
$z^{(i)}_{u_1}z^{(i)}_{u_2}$:
\begin{align}
\frac{\partial \gamma^2_{{\rm net}}}{\partial \gamma^2_u} &= 
(1-\alpha)\sum_{v\in N^\Out(u)} \frac{\partial \gamma^2_{{\rm 
net}}}{\partial \gamma^2_v} w^2_{u\rightarrow v}
\end{align}
\begin{equation}
\frac{\partial \gamma^2_{{\rm net}}}{\partial 
(z_{u_1}^{(i)}z_{u_2}^{(i)})} = \alpha \left[\frac{ \partial 
\gamma^2_{\rm net}}{\partial \gamma^2_{u_1}}\right]_{u_1=u_2} + 
\left[\sum_{(v_1,v_2)\in \left(N^\Out(u_1)\right)^2} \frac{\partial 
\gamma^2_{\rm net}}{\partial 
(z_{v_1}^{(i)}z_{v_2}^{(i)})}w_{u_1\rightarrow v_1} w_{u_2 
\rightarrow v_2} \right]_{z^{(i)}_{u_1}>0, z^{(i)}_{u_2}>0}
\end{equation}
Now we can calculate the partials for $w_{u\rightarrow v}$ as follows:
\begin{equation}
\frac{\partial \gamma^2_{\rm net}}{\partial w_{u\rightarrow v}} = 
2(1-\alpha)\frac{\partial \gamma^2_{\rm net}}{\partial \gamma^2_v} 
\gamma^2_u w_{u\rightarrow v} + 2\sum_{i=1}^n\sum_{v'\in N^\Out(u)} 
\frac{\partial \gamma^2_{\rm net}}{\partial (z^{(i)}_v z^{(i)}_{v'})} 
h_{u}^{(i)}z_{v'}^{(i)}
\end{equation}
Since the partials $\frac{\partial \gamma^2_{{\rm net}}}{\partial 
\gamma^2_u}$ and $\frac{\partial \gamma^2_{{\rm net}}}{\partial 
(z_{u_1}^{(i)}z_{u_2}^{(i)})}$ do not depend on $w_{u\rightarrow v}$, 
the second order derivative can be calculated directly:
\begin{equation}
\kappa_{u\rightarrow v}(\vecw)=\frac{1}{2}\frac{\partial^2 \gamma^2_{\rm 
net}}{\partial w_{u\rightarrow v}^2} = (1-\alpha)\frac{\partial 
\gamma^2_{\rm net}}{\partial \gamma^2_v} \gamma^2_u + 
\sum_{i=1}^n\frac{\partial \gamma^2_{\rm net}}{\partial 
\left({z^{(i)}_v}^2\right)}\left(h^{(i)}_u\right)^2
\end{equation}

\subsection{Natural Gradient}\label{sec:ng}
The natural gradient algorithm \citep{Ama98} achieves invariance 
by applying the inverse of the Fisher information
matrix $F(\vecw^{(t)})$ at the current parameter $\vecw^{(t)}$ to the 
negative gradient direction
as follows:
 \begin{align}
 \vecw^{(t+1)} &=  \vecw^{(t)} + \eta\vecDelta^{(\rm natural)},\notag
 \intertext{where}
 \label{eq:ng-argmin}
 \Delta^{(\rm natural)} &=
   \argmin{\Delta\in\mathR^{|E|}}
  \inner{-\frac{\partial L}{\partial w}(\vecw^{(t)})}{\vecDelta}
  ,\quad {\rm s.t.}\quad \vecDelta\top F(\vecw^{(t)})\vecDelta\leq \delta^2\\
 \label{eq:ng}
&=-F^{-1}(\vecw^{(t)})\frac{\partial L}{\partial w}(\vecw^{(t)}).
 \end{align}
Here $F(\vecw)$ is the Fisher information matrix at point $\vecw$ and is
defined with respect to the probabilistic view of the feedforward neural network
model, which we describe in more detail below.

Suppose that we are solving a classification problem and the final 
layer
of the network is fed into a softmax layer that determines the
probability of candidate classes given the input $x$. Then the neural
network with the softmax layer can be viewed as a conditional
probability distribution
 \begin{align}
  \label{eq:cond-prob}
 q(y|\vecx)= \frac{\exp(f_{\vecw}(\vecx)[v_y])}{\sum_{v\in  V_{\rm 
out}}\exp(f_{\vecw}(\vecx)[v])},
 \end{align}
where $v_y$ is the output node corresponding to class $y$.
If we are solving a regression problem a Gaussian distribution is
probably more appropriate for $q(y|\vecx)$.

Given the conditional probability distribution $q(y|\vecx)$, the Fisher
information matrix can be defined as follows:
 \begin{align}
  \label{eq:fisher-information}
 F(\vecw)[e,e'] = \mathE_{\vecx\sim p(\vecX)}\mathE_{y\sim q(y|\vecx)}\left[
\frac{\partial \log q(y|\vecx)}{\partial w_e}
\frac{\partial \log q(y|\vecx)}{\partial w_{e'}}
 \right],
 \end{align}
where $p(x)$ is the marginal distribution of the data.

Since we have
 \begin{align}
  \label{eq:partial-logq}
 \frac{\partial \log q(y|\vecx)}{\partial w_{u\rightarrow v}}
 =\frac{\partial \log q(y|\vecx)}{\partial z_v}\cdot h_u
 =\sum_{v'\in V_{\rm out}}
 \frac{\partial \log q(y|\vecx)}{\partial z_{v'}}\cdot
 \frac{\partial z_{v'}}{\partial z_{v}}\cdot h_u
 \end{align}
using the chain rule, each entry of the Fisher information matrix can 
be
computed efficiently by forward and backward propagations on a 
minibatch.

\phantomsection
\addcontentsline{toc}{chapter}{\Large  Conclusion}
\chapter*{Conclusion} \label{chap:conclusion}
In this dissertation, we tried to explain generalization in deep learning with a view that is central around implicit regularization by the optimization algorithm, showing that the implicit regularization is the main component that should be taken into account. We proved several generalization guarantees based on different complexity measures for neural networks and investigated whether implicit regularization is indeed penalizing the complexity of the model based on any of those measures. Finally, we designed optimization algorithms to implicitly regularize complexity measures that are more suitable for neural networks and provided empirical evidence indicating that these algorithms lead to better generalization than SGD for feedforward and recurrent networks.
\phantomsection
\addcontentsline{toc}{chapter}{\Large Bibliography}
\bibliography{ref}
\end{document}